\documentclass{article}
\usepackage{arxiv}
\usepackage{amsfonts,amsmath,amssymb,amsthm}
\usepackage{array}
\usepackage{bm,bbm}
\usepackage{booktabs}
\usepackage{graphicx}
\usepackage{here}
\usepackage{float}
\usepackage{framed}
\usepackage{latexsym}
\usepackage{multirow}
\usepackage{natbib}
\usepackage{overpic}
\usepackage{hyperref}
\usepackage{url}
\usepackage{xcolor}
\bibpunct[, ]{(}{)}{;}{; }{,}{,}
\hypersetup{bookmarksnumbered=true, bookmarksopen=true, colorlinks=true, linkcolor=blue, citecolor=blue,}
\newtheorem{theorem}{Theorem}
\newtheorem{corollary}{Corollary}

\newtheorem{definition}{Definition}
\newtheorem{example}{Example}
\newcommand{\zo}{{\rm zo}}
\newcommand{\ab}{{\rm ab}}
\newcommand{\sq}{{\rm sq}}
\newcommand{\logi}{{\rm logi}}
\newcommand{\hing}{{\rm hing}}
\newcommand{\ramp}{{\rm ramp}}
\newcommand{\smhi}{{\rm smhi}}
\newcommand{\sqhi}{{\rm sqhi}}
\newcommand{\expo}{{\rm expo}}
\newcommand{\abso}{{\rm abso}}
\newcommand{\squa}{{\rm squa}}
\newcommand{\logiit}{{\rm logi\text{-}it}}
\newcommand{\hingit}{{\rm hing\text{-}it}}

\newcommand{\expoit}{{\rm expo\text{-}it}}

\newcommand{\squait}{{\rm squa\text{-}it}}
\newcommand{\logiat}{{\rm logi\text{-}at}}
\newcommand{\hingat}{{\rm hing\text{-}at}}

\newcommand{\expoat}{{\rm expo\text{-}at}}
\newcommand{\absoat}{{\rm abso\text{-}at}}
\newcommand{\squaat}{{\rm squa\text{-}at}}
\newcommand{\tot}{{\rm to}}

\newcommand{\ord}{{\rm ord}}

\newcommand{\thr}{{\rm thr}}
\newcommand{\cl}{{\rm cl}}
\newcommand{\acl}{{\rm acl}}
\DeclareMathOperator{\argmax}{{arg\,max}}
\DeclareMathOperator{\argmin}{{arg\,min}}
\newcommand{\bb}{{\bm{b}}}
\newcommand{\bc}{{\bm{c}}}
\newcommand{\bp}{{\bm{p}}}
\newcommand{\bt}{{\bm{t}}}

\newcommand{\bw}{{\bm{w}}}
\newcommand{\bx}{{\bm{x}}}
\newcommand{\bI}{{\mathbf{I}}}
\newcommand{\bD}{{\mathbf{D}}}
\newcommand{\bX}{{\bm{X}}}
\newcommand{\calA}{{\mathcal{A}}}
\newcommand{\calB}{{\mathcal{B}}}
\newcommand{\calO}{{\mathcal{O}}}
\newcommand{\calS}{{\mathcal{S}}}
\newcommand{\calX}{{\mathcal{X}}}
\DeclareMathOperator{\Uni}{{\mathrm{Uni}}}
\DeclareMathOperator{\Cat}{{\mathrm{Cat}}}
\DeclareMathOperator{\bbE}{{\mathbb{E}}}
\DeclareMathOperator{\bbI}{{\mathbbm{1}}}
\newcommand{\bbR}{{\mathbb{R}}}
\newcommand{\bbN}{{\mathbb{N}}}
\newcommand{\scs}{\scriptsize}

\newcommand{\tcr}[1]{\textcolor{red}{#1}}
\newcommand{\tcb}[1]{\textcolor{blue}{#1}}
\newcommand{\hyl}[1]{{\rm(\hyperlink{#1}{\tcb{#1}})}}
\newcommand{\hyt}[1]{{\hypertarget{#1}{{\rm(\tcb{#1})}}}}
\title{Remarks on Loss Function of Threshold Method for Ordinal Regression Problem}
\author{Ryoya Yamasaki\And Toshiyuki Tanaka}
\begin{document}
\maketitle
\begin{abstract}
Threshold methods are popular for ordinal regression problems, 
which are classification problems for data with a natural ordinal relation.
They learn a one-dimensional transformation (1DT) of observations 
of the explanatory variable, and then assign label predictions to 
the observations by thresholding their 1DT values.
In this paper, we study the influence of the underlying data 
distribution and of the learning procedure of the 1DT on 
the classification performance of the threshold method 
via theoretical considerations and numerical experiments.
Consequently, for example, we found that threshold methods 
based on typical learning procedures may perform poorly 
when the probability distribution of the target variable conditioned on 
an observation of the explanatory variable tends to be non-unimodal.
Another instance of our findings is that learned 1DT values are
concentrated at a few points under the learning procedure based on 
a piecewise-linear loss function, 
which can make difficult to classify data well.
\end{abstract}
\section{Introduction}
\label{sec:Introduction}
\emph{Ordinal regression} (\emph{OR}, or called ordinal classification) is 
the classification of \emph{ordinal data} in which the underlying target variable is 
labeled from a categorical \emph{ordinal scale} that is considered to be equipped 
with a \emph{natural ordinal relation} for the underlying explanatory variable,
as formalized in Section~\ref{sec:ORT-OD}.
The ordinal scale is typically formed as 
a graded summary of objective indicators like age groups 
\{`0--9', `10--19', \ldots, `90--99', `100--'\} 
or graded evaluation of subjectivity like human rating 
\{`excellent', `good', `average', `bad', `terrible'\}.
OR techniques are employed in a variety of practical applications,
for example, age estimation \citep{niu2016ordinal, cao2020rank}, 
information retrieval \citep{liu2009learning},
movie rating \citep{yu2006collaborative}, and
questionnaire survey \citep{burkner2019ordinal}.

\emph{Threshold methods} are popular for OR problems 
as a simple way to capture the ordinal relation of ordinal data,
and have been studied vigorously in machine learning research
\citep{shashua2003ranking, lin2006large, chu2007support, lin2012reduction, 
li2007ordinal, pedregosa2017consistency, yamasaki2022optimal}.
Those methods learn a \emph{one-dimensional transformation (1DT)} 
of the observation of the explanatory variable so that an observation 
with a larger class label tends to have a larger 1DT value;
they then assign a label prediction to the learned 1DT according to 
the rank of an interval to which the 1DT belongs among intervals 
on the real line separated by \emph{threshold parameters}.
See also Section~\ref{sec:Threshold} for their formulation.

Despite its popularity, the question of what data distribution 
threshold methods will work well or poorly for is not sufficiently elucidated.
For example, \citet{rennie2005loss, gutierrez2015ordinal} 
performed numerical experiments to compare the classification 
performance of OR methods including threshold methods, 
but did not attempt to elucidate factors behind the difference in those performances.
Therefore, with the aim of providing an answer to this question,
we study the influence of the underlying data 
distribution and of the learning procedure of the 1DT on 
the classification performance of the threshold method
(see Section~\ref{sec:Policy} for a mathematical
formulation of the research objectives of this work).
We present theoretical analysis and conjectures in Section~\ref{sec:AC},
and verify these results via numerical experiments as described 
in Sections~\ref{sec:NE}
and Appendix~\ref{sec:SER}.\footnote{%
See \url{https://github.com/yamasakiryoya/RLTM} 
for program codes for all the experiments in this work.}
These construct qualitative understanding of what properties 
of the ordinal data or of the learning procedure 
contribute to the success or failure of a threshold method.
%
Refer to Section~\ref{sec:Conclusion} for summary of 
results obtained by this work and for future works, and 
Appendix~\ref{sec:Proof} for proofs of theoretical results.

\section{Preparation}
\label{sec:Preparation}
\subsection{Formulation of Ordinal Regression Tasks and Ordinal Data}
\label{sec:ORT-OD}
\subsubsection{Ordinal Regression Tasks}
\label{sec:ORT}
Suppose that we are interested in behaviors of the categorical target variable 
$Y\in[K]:=\{1,\ldots,K\}$ given a value of the explanatory variable $\bX\in\bbR^d$,
and that we can access data $(\bx_1,y_1),\ldots,(\bx_n,y_n)\in\bbR^d\times[K]$
that are considered to be independent and identically distributed observations 
of the pair $(\bX,Y)$ and have a natural ordinal relation like examples described 
in Section~\ref{sec:Introduction}, where $K, d, n\in\bbN$ such that $K\ge3$;
we call such data ordinal data.
An OR task is a classification task for ordinal data 
and often formulated as searching for a classifier $f:\bbR^d\to[K]$ 
that is good in minimization of the \emph{task risk} $\bbE[\ell(f(\bX),Y)]$ 
for a user-specified \emph{task loss function} $\ell:[K]^2\to[0,+\infty)$, 
where the expectation $\bbE[\cdot]$ is taken for 
all included random variables (here $\bX$ and $Y$).\footnote{%
For evaluation in OR tasks, one may use criteria that cannot be 
decomposed into a sum of losses for each sample point such as 
quadratic weighted kappa \citep{cohen1960coefficient, cohen1968weighted}. 
One of the main topics in this paper concerns troubles caused in relation 
to the learning procedure (defined later in Section~\ref{sec:Learning}), 
and those troubles are there similarly even for such other criteria.
Our formulation is not comprehensive for the OR task, but 
restricting the formulation does not seriously affect our discussion.}
Popular task loss functions for OR tasks include not only 
the zero-one task loss $\ell_\zo(k,l):=\bbI(k\neq l)$ 
for minimization of misclassification rate, but also V-shaped losses 
for cost-sensitive classification tasks that reflect user's preference 
of smaller prediction errors over larger ones such as 
the absolute task loss $\ell_\ab(k,l):=|k-l|$,
squared task loss $\ell_\sq(k,l):=(k-l)^2$, and 
$\ell_{\zo,\epsilon}(k,l):=\bbI(|k-l|>\epsilon)$ with $\epsilon>0$,
where $\bbI(c)$ is 1 if the condition $c$ is true and 0 otherwise.

\subsubsection{Ordinal Data}
\label{sec:OD}
OR tasks assume that ordinal data have a natural ordinal relation.
As the natural ordinal relation, in this work, 
we basically suppose the \emph{unimodality hypothesis} adopted in 
\citet{da2008unimodal, beckham2017unimodal, yamasaki2022unimodal}.
Namely, we assume that many ordinal data are almost unimodal, 
on the basis of the following terminologies:
For a \emph{probability mass function} (\emph{PMF}) $\bp=(p_k)_{k\in[K]}\in\Delta_{K-1}
:=\{(q_k)_{k\in[K]}\mid\sum_{k=1}^K q_k=1; q_k\in[0,1], \forall k=[K]\}$, 
we call $M(\bp)\in\argmax_k(p_k)_{k\in[K]}$ a \emph{mode} of $\bp$, 
and say that $\bp$ is \emph{unimodal} if it satisfies $p_1\le\cdots\le p_{M(\bp)}$
and $p_{M(\bp)}\ge\cdots\ge p_K$ for every mode $M(\bp)$.
Moreover, if the \emph{conditional probability distribution} (\emph{CPD}) 
$(\Pr(Y=y|\bX=\bx))_{y\in[K]}$ of the target variable $Y$ conditioned 
on an observation $\bx$ of the explanatory variable $\bX$
is unimodal for any $\bx$ in the support $\calX$ 
of the probability distribution of $\bX$,
we say that the data is \emph{unimodal};
if the CPD $(\Pr(Y=y|\bX=\bx))_{y\in[K]}$ is unimodal at any $\bx$ in 
a sub-domain $\calX_0\subseteq\calX$ with a large or small probability $\Pr(\bX\in\calX_0)$, 
we say that the data is \emph{almost unimodal} or \emph{almost non-unimodal}.
\citet{yamasaki2022unimodal} verified that many data, 
which have been treated as ordinal data by other previous OR 
studies \citep{chu2005gaussian, gutierrez2015ordinal},
tend almost unimodal.

As further preparation, we introduce qualitative notions regarding scale of the data as well:
For a PMF $\bp=(p_k)_{k\in[K]}\in\Delta_{K-1}$ having a mode $m\in[K]$,
we call the degree of spread of $\bp$ around $m$ the \emph{scale} of $\bp$.
Refer to Figure~\ref{fig:Scale};
we say that a more right PMF in Figure~\ref{fig:Scale} 
has a larger scale than a more left one.
If the scale of the CPD $\bp(\bx)=(\Pr(Y=y|\bX=\bx))_{y\in[K]}$ tends to be large or small
over whole the support of the probability distribution of the explanatory variable $\bX$, 
we say that the data has large scale or small scale.
Also, if the scale of the CPD $\bp(\bx)$ is similar or dissimilar over whole the support, 
we say that the data is \emph{homoscedastic} or \emph{heteroscedastic}.
In particular, we refer to heteroscedastic data as 
\emph{mode-wise heteroscedastic} or \emph{overall heteroscedastic} data,
if the scale of their CPD $\bp(\bx)$ is similar or can be dissimilar in 
each domain where the conditional mode $M(\bp(\bx))$ is the same.

\begin{figure}[t]
\renewcommand{\arraystretch}{0.1}\renewcommand{\tabcolsep}{5pt}\centering%
\begin{tabular}{ccc}\hskip-6pt
\includegraphics[width=4.8cm, bb=0 0 946 442]{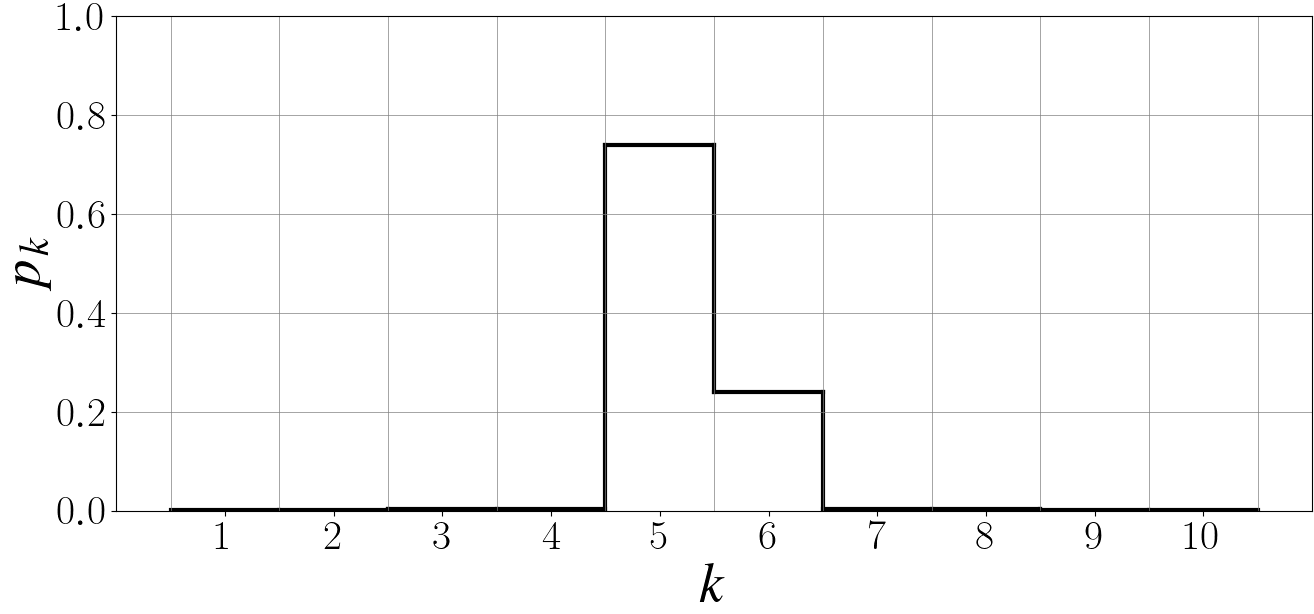}&\hskip-6pt
\includegraphics[width=4.8cm, bb=0 0 946 442]{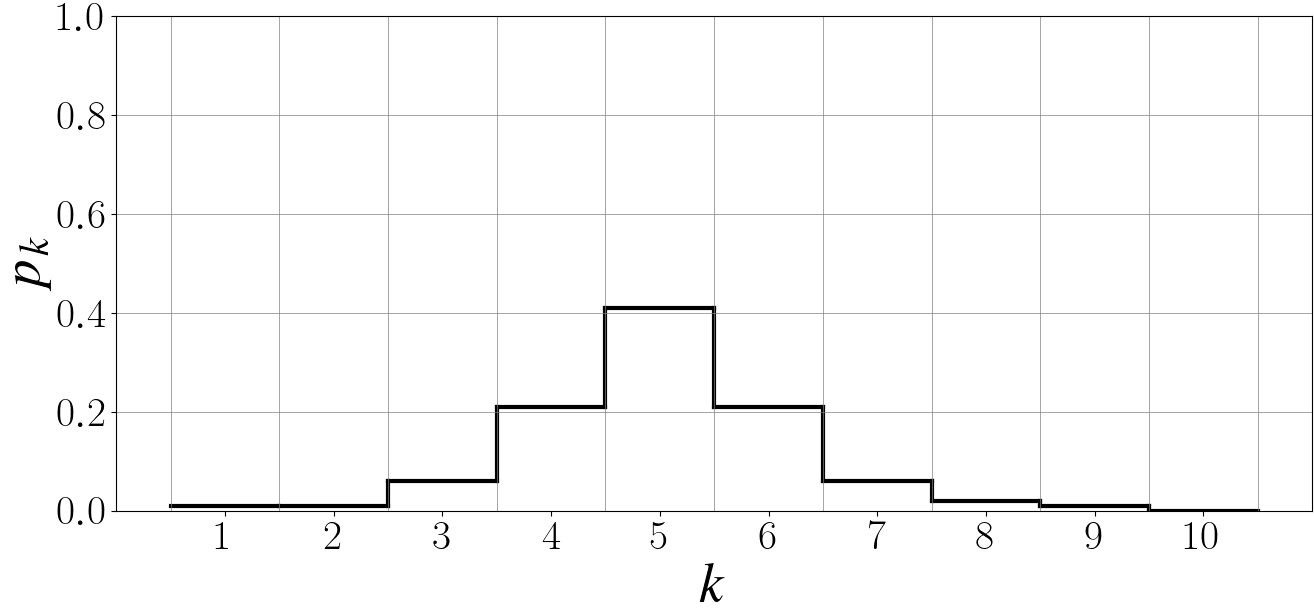}&\hskip-6pt
\includegraphics[width=4.8cm, bb=0 0 946 442]{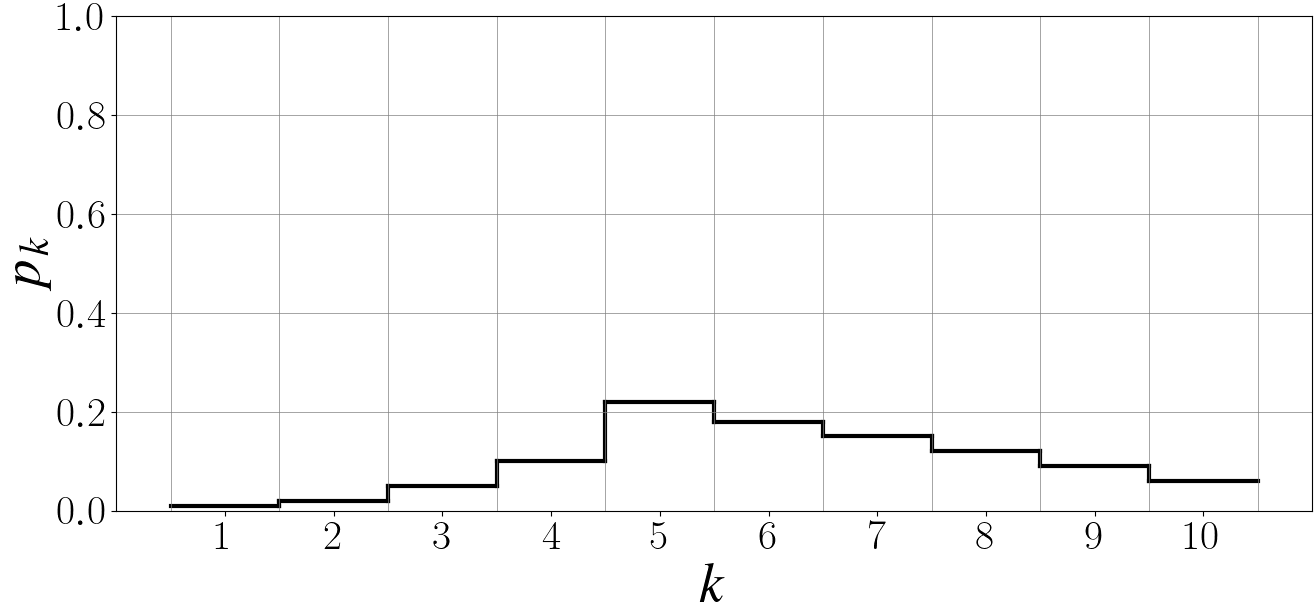}
\end{tabular}
\caption{%
Instances of the 10-dimensional unimodal PMF $\bp=(p_k)_{k\in[10]}$ with the mode 5.}
\label{fig:Scale}
\end{figure}

\subsection{Formulation of Threshold Methods}
\label{sec:Threshold}
\subsubsection{Overview}
\label{sec:OverviewTM}
A \emph{1DT-based method} \citep{yamasaki2022optimal} is formulated with two procedures, 
\emph{learning procedure} and \emph{labeling procedure} of the 1DT.
In the learning procedure, 1DT-based methods learn 
a 1DT (a real-valued function defined on $\bbR^d$) of 
the explanatory variable $\bX$ based on the data $(\bx_1,y_1),\ldots,(\bx_n,y_n)$.
In the labeling procedure, 
1DT-based methods construct a classifier $f$ as $f=h\circ\hat{a}$ 
with a learned 1DT $\hat{a}$ and a \emph{labeling function} $h:\bbR\to[K]$.
Many of those methods predict a label for $\bX=\bx$ as $f(\bx)=h_\thr(\hat{a}(\bx);\bt)$ 
using a \emph{threshold labeling function},
\begin{align}
\label{eq:ThrLab}
    h_\thr(u;\bt):= 1+\sum_{k=1}^{K-1}\bbI(u\ge t_k)
\end{align}
with a \emph{threshold parameter vector} $\bt=(t_k)_{k\in[K-1]}\in\bbR^{K-1}$.
In particular, such 1DT-based methods that employ a threshold labeling 
function are called threshold methods or threshold models, 
and have been studied actively in machine learning research
\citep{shashua2002taxonomy, chu2005new, rennie2005loss, lin2006large, 
lin2012reduction, pedregosa2017consistency, cao2020rank, yamasaki2022optimal}.
The succeeding two sections detail the learning and labeling procedures further.

\subsubsection{Learning Procedure}
\label{sec:Learning}
The learning procedure of many existing threshold methods is designed 
according to the idea of the \emph{surrogate risk minimization} 
\citep{pedregosa2017consistency} that allows for easier 
continuous optimization instead of directly minimizing the task risk.
In this work, we study threshold methods that can be seen 
to learn a 1DT $a:\bbR^d\to\bbR$ from a function class $\calA$ and a 
\emph{bias parameter vector} $\bb=(b_k)_{k\in[K-1]}$ from a class $\calB$ 
by minimizing the \emph{empirical surrogate risk} 
for a specified \emph{surrogate loss function} $\phi$,
\begin{align}
\label{eq:SRM}
    \min_{a\in\calA, \bb\in\calB}\frac{1}{n}\sum_{i=1}^n\phi(a(\bx_i),\bb,y_i)
\end{align}
(we denote the minimizer $(\hat{a}, \hat{\bb})$ 
with $\hat{\bb}:=(\hat{b}_k)_{k\in[K-1]}$)
or its regularized version.
Discussion in the main part of this paper focuses on 
the non-regularized (empirical) surrogate risk minimization;
see Appendix~\ref{sec:Proof} for supplemental discussion on the regularized version.

In this work, we consider popular surrogate loss functions 
in the class of \emph{all-threshold (AT) losses} or
\emph{immediate-threshold (IT) losses} \citep{pedregosa2017consistency}.
An AT loss is representable as
\begin{align}
\label{eq:AT}
    \phi(a(\bx),\bb,y)
    =\begin{cases}
    \sum_{k=1}^{K-1}\varphi(b_k-a(\bx))&\text{if~}y=1,\\
    \sum_{k=1}^{K-1}\varphi(a(\bx)-b_k)&\text{if~}y=K,\\
    \sum_{k=1}^{y-1}\varphi(a(\bx)-b_k)+\sum_{k=y}^{K-1}\varphi(b_k-a(\bx))&\text{otherwise},
    \end{cases}
\end{align}
and an IT loss is representable as
\begin{align}
\label{eq:IT}
    \phi(a(\bx),\bb,y)
    =\begin{cases}
    \varphi(b_1-a(\bx))&\text{if~}y=1,\\
    \varphi(a(\bx)-b_{K-1})&\text{if~}y=K,\\
    \varphi(a(\bx)-b_{y-1})+\varphi(b_y-a(\bx))&\text{otherwise},
    \end{cases}
\end{align}
with a certain function $\varphi:\bbR\to[0,\infty)$.
As a function $\varphi$, one often applies a convex loss function 
that is an upper bound of the function $u\mapsto\bbI(u\le 0)$
and commonly used in binary classification.
Table~\ref{tab:Loss} summarizes instances of the function $\varphi$ that we 
consider in this paper, corresponding AT and IT losses, and reference papers for them.
To represent these functions, we subscript their abbreviation name 
in lowercase to $\varphi$ and $\phi$, like $\varphi_\logi(u)=\log(1+e^{-u})$ 
and $\phi_\logiat$ and $\phi_\logiit$ for the corresponding AT and IT losses.

\begin{table}[t]
\renewcommand{\arraystretch}{0.75}\renewcommand{\tabcolsep}{5pt}\centering%
\caption{Name in this paper (its abbreviation) and functional form of $\varphi$, 
and reference and name therein of the corresponding AT/IT loss $\phi$.}
\label{tab:Loss}\scalebox{0.75}{\begin{tabular}{lll}
\toprule
\multicolumn{1}{c}{\scs Name of $\varphi$ in this paper (abbr.)} & 
\multicolumn{1}{c}{\scs Functional form of $\varphi(u)$} & 
\multicolumn{1}{c}{\scs Reference and name therein of AT/IT loss $\phi$} \\
\midrule
Logistic (logi) & $\log(1+e^{-u})$ & \citep{rennie2005loss}, CORAL \citep{cao2020rank} \\
Hinge (hing) & $(1-u)_+$ & SVOR-EXC/SVOR-IMC \citep{chu2005new} \\
Ramp (ramp) & $\min\{(1-u)_+,s\}$ with $s>0$ & \citep{pfannschmidt2020feature} \\
Smoothed-Hinge (smhi) & $1-2u~(u\le0), \{(1-u)_+\}^2~(u>0)$ & \citep{rennie2005loss} \\
Squared-Hinge (sqhi) & $\{(1-u)_+\}^2$ & AT/IT construction with modified least squares \citep{rennie2005loss} \\
Exponential (expo) & $e^{-u}$ & ORBoost-LR/ORBoost-ALL \citep{lin2006large, lin2012reduction, li2007ordinal} \\
Absolute (abso) & $|1-u|$ & -- \\
Squared (squa) & $(1-u)^2$ & Squared AT \citep{pedregosa2017consistency} \\
\bottomrule
\end{tabular}}
\end{table}

As a class of the 1DT $\calA$, one applies, for example, a linear model class 
$\{\bx\mapsto\bw^\top\bx+w_0\mid \bw\in\bbR^d, w_0\in\bbR\}$ and 
a neural network model with a fixed network architecture 
and learnable bias and weight parameters.
In this paper, we do not discuss the design of the class of the 1DT, 
and simply assume $\calA\subseteq\{a:\bbR^d\to\bbR\}$ 
in the following discussion.

The AT and IT surrogate losses 
$\phi(a(\bx),\bb,y)$ have the translation invariance:
$\phi(a(\bx),\bb,y)=\phi(a(\bx)+c,\bb+c\bm{1}_{K-1},y)$ with 
all-1 vector $\bm{1}_{K-1}\in\bbR^{K-1}$ for any $c\in\bbR$.
Also, many methods incorporate 
the order constraint $b_1\le\cdots\le b_{K-1}$ on the bias parameter vector $\bb$.
Thus, we study the popular instances of the class $\calB$,
$\calB_0:=\{\bb\in\bbR^{K-1}\mid b_1=0\}$ (\emph{non-ordered class}) and 
$\calB^\ord_0:=\{\bb\in\bbR^{K-1}\mid b_1=0\le b_2\le\cdots\le b_{K-1}\}$ (\emph{ordered class}).
The order constraint on the bias parameter vector $\bb$ can be implemented, 
for example, by $b_{k+1}=b_k+c_k^2$ for $k\in[K-2]$ with other parameters 
$c_1, \ldots,c_{K-2}\in\bbR$ \citep{franses2001quantitative}.

\subsubsection{Labeling Procedure}
\label{sec:Labeling}
The threshold labeling $h_\thr(u;\bt)$ is non-decreasing in $u\in\bbR$ 
and the simplest labeling function in the sense that the resulting 
1DT-based classifier has only $(K-1)$ decision boundaries 
for the learned 1DT at most \citep{yamasaki2022optimal}.
Partly because of this simplicity, the threshold labeling 
is preferred for many 1DT-based methods.
\citet{pedregosa2017consistency, cao2020rank} adopt
the learned bias parameter vector $\hat{\bb}$ to the threshold parameter vector $\bt$,
and \citet{shashua2003ranking, chu2007support} use 
the labeling function $h(u)=\min(\{k\in[K-1]\mid u<\hat{b}_k\}\cup\{K\})$.
The latter labeling function is also a certain threshold labeling,
which is same as (resp.\;different from) the former $h_\thr(\cdot;\hat{\bb})$ 
when $\hat{\bb}$ are ordered (resp.\;non-ordered)
\citep[Proposition 2]{yamasaki2022optimal}.
\citet{yamasaki2022optimal} proposed to use the 
\emph{empirical optimal threshold parameter vector} 
that minimizes the empirical task risk, 
$\hat{\bt}\in\argmin_{\bt\in\bbR^{K-1}}
\frac{1}{n}\sum_{i=1}^n\ell(h_\thr(\hat{a}(\bx_i);\bt), y_i)$ 
for the learned 1DT $\hat{a}$ and specified task loss $\ell$.
This minimization problem can be solved by 
a dynamic-programming-based algorithm \citep{lin2006large} 
or computationally efficient parallel algorithm \citep{yamasaki2024parallel} 
for tasks with a task loss function 
that is convex with respect to the first argument.
Also, \citet{yamasaki2022optimal} experimentally confirmed that 
this labeling function can lead to better classification performance
than other labeling functions including the above-mentioned two previous ones.
Therefore, we suppose to employ the empirical optimal 
threshold parameter vector $\hat{\bt}$ in this work.

\subsection{Research Objectives of This Work and Related Works}
\label{sec:Policy}
We are ultimately interested in the \emph{generalization error} 
(or \emph{prediction error}) $\bbE[\ell(f(\bX),Y)]$ for the classifier 
$f(\cdot)=\hat{f}(\cdot):= h_\thr(\hat{a}(\cdot);\hat{\bt})$ 
of the threshold method, where the learned 1DT $\hat{a}$ 
and learned threshold parameter vector $\hat{\bt}$ 
are obtained by the two-step optimization
\begin{align}
\label{eq:Estiab}
    (\hat{a},\hat{\bb})\in\underset{a\in\calA,\bb\in\calB}{\argmin}\,
    \frac{1}{n}\sum_{i=1}^n\phi(a(\bx_i),\bb,y_i),\quad
    \hat{\bt}\in\underset{\bt\in\bbR^{K-1}}{\argmin}\,
    \frac{1}{n}\sum_{i=1}^n\ell(h_\thr(\hat{a}(\bx_i);\bt), y_i).
\end{align}
The generalization error $\bbE[\ell(\hat{f}(\bX),Y)]$ is 
often understood through the decomposition into 
the \emph{approximation error} and \emph{estimation error},
\begin{align}
    \bbE[\ell(\hat{f}(\bX),Y)]
    =\underbrace{\bbE[\ell(\bar{f}(\bX),Y)]}_{\text{approximation error}}
    +\underbrace{\bbE[\ell(\hat{f}(\bX),Y)]-\bbE[\ell(\bar{f}(\bX),Y)]}_{\text{estimation error}},
\end{align}
like \citet[Section 5.2]{shalev2014understanding}, where 
the classifier $\bar{f}(\cdot):= h_\thr(\bar{a}(\cdot);\bar{\bt})$ 
is defined with the surrogate risk minimizer $(\bar{a},\bar{\bb})$
and the optimal threshold parameter vector $\bar{\bt}$:
\begin{align}
\label{eq:Idolab}
    (\bar{a},\bar{\bb})\in\underset{a\in\calA,\bb\in\calB}{\argmin}\,
    \bbE[\phi(a(\bX),\bb,Y)],\quad
    \bar{\bt}\in\underset{\bt\in\bbR^{K-1}}{\argmin}\,
    \bbE[\ell(h_\thr(\bar{a}(\bX);\bt),Y)].
\end{align}

Behaviors of the generalization error $\bbE[\ell(f(\bX),Y)]$ for 
a classifier of the threshold method have not been sufficiently understood.
One important theoretical analysis of the threshold method 
is \citet{lin2006large, lin2012reduction, li2007ordinal}:
they gave a sample-based probably-approximately-correct 
bound of the generalization error $\bbE[\ell(f(\bX),Y)]$ 
for a classifier $f(\cdot)=h_\thr(a(\cdot);\bt)$ 
with a boosting-based 1DT $a$ and ordered threshold parameter vector $\bt$.
Their bound has a term of the order $\calO(\log(n)/\sqrt{n})$ regarding 
the sample size $n$ and another sample-based constant-order term,
which are interpretable as corresponding to the estimation and approximation errors respectively.
However, the influence of the learning procedure of the threshold method on 
that constant-order approximation error term has not been sufficiently studied.

Therefore, in this work, we focus only on the approximation error as a first step.
In general, as the training sample size $n$ increases, 
the estimation error is expected to approach zero, 
and in such cases the approximation error becomes 
dominant in the generalization error.
Thus, we expect that our consideration will be 
suggestive for behaviors of the generalization error of 
the threshold method with the sufficiently large sample size $n$.

Typical binary or multi-class classification methods such as 
logistic regression have a guarantee that their surrogate risk 
minimization with the most representable model class possible 
(counterpart of $(\calA,\calB)$) yields a \emph{Bayes optimal classifier} 
(i.e., approximation error becomes the so-called \emph{Bayes error}
$\min_{f:\bbR^d\to[K]}\bbE[\ell(f(\bX),Y)]$) 
for any probability distribution of $(\bX,Y)$;
such a guarantee is known as \emph{Fisher consistency} or classification calibration
\citep{bartlett2006convexity, tewari2007consistency, liu2007fisher, pires2013cost}.
In contrast, most threshold methods do not, 
as described in \citet[Section 4]{pedregosa2017consistency},
namely, $\bbE[\ell(\bar{f}(\bX),Y)]$ defined via \eqref{eq:Idolab} is 
not necessarily equal to $\min_{f:\bbR^d\to[K]}\bbE[\ell(f(\bX),Y)]$
even if $\calA=\{a:\bbR^d\to\bbR\}$.
In other words, deviation of the approximation error of a threshold method from 
the Bayes error generally varies depending on the underlying data distribution
and its learning procedure.

Therefore, we study the dependence of the underlying data distribution and of 
the learning procedure of the 1DT on the approximation error of the threshold method, 
in order to qualitatively understand what threshold methods work well or poorly for what data.
It aims to advance the previous theoretical result \citep[Section 4]{pedregosa2017consistency} 
that just suggests that many threshold methods do not have Fisher consistency, 
and refine the experimental comparison of loss functions 
regarding the classification performance of the classifier $h_\thr(\hat{a}(\cdot);\hat{\bb})$ 
by \citet{rennie2005loss, gutierrez2015ordinal} 
to a theoretical comparison regarding the classification performance of the classifier $\hat{f}$.
Also, this study is the first attempt to discuss properties of threshold methods 
under a more in-depth characterization of the natural ordinal relation 
(here, the unimodal hypothesis) as far as we know.

\section{Analyses and Conjectures}
\label{sec:AC}
\subsection{All-Threshold (AT) Loss}
\label{sec:AT-Loss}
\subsubsection{Order of Optimized Bias Parameters}
\label{sec:AT-BP}
In Section~\ref{sec:AT-Loss}, 
we give analysis of the surrogate risk minimization of 
the threshold method that uses an AT surrogate loss function,
and conjecture about what data distributions 
the resulting classifier will perform well or poorly for.

We first discuss behaviors of the optimized bias parameter vector in Section~\ref{sec:AT-BP}.
The following theorem states that, for many popular AT surrogate loss functions, 
minimization of the (empirical) surrogate risk yields the bias parameter vector 
that satisfies the order condition even without imposing the order condition.
\begin{theorem}
\label{thm:AT-BP-Ord}
Let $\calA\subseteq\{a:\bbR^d\to\bbR\}$ and $\calB_0^\ord\subseteq\calB\subseteq\calB_0$
and introduce the conditions
\begin{itemize}\setlength{\parskip}{0pt}\setlength{\itemindent}{0pt}
\item[\hyt{a1}] $\varphi(u_1)-\varphi(-u_1)-\varphi(u_2)+\varphi(-u_2)\ge0$ for any $u_1,u_2\in\bbR$ s.t.\;$u_1<u_2$;
\item[\hyt{a2}] $\varphi(u)$ is non-increasing in $u\in\bbR$ (e.g., $\varphi=\varphi_\ramp$);
\item[\hyt{a3}] $\varphi(u)-\varphi(-u)$ is non-increasing in $u\in\bbR$ (e.g., $\varphi=\varphi_\abso$);
\item[\hyt{a4}] $\varphi(u_1)-\varphi(-u_1)-\varphi(u_2)+\varphi(-u_2)>0$ for any $u_1,u_2\in\bbR$ s.t.\;$u_1<u_2$;
\item[\hyt{a5}] $\varphi(u)$ is decreasing in $u\in\bbR$ (e.g., $\varphi=\varphi_\logi, \varphi_\expo$);
\item[\hyt{a6}] $\varphi(u)-\varphi(-u)$ is decreasing in $u\in\bbR$ 
(e.g., $\varphi=\varphi_\hing, \varphi_\smhi, \varphi_\sqhi, \varphi_\squa$);
\item[\hyt{a7}] $\{i\in[n]\mid y_i=y\}\neq\emptyset$ for all $y\in[K]$;
\item[\hyt{a8}] $\Pr(Y=y)>0$ for all $y\in[K]$.
\end{itemize}
Then, it holds that
\begin{itemize}\setlength{\parskip}{0pt}\setlength{\itemindent}{0pt}
\item[\hyt{a9}] 
\hyl{a1} holds if \hyl{a2} or \hyl{a3} holds,
and \hyl{a4} holds if \hyl{a5} or \hyl{a6} holds;
\item[\hyt{a10}] 
when $\phi$ is an AT loss \eqref{eq:AT} with $\varphi$ satisfying \hyl{a1},
there exists $(\hat{a},\hat{\bb})$ defined by \eqref{eq:Estiab} that 
satisfies $\hat{b}_1\le\cdots\le\hat{b}_{K-1}$ if \hyl{a7} holds,
and there exists $(\bar{a},\bar{\bb})$ defined by \eqref{eq:Idolab} that 
satisfies $\bar{b}_1\le\cdots\le\bar{b}_{K-1}$ if \hyl{a8} holds;
\item[\hyt{a11}] 
when $\phi$ is an AT loss \eqref{eq:AT} with $\varphi$ satisfying \hyl{a4},
any $(\hat{a},\hat{\bb})$ defined by \eqref{eq:Estiab} satisfies 
$\hat{b}_1\le\cdots\le\hat{b}_{K-1}$ if \hyl{a7} holds,
and any $(\bar{a},\bar{\bb})$ defined by \eqref{eq:Idolab} satisfies 
$\bar{b}_1\le\cdots\le\bar{b}_{K-1}$ if \hyl{a8} holds.
\end{itemize}
\end{theorem}

There are existing results similar to this theorem:
\citet{chu2005new} showed \hyl{a10}
for the Hinge-AT surrogate loss $\phi=\phi_\hingat$ in their Theorem 1,
and \citet{li2007ordinal} showed \hyl{a10} for $\varphi$ satisfying \hyl{a2}
in their Theorem 2.
We found more general conditions \hyl{a1} and \hyl{a4} and 
other sufficient conditions \hyl{a3} and \hyl{a6} for the general ones.
Consequently, Theorem~\ref{thm:AT-BP-Ord} newly 
covers the losses $\phi=\phi_\absoat, \phi_\squaat$.

This theoretical guarantee differs from \citet[Lemma 1]{crammer2002pranking} 
for a perceptron-based algorithm as also pointed out by \citet{chu2005new}: 
%
this lemma guarantees that a learned bias parameter vector 
always satisfies the order condition on the basis of the nature of 
their optimization procedure (initialization and update rule for perceptron),
while Theorem~\ref{thm:AT-BP-Ord} guarantees it on the basis of 
the global optimality in the (empirical) surrogate risk minimization.
Thus, it should be noted that, for example, when the bias parameter vector 
is learned by some iterative algorithm that adopts the empirical surrogate risk 
as the objective function under the setting of Theorem~\ref{thm:AT-BP-Ord}, 
it may be trapped by a local optimum and may not satisfy the order condition in practice.
%
So as to evade such an unexpected behavior, we hereafter adopt 
the ordered class of the bias parameter vector $\calB=\calB^\ord_0$ 
when we adopt an AT loss that satisfies \hyl{a1} or \hyl{a4}.

\subsubsection{Analysis for Special Instances: Logistic- and Exponential-AT Losses}
\label{sec:AT-LEloss}
We are interested in what data distributions a classifier obtained via the 
surrogate risk minimization with an AT loss will perform well or poorly for.
As one answer to this question, in this section,
we clarify special instances of the data distributions 
for which classifiers of the threshold methods based on 
the Logistic-AT or Exponential-AT loss perform well,
and analyze those special instances.

The following theorem shows that the threshold methods 
based on the Logistic-AT or Exponential-AT loss perform 
well when the data follow a certain likelihood model, 
which is well known as the (\emph{proportional odds}) 
\emph{cumulative logit} (\emph{CL}) \emph{model} 
in the OR research \citep{mccullagh1980regression, agresti2010analysis}.
\begin{theorem}
\label{thm:AT-Consistency}
Assume that the random variable $(\bX,Y)$ has 
the conditional probability of $Y=y$ at $\bX=\bx$,
\begin{align}
\label{eq:CL}
    \Pr(Y=y|\bX=\bx)=P_\cl(y;\tilde{a}(\bx), \tilde{\bb})
    :=
    \begin{cases}
    \frac{1}{1+e^{-(\tilde{b}_1-\tilde{a}(\bx))}}&\text{if }y=1,\\
    1-\frac{1}{1+e^{-(\tilde{b}_{K-1}-\tilde{a}(\bx))}}&\text{if }y=K,\\
    \frac{1}{1+e^{-(\tilde{b}_y-\tilde{a}(\bx))}}-
    \frac{1}{1+e^{-(\tilde{b}_{y-1}-\tilde{a}(\bx))}}&\text{otherwise},
    \end{cases}
\end{align}
with $\tilde{a}:\bbR^d\to\bbR$ and $\tilde{\bb}\in\calB_0^\ord$, 
for any $\bx$ in the support of the probability distribution of $\bX$ and every $y\in[K]$.
Let $\phi$ be the Logistic-AT loss ($\phi=\phi_\logiat$) or Exponential-AT loss ($\phi=\phi_\expoat$), 
and $\calA\times\calB\subseteq\{a:\bbR^d\to\bbR\}\times\calB_0^\ord$ 
include $(\tilde{a},\tilde{\bb})$ or $(\tilde{a}(\cdot)/2,\tilde{\bb}/2)$ 
respectively when $\phi=\phi_\logiat$ or $\phi=\phi_\expoat$.
Then, it holds that
\begin{itemize}\setlength{\parskip}{0pt}\setlength{\itemindent}{0pt}
\item[\hyt{b1}]
$(\bar{a},\bar{\bb})$ defined by \eqref{eq:Idolab} satisfies 
$(\bar{a}(\bX),\bar{\bb})=(\tilde{a}(\bX),\tilde{\bb})$ or 
$(\bar{a}(\bX),\bar{\bb})=(\tilde{a}(\bX)/2,\tilde{\bb}/2)$ almost surely
respectively when $\phi=\phi_\logiat$ or $\phi=\phi_\expoat$;
\item[\hyt{b2}]
$\bar{f}(\cdot)=h_\thr(\bar{a}(\cdot);\bar{\bt})$ 
with $\bar{a}$ and $\bar{\bt}$ defined by \eqref{eq:Idolab} satisfies 
$\bbE[\ell(\bar{f}(\bX),Y)]=\min_{f:\bbR^d\to[K]}\bbE[\ell(f(\bX),Y)]$,
if $\ell=\ell_\zo, \ell_{\zo,\epsilon}$ with $\epsilon\in[0,\lfloor K/2\rfloor)$
or if $\ell$ satisfies \hyl{b3} (e.g., $\ell=\ell_\ab, \ell_\sq$) described below.
\end{itemize}
Here, $\lfloor u\rfloor$ for $u\in\bbR$ is the greatest integer less than or equal to $u$,
and the condition \hyl{b3} is as follows:
\begin{itemize}\setlength{\parskip}{0pt}\setlength{\itemindent}{0pt}
\item[\hyt{b3}]
$\ell(k,l)$ at each fixed $k\in[K]$ is non-increasing in $l$ for $l\le k$ 
and non-decreasing in $l$ for $l\ge k$, and 
$\ell(k,j)-\ell(k,j+1)-\ell(l,j)+\ell(l,j+1)$ at each fixed 
$k,l\in[K]$ s.t.\;$k<l$ is non-positive for all $j\in[K-1]$.
\end{itemize}
\end{theorem}
\noindent%
For the Logistic-AT surrogate loss $\phi=\phi_\logiat$,
the statement \hyl{b1} is proved in \citet[Theorem 2]{yamasaki2022optimal},
and the statement \hyl{b2} relies on \citet[Theorems 3 and 4]{yamasaki2022optimal}.
The statements for the Exponential-AT surrogate loss $\phi=\phi_\expoat$ are novel findings.

For the CL model \eqref{eq:CL}, we found the following properties.
\begin{theorem}
\label{thm:POCL-shape}
Assume $\bb\in\calB_0^\ord$.
Then, it holds that 
\begin{itemize}\setlength{\parskip}{0pt}\setlength{\itemindent}{0pt}
\item[\hyt{c1}]
$\sum_{l=1}^k P_\cl(l;u,\bb)=\frac{1}{1+e^{-(b_k-u)}}$ is decreasing in $u$, 
and $\sum_{l=k+1}^K P_\cl(l;u,\bb)=\frac{1}{1+e^{(b_k-u)}}$ is increasing in $u$,
for $k=1,\ldots,K-1$;
\item[\hyt{c2}]
$P_\cl(1;u,\bb)\to1$ and $P_\cl(k;u,\bb)\to0$ for $k=2,\ldots,K$ as $u\to-\infty$, 
and $P_\cl(K;u,\bb)\to1$ and $P_\cl(k;u,\bb)\to0$ for $k=1,\ldots,K-1$ as $u\to+\infty$;
\end{itemize}
and that, for each $k\in\{2,\ldots,K-1\}$, 
\begin{itemize}\setlength{\parskip}{0pt}\setlength{\itemindent}{0pt}
\item[\hyt{c3}]
$P_\cl(k;u,\bb)$ is symmetric in $u$ around $u=\frac{b_{k-1}+b_k}{2}$;
\item[\hyt{c4}]
$P_\cl(k;u,\bb)$ is maximized in $u$ at $u=\frac{b_{k-1}+b_k}{2}$, 
increasing in $u$ if $u<\frac{b_{k-1}+b_k}{2}$, 
and decreasing in $u$ if $u>\frac{b_{k-1}+b_k}{2}$;
\item[\hyt{c5}]
$P_\cl(k;u,\bb)$ with fixed $u$ and $\frac{b_{k-1}+b_k}{2}$ is increasing in $b_k-b_{k-1}$.
\end{itemize}
Assume further that $\bb=\bb^{[\Delta]}$ for some $\Delta>0$, 
and that $b_{m-1}\le u<b_m$ for some $m\in[K]$, 
where $\bb^{[\Delta]}:=\Delta\cdot(0,1,\ldots,K-2)^\top$, 
$b_0:=-\infty$, and $b_K:=+\infty$.
Then, it holds that 
\begin{itemize}\setlength{\parskip}{0pt}\setlength{\itemindent}{0pt}
\item[\hyt{c6}]
$P_\cl(2;u,\bb)\le\cdots\le P_\cl(m;u,\bb)$ if $m\ge 2$, and
$P_\cl(m;u,\bb)\ge\cdots\ge P_\cl(K-1;u,\bb)$ if $m\le K-1$;
\item[\hyt{c7}]
$P_\cl(1;u,\bb)\le\cdots\le P_\cl(m;u,\bb)$ if $m\ge 2$ and $\Delta\ge\Delta_1$, and
$P_\cl(m;u,\bb)\ge\cdots\ge P_\cl(K;u,\bb)$ if $m\le K-1$ and $\Delta\ge\Delta_2$,
where $\Delta_1$ is $\Delta$ satisfying $\frac{2}{1+e^{-(b_1-u)}}=\frac{1}{1+e^{-(b_2-u)}}$ and 
$\Delta_2$ is $\Delta$ satisfying $\frac{2}{1+e^{(b_{K-1}-u)}}=\frac{1}{1+e^{(b_{K-2}-u)}}$, 
and where $\Delta_1$ and $\Delta_2$ are uniquely determined for each $u$.
\end{itemize}
\end{theorem}
\noindent%
The results \hyl{c1} and \hyl{c2} would be trivial,
and the results \hyl{c3}--\hyl{c5} have been described in \citet[Lemma 1]{yamasaki2022optimal},
while the results \hyl{c6} and \hyl{c7} are novel and important for our subsequent discussion.

\begin{figure}[t]
\renewcommand{\arraystretch}{0.1}\renewcommand{\tabcolsep}{5pt}\centering%
\begin{tabular}{ccc}\hskip-6pt
{\tiny\hyt{d1} $\bb=\bb^{[\Delta]}$, $\Delta=1/3$}&\hskip-6pt
{\tiny\hyt{d2} $\bb=\bb^{[\Delta]}$, $\Delta=1$}&\hskip-6pt
{\tiny\hyt{d3} $\bb=\bb^{[\Delta]}$, $\Delta=3$}\\\hskip-6pt
\includegraphics[width=4.8cm, bb=0 0 946 442]{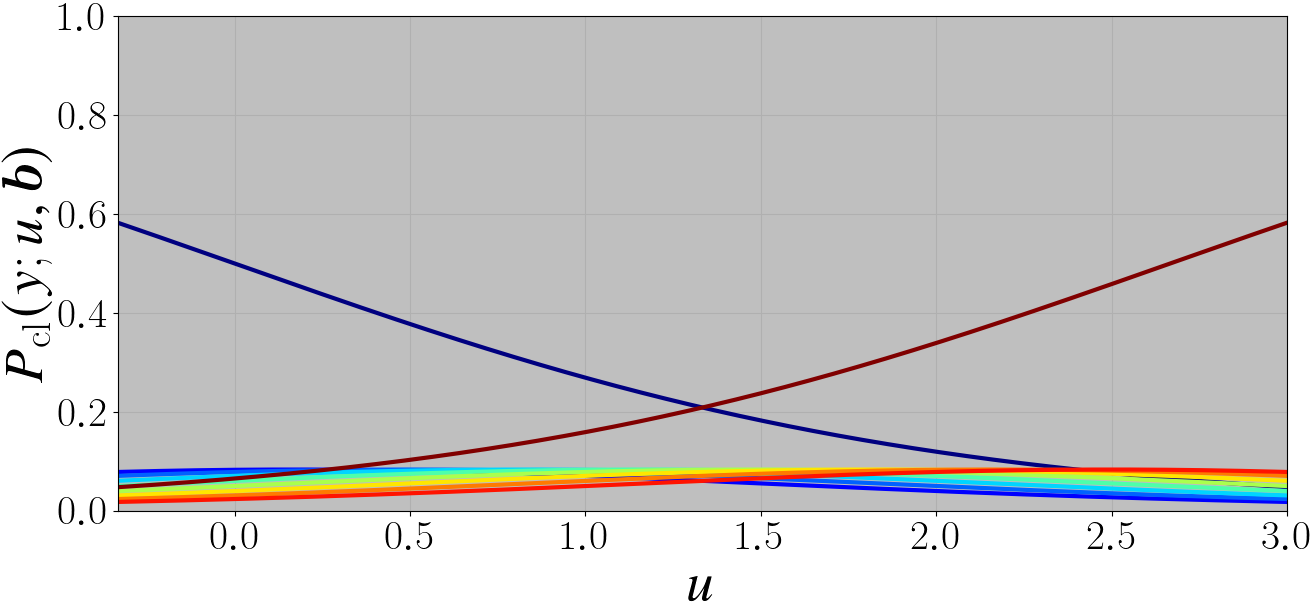}&\hskip-6pt
\includegraphics[width=4.8cm, bb=0 0 946 442]{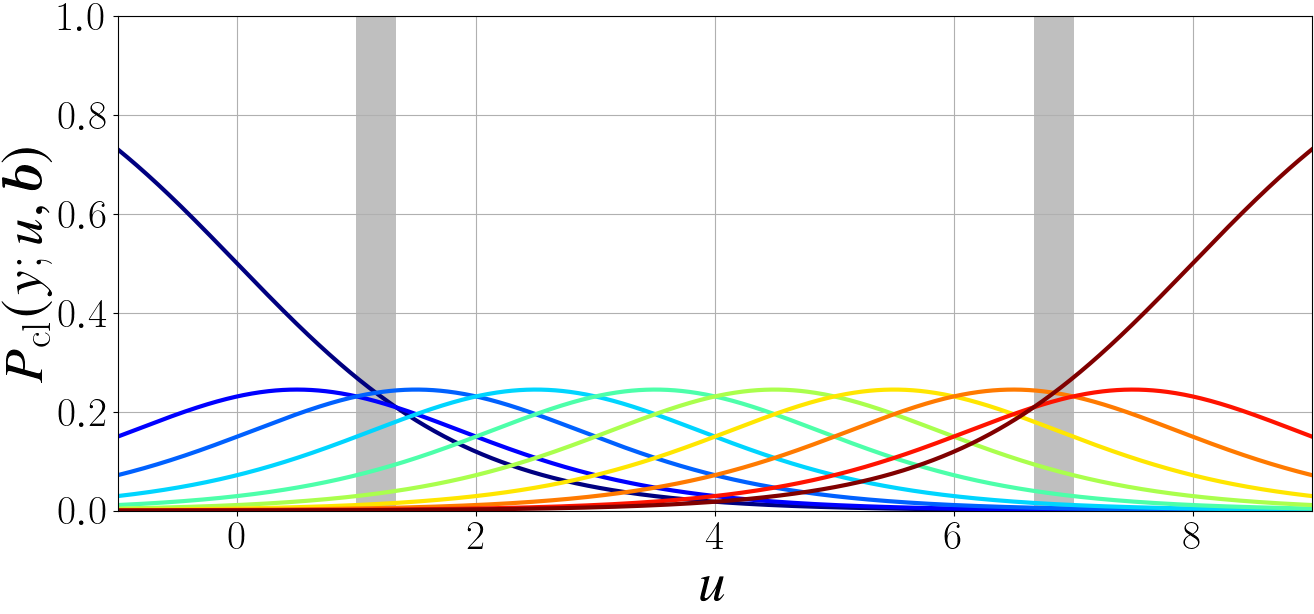}&\hskip-6pt
\includegraphics[width=4.8cm, bb=0 0 946 442]{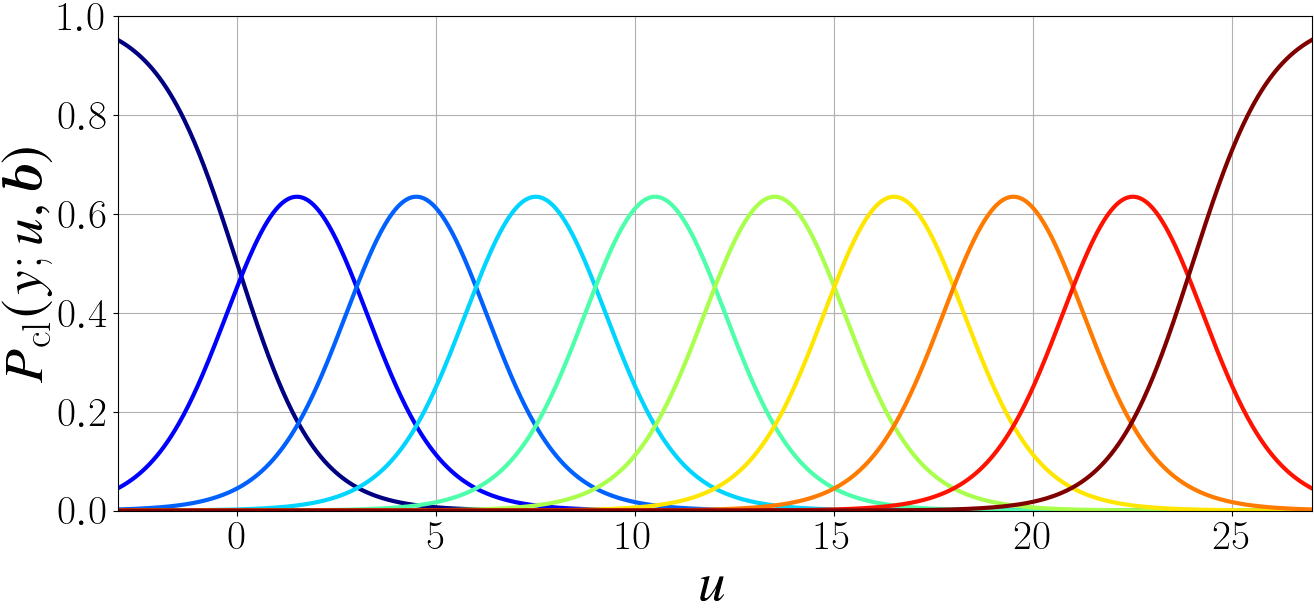}\\\hskip-6pt
{\tiny\hyt{e1} $\bb=\acute{\bb}^{[\Delta]}$, $\Delta=1/3$}&\hskip-6pt
{\tiny\hyt{e2} $\bb=\acute{\bb}^{[\Delta]}$, $\Delta=1$}&\hskip-6pt
{\tiny\hyt{e3} $\bb=\acute{\bb}^{[\Delta]}$, $\Delta=3$}\\\hskip-6pt
\includegraphics[width=4.8cm, bb=0 0 946 442]{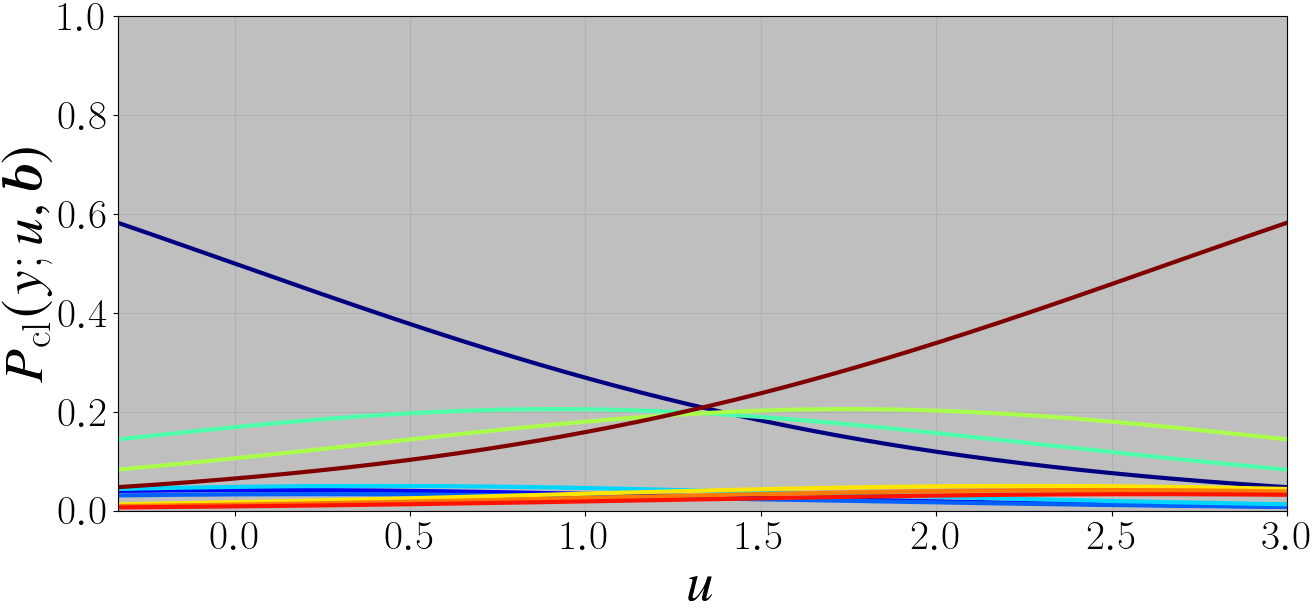}&\hskip-6pt
\includegraphics[width=4.8cm, bb=0 0 946 442]{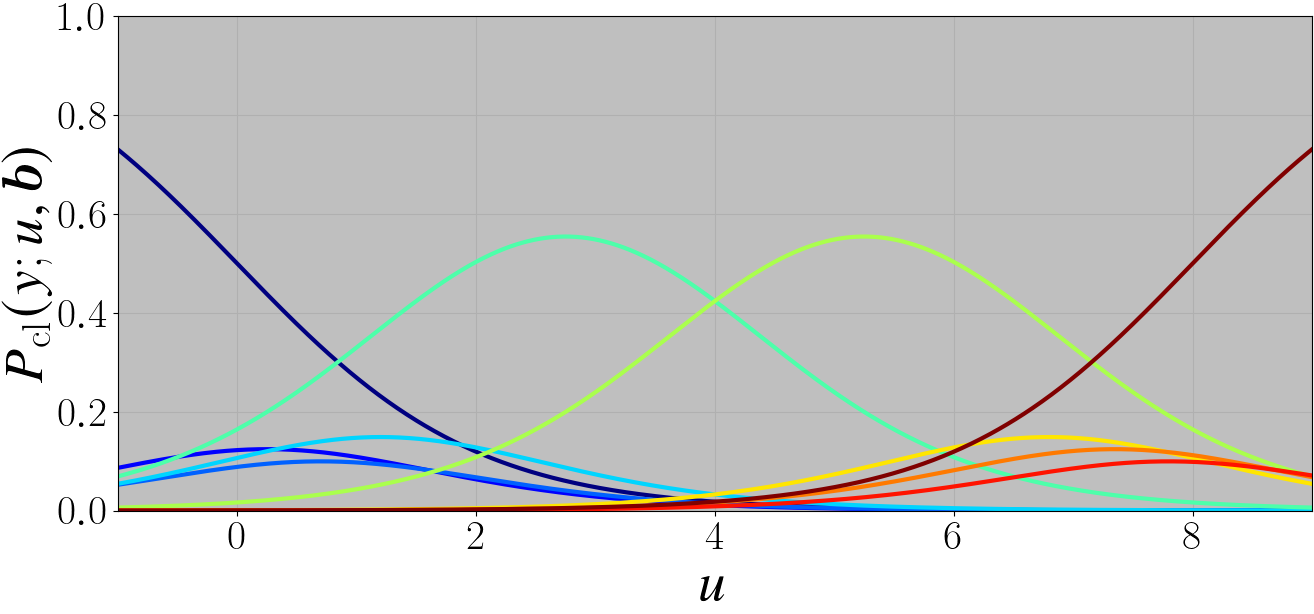}&\hskip-6pt
\includegraphics[width=4.8cm, bb=0 0 946 442]{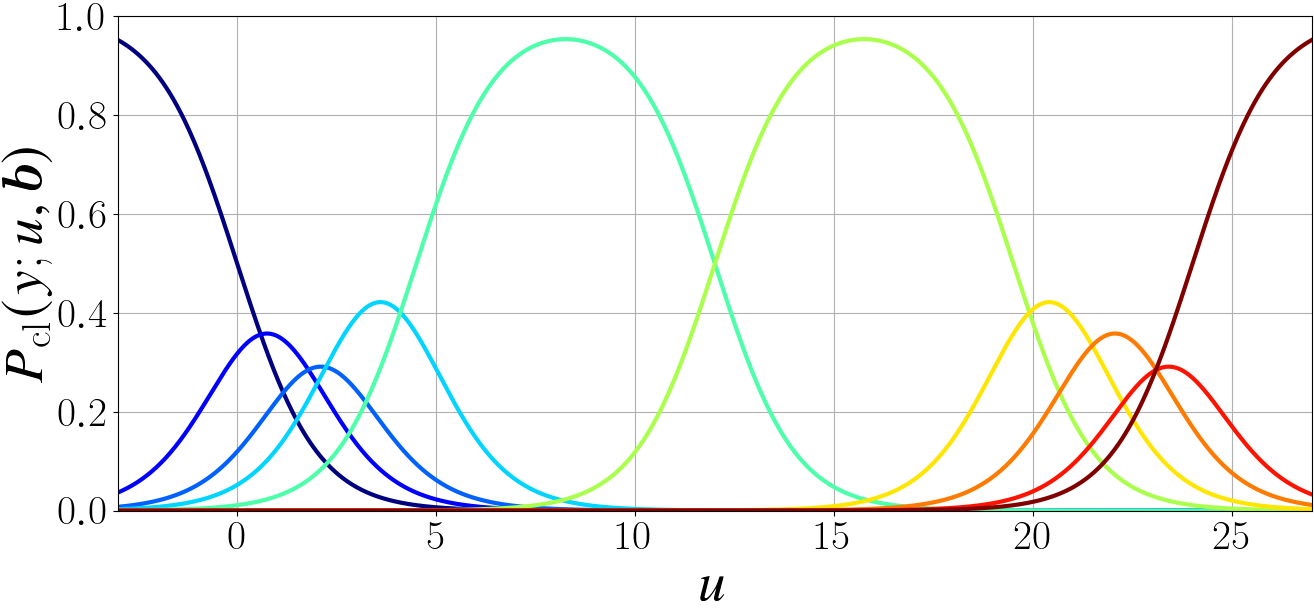}
\end{tabular}\\\hskip-6pt
\includegraphics[width=10cm, bb=0 0 1584 58]{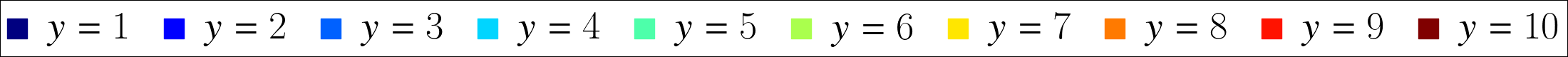}
\caption{%
Instances of the CL model of $K=10$.
Figures show $P_\cl(y;u,\bb)$ for $y=1,\ldots,10$, $u\in[-\Delta,9\Delta]$, 
$\bb=\bb^{[\Delta]}$ (plates~\protect\hyl{d1}--\protect\hyl{d3}), 
$\acute{\bb}^{[\Delta]}$ (plates~\protect\hyl{e1}--\protect\hyl{e3}), 
and $\Delta=1/3,1,3$.
At $u$ in the region where the background color is white or gray, 
the PMF $(P_\cl(y;u,\bb))_{y\in[10]}$ is unimodal or not.}
\label{fig:CL}
\end{figure}

Figure~\ref{fig:CL} shows instances of the CL model 
$(P_\cl(y;u,\bb))_{y\in[K]}$ of $K=10$ with the ordered equal-interval bias parameter vector 
$\bb=\bb^{[\Delta]}=\Delta\cdot(0,1,\ldots,8)^\top$ and ordered unequal-interval bias parameter vector 
$\bb=\acute{\bb}^{[\Delta]}:=\Delta\cdot(0,0.5,0.9,1.5,4,6.5,7.1,7.6,8)^\top$ 
for $\Delta=1/3,1,3$.
As Theorem~\ref{thm:POCL-shape} \hyl{c6} and \hyl{c7} show,
the CL model with the ordered equal-interval bias parameter vector $\bb=\bb^{[\Delta]}$
is unimodal when the gap $\Delta$ of adjacent bias parameters is large;
see Figure~\ref{fig:CL} \hyl{d1}--\hyl{d3}.
Therefore, one can find that the CL model is well suited 
for representing homoscedastic unimodal data 
(especially, data having a CPD in which the scale tends to be small overall).
The CL model with the ordered unequal-interval bias parameter vector $\bb=\acute{\bb}^{[\Delta]}$
shown in Figure~\ref{fig:CL} \hyl{e1}--\hyl{e3} are mode-wise 
heteroscedastic, and the CPD with $\Delta=3$ is unimodal overall.
This result suggests that the CL model can represent 
mode-wise heteroscedastic unimodal data as well.
On the other hand, it should be noted that the CL model 
cannot represent overall heteroscedastic data, for example, 
data that follow the large-scale CPD in Figure~\ref{fig:CL} \hyl{d2} in a certain domain 
and the small-scale CPD in Figure~\ref{fig:CL} \hyl{d3} in another domain.

\subsubsection{Analysis for a Special Instance: Squared-AT Loss}
\label{sec:AT-SQloss}
The Absolute- and Squared-AT losses adopt a non-monotonic function $\varphi$, 
in which the shape is not similar to those we analyzed in Section~\ref{sec:AT-LEloss}.
Therefore, analogies from behaviors of the surrogate risk minimizer 
with the Logistic- and Exponential-AT losses may not be convincing 
in understanding those with the Absolute- and Squared-AT losses.
We thus studied behaviors of the surrogate risk minimizer 
with the Squared-AT loss and the corresponding classifier, 
and obtained the following result
(see Section~\ref{sec:PL-Loss} for results for the Absolute-AT loss):
\begin{theorem}
\label{thm:ATSQ-Consistency}
Let $\phi$ be the Squared-AT loss ($\phi=\phi_\squaat$), 
and $\calA\times\calB\subseteq\{a:\bbR^d\to\bbR\}\times\calB_0^\ord$ 
include $(\tilde{a},\tilde{\bb})$ that satisfies, 
with $c_1:= \Pr(Y\le 1)$ and $c_2:=1+\frac{2}{K-1}(\sum_{y=1}^{K-1}\Pr(Y\le y)-K)-c_1$,
\begin{align}
\label{eq:ATSQSRM}
    \tilde{a}(\bx)
    =c_2+\frac{2}{K-1}\sum_{y=1}^{K-1}y\Pr(Y=y|\bX=\bx),\quad
    \tilde{b}_y
    =\Pr(Y\le y)-c_1
\end{align}
for any $\bx$ in the support of the probability distribution of $\bX$ and every $y\in[K]$.
Then, $(\bar{a},\bar{\bb})$ defined by \eqref{eq:Idolab} satisfies 
$(\bar{a}(\bX),\bar{\bb})=(\tilde{a}(\bX),\tilde{\bb})$ almost surely.
Also, $\bar{f}(\cdot)=h_\thr(\bar{a}(\cdot);\bar{\bt})$ 
with $\bar{\bt}$ defined by
\begin{align}
    \bar{t}_k=c_2+\frac{2}{K-1}(k+0.5)\text{ for }k\in[K-1]
\end{align}
satisfies $\bbE[\ell(\bar{f}(\bX),Y)]=\min_{f:\bbR^d\to[K]}\bbE[\ell(f(\bX),Y)]$,
if $\ell=\ell_\sq$.
\end{theorem}
\noindent%
This result guarantees the optimality of the classifier of the threshold method
independent of the underlying data distribution
(only for the case with $(\phi,\ell)=(\phi_\squaat,\ell_\sq)$).
Even in cases other than $\ell=\ell_\sq$, 
which are not covered by this theorem, 
the expression of the surrogate risk minimizer suggests that 
the threshold method with $\phi=\phi_\squaat$ works well 
when the data follow a unimodal CPD with a small scale:
Consider that, when $\Pr(Y=y|\bX=\bx)\approx1$ if $y=M((\Pr(Y=y|\bX=\bx))_{y\in[K]})$ and $0$ otherwise,
one has $\bar{a}(\bx)\approx c_2+\frac{2}{K-1}M((\Pr(Y=y|\bX=\bx))_{y\in[K]})$
(on the other hand, when $\Pr(Y=y|\bX=\bx)\approx\frac{1}{K}$ 
for all $y\in[K]$, one has $\bar{a}(\bx)\approx c_2+1$).

\subsubsection{Conjecture}
\label{sec:AT-Conj}
Theorem~\ref{thm:AT-Consistency} guarantees 
the optimality of the surrogate risk minimizer and 
of the resulting classifier under several popular tasks 
only for the Logistic-AT and Exponential-AT losses 
and restricted class of the data distribution.
However, we conjecture that the goodness of the surrogate risk minimizer 
and resulting classifier will be maintained to some extent even if 
the data distribution and loss function change their shape somewhat.
Thus, from Theorem~\ref{thm:POCL-shape}, we conjecture that, 
for AT losses based on a non-increasing function $\varphi$
that is similar to the Logistic and Exponential losses and
has no special property (see Section~\ref{sec:PL-Loss}),
a classifier of threshold methods tends to perform well for 
homoscedastic unimodal data and mode-wise heteroscedastic unimodal data.
In contrast, such threshold methods may perform poorly for 
overall heteroscedastic unimodal data and almost non-unimodal data.
Also, Theorem~\ref{thm:ATSQ-Consistency} gives a closed-form expression 
of the surrogate risk minimizer based on $\phi=\phi_\squaat$ 
and guarantees the optimality of the resulting classifier independent 
of the data distribution for the task with $\ell=\ell_\sq$.
According to that expression, we conjecture that the learning procedure 
with an AT loss based on a V-shaped function $\varphi$ like $\varphi_\squa$
(with no special property; see Section~\ref{sec:PL-Loss}) also 
performs well when the data follow a unimodal CPD with a small scale.

\subsection{Immediate-Threshold (IT) Loss}
\label{sec:IT-Loss}
\subsubsection{Analysis for Special Instances: Logistic- and Exponential-IT Losses}
\label{sec:IT-LEloss}
In Section~\ref{sec:IT-Loss}, 
we give analysis of the surrogate risk minimization of 
the threshold method with an IT surrogate loss function,
and conjecture about what data distributions 
the resulting classifier will perform well or poorly for.

The following theorem suggests that 
threshold methods based on the Logistic-IT or Exponential-IT loss 
perform well when the data follow a certain likelihood model, 
known as the (\emph{proportional odds}) \emph{adjacent categories logit} 
(\emph{ACL}) \emph{model} \citep{agresti2010analysis, yamasaki2022unimodal}.
\begin{theorem}
\label{thm:IT-Consistency}
Assume that the random variable $(\bX,Y)$ has 
the conditional probability of $Y=y$ at $\bX=\bx$,
\begin{align}
\label{eq:ACL}
    \Pr(Y=y|\bX=\bx)=P_\acl(y;\tilde{a}(\bx), \tilde{\bb}):=
    \frac{e^{-\sum_{k=1}^{y-1}(\tilde{b}_k-\tilde{a}(\bx))}}{\sum_{l=1}^Ke^{-\sum_{k=1}^{l-1}(\tilde{b}_k-\tilde{a}(\bx))}},
\end{align}
with $\tilde{a}:\bbR^d\to\bbR$ and $\tilde{\bb}\in\calB_0$, 
for any $\bx$ in the support of the probability distribution 
of $\bX$ and every $y\in[K]$.
Let $\phi$ be the Logistic-IT loss ($\phi=\phi_\logiit$) or Exponential-IT loss ($\phi=\phi_\expoit$), 
and $\calA\times\calB\subseteq\{a:\bbR^d\to\bbR\}\times\calB_0$ 
include $(\tilde{a},\tilde{\bb})$ or $(\tilde{a}(\cdot)/2,\tilde{\bb}/2)$ 
respectively when $\phi=\phi_\logiit$ or $\phi=\phi_\expoit$.
Then, it holds that
\begin{itemize}\setlength{\parskip}{0pt}\setlength{\itemindent}{0pt}
\item[\hyt{f1}]
$(\bar{a},\bar{\bb})$ defined by \eqref{eq:Idolab} satisfies 
$(\bar{a}(\bX),\bar{\bb})=(\tilde{a}(\bX),\tilde{\bb})$ or 
$(\bar{a}(\bX),\bar{\bb})=(\tilde{a}(\bX)/2,\tilde{\bb}/2)$ almost surely
respectively when $\phi=\phi_\logiit$ or $\phi=\phi_\expoit$;
\item[\hyt{f2}]
$\bar{f}=h_\thr(\bar{a}(\cdot);\bar{\bt})$ with 
$\bar{a}$ and $\bar{\bt}$ defined by \eqref{eq:Idolab} satisfies
$\bbE[\ell(\bar{f}(\bX),Y)]=\min_{f:\bbR^d\to[K]}\bbE[\ell(f(\bX),Y)]$,
if $\ell=\ell_\zo$;
\item[\hyt{f3}]
$\bbE[\ell(h_\thr(\bar{a}(\bX);\bar{\bt}),Y)]=\bbE[\ell(h_\thr(\bar{a}(\bX);\bar{\bb}),Y)]$
with $(\bar{a},\bar{\bb})$ and $\bar{\bt}$ defined by \eqref{eq:Idolab},
if $\ell=\ell_\zo$ and $\tilde{\bb}\in\calB_0^\ord$.
\end{itemize}
Alternatively let $\calA=\{a:\bbR^d\to\bbR\}$ and $\calB=\calB_0^\ord$.
Then, it holds that
\begin{itemize}\setlength{\parskip}{0pt}\setlength{\itemindent}{0pt}
\item[\hyt{f4}]
$\bar{\bb}$ defined by \eqref{eq:Idolab} has overlapped components
(i.e., $\bar{b}_k=\bar{b}_l$ for some $k,l\in[K-1]$ s.t.\;$k\neq l$),
if $\tilde{\bb}\not\in\calB_0^\ord$.
\end{itemize}
\end{theorem}
\noindent%
To our knowledge, this type of result is the first for IT losses.

For the ACL model \eqref{eq:ACL}, we found the following properties.
\begin{theorem}
\label{thm:POACL-shape}
It holds that, with $b_0:=-\infty$ and $b_K:=+\infty$,
\begin{itemize}\setlength{\parskip}{0pt}\setlength{\itemindent}{0pt}
\item[\hyt{g1}]
for each $k\in[K-1]$,
$P_\acl(k+1;u,\bb)/P_\acl(k;u,\bb)
=(\frac{1}{1+e^{(b_k-u)}})/(\frac{1}{1+e^{-(b_k-u)}})
=e^{-(b_k-u)}$;
\item[\hyt{g2}]
$(P_\acl(y;u,\bb))_{y\in[K]}$ is unimodal at any $u\in\bbR$
(i.e., $P_\acl(1;u,\bb)\le\cdots\le P_\acl(m;u,\bb)$ 
and $P_\acl(m;u,\bb)\ge\cdots\ge P_\acl(K;u,\bb)$
when $b_{m-1}\le u\le b_m$ for $m\in[K]$),
if $\bb\in\calB_0^\ord$;
\item[\hyt{g3}]
$(P_\acl(y;u,\bb))_{y\in[K]}$ is unimodal,
if $\max(\{1\}\cup\{k\in[K-1]\mid b_k\le u\})\le\min(\{k\in[K-1]\mid b_k\ge u\}\cup\{K\})$;
\item[\hyt{g4}]
$(P_\acl(y;u,\bb))_{y\in[K]}$ is not unimodal at any $u\in(b_l,b_k)$,
if $b_k>b_l$ for $k,l\in[K-1]$ s.t.\;$k<l$.
\end{itemize}
\end{theorem}
\noindent%
The results \hyl{g1} and \hyl{g2} are shown in 
\citet[Theorem 3]{yamasaki2022unimodal},
while the results \hyl{g3} and \hyl{g4} are novel findings.

\begin{figure}[t]
\renewcommand{\arraystretch}{0.1}\renewcommand{\tabcolsep}{5pt}\centering%
\begin{tabular}{ccc}\hskip-6pt
{\tiny\hyt{h1} $\bb=\bb^{[\Delta]}$, $\Delta=1/3$}&\hskip-6pt
{\tiny\hyt{h2} $\bb=\bb^{[\Delta]}$, $\Delta=1$}&\hskip-6pt
{\tiny\hyt{h3} $\bb=\bb^{[\Delta]}$, $\Delta=3$}\\\hskip-6pt
\includegraphics[width=4.8cm, bb=0 0 946 442]{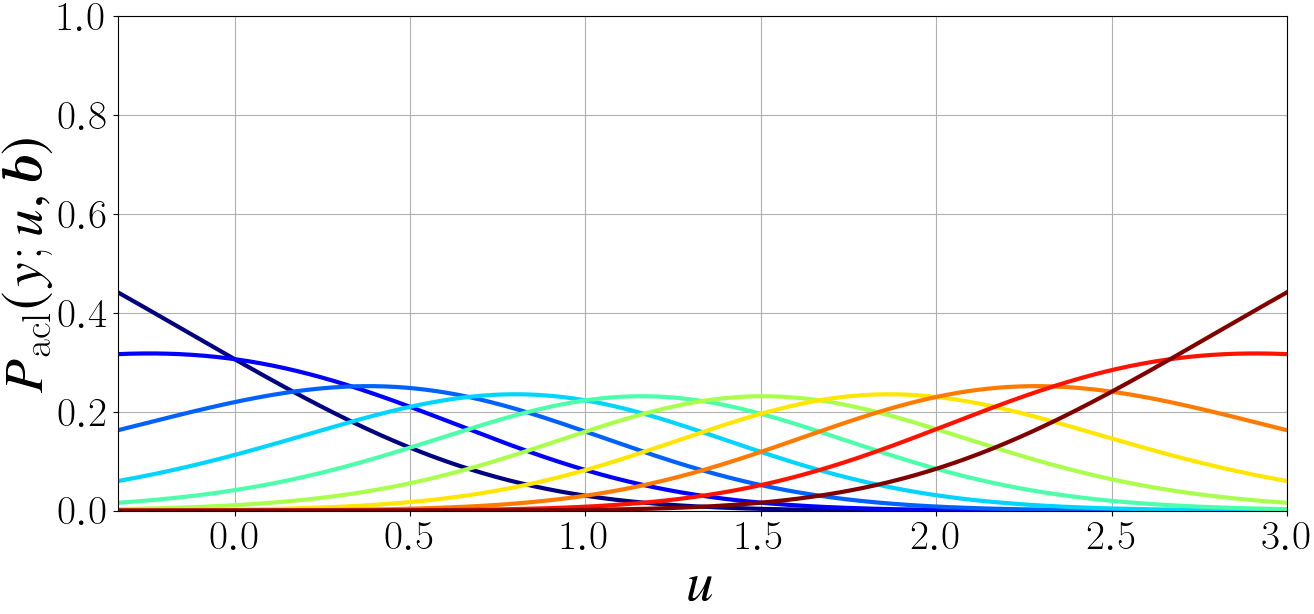}&\hskip-6pt
\includegraphics[width=4.8cm, bb=0 0 946 442]{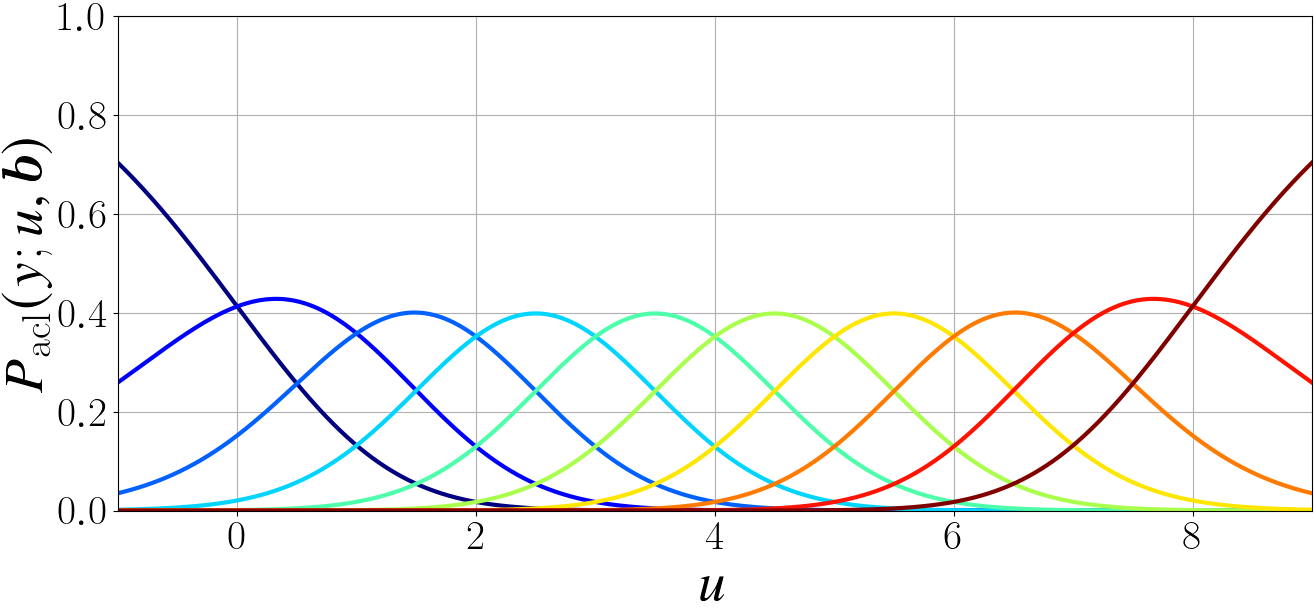}&\hskip-6pt
\includegraphics[width=4.8cm, bb=0 0 946 442]{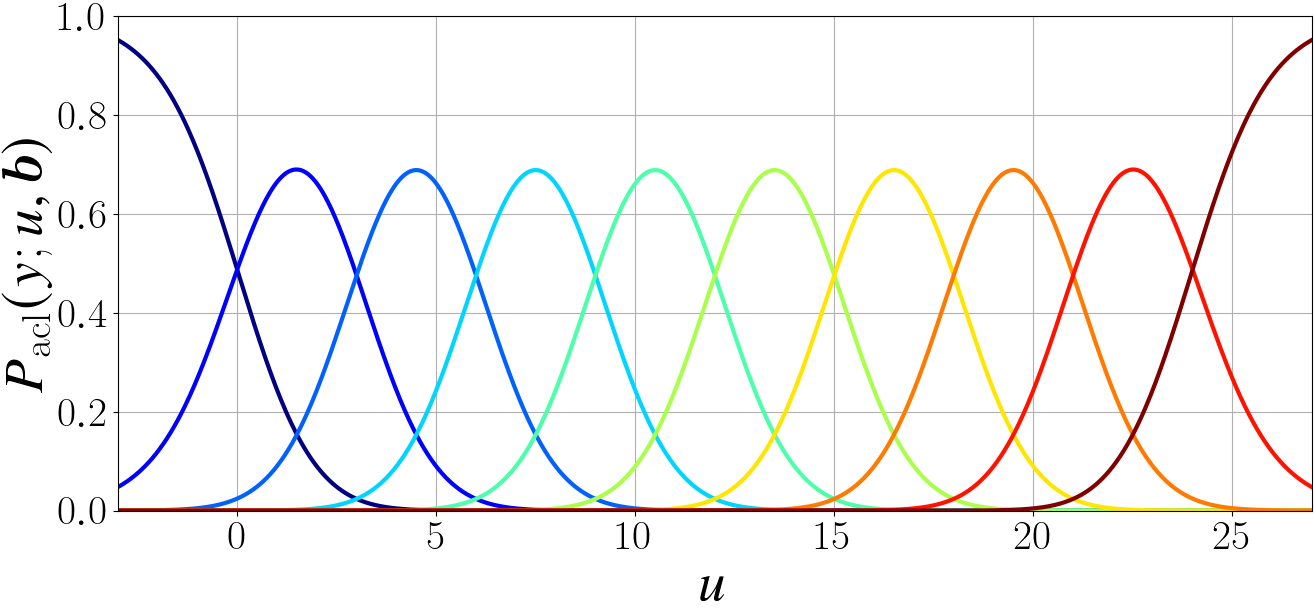}\\\hskip-6pt
{\tiny\hyt{i1} $\bb=\acute{\bb}^{[\Delta]}$, $\Delta=1/3$}&\hskip-6pt
{\tiny\hyt{i2} $\bb=\acute{\bb}^{[\Delta]}$, $\Delta=1$}&\hskip-6pt
{\tiny\hyt{i3} $\bb=\acute{\bb}^{[\Delta]}$, $\Delta=3$}\\\hskip-6pt
\includegraphics[width=4.8cm, bb=0 0 946 442]{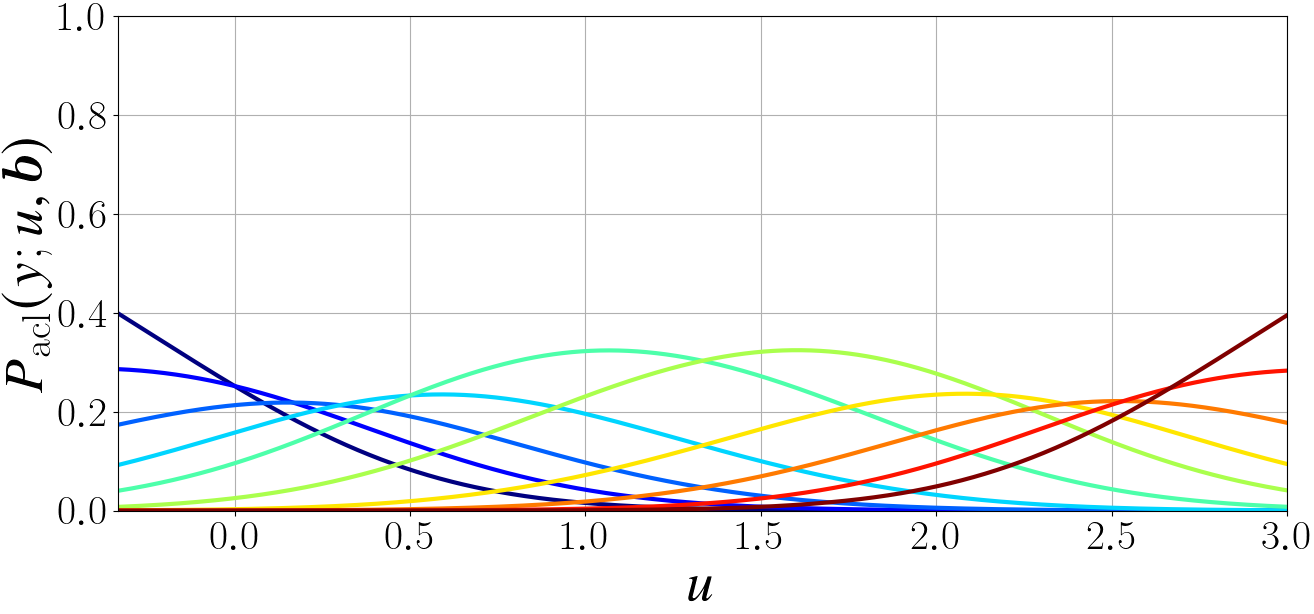}&\hskip-6pt
\includegraphics[width=4.8cm, bb=0 0 946 442]{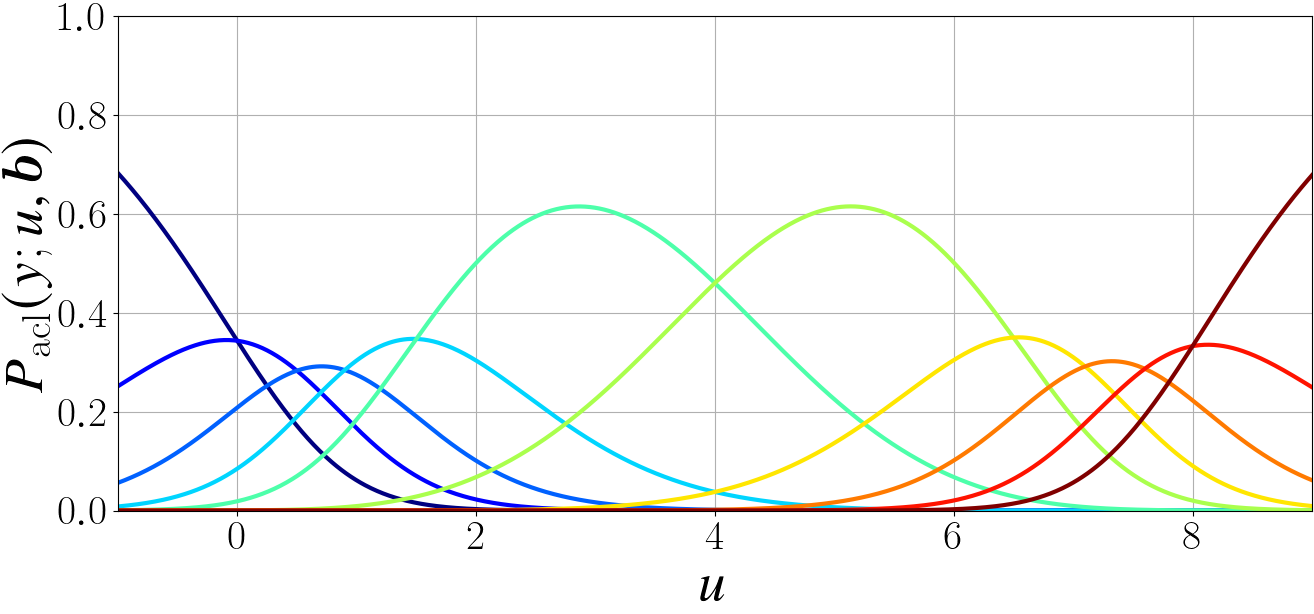}&\hskip-6pt
\includegraphics[width=4.8cm, bb=0 0 946 442]{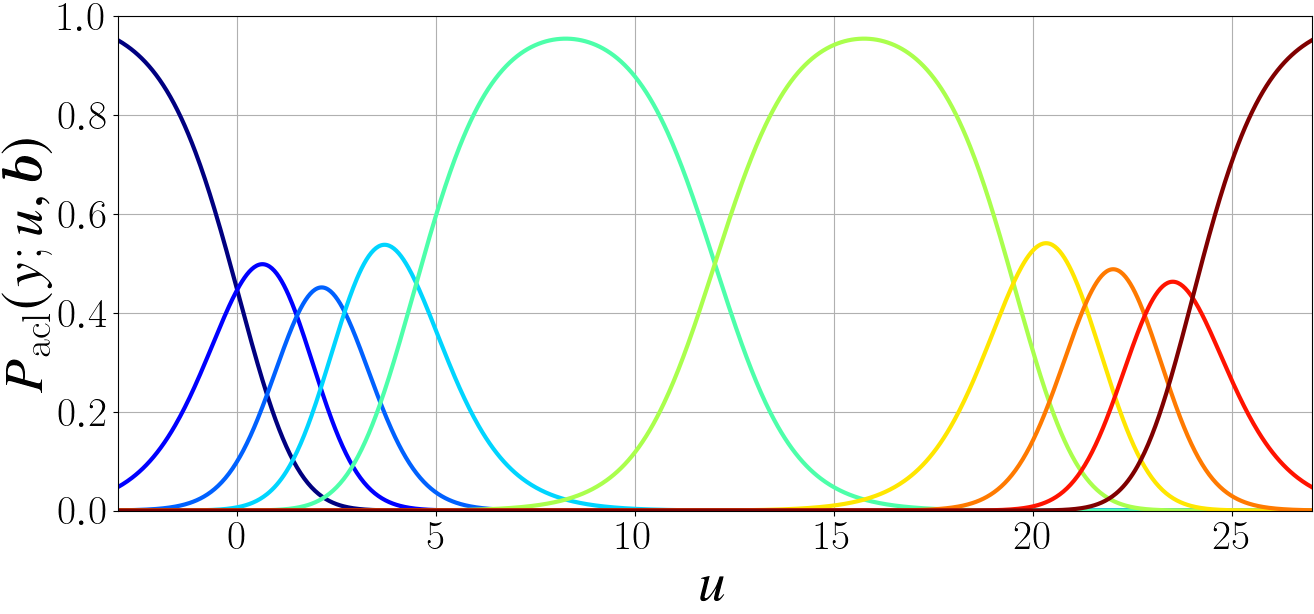}\\\hskip-6pt
{\tiny\hyt{j1} $\bb=\grave{\bb}^{[\Delta]}$, $\Delta=1/3$}&\hskip-6pt
{\tiny\hyt{j2} $\bb=\grave{\bb}^{[\Delta]}$, $\Delta=1$}&\hskip-6pt
{\tiny\hyt{j3} $\bb=\grave{\bb}^{[\Delta]}$, $\Delta=3$}\\\hskip-6pt
\includegraphics[width=4.8cm, bb=0 0 946 442]{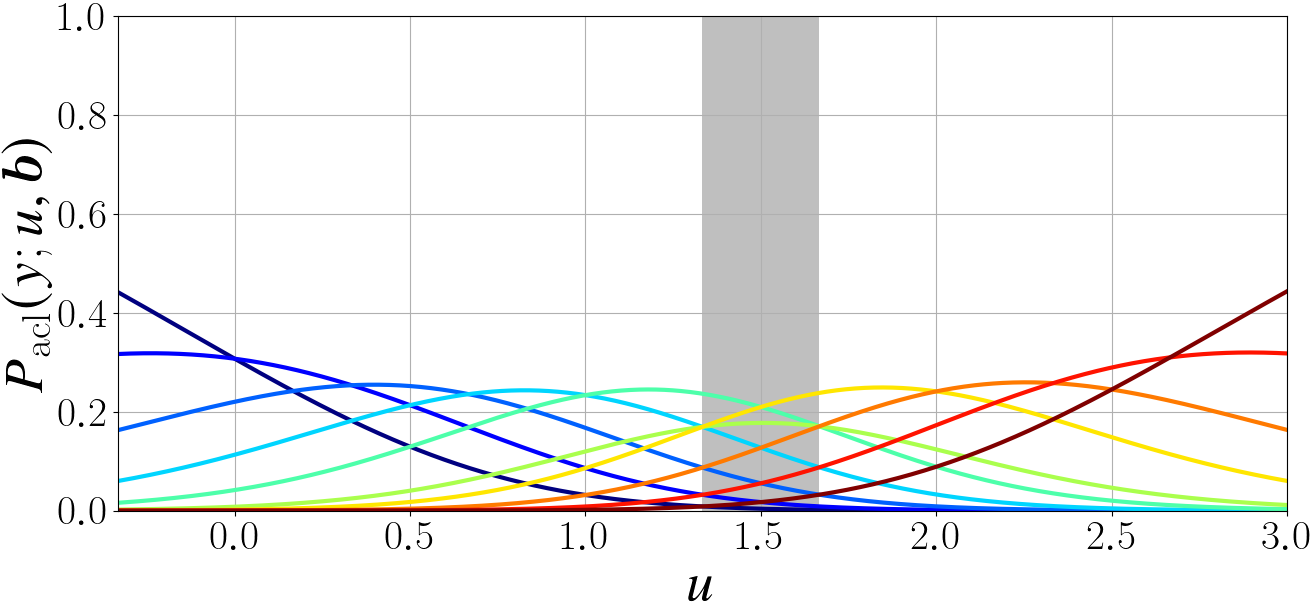}&\hskip-6pt
\includegraphics[width=4.8cm, bb=0 0 946 442]{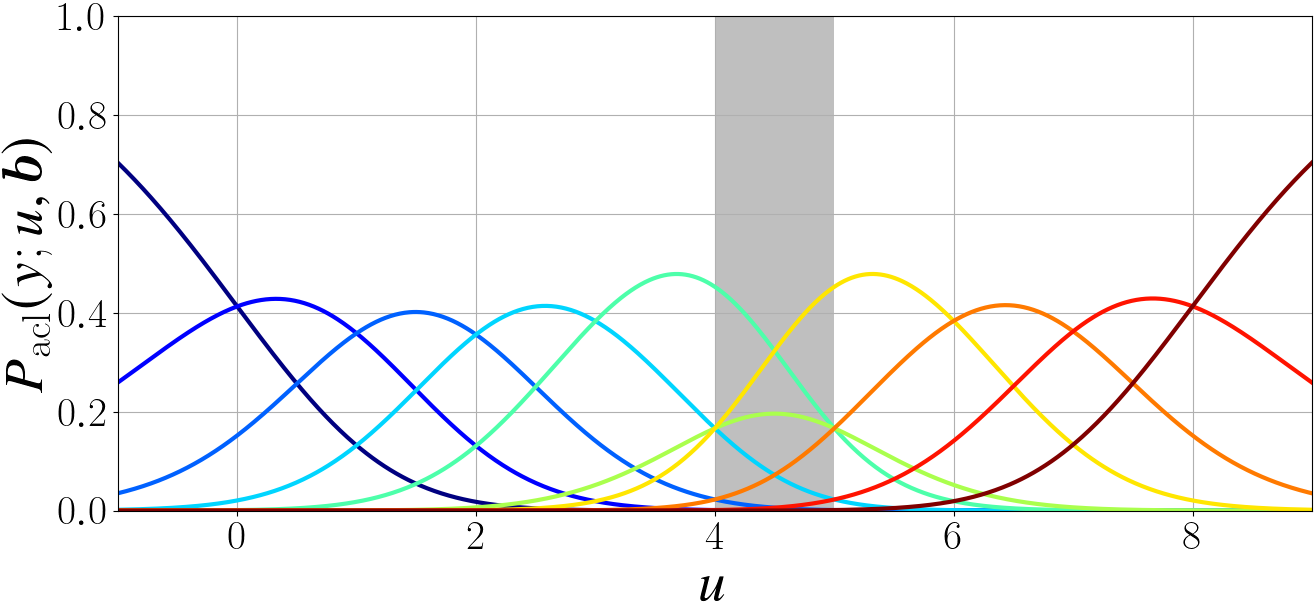}&\hskip-6pt
\includegraphics[width=4.8cm, bb=0 0 946 442]{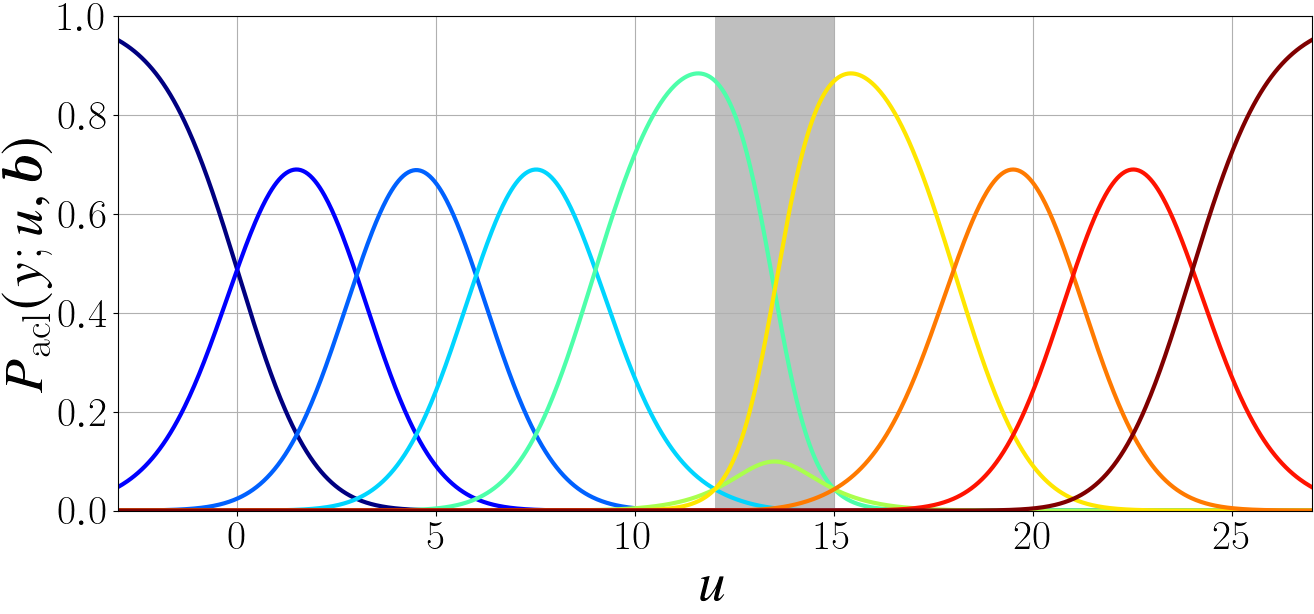}\\\hskip-6pt
\end{tabular}\\\hskip-6pt
\includegraphics[width=10cm, bb=0 0 1584 58]{./image/legend3.png}
\caption{%
Instances of the ACL model of $K=10$.
Figures show $P_\acl(y;u,\bb)$ for $y=1,\ldots,10$, $u\in[-\Delta,9\Delta]$, 
$\bb=\bb^{[\Delta]}$ (plates~\protect\hyl{h1}--\protect\hyl{h3}), 
$\acute{\bb}^{[\Delta]}$ (plates~\protect\hyl{i1}--\protect\hyl{i3}), 
$\grave{\bb}^{[\Delta]}$ (plates~\protect\hyl{j1}--\protect\hyl{j3}), 
and $\Delta=1/3,1,3$.
At $u$ in the region where the background color is white or gray, 
the PMF $(P_\acl(y;u,\bb))_{y\in[10]}$ is unimodal or not.}
\label{fig:ACL}
\end{figure}

Figure~\ref{fig:ACL} shows instances of the ACL model 
$(P_\acl(y;u,\bb))_{y\in[K]}$ of $K=10$ with the ordered equal-interval bias parameter vector 
$\bb=\bb^{[\Delta]}$, ordered unequal-interval bias parameter vector
$\bb=\acute{\bb}^{[\Delta]}$, and non-ordered bias parameter vector
$\bb=\grave{\bb}^{[\Delta]}:=\Delta\cdot(0,1,2,3,5,4,6,7,8)^\top$
for $\Delta=1/3,1,3$.
Theorem~\ref{thm:POACL-shape} \hyl{g2}--\hyl{g4} 
can be confirmed from Figure~\ref{fig:ACL}:
the ACL model $(P_\acl(y;u,\bb))_{y\in[10]}$ is unimodal for 
the ordered bias parameter vectors $\bb=\bb^{[\Delta]}, \acute{\bb}^{[\Delta]}$,
and partly non-unimodal for the non-ordered bias parameter vector $\bb=\grave{\bb}^{[\Delta]}$.
The CL and ACL models with the ordered equal-interval bias 
parameter vector $\bb=\bb^{[\Delta]}$ are similar;
compare Figure~\ref{fig:CL} \hyl{d1}--\hyl{d3} 
and Figure~\ref{fig:ACL} \hyl{h1}--\hyl{h3}.
Therefore, one can find that the ACL model is also well 
suited for representing homoscedastic unimodal data.
The ACL model with the ordered unequal-interval bias parameter vector
$\bb=\acute{\bb}^{[\Delta]}$ shown in Figure~\ref{fig:ACL} \hyl{i1}--\hyl{i3} 
is mode-wise heteroscedastic, and similar to the counterpart of 
the CL model shown in Figure~\ref{fig:CL} \hyl{e1}--\hyl{e3}.
This result suggests that the ACL model can represent 
mode-wise heteroscedastic unimodal data as well.
On the other hand, as the CL model,
the ACL model cannot represent overall heteroscedastic data well.

\subsubsection{Analysis for a Special Instance: Squared-IT Loss}
\label{sec:IT-SQloss}
A function $\varphi$ adopted for the Absolute- and Squared-IT losses 
is non-monotonic and dissimilar to those we analyzed in Section~\ref{sec:IT-LEloss}.
Therefore, analogies from behaviors of the surrogate risk minimizer 
with the Logistic- and Exponential-IT losses may not be convincing 
in understanding those with the Absolute- and Squared-IT losses.
We thus studied behaviors of the surrogate risk minimizer 
with the Squared-IT loss and the corresponding classifier, 
and obtained the following result
(see Section~\ref{sec:PL-Loss} for results for the Absolute-IT loss):
\begin{theorem}
\label{thm:ITSQ-Consistency}
Assume that $\Pr(Y\in\{y,y+1\})>0$ for all $y\in[K-1]$,
and let $\phi$ be the Squared-IT loss ($\phi=\phi_\squait$), 
and $\calA\times\calB\subseteq\{a:\bbR^d\to\bbR\}\times\calB_0$ 
include $(\tilde{a},\tilde{\bb})$ that satisfies
\begin{align}
\label{eq:ITSQSRM}
\begin{split}
%
%
    &\tilde{a}(\bx)
    =\frac{\sum_{y=1}^{K-1}\Pr(Y\in\{y,y+1\}|\bX=\bx)\tilde{b}_y-\Pr(Y=1|\bX=\bx)+\Pr(Y=K|\bX=\bx)}{2-\Pr(Y\in\{1,K\}|\bX=\bx)},\\
    &\tilde{b}_y
    =\frac{\Pr(Y=y)-\Pr(Y=y+1)+\bbE[\Pr(Y\in\{y,y+1\}|\bX)\tilde{a}(\bX)]}{\Pr(Y\in\{y,y+1\})}
%
\end{split}
\end{align}
for any $\bx$ in the support of the probability distribution of $\bX$ and every $y\in[K]$
and $\tilde{b}_1=0$, where the expectation $\bbE[\cdot]$ is taken for $\bX$.
Then, $(\bar{a},\bar{\bb})$ defined by \eqref{eq:Idolab} satisfies 
$(\bar{a}(\bX),\bar{\bb})=(\tilde{a}(\bX),\tilde{\bb})$ almost surely.
\end{theorem}
\noindent%
Note that \eqref{eq:ITSQSRM} can be solved by 
$\bar{\bb}=\bD\bar{\bb}+\bc$, that is, $\bar{\bb}=(\bI-\bD)^{-1}\bc$
under the invertibility of the matrix $\bI-\bD$ with 
the $(K-1)\times(K-1)$ identity matrix $\bI$, $\bc=(c_y)_{y\in[K-1]}$, 
and $\bD=(d_{y,k})_{y,k\in[K-1]}$ that are consisted of
\begin{align}
\begin{split}
    &c_y:=
    \frac{\Pr(Y=y)-\Pr(Y=y+1)-\bbE\bigl[\frac{\Pr(Y\in\{y,y+1\}|\bX)\{\Pr(Y=1|\bX)-\Pr(Y=K|\bX)\}}{2-\Pr(Y\in\{1,K\}|\bX)}\bigl]}{\Pr(Y\in\{y,y+1\})},\\
    &d_{y,k}:=
    \begin{cases}
    0&\text{if }k=1,\\
    \frac{\bbE\bigl[\frac{\Pr(Y\in\{y,y+1\}|\bX)\Pr(Y\in\{k,k+1\}|\bX)}{2-\Pr(Y\in\{1,K\}|\bX)}\bigl]}{\Pr(Y\in\{y,y+1\})}&\text{otherwise}.
    \end{cases}
\end{split}
\end{align}
The expression of the surrogate risk minimizer suggests that 
the threshold method with $\phi=\phi_\squait$ works well 
when the data follow a unimodal CPD with a small scale:
consider that, when $\Pr(Y=y|\bX=\bx)\approx1$ if $y=M((\Pr(Y=y|\bX=\bx))_{y\in[K]})$ and $0$ otherwise
for any $\bx$ and $\Pr(Y=y)\approx\frac{1}{K}$ for all $y\in[K]$,
one has $\bar{b}_k\approx2(k-1)$ for all $k\in[K-1]$ 
and $\bar{a}(\bx)\approx2M((\Pr(Y=y|\bX=\bx))_{y\in[K]}-3$
(on the other hand, when $\Pr(Y=y|\bX=\bx)\approx\frac{1}{K}$ for any $\bx$
and $\Pr(Y=y)\approx\frac{1}{K}$ for all $y\in[K]$,
one has that $\bar{b}_k\approx0$ for $y=1$, $\frac{1}{2}$ for $y=2,\ldots,K-2$, 
and $1$ for $y=K-1$ and $\bar{a}(\bx)\approx\frac{1}{2}$).

\subsubsection{Overlap of Optimized Bias Parameters}
\label{sec:IT-BP}
For an IT loss $\phi$, a minimizer $(a,\bb)=(\bar{a}',\bar{\bb}')$ 
of the surrogate risk $\bbE[\phi(a(\bX),\bb,Y)]$ over $\calA\times\calB_0$ 
has no guarantee to satisfy the order condition unlike that for an AT loss, 
namely, it may satisfy $\bar{\bb}'\not\in\calB_0^\ord$.
When $\bar{\bb}'\not\in\calB_0^\ord$, 
the bias parameter vector of the surrogate risk minimizer 
$(a,\bb)=(\bar{a},\bar{\bb})$ over $\calA\times\calB_0^\ord$
has overlapped components.
Such a situation is not only exemplified in 
Theorem~\ref{thm:IT-Consistency} \hyl{f3}, but occurs frequently in practice.
In this section, we describe an inherent trouble of the threshold method 
that adopts an IT loss function and the ordered class of the bias parameter vector,
owing to the overlap of optimized bias parameters.

Assume that the surrogate risk minimizer $(\bar{a},\bar{\bb})$ satisfies
\begin{align}
\label{eq:concentrate}
    \bar{b}_k=c\text{ for all }k\in\{l,\ldots,m\}
    \text{ with some }c\in\bbR, l, m\in[K-1]\text{ s.t. }l<m.
\end{align}
Under this setting, the conditional surrogate risk at $\bX=\bx$ becomes
\begin{align}
    \begin{split}
    \bbE[\phi(\bar{a}(\bx),\bar{\bb},Y)]
    &=\Pr(Y\in\{l+1,\ldots,m\}|\bX=\bx)\{\varphi(\bar{a}(\bx)-c)+\varphi(c-\bar{a}(\bx))\}\\
    &+\sum_{y\in\{\ldots,l,m+1,\ldots\}}\Pr(Y=y|\bX=\bx)\{\varphi(\bar{a}(\bx)-\bar{b}_{y-1})+\varphi(\bar{b}_y-\bar{a}(\bx))\}.
    \end{split}
\end{align}
If $\Pr(Y\in\{\ldots,l,m+1,\ldots\}|\bX=\bx)=0$ for $\bx\in\calX_0$ 
with a some domain $\calX_0\subseteq\bbR^d$,
the optimized 1DT value $\bar{a}(\bx)\in\argmin_{a(\bx)\in\bbR}\bbE[\phi(a(\bx),\bar{\bb},Y)]$ 
for $\bx\in\calX_0$ is determined independently of the conditional probabilities
$\Pr(Y=l+1|\bX=\bx),\ldots,\Pr(Y=m|\bX=\bx)$.
Especially when $\varphi$ is strictly convex, 
optimized 1DT values are concentrated like
$\bar{a}(\bx)=c$ for any $\bx\in\calX_0$.
Then, the resulting classifier cannot classify data in $\calX_0$ well.

The above demonstration is on an extreme situation.
However, even in more mild situation,
overlap of the optimized bias parameter vector makes it difficult 
to reflect variation of the data distribution in the learning of the 1DT.
This trouble is inherent to the combination of 
the IT loss and ordered class $\calB_0^\ord$ of bias parameter vector 
and may lead to the sub-optimal approximation error.

\subsubsection{Conjecture}
\label{sec:IT-Conj}
According to the discussion (Theorems~\ref{thm:IT-Consistency} 
and \ref{thm:POACL-shape}) in Section~\ref{sec:IT-LEloss},
for the case using IT losses we present a conjecture similar to that for AT losses:
for IT losses based on a non-increasing function $\varphi$ that has no special property,
the classifier of threshold methods would tend to perform well for 
homoscedastic unimodal data and mode-wise heteroscedastic unimodal data,
and may perform poorly for overall heteroscedastic unimodal data and almost non-unimodal data.
Also, Theorem~\ref{thm:ITSQ-Consistency} gives a closed-form expression 
of the surrogate risk minimizer based on $\phi=\phi_\squait$.
According to that expression, we conjecture that the learning procedure 
with an IT loss based on a V-shaped function $\varphi$ like $\varphi_\squa$
(with no special property; see Section~\ref{sec:PL-Loss})
performs well when the data follow a unimodal CPD with a small scale.
Furthermore, the order constraint on the bias parameter vector $\calB=\calB_0^\ord$ 
for IT losses may cause the overlap of learned bias parameter vector, 
which has a risk to negatively influence the classification performance,
as described in Section~\ref{sec:IT-BP}.

\subsection{Piecewise-Linear (PL) Loss}
\label{sec:PL-Loss}
We here study the learning procedure based on 
a \emph{piecewise-linear} (\emph{PL}) function $\varphi$.

We define the piecewise-linearity as follows:
\begin{definition}
\label{def:PL}
If a function $\varphi:\bbR\to\bbR$ satisfies that $\varphi(u)=a_i+b_i u$ for $u$ in $i$-th one 
of the intervals $(-\infty,c_1], (c_1,c_2], \ldots, (c_{I-1},c_I], (c_I, +\infty)$ for all $i\in[I+1]$,
with $a_1,\ldots,a_{I+1},b_1,\ldots,b_{I+1},c_1,\ldots,c_I\in\bbR$ satisfying 
$a_i\neq a_{i+1}$ or $b_i\neq b_{i+1}$ for all $i\in[I]$ and $c_1<\cdots<c_I$,
we say that $\varphi$ is PL with $I$ \emph{knots} $c_1,\ldots,c_I$.
\end{definition}
\noindent%
For example, the hinge loss function $\varphi_\hing$ 
and the absolute loss function $\varphi_\abso$ are PL with 1 knot,
and the ramp loss function $\varphi_\ramp$ is PL with 2 knots.

When $\varphi$ is PL, the corresponding AT and IT loss functions and 
the conditional surrogate risk $\bbE[\phi(a(\bx),\bb,Y)]$ are also PL:
\begin{theorem}
\label{thm:PLAIT}
Assume that $\varphi$ is PL with $I$ knots,
and $\bb=(b_k)_{k\in[K-1]}$ consists of different $J$ values,
with $I\in\bbN$ and $J\in[K-1]$.
Then, the AT loss $\phi(\cdot,\bb,y)$ defined by \eqref{eq:AT} with $\varphi$ 
is PL with at most $IJ$ knots,
and the IT loss $\phi(\cdot,\bb,y)$ defined by \eqref{eq:IT} with $\varphi$ 
is PL with $I$ knots if $y\in\{1,K\}$ or at most $2I$ knots otherwise.
Moreover, for both of the AT and IT losses,
the conditional surrogate risk $\bbE[\phi(\cdot,\bb,Y)]$
is PL with at most $2IJ$ knots.
\end{theorem}

In general, the PL conditional surrogate risk $\bbE[\phi(a(\bx),\bb,Y)]$ 
is minimized with respect to $a(\bx)\in\bbR$ at one of its knots 
(except for, when the CPD $(\Pr(Y=y|\bX=\bx))_{y\in[K]}$ is special 
so that $\bbE[\phi(a(\bx),\bb,Y)]$ has a plateau regarding $a(\bx)$
and its minimizer is not uniquely determined).
When the optimized 1DT $\bar{a}(\bx)$ takes the same value for $\bx$'s with different CPD values,
the resulting classifier cannot classify data on such $\bx$'s successfully, 
as the following example:
\begin{example}
\label{ex:3hingit}
Consider the setting with $K=3$, $\phi=\phi_\hingit$, $\calA=\{a:\bbR^d\to\bbR\}$, and $\calB=\calB_0^\ord$.
In this setting, Theorem~\ref{thm:FBLAIT} in Appendix~\ref{sec:Proof} shows $\bar{b}_2\in[0,2]$,
and the conditional surrogate risk for $(a,\bar{\bb})$ given $\bX=\bx$ 
at its knots $-1, \bar{b}_2-1, 1, \bar{b}_2+1$ becomes
\begin{align}
    &\bbE[\phi(a(\bx),\bar{\bb},Y)]\\
    &=\begin{cases}
    2\Pr(Y=2|\bX=\bx)+(2+\bar{b}_2)\Pr(Y=3|\bX=\bx)&\text{if~}a(\bx)=-1,\\
    \bar{b}_2\Pr(Y=1|\bX=\bx)+(2-\bar{b}_2)\Pr(Y=2|\bX=\bx)+2\Pr(Y=3|\bX=\bx)&\text{if~}a(\bx)=\bar{b}_2-1,\\
    2\Pr(Y=1|\bX=\bx)+(2-\bar{b}_2)\Pr(Y=2|\bX=\bx)+\bar{b}_2\Pr(Y=3|\bX=\bx)&\text{if~}a(\bx)=1,\\
    (2+\bar{b}_2)\Pr(Y=1|\bX=\bx)+2\Pr(Y=2|\bX=\bx)&\text{if~}a(\bx)=\bar{b}_2+1.
    \end{cases}\nonumber
\end{align}
By comparing the above four terms, one can see that the minimizer 
$\bar{a}(\bx)=\argmin_{a(\bx)\in\bbR}\bbE[\phi(a(\bx),\bar{\bb},Y)]$ is
\begin{align}
    \bar{a}(\bx)
    =\begin{cases}
    -1&\text{if~}\bx\in\calX_1,\\
    \bar{b}_2-1&\text{if~}\bx\in\calX_2,\\
    1&\text{if~}\bx\in\calX_3,\\
    \bar{b}_2+1&\text{if~}\bx\in\calX_4,
    \end{cases}
\end{align}
where the domains $\calX_1,\ldots,\calX_4\subseteq\bbR^d$ 
splitting the support of the distribution of $\bX$ are defined as
\begin{align}
    \begin{split}
    &\calX_1:=\{\bx\mid \Pr(Y=1|\bX=\bx)>\Pr(Y=2|\bX=\bx)+\Pr(Y=3|\bX=\bx)\},\\
    &\calX_2:=\{\bx\mid \Pr(Y=2|\bX=\bx)>\Pr(Y=1|\bX=\bx)-\Pr(Y=3|\bX=\bx)>0\},\\
    &\calX_3:=\{\bx\mid \Pr(Y=2|\bX=\bx)>\Pr(Y=3|\bX=\bx)-\Pr(Y=1|\bX=\bx)>0\},\\
    &\calX_4:=\{\bx\mid \Pr(Y=3|\bX=\bx)>\Pr(Y=1|\bX=\bx)+\Pr(Y=2|\bX=\bx)\}.
    \end{split}
\end{align}
The surrogate risk becomes
\begin{align}
	\begin{split}
	\bbE[\phi(\bar{a}(\bX),\bar{\bb},Y)]
	&=\bar{b}_2\bigl\{
	\Pr(\bX\in\calX_2\cup\calX_4, Y=1)
	-\Pr(\bX\in\calX_2\cup\calX_3, Y=2)\\
	&\hphantom{=\bar{b}_2\bigl\{}
	+\Pr(\bX\in\calX_1\cup\calX_3, Y=3)\bigr\}
	+(\text{$\bar{b}_2$-independent term}),
	\end{split}
\end{align}
which implies that the surrogate risk minimizer $(\bar{a},\bar{\bb})$ depends on the condition
\begin{align}
\label{eq:PhaTra}
    \Pr(\bX\in\calX_2\cup\calX_4, Y=1)+\Pr(\bX\in\calX_1\cup\calX_3, Y=3)
    \gtrless \Pr(\bX\in\calX_2\cup\calX_3, Y=2).
\end{align}
If the left-hand side is greater than the right-hand side in the above, 
then the surrogate risk minimizer becomes 
`phase-1': $\bar{b}_1=\bar{b}_2=0$ and
\begin{align}
    \bar{a}(\bx)
    =\begin{cases}
    -1&\text{if~}\bx\in\calX_1\cup\calX_2,\\
    1&\text{if~}\bx\in\calX_3\cup\calX_4.
    \end{cases}
\end{align}
In the other case, the surrogate risk minimizer becomes
`phase-2': $\bar{b}_1=0$, $\bar{b}_2=2$, and
\begin{align}
    \bar{a}(\bx)
    =\begin{cases}
    -1&\text{if~}\bx\in\calX_1,\\
    1&\text{if~}\bx\in\calX_2\cup\calX_3,\\
    3&\text{if~}\bx\in\calX_4.
    \end{cases}
\end{align}
The phase-1 may be particularly problematic for the 3-class classification problem.
We demonstrate below that 
the surrogate risk minimizer tends to be in the phase-1
especially when the scale of the CPD is large.
Assume further that $\bX$ has probabilities
$\Pr(\bX=\bx^{[i]})=p_i$ for $i\in[4]$ with $(p_i)_{i\in[4]}\in\Delta_3$
at four different points $\bx^{[1]}, \ldots,\bx^{[4]}$,
and the CPD of $Y$ at each of these points is unimodal and
$(\Pr(Y=y|\bX=\bx^{[i]}))_{y\in[3]}=(q_1,q_2,q_3),(q_1,q_3,q_2),(q_2,q_3,q_1),(q_3,q_2,q_1)$
for $i\in[4]$ with $(q_i)_{i\in[3]}\in\Delta_2$ s.t.\;$0\le q_1<q_2<q_3$.

If $q_3>q_1+q_2$,
it holds that $\calX_1=\{\bx^{[4]}\}$, $\calX_2=\{\bx^{[3]}\}$, 
$\calX_3=\{\bx^{[2]}\}$, $\calX_4=\{\bx^{[1]}\}$,
and $\Pr(\bX\in\calX_2\cup\calX_4, Y=1)=p_3q_2+p_1q_1$, 
$\Pr(\bX\in\calX_2\cup\calX_3, Y=2)=p_3q_3+p_2q_3$, 
$\Pr(\bX\in\calX_1\cup\calX_3, Y=3)=p_4q_1+p_2q_2$,
and the condition for the phase transition \eqref{eq:PhaTra} becomes
\begin{align}
    (p_1+p_4)q_1\gtrless (p_2+p_3)(q_3-q_2).
\end{align}
If $q_3<q_1+q_2$,
it holds that $\calX_1=\emptyset$, $\calX_2=\{\bx^{[3]},\bx^{[4]}\}$, 
$\calX_3=\{\bx^{[1]},\bx^{[2]}\}$, $\calX_4=\emptyset$,
and $\Pr(\bX\in\calX_2\cup\calX_4, Y=1)=p_3q_2+p_4q_3$, 
$\Pr(\bX\in\calX_2\cup\calX_3, Y=2)=p_1q_2+p_2q_3+p_3q_3+p_4q_2$, 
$\Pr(\bX\in\calX_1\cup\calX_3, Y=3)=p_1q_3+p_2q_2$,
and the condition for the phase transition \eqref{eq:PhaTra} becomes
\begin{align}
    (p_1+p_4)\gtrless (p_2+p_3).
\end{align}

\begin{figure}[t]
\renewcommand{\arraystretch}{0.1}\renewcommand{\tabcolsep}{5pt}\centering%
\begin{tabular}{cccc}\hskip-6pt
~~~~{\tiny\hyt{k1} $(p_1,p_2)=\ldots,(0.2,0.3)$}&\hskip-10pt
~~~~~{\tiny\hyt{k2} $(p_1,p_2)=(0.3,0.2)$}&\hskip-10pt
~~~~~{\tiny\hyt{k3} $(p_1,p_2)=(0.4,0.1)$}&\hskip-10pt
~~~~{\tiny\hyt{k4} $(p_1,p_2)=(0.5,0)$}\\\hskip-6pt
\includegraphics[width=3.75cm, bb=0 0 652 639]{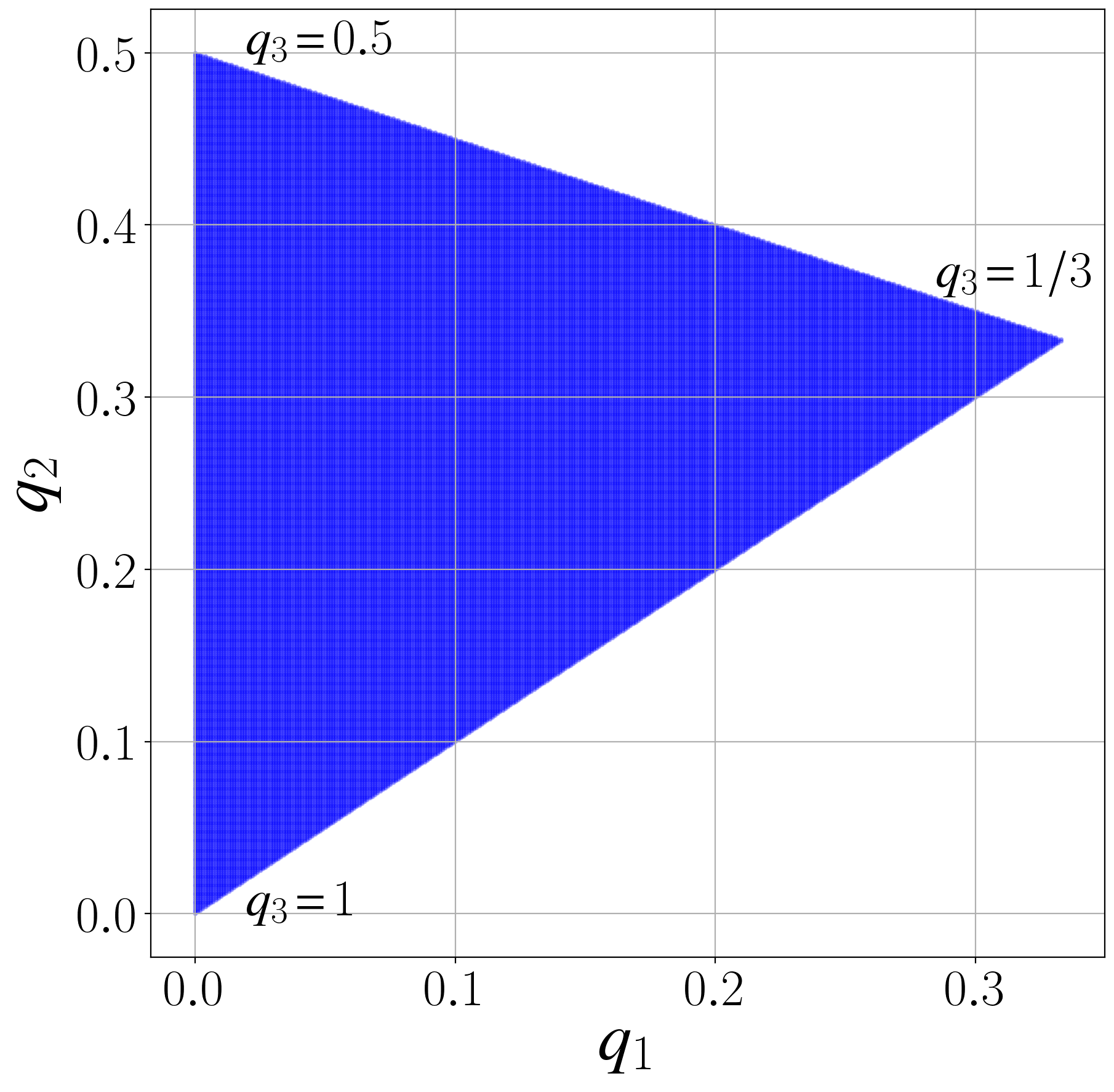}&\hskip-10pt
\includegraphics[width=3.75cm, bb=0 0 652 639]{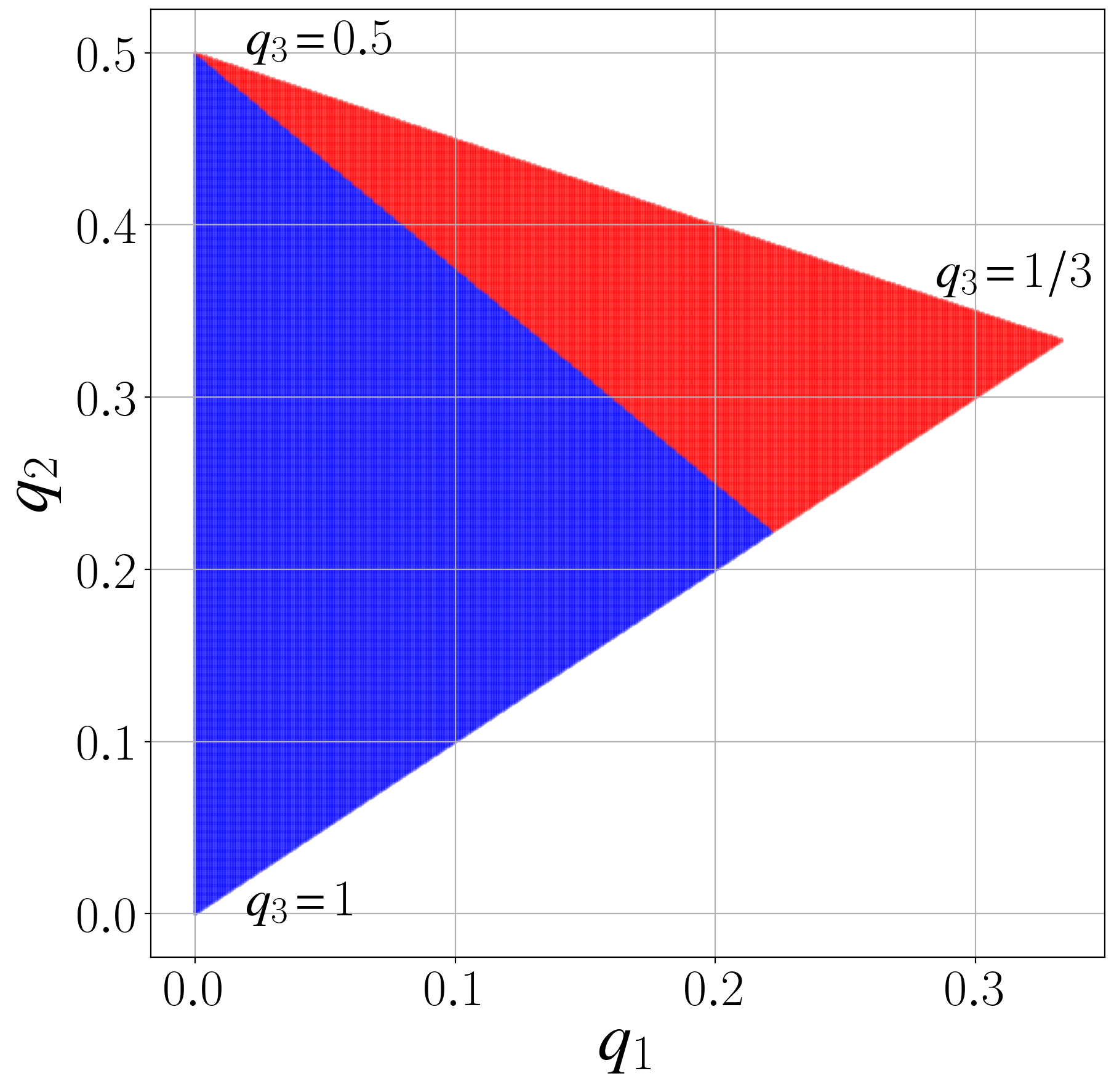}&\hskip-10pt
\includegraphics[width=3.75cm, bb=0 0 652 639]{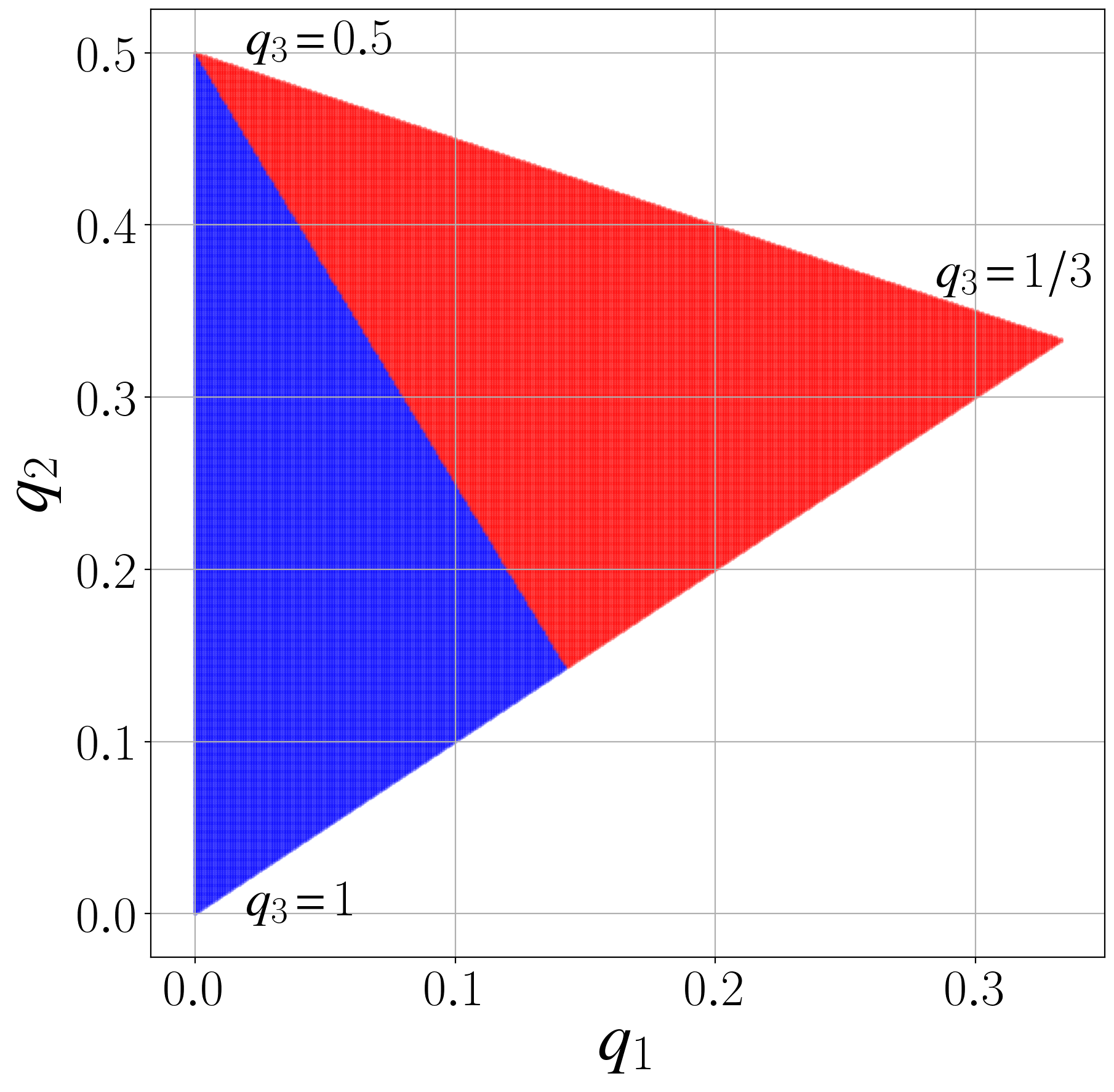}&\hskip-10pt
\includegraphics[width=3.75cm, bb=0 0 652 639]{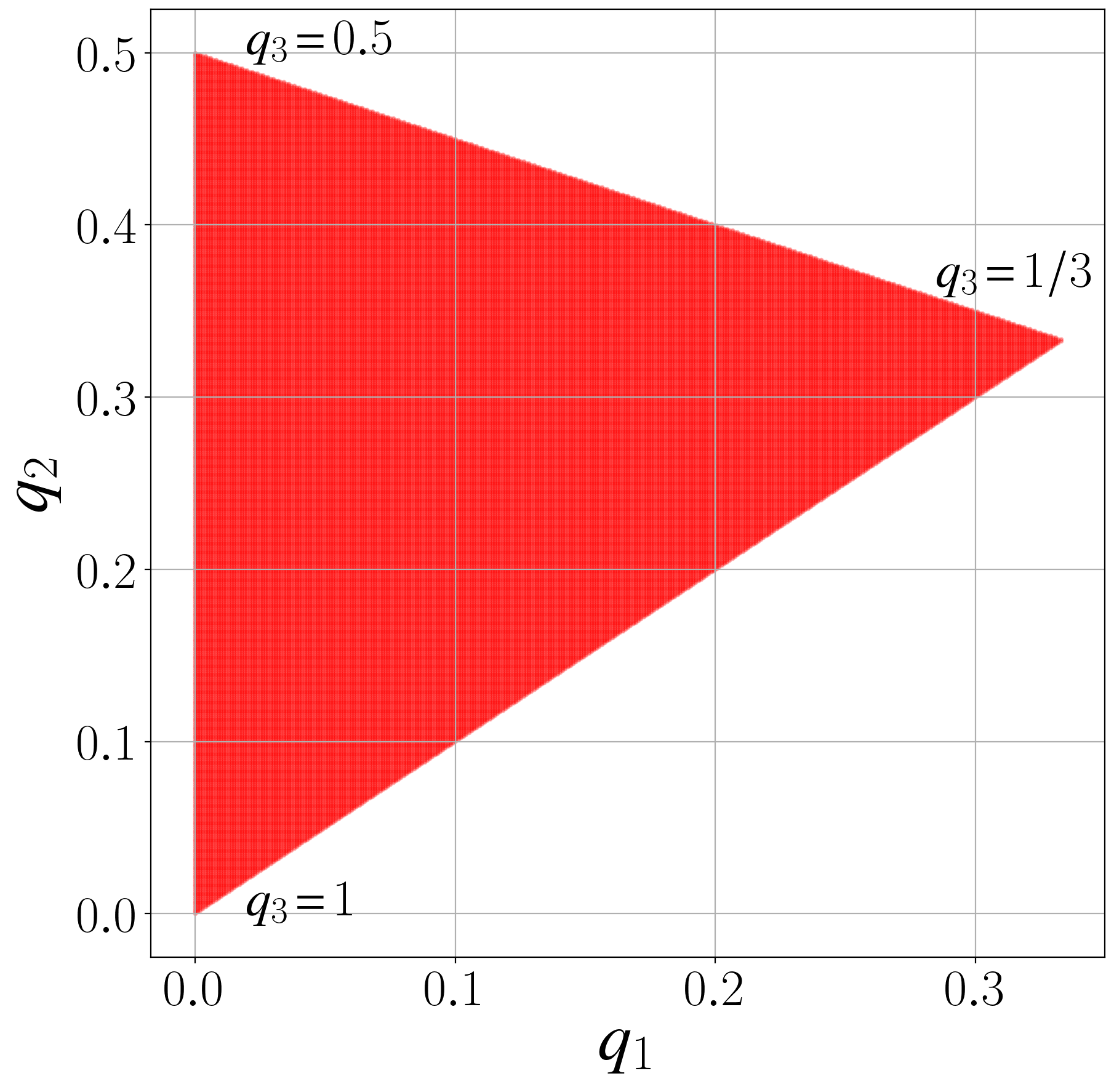}
\end{tabular}
\caption{%
Phase transition of the surrogate risk minimizer for Example~\ref{ex:3hingit}
with $(p_1,p_2)=(p_4,p_3)=(0,0.5),(0.1,0.4),\ldots,(0.5,0)$.
The surrogate risk minimizer for $(q_1,q_2,q_3)$ in the red (resp.\;blue) region is in the phase-1 (resp.\;phase-2).
For example, the CPD has a small scale for $(q_1,q_2,q_3)\approx(0,0,1)$ 
or a large scale for $(q_1,q_2,q_3)\approx(1/3,1/3,1/3)$.}
\label{fig:PT}
\end{figure}
Figure~\ref{fig:PT} shows the phase transition of the surrogate risk minimizer.
From this figure, one can find that the larger the scale of the CPD, 
the more likely the surrogate risk minimizer tends to be in the phase-1.
\end{example}

Binary classification methods such as support vector classification \citep{cortes1995support}
also apply a PL surrogate loss so that outputs of the optimized model can be concentrated at a few points, 
but they have no trouble that the approximation error becomes sub-optimal;
the trouble of the PL loss in threshold methods becomes striking when 
learned 1DT values are concentrated on a number of points less than $K$ 
under a $K$-class problem.

According to the consideration in this section, 
one can see that the combination of an IT surrogate loss 
and the ordered class of the bias parameter vector 
should be used with caution and may not work well 
when the data have large scale, even if they are unimodal.

\section{Numerical Experiments}
\label{sec:NE}
\subsection{Simulation Study}
\label{sec:SS}
With the aim of understanding how the design of the learning procedure 
for a threshold method affects the approximation error of that method,
we presented several theorems and conjectures in Section~\ref{sec:AC}.
In order to verify these results experimentally, we performed the following simulation study:
As a simulation of the surrogate risk minimization 
$\min_{a\in\calA,\bb\in\calB}\bbE[\phi(a(\bX),\bb,Y)]$, 
we solved the numerical optimization problem
\begin{align}
\label{eq:SimSRM}
    \underset{a_1,\ldots,a_N\in\bbR,\bb\in\calB}{\min}\,
    \frac{1}{N}\sum_{i=1}^N\sum_{y=1}^K p_{i,y}\phi(a_i,\bb,y)
\end{align}
with $N=100$ and $K=10$.
This problem simulates the learning procedure of a threshold 
method with the surrogate loss function $\phi$, class of the 1DT 
$\calA=\{a:\bbR^d\to\bbR\}$, and class of the bias parameter vector $\calB$ 
for data in which the underlying explanatory variable $\bX$ follows 
the uniform distribution on $N$ different points $\bx^{[1]},\ldots,\bx^{[N]}$ 
and in which the CPD of the $K$-class target variable $Y$ 
takes $(p_{i,y})_{y\in[K]}$ at $i$-th point $\bx^{[i]}$.

\paragraph{Data Distributions}
Let $\tilde{a}^{[N,\Delta]}_i:=(\tfrac{(i-1)K}{N-1}-1)\Delta$.
We considered the following 15 instances of the CPD:
homoscedastic unimodal data distribution H-$\Delta$ with $\Delta=1/3,1,3$,
\begin{align}
\label{eq:U}
    p_{i,y}=P_\acl(y;\tilde{a}^{[N,\Delta]}_i,\bb^{[\Delta]})
    \text{~for~}i\in[N], y\in[K];
\end{align}
mode-wise heteroscedastic unimodal data distribution M-$\Delta$ with $\Delta=1/3,1,3$,
\begin{align}
\label{eq:M}
    p_{i,y}=P_\acl(y;\tilde{a}^{[N,\Delta]}_i,\acute{\bb}^{[\Delta]})
    \text{~for~}i\in[N], y\in[K];
\end{align}
almost unimodal (partly non-unimodal) data distribution A-$\Delta$ with $\Delta=1/3,1,3$,
\begin{align}
\label{eq:A}
    p_{i,y}=P_\acl(y;\tilde{a}^{[N,\Delta]}_i,\grave{\bb}^{[\Delta]})
    \text{~for~}i\in[N], y\in[K];
\end{align}
overall heteroscedastic data unimodal distribution O-($\Delta_1,\Delta_2$)
with $(\Delta_1,\Delta_2)=(1/3,1),(1/3,3),(1,3)$,
\begin{align}
\label{eq:O}
    p_{i,y}=P_\acl(y;\tilde{a}^{[N/2,\Delta_1]}_i, \bb^{[\Delta_1]})\text{ and }
    p_{i+N/2,y}=P_\acl(y;\tilde{a}^{[N/2,\Delta_2]}_i, \bb^{[\Delta_2]})
    \text{~for~}i=[\tfrac{N}{2}], y\in[K];
\end{align}
almost non-unimodal data distribution N-$\Delta$ with $\Delta=1/3,1,3$,
\begin{align}
\label{eq:N}
    p_{i,y}=P_\acl(\pi(y);\tilde{a}^{[N,\Delta]}_i,\bb^{[\Delta]})
    \text{~for~}i\in[N], y\in[K],
\end{align}
where the function $\pi$ maps $1,\ldots,10$ to $1,10,2,9,3,8,4,7,5,6$.
These instances are based on those shown in Figure~\ref{fig:ACL}.
Here, $\bb^{[\Delta]}=\Delta\cdot(0,1,\ldots,8)^\top$ is ordered equal-interval,
$\acute{\bb}^{[\Delta]}=\Delta\cdot(0,0.5,0.9,1.5,4,6.5,7.1,7.6,8)^\top$ 
is ordered unequal-interval, and 
$\grave{\bb}^{[\Delta]}=\Delta\cdot(0,1,2,3,5,4,6,7,8)^\top$ is non-ordered.
Also, $\Delta$, $\Delta_1$, and $\Delta_2$ control the scale of the data distribution;
the larger they are, the smaller scale the associated data distribution have.
We basically suppose that the ordinal data tend to have the unimodality, 
and hence consider that the data distributions 
H-$\Delta$, M-$\Delta$, A-$\Delta$, and O-($\Delta_1,\Delta_2$)
can be adequate as a data distribution underlying the ordinal data,
but N-$\Delta$ is not.

\paragraph{Methods}
As the surrogate loss function $\phi$, we tried the Logistic, 
Hinge, Smoothed-Hinge, Squared-Hinge, Exponential, 
Absolute, and Squared-AT and -IT losses.
As the class of the bias parameter vector $\calB$,
we considered $\calB_0^\ord$ (ordered class) for AT losses, and
$\calB_0$ (non-ordered class) and $\calB_0^\ord$ for IT losses;
we denote these three settings AT-O, IT-N, and IT-O, respectively.
Under these settings, 
the problem \eqref{eq:SimSRM} is a convex optimization problem, 
and hence it is relatively easy to seek the global minimizer
(this is a reason why we did not try the Ramp-AT and -IT losses).
We solved the problem \eqref{eq:SimSRM} numerically
using full-batch Adam optimization with 
the descending learning rate $0.1^{1+3t/T}$ 
at $t$-th epoch during $T=10^5$ epochs.

\paragraph{Evaluation}
We considered three popular OR tasks (Tasks-Z, -A, and -S) 
with the task loss function $\ell=\ell_\zo,\ell_\ab,\ell_\sq$.
We denote the solution of \eqref{eq:SimSRM} $\bar{a}_1,\ldots,\bar{a}_N,\bar{\bb}$,
omitting the dependency on $p_{i,y}$'s, $\phi$, and $\calB$.
For the task loss $\ell$ and learned 1DT values $\bar{a}_1,\ldots,\bar{a}_N$,
we calculated the optimal threshold parameter vector
$\bar{\bt}_\ell\in\argmin_{\bt\in\bbR^{K-1}}\frac{1}{N}\sum_{i=1}^N
\sum_{y=1}^K p_{i,y}\ell(h_\thr(\bar{a}_i;\bt), y)$ 
following a way that modifies a dynamic-programming-based algorithm 
\citep{lin2006large, yamasaki2022optimal} that is for a sample-based situation.
To evaluate the method's label prediction 
$\bar{f}_{\ell,i}=h_\thr(\bar{a}_i;\bar{\bt}_\ell)$, 
we calculated the approximation errors, 
$\frac{1}{N}\sum_{i=1}^N\sum_{y=1}^K p_{i,y}\ell(\bar{f}_{\ell,i}, y)$ with 
$\ell=\ell_\zo,\ell_\ab,\ell_\sq$ (mean zero-one error (MZE), mean absolute error 
(MAE), and mean squared error (MSE)) respectively for Tasks-Z, -A, and -S.
Under the task with the task loss $\ell$,
the optimal label prediction for $i$-th point $\bx^{[i]}$ is 
$\tilde{f}_{\ell,i}\in\argmin_{k\in[K]}\sum_{y=1}^K p_{i,y}\ell(k,y)$,
and the optimal error (Bayes error) is
$\frac{1}{N}\sum_{i=1}^N\sum_{y=1}^K p_{i,y}\ell(\tilde{f}_{\ell,i}, y)$.
Note that, in the following, the evaluation for Task-S 
is reported as `root of MSE' (RMSE), not as MSE itself.

\paragraph{Results}
We show results for all the combinations of 15 data distributions, 
21 threshold methods, and 3 tasks and optimal values 
in Tables~\ref{tab:SimRes-H}--\ref{tab:SimRes-N} in Appendix~\ref{sec:ApeSS},
and representative behaviors of learned 1DT values $\bar{a}_i$'s 
in Figure~\ref{fig:SimRes} in Appendix~\ref{sec:ApeSS}.
We observed the following facts:
\begin{itemize}\setlength{\parskip}{0pt}\setlength{\itemindent}{0pt}
\item[\hyt{l1}]
Non-PL losses gave optimal approximation errors for H-$\Delta$, M-$\Delta$, and A-$\Delta$.
Not only squa-AT-O for Task-S,
logi-IT-N and expo-IT-N for all of these data distributions, 
and logi-IT-O and expo-IT-O for H-$\Delta$ and M-$\Delta$
with the optimality guarantee 
(Theorems~\ref{thm:ATSQ-Consistency} and \ref{thm:IT-Consistency}), 
but also logi-AT-O, smhi-AT-O, and others performed well;
see Tables~\ref{tab:SimRes-H}--\ref{tab:SimRes-A}.
Figure~\ref{fig:SimRes} \hyl{m1}--\hyl{m3} show plots of $(i,\bar{a}_i)$'s 
only for logi-AT-O, -IT-N, and -IT-O, but other non-PL losses gave similar results.
\item[\hyt{l2}]
All the losses gave sub-optimal approximation errors for 
many O-$(\Delta_1,\Delta_2)$ and N-$\Delta$ with a large scale;
see Tables~\ref{tab:SimRes-O} and \ref{tab:SimRes-N}
and Figure~\ref{fig:SimRes} \hyl{m4} and \hyl{m5}.
The deviation of the approximation error from the optimal value 
was small for O-$(\Delta_1,\Delta_2)$ and large for N-$\Delta$.
\item[\hyt{l3}]
PL losses (Hinge and Absolute-AT and -IT losses) gave sub-optimal 
approximation errors for many data distributions with a large scale.
Learned 1DT values $\bar{a}_i$'s with different CPD values were concentrated at a few points,
which resulted in the deterioration of the approximation error of the resulting classifier in many cases.
Figure~\ref{fig:SimRes} \hyl{m6}--\hyl{m8} show plots of $(i,\bar{a}_i)$'s 
only for hing-AT-O, -IT-N, and -IT-O, but plots for abso-AT-O, -IT-N, and -IT-O 
similarly indicated the concentration of $\bar{a}_i$'s.
\end{itemize}

\subsection{Synthesis Data Experiment}
\label{sec:SDE}
Our ultimate interest is in the prediction performance of threshold methods.
We have proceeded with the study under the working hypothesis 
that the approximation error dominates the prediction error 
when the number of training data is sufficiently large.
To test this working hypothesis and 
to show that our discussion is suggestive in studying the ultimate interest, 
we performed a performance comparison experiment with synthesis data, 
which were generated from the data distributions we considered in Section~\ref{sec:SS}.

\paragraph{Datasets}
Let $N=100$, $K=10$, $n_\tot=11000$, 
$\tilde{a}^{[N,\Delta]}_i=(\tfrac{(i-1)K}{N-1}-1)\Delta$,
$\Uni(\calS)$ be the uniform distribution on the set $\calS$, 
and $\Cat(\bp)$ be the categorical distribution for $K$ 
categories with event probabilities $p_1,\ldots,p_K$ 
of $\bp=(p_k)_{k\in[K]}\in\Delta_{K-1}$.
We used the following datasets:
homoscedastic unimodal data H-$\Delta$ with $\Delta=1/3,1,3$,
\begin{align}
    x_i\sim\Uni\bigl(\{\tilde{a}^{[N,\Delta]}_j\}_{j\in[N]}\bigr)\text{~and~}
    y_i\sim\Cat\bigl((P_\acl(k;x_i,\bb^{[\Delta]}))_{k\in[K]}\bigr)
    \text{~for~}i\in[n_\tot];
\end{align}
mode-wise heteroscedastic unimodal data M-$\Delta$ with $\Delta=1/3,1,3$,
\begin{align}
    x_i\sim\Uni\bigl(\{\tilde{a}^{[N,\Delta]}_j\}_{j\in[N]}\bigr)\text{~and~}
    y_i\sim\Cat\bigl((P_\acl(k;x_i,\acute{\bb}^{[\Delta]}))_{k\in[K]}\bigr)
    \text{~for~}i\in[n_\tot];
\end{align}
almost unimodal (partly non-unimodal) data A-$\Delta$ with $\Delta=1/3,1,3$,
\begin{align}
    x_i\sim\Uni\bigl(\{\tilde{a}^{[N,\Delta]}_j\}_{j\in[N]}\bigr)\text{~and~}
    y_i\sim\Cat\bigl((P_\acl(k;x_i,\grave{\bb}^{[\Delta]}))_{k\in[K]}\bigr)
    \text{~for~}i\in[n_\tot];
\end{align}
overall heteroscedastic unimodal data O-($\Delta_1,\Delta_2$) 
with $(\Delta_1,\Delta_2)=(1/3,1),(1/3,3),(1,3)$,
\begin{align}
    \begin{split}
    &x_i\sim\Uni\bigl(\{\tilde{a}^{[N/2,\Delta_1]}_j\}_{j\in[N/2]}\bigr)\text{~and~}
    y_i\sim\Cat\bigl((P_\acl(k;x_i,\bb^{[\Delta_1]}))_{k\in[K]}\bigr),\text{~or}\\
    &x_i\sim\Uni\bigl(\{\tilde{a}^{[N/2,\Delta_2]}_j+c\}_{j\in[N/2]}\bigr)\text{~and~}
    y_i\sim\Cat\bigl((P_\acl(k;x_i,\bb^{[\Delta_2]}+c\cdot\bm{1}_{K-1}))_{k\in[K]}\bigr)\\
    &\text{with equal probabilities for~}i\in[n_\tot],
    \end{split}
\end{align}
where we added $c=K\Delta_1+\Delta_2$ to prevent 
the overlap of the domains of $x_i$'s for the two cases,
which changes the CPD there from \eqref{eq:O};
almost non-unimodal data N-$\Delta$ with $\Delta=1/3,1,3$,
\begin{align}
    x_i\sim\Uni\bigl(\{\tilde{a}^{[N,\Delta]}_j\}_{j\in[N]}\bigr)\text{~and~}
    y_i\sim\Cat\bigl((P_\acl(\pi(k);x_i,\bb^{[\Delta]}))_{k\in[K]}\bigr)
    \text{~for~}i\in[n_\tot].
\end{align}
We used each of these datasets with randomly splitting them into the training set of 900 data points, 
validation set of 100 data points, and test set of 10000 data points.

\paragraph{Methods}
We considered threshold methods based on the two-step optimization \eqref{eq:Estiab}.
As their surrogate loss function $\phi$ and class of the bias parameter vector $\calB$,
we tried the 21 settings same as those of the simulation study in Section~\ref{sec:SS}.
As their class of the 1DT $\calA$, we used a 4-layer neural network model,
in which each hidden layer has 100 nodes and every activation function is ReLU.
For comparison, 
we also experimented with multinomial logistic regression (MLR),
which applies the negative-log-likelihood as a surrogate loss 
and a similar 4-layer neural network model to the above.
We learned models by Adam optimization with 
the empirical surrogate risk as the objective function, 
the mini-batch size 4, and the descending learning rate 
$0.1^{2.5+t/T}$ at $t$-th epoch during $T=2000$ epochs.

\paragraph{Evaluation}
During 2000 training epochs, we evaluated 
the empirical surrogate risk for the validation set after every epoch.
For a model at the point in time when the validation error was minimized, 
we calculated the empirical task risks with the task loss 
$\ell=\ell_\zo,\ell_\ab,\ell_\sq$ for the test set.
To judge the significance, 
we performed 20 independent trials and applied the one-sided Wilcoxon rank-sum test 
with the $p$-value 0.05 (Wilcoxon test) for 20 observations of each evaluation criterion.

\paragraph{Results}
Tables~\ref{tab:SDERes-H}--\ref{tab:SDERes-N} in Appendix~\ref{sec:ApeSDE} show
the mean and standard deviation (SD) of MZE, MAE, and RMSE over 20 trials,
and Figure~\ref{fig:SDERes} in Appendix~\ref{sec:ApeSDE} shows
representative behaviors of learned 1DT values $\hat{a}(\bx^{[i]})$'s. 
The experiment in this section has two differences 
from the simulation in Section~\ref{sec:SS}:
In this experiment,
we considered a finite-sample situation, and there is no guarantee 
that we obtained the global minimizer of the empirical surrogate risk 
(i.e., optimization error may exist, for example, 
because the optimization may be trapped in a local minimizer).
Even under such differences, 
we could see the following observations similar to 
those observed in the simulation in Section~\ref{sec:SS},
as expected under our working hypothesis:
\begin{itemize}\setlength{\parskip}{0pt}\setlength{\itemindent}{0pt}
\item[\hyt{n1}]
Non-PL losses gave competitive or smaller test errors for unimodal data, 
H-$\Delta$, M-$\Delta$, A-$\Delta$, and O-$(\Delta_1,\Delta_2)$,
compared with MLR;
see Tables~\ref{tab:SDERes-H}--\ref{tab:SDERes-O}
and Figure~\ref{fig:SDERes} \hyl{o1}--\hyl{o4}.
We consider that, for O-$(\Delta_1,\Delta_2)$, 
that the approximation error of these threshold methods was
slightly larger than that of MLR (see Table~\ref{tab:SimRes-O} again)
but the estimation error was smaller than that of MLR, 
and hence those threshold methods performed well.
\item[\hyt{n2}]
Most threshold methods performed poorly than MLR for N-$\Delta$;
see Table~\ref{tab:SDERes-N} and Figure~\ref{fig:SDERes} \hyl{o5}.
\item[\hyt{n3}]
PL losses resulted in the concentration of learned 1DT values $\hat{a}(\bx^{[i]})$'s
especially for large-scale data (see Figure~\ref{fig:SDERes} \hyl{o6}--\hyl{o8}),
and provided larger test errors than non-PL losses in many cases.
\end{itemize}

\subsection{Real-World Data Experiment}
\label{sec:RWDE}
Finally, we performed an OR experiment with real-world data,
mainly to confirm troubles owing to 
the use of an IT loss and ordered class of the bias parameter vector
and the use of a PL loss for practical OR tasks.
We here addressed an OR problem of estimating age from facial image 
by a discrete-value prediction, with reference to the experiments in 
the previous OR studies \citep{cao2020rank, yamasaki2022optimal}.
%
It would be supposed that the distribution of age of people 
conditioned on their facial image tends to be unimodal.

\paragraph{Datasets}
We used MORPH-2, CACD, and AFAD datasets 
\citep{ricanek2006morph, chen2014cross, niu2016ordinal}.
We purchased MORPH-2 (MORPH Album2) dataset at 
\url{https://ebill.uncw.edu/C20231_ustores/web/} and preprocessed 
it so that the face spanned the whole image with the nose tip, which was 
located by facial landmark detection \citep{sagonas2016300}, at the center. 
We used 55,013 instances with ages from 16 to 70 in this dataset. 
We downloaded CACD dataset from \url{https://bcsiriuschen.github.io/CARC/},
and preprocessed this dataset similarly to MORPH-2.
Excluding images, in which no face or more than two faces were 
detected in the preprocessing, 
we used 159,402 facial images in the age range of 14--62 years.
For AFAD dataset (\url{https://github.com/afad-dataset/tarball}), 
because faces in its images were already centered, we took no preprocessing.
We used its 164,418 instances with ages 15--40. 
Because data of MORPH-2 were collected in a laboratory setting
and data of CACD and AFAD were collected by web scraping, 
CACD and AFAD are inferred to have a larger scale
(this can also be inferred from that errors for CACD and AFAD were larger
as shown in Table~\ref{tab:FAE} in Appendix~\ref{sec:ApeRWDE}).
For each dataset, we resized all images to $128\times 128$ pixels of 3 RGB channels and 
randomly split the dataset into 72\,\% training, 8\,\% validation, and 20\,\% test sets.
The training phase used images randomly cropped to $120 \times 120$ pixels of 3 channels
as input to improve the stability of the model against the difference of facial positions, 
and validation and test phases used images center-cropped to the same size.

\paragraph{Methods}
We tried 21 threshold methods and MLR as in 
the experiment in Section~\ref{sec:SDE}.
We implemented all the methods with ResNet34 \citep{he2016deep}.
We trained the network using Adam optimization 
with the mini-batch size 256 and the descending learning rate 
$0.1^{2.5+t/T}$ at $t$-th epoch during $T=200$ epochs.

\paragraph{Evaluation}
We took 5 trials, and considered the three OR tasks with the task loss 
$\ell=\ell_\zo,\ell_\ab,\ell_\sq$ and used the test MZE, MAE, and RMSE 
for the evaluation similarly to the experiment in Section~\ref{sec:SDE}.

\paragraph{Results}
Table~\ref{tab:FAE} in Appendix~\ref{sec:ApeRWDE} 
shows the mean and SD of MZE, MAE, and RMSE over the 5 trials,
and Figure~\ref{fig:Violin} in Appendix~\ref{sec:ApeRWDE} 
shows plots of $(y_i,\hat{a}(\bx_i))$'s.
These results lead to the following findings:
\begin{itemize}\setlength{\parskip}{0pt}\setlength{\itemindent}{0pt}
\item[\hyt{p1}]
Non-PL losses, especially, non-PL AT losses, 
gave the prediction performance comparable to MLR.
Plots of $(y_i,\hat{a}(\bx_i))$'s for other non-PL losses were 
similar to those for logi-AT-O, -IT-N, and -IT-O 
(Figure~\ref{fig:Violin} \hyl{q1}--\hyl{q3}).
\item[\hyt{p2}]
Plots of $(y_i,\hat{a}(\bx_i))$'s for abso-AT-O, -IT-N, and -IT-O were 
similar to those for hing-AT-O, -IT-N, and -IT-O 
(Figure~\ref{fig:Violin} \hyl{q4}--\hyl{q6}).
Although there was a little blur probably due to the optimization error, 
1DT values $\hat{a}(\bx_i)$'s learned with PL losses were concentrated at a few points.
Partly because of such concentration phenomenon, 
PL losses gave worse prediction performance than non-PL losses in many cases
(especially for CACD and AFAD that are inferred to have a larger scale).
\item[\hyt{p3}]
Table~\ref{tab:Bias} shows the mean and SD of the number 
of unique elements of learned bias parameter vector,
and indicates that the setting IT-O resulted in much heavy 
overlap of the learned bias parameter vector than the setting AT-O.
This phenomenon might contribute to low performance 
of the threshold methods with the setting IT-O.
\end{itemize}

\section{Conclusion and Future Works}
\label{sec:Conclusion}
We studied the influence of the underlying data distribution 
and of the learning procedure of the 1DT 
on the approximation error of threshold methods,
under the expectation that it would lead to understand
the prediction performance of those methods
learned with a sufficiently large-size sample.
First, Theorems~\ref{thm:AT-Consistency} and \ref{thm:IT-Consistency} 
clarify data distributions for which the threshold methods based on 
the learning procedure with the Logistic and Exponential-AT and -IT losses 
can yield the minimum approximation error.
On the ground of these results and Theorems~\ref{thm:POCL-shape} and \ref{thm:POACL-shape}, 
we conjecture that threshold methods using 
the Logistic and Exponential-AT and -IT losses or similarly-shaped losses 
will perform well for homoscedastic unimodal data and mode-wise heteroscedastic unimodal data
and may perform poorly for overall heteroscedastic unimodal data and almost non-unimodal data.
Second, according to Theorems~\ref{thm:ATSQ-Consistency} and \ref{thm:ITSQ-Consistency},
which give a closed-form expression of the surrogate risk minimizer 
with the Squared-AT and -IT losses respectively, we also conjecture that 
threshold methods using the Squared-AT and -IT losses or similarly-shaped losses 
will perform well for unimodal data with a small scale.
Third, the learning procedure based on an IT loss and ordered class of the bias parameter vector
can output overlapped learned bias parameter vector:
In Section~\ref{sec:IT-BP}, we demonstrated that such overlap 
may deteriorate the ability of the threshold method to reflect 
variation in the underlying data distribution.
Fourth, we showed in Section~\ref{sec:PL-Loss} that, 
under the learning procedure based on a PL loss,
learned 1DT values are concentrated at a few points.
It is considered that these troubles can act as negative factors 
that degrade the prediction performance.
We confirmed these results and conjectures as 
\hyl{l1}--\hyl{l3}, \hyl{n1}--\hyl{n3}, and \hyl{p1}--\hyl{p3} experimentally.
Consequently, these results advance 
the previous theoretical work \citep[Section 4]{pedregosa2017consistency},
which suggests the absence of Fisher consistency for many threshold methods,
and the experimental comparisons \citep{rennie2005loss, gutierrez2015ordinal}
of threshold methods using various surrogate loss functions.

The results of this study may be useful in developing 
novel learning procedures for threshold methods.
On the other hand, these results also indicate that 
a classifier of several threshold methods based on 
the strict surrogate risk minimization may not be good.
Devised design of the optimization procedure, 
including the setting of initial parameters and appropriate stopping, 
may improve the practical classification performance of the classifier.
Also, it would be still important to analyze the estimation error of 
the classifier based on the two-step optimization \eqref{eq:Estiab}.
These topics are future works.



\bibliographystyle{abbrvnat}
\bibliography{multitask_learning, ordinal_regression, SVM_SVR,machine_learning, sparse, age_estimation}

\begin{thebibliography}{38}
\providecommand{\natexlab}[1]{#1}
\providecommand{\url}[1]{\texttt{#1}}
\expandafter\ifx\csname urlstyle\endcsname\relax
  \providecommand{\doi}[1]{doi: #1}\else
  \providecommand{\doi}{doi: \begingroup \urlstyle{rm}\Url}\fi

\bibitem[Agresti(2010)]{agresti2010analysis}
A.~Agresti.
\newblock \emph{Analysis of Ordinal Categorical Data}, volume 656.
\newblock John Wiley \& Sons, 2nd edition, 2010.

\bibitem[Bartlett et~al.(2006)Bartlett, Jordan, and
  McAuliffe]{bartlett2006convexity}
P.~L. Bartlett, M.~I. Jordan, and J.~D. McAuliffe.
\newblock Convexity, classification, and risk bounds.
\newblock \emph{Journal of the American Statistical Association}, 101\penalty0
  (473):\penalty0 138--156, 2006.

\bibitem[Beckham and Pal(2017)]{beckham2017unimodal}
C.~Beckham and C.~Pal.
\newblock Unimodal probability distributions for deep ordinal classification.
\newblock In \emph{Proceedings of the International Conference on Machine
  Learning}, pages 411--419, 2017.

\bibitem[B{\"u}rkner and Vuorre(2019)]{burkner2019ordinal}
P.-C. B{\"u}rkner and M.~Vuorre.
\newblock Ordinal regression models in psychology: A tutorial.
\newblock \emph{Advances in Methods and Practices in Psychological Science},
  2\penalty0 (1):\penalty0 77--101, 2019.

\bibitem[Cao et~al.(2020)Cao, Mirjalili, and Raschka]{cao2020rank}
W.~Cao, V.~Mirjalili, and S.~Raschka.
\newblock Rank consistent ordinal regression for neural networks with
  application to age estimation.
\newblock \emph{Pattern Recognition Letters}, 140:\penalty0 325--331, 2020.

\bibitem[Chen et~al.(2014)Chen, Chen, and Hsu]{chen2014cross}
B.-C. Chen, C.-S. Chen, and W.~H. Hsu.
\newblock Cross-age reference coding for age-invariant face recognition and
  retrieval.
\newblock In \emph{Proceedings of the European Conference on Computer Vision},
  pages 768--783, 2014.

\bibitem[Chu and Ghahramani(2005)]{chu2005gaussian}
W.~Chu and Z.~Ghahramani.
\newblock Gaussian processes for ordinal regression.
\newblock \emph{Journal of Machine Learning Research}, 6\penalty0
  (Jul):\penalty0 1019--1041, 2005.

\bibitem[Chu and Keerthi(2005)]{chu2005new}
W.~Chu and S.~S. Keerthi.
\newblock New approaches to support vector ordinal regression.
\newblock In \emph{Proceedings of the International Conference on Machine
  Learning}, pages 145--152, 2005.

\bibitem[Chu and Keerthi(2007)]{chu2007support}
W.~Chu and S.~S. Keerthi.
\newblock Support vector ordinal regression.
\newblock \emph{Neural Computation}, 19\penalty0 (3):\penalty0 792--815, 2007.

\bibitem[Cohen(1960)]{cohen1960coefficient}
J.~Cohen.
\newblock A coefficient of agreement for nominal scales.
\newblock \emph{Educational and Psychological Measurement}, 20\penalty0
  (1):\penalty0 37--46, 1960.

\bibitem[Cohen(1968)]{cohen1968weighted}
J.~Cohen.
\newblock Weighted kappa: nominal scale agreement provision for scaled
  disagreement or partial credit.
\newblock \emph{Psychological Bulletin}, 70\penalty0 (4):\penalty0 213, 1968.

\bibitem[Cortes and Vapnik(1995)]{cortes1995support}
C.~Cortes and V.~Vapnik.
\newblock Support-vector networks.
\newblock \emph{Machine Learning}, 20:\penalty0 273--297, 1995.

\bibitem[Crammer and Singer(2002)]{crammer2002pranking}
K.~Crammer and Y.~Singer.
\newblock Pranking with ranking.
\newblock In \emph{Advances in Neural Information Processing Systems}, pages
  641--647, 2002.

\bibitem[da~Costa et~al.(2008)da~Costa, Alonso, and Cardoso]{da2008unimodal}
J.~F.~P. da~Costa, H.~Alonso, and J.~S. Cardoso.
\newblock The unimodal model for the classification of ordinal data.
\newblock \emph{Neural Networks}, 21\penalty0 (1):\penalty0 78--91, 2008.

\bibitem[Franses and Paap(2001)]{franses2001quantitative}
P.~H. Franses and R.~Paap.
\newblock \emph{Quantitative Models in Marketing Research}.
\newblock Cambridge University Press, 2001.

\bibitem[Gutierrez et~al.(2015)Gutierrez, Perez-Ortiz, Sanchez-Monedero,
  Fernandez-Navarro, and Hervas-Martinez]{gutierrez2015ordinal}
P.~A. Gutierrez, M.~Perez-Ortiz, J.~Sanchez-Monedero, F.~Fernandez-Navarro, and
  C.~Hervas-Martinez.
\newblock Ordinal regression methods: survey and experimental study.
\newblock \emph{IEEE Transactions on Knowledge and Data Engineering},
  28\penalty0 (1):\penalty0 127--146, 2015.

\bibitem[He et~al.(2016)He, Zhang, Ren, and Sun]{he2016deep}
K.~He, X.~Zhang, S.~Ren, and J.~Sun.
\newblock Deep residual learning for image recognition.
\newblock In \emph{Proceedings of the IEEE Conference on Computer Vision and
  Pattern Recognition}, pages 770--778, 2016.

\bibitem[Li and Lin(2007)]{li2007ordinal}
L.~Li and H.-T. Lin.
\newblock Ordinal regression by extended binary classification.
\newblock In \emph{Advances in Neural Information Processing Systems}, pages
  865--872, 2007.

\bibitem[Lin and Li(2006)]{lin2006large}
H.-T. Lin and L.~Li.
\newblock Large-margin thresholded ensembles for ordinal regression: Theory and
  practice.
\newblock In \emph{Algorithmic Learning Theory}, pages 319--333, 2006.

\bibitem[Lin and Li(2012)]{lin2012reduction}
H.-T. Lin and L.~Li.
\newblock Reduction from cost-sensitive ordinal ranking to weighted binary
  classification.
\newblock \emph{Neural Computation}, 24\penalty0 (5):\penalty0 1329--1367,
  2012.

\bibitem[Liu(2011)]{liu2009learning}
T.-Y. Liu.
\newblock \emph{Learning to Rank for Information Retrieval}.
\newblock Springer Science \& Business Media, 2011.

\bibitem[Liu(2007)]{liu2007fisher}
Y.~Liu.
\newblock Fisher consistency of multicategory support vector machines.
\newblock In \emph{Proceedings of the International Conference on Artificial
  Intelligence and Statistics}, pages 291--298, 2007.

\bibitem[McCullagh(1980)]{mccullagh1980regression}
P.~McCullagh.
\newblock Regression models for ordinal data.
\newblock \emph{Journal of the Royal Statistical Society: Series B
  (Methodological)}, 42\penalty0 (2):\penalty0 109--127, 1980.

\bibitem[Niu et~al.(2016)Niu, Zhou, Wang, Gao, and Hua]{niu2016ordinal}
Z.~Niu, M.~Zhou, L.~Wang, X.~Gao, and G.~Hua.
\newblock Ordinal regression with multiple output cnn for age estimation.
\newblock In \emph{Proceedings of the IEEE Conference on Computer Vision and
  Pattern Recognition}, pages 4920--4928, 2016.

\bibitem[Pedregosa et~al.(2017)Pedregosa, Bach, and
  Gramfort]{pedregosa2017consistency}
F.~Pedregosa, F.~Bach, and A.~Gramfort.
\newblock On the consistency of ordinal regression methods.
\newblock \emph{Journal of Machine Learning Research}, 18\penalty0
  (Jan):\penalty0 1769--1803, 2017.

\bibitem[Pfannschmidt et~al.(2020)Pfannschmidt, Jakob, Hinder, Biehl, Tino, and
  Hammer]{pfannschmidt2020feature}
L.~Pfannschmidt, J.~Jakob, F.~Hinder, M.~Biehl, P.~Tino, and B.~Hammer.
\newblock Feature relevance determination for ordinal regression in the context
  of feature redundancies and privileged information.
\newblock \emph{Neurocomputing}, 416:\penalty0 266--279, 2020.

\bibitem[Pires et~al.(2013)Pires, Szepesvari, and Ghavamzadeh]{pires2013cost}
B.~A. Pires, C.~Szepesvari, and M.~Ghavamzadeh.
\newblock Cost-sensitive multiclass classification risk bounds.
\newblock In \emph{Proceedings of the International Conference on Machine
  Learning}, pages 1391--1399, 2013.

\bibitem[Rennie and Srebro(2005)]{rennie2005loss}
J.~D. Rennie and N.~Srebro.
\newblock Loss functions for preference levels: Regression with discrete
  ordered labels.
\newblock In \emph{Proceedings of the IJCAI Multidisciplinary Workshop on
  Advances in Preference Handling}, pages 180--186, 2005.

\bibitem[Ricanek and Tesafaye(2006)]{ricanek2006morph}
K.~Ricanek and T.~Tesafaye.
\newblock Morph: A longitudinal image database of normal adult age-progression.
\newblock In \emph{Proceedings of the IEEE International Conference on
  Automatic Face and Gesture Recognition}, pages 341--345, 2006.

\bibitem[Sagonas et~al.(2016)Sagonas, Antonakos, Tzimiropoulos, Zafeiriou, and
  Pantic]{sagonas2016300}
C.~Sagonas, E.~Antonakos, G.~Tzimiropoulos, S.~Zafeiriou, and M.~Pantic.
\newblock 300 faces in-the-wild challenge: Database and results.
\newblock \emph{Image and Vision Computing}, 47:\penalty0 3--18, 2016.

\bibitem[Shalev-Shwartz and Ben-David(2014)]{shalev2014understanding}
S.~Shalev-Shwartz and S.~Ben-David.
\newblock \emph{Understanding Machine Learning: From Theory to Algorithms}.
\newblock Cambridge University Press, 2014.

\bibitem[Shashua and Levin(2002)]{shashua2002taxonomy}
A.~Shashua and A.~Levin.
\newblock Taxonomy of large margin principle algorithms for ordinal regression
  problems.
\newblock In \emph{Advances in Neural Information Processing Systems}, pages
  937--944, 2002.

\bibitem[Shashua and Levin(2003)]{shashua2003ranking}
A.~Shashua and A.~Levin.
\newblock Ranking with large margin principle: Two approaches.
\newblock In \emph{Advances in Neural Information Processing Systems}, pages
  961--968, 2003.

\bibitem[Tewari and Bartlett(2007)]{tewari2007consistency}
A.~Tewari and P.~L. Bartlett.
\newblock On the consistency of multiclass classification methods.
\newblock \emph{Journal of Machine Learning Research}, 8\penalty0
  (May):\penalty0 1007--1025, 2007.

\bibitem[Yamasaki(2022)]{yamasaki2022unimodal}
R.~Yamasaki.
\newblock Unimodal likelihood models for ordinal data.
\newblock \emph{Transactions on Machine Learning Research}, 2022.
\newblock URL \url{https://openreview.net/forum?id=1l0sClLiPc}.

\bibitem[Yamasaki(2023)]{yamasaki2022optimal}
R.~Yamasaki.
\newblock Optimal threshold labeling for ordinal regression methods.
\newblock \emph{Transactions on Machine Learning Research}, 2023.
\newblock URL \url{https://openreview.net/forum?id=mHSAy1n65Z}.

\bibitem[Yamasaki and Tanaka(2024)]{yamasaki2024parallel}
R.~Yamasaki and T.~Tanaka.
\newblock Parallel algorithm for optimal threshold labeling of ordinal
  regression methods.
\newblock \emph{arXiv preprint arXiv:2405.12756v1}, 2024.

\bibitem[Yu et~al.(2006)Yu, Yu, Tresp, and Kriegel]{yu2006collaborative}
S.~Yu, K.~Yu, V.~Tresp, and H.-P. Kriegel.
\newblock Collaborative ordinal regression.
\newblock In \emph{Proceedings of the International Conference on Machine
  Learning}, pages 1089--1096, 2006.

\end{thebibliography}

\appendix
\section{Tables and Figures about Experimental Results}
\label{sec:SER}
\subsection{Results for Simulation Study in Section~\protect\ref{sec:SS}}
\label{sec:ApeSS}
Tables~\ref{tab:SimRes-H}--\ref{tab:SimRes-N} show
the simulated approximation errors of 21 threshold methods and optimal values (Bayes errors)
for the simulation study in Section~\ref{sec:SS}.
Results for a threshold method with the error equal to 
the optimal value are omitted or marked as `-'.

\begin{table}[H]
\renewcommand{\arraystretch}{0.75}\renewcommand{\tabcolsep}{5pt}\centering%
\caption{Results for the simulation in Section~\ref{sec:SS} with the data distributions $\text{H-}1/3$, $\text{H-}1$, and $\text{H-}3$.}
\label{tab:SimRes-H}\scalebox{0.75}{\begin{tabular}{cc|ccc|ccc|ccc}\toprule
&&\multicolumn{3}{c|}{\scs H-1/3}&\multicolumn{3}{c|}{\scs H-1}&\multicolumn{3}{c}{\scs H-3}\\
&&{\scs MZE}&{\scs MAE}&{\scs RMSE}&{\scs MZE}&{\scs MAE}&{\scs RMSE}&{\scs MZE}&{\scs MAE}&{\scs RMSE}\\
\midrule
\multicolumn{2}{c|}{\scs optimal}&.726&1.144&1.492&.571&.686&.959&.349&.356&.609\\
\midrule\multirow{2}{*}{\rotatebox{90}{\tiny AT-O\hspace{0.1em}}}
&{\scs hing}&.727&1.150&1.502&.575&.690&.963&-&-&-\\
&{\scs abso}&.746&1.233&1.585&.657&.878&1.183&.592&.775&1.052\\
\midrule\multirow{2}{*}{\rotatebox{90}{\tiny IT-N\hspace{0.1em}}}
&{\scs hing}&.728&1.150&1.502&.577&.695&.972&-&-&-\\
&{\scs abso}&.731&1.164&1.520&.586&.725&1.010&.378&.390&.644\\
\midrule\multirow{2}{*}{\rotatebox{90}{\tiny IT-O\hspace{0.1em}}}
&{\scs hing}&.727&1.148&1.497&.575&.692&.966&-&-&-\\
&{\scs abso}&.730&1.161&1.517&.578&.700&.976&.366&.376&.630\\
\bottomrule
\end{tabular}}\end{table}
\begin{table}[H]
\renewcommand{\arraystretch}{0.75}\renewcommand{\tabcolsep}{5pt}\centering%
\caption{Results for the simulation in Section~\ref{sec:SS} with the data distributions $\text{M-}1/3$, $\text{M-}1$, and $\text{M-}3$.}
\label{tab:SimRes-M}\scalebox{0.75}{\begin{tabular}{cc|ccc|ccc|ccc}\toprule
&&\multicolumn{3}{c|}{\scs M-1/3}&\multicolumn{3}{c|}{\scs M-1}&\multicolumn{3}{c}{\scs M-3}\\
&&{\scs MZE}&{\scs MAE}&{\scs RMSE}&{\scs MZE}&{\scs MAE}&{\scs RMSE}&{\scs MZE}&{\scs MAE}&{\scs RMSE}\\
\midrule
\multicolumn{2}{c|}{\scs optimal}&.716&1.122&1.452&.529&.660&.952&.296&.318&.597\\
\midrule\multirow{2}{*}{\rotatebox{90}{\tiny AT-O\hspace{0.1em}}}
&{\scs hing}&.717&1.126&1.457&.530&.665&.957&-&-&-\\
&{\scs abso}&.786&1.503&1.887&.619&.957&1.335&.484&.666&1.042\\
\midrule\multirow{2}{*}{\rotatebox{90}{\tiny IT-N\hspace{0.1em}}}
&{\scs hing}&.721&1.143&1.479&.531&.671&.966&-&-&.598\\
&{\scs abso}&.730&1.215&1.585&.536&.684&.974&.303&.330&.619\\
\midrule\multirow{2}{*}{\rotatebox{90}{\tiny IT-O\hspace{0.1em}}}
&{\scs hing}&.720&1.141&1.478&.531&.666&.961&-&-&-\\
&{\scs abso}&.722&1.172&1.527&.541&.694&.983&.299&.326&.616\\
\bottomrule
\end{tabular}}\end{table}
\begin{table}[H]
\renewcommand{\arraystretch}{0.75}\renewcommand{\tabcolsep}{5pt}\centering%
\caption{Results for the simulation in Section~\ref{sec:SS} with the data distributions $\text{A-}1/3$, $\text{A-}1$, and $\text{A-}3$.}
\label{tab:SimRes-A}\scalebox{0.75}{\begin{tabular}{cc|ccc|ccc|ccc}\toprule
&&\multicolumn{3}{c|}{\scs A-1/3}&\multicolumn{3}{c|}{\scs A-1}&\multicolumn{3}{c}{\scs A-3}\\
&&{\scs MZE}&{\scs MAE}&{\scs RMSE}&{\scs MZE}&{\scs MAE}&{\scs RMSE}&{\scs MZE}&{\scs MAE}&{\scs RMSE}\\
\midrule
\multicolumn{2}{c|}{\scs optimal}&.720&1.141&1.493&.553&.681&.960&.309&.335&.606\\
\midrule\multirow{2}{*}{\rotatebox{90}{\tiny AT-O\hspace{0.1em}}}
&{\scs hing}&.722&1.150&1.502&.554&.683&.963&-&-&-\\
&{\scs abso}&.750&1.256&1.626&.662&.918&1.225&.528&.727&1.042\\
\midrule\multirow{2}{*}{\rotatebox{90}{\tiny IT-N\hspace{0.1em}}}
&{\scs hing}&.728&1.172&1.530&.560&.707&.991&.313&.343&.612\\
&{\scs abso}&.728&1.174&1.537&.574&.737&1.043&.357&.409&.688\\
\midrule\multirow{2}{*}{\rotatebox{90}{\tiny IT-O\hspace{0.1em}}}
&{\scs hing}&.722&1.149&1.503&.557&.691&.977&-&-&.607\\
&{\scs abso}&.728&1.180&1.548&.562&.710&.995&.332&.360&.635\\
\bottomrule
\end{tabular}}\end{table}
\begin{table}[H]
\renewcommand{\arraystretch}{0.75}\renewcommand{\tabcolsep}{5pt}\centering%
\caption{Results for the simulation in Section~\ref{sec:SS} with the data distributions $\text{O-}(1/3,1)$, $\text{O-}(1/3,3)$, and $\text{O-}(1/3,3)$.}
\label{tab:SimRes-O}\scalebox{0.75}{\begin{tabular}{cc|ccc|ccc|ccc}\toprule
&&\multicolumn{3}{c|}{\scs O-(1/3,1)}&\multicolumn{3}{c|}{\scs O-(1/3,3)}&\multicolumn{3}{c}{\scs O-(1,3)}\\
&&{\scs MZE}&{\scs MAE}&{\scs RMSE}&{\scs MZE}&{\scs MAE}&{\scs RMSE}&{\scs MZE}&{\scs MAE}&{\scs RMSE}\\
\midrule
\multicolumn{2}{c|}{\scs optimal}&.646&.911&1.250&.535&.746&1.135&.457&.517&.800\\
\midrule\multirow{7}{*}{\rotatebox{90}{\tiny AT-O\hspace{1em}}}
&{\scs logi}&.650&-&-&.543&.748&-&.458&.518&-\\
&{\scs hing}&.652&.918&1.261&.543&.748&1.137&-&-&-\\
&{\scs smhi}&.650&-&-&.543&-&-&.458&-&-\\
&{\scs sqhi}&.650&-&-&.543&.748&-&.458&.518&-\\
&{\scs expo}&.654&.913&-&.544&.748&1.136&.459&.518&.801\\
&{\scs abso}&.694&1.053&1.395&.635&.954&1.356&.587&.780&1.104\\
&{\scs squa}&.650&-&-&.543&.748&-&.458&.518&-\\
\midrule\multirow{7}{*}{\rotatebox{90}{\tiny IT-N\hspace{1em}}}
&{\scs logi}&.649&-&-&.540&-&1.136&.458&-&-\\
&{\scs hing}&.650&.915&1.256&.546&.755&1.146&.458&-&-\\
&{\scs smhi}&.648&-&1.251&.538&-&-&.458&-&-\\
&{\scs sqhi}&.649&-&-&.540&-&-&.458&-&-\\
&{\scs expo}&.649&-&-&.542&-&-&.458&.518&-\\
&{\scs abso}&.656&.932&1.273&.578&.809&1.182&.468&.536&.820\\
&{\scs squa}&.649&-&-&.542&-&-&.458&-&-\\
\midrule\multirow{7}{*}{\rotatebox{90}{\tiny IT-O\hspace{1em}}}
&{\scs logi}&.649&-&-&.540&-&1.136&.458&-&-\\
&{\scs hing}&.648&.912&1.252&.552&.762&1.144&-&-&-\\
&{\scs smhi}&.648&-&1.251&.538&-&-&.458&-&-\\
&{\scs sqhi}&.649&-&-&.540&-&-&.458&-&-\\
&{\scs expo}&.649&-&-&.542&-&-&.458&.518&-\\
&{\scs abso}&.651&.921&1.261&.555&.770&1.158&.466&.531&.815\\
&{\scs squa}&.649&-&-&.542&-&-&.458&-&-\\
\bottomrule
\end{tabular}}\end{table}
\begin{table}[H]
\renewcommand{\arraystretch}{0.75}\renewcommand{\tabcolsep}{5pt}\centering%
\caption{Results for the simulation in Section~\ref{sec:SS} with the data distributions $\text{N-}1/3$, $\text{N-}1$, and $\text{N-}3$.}
\label{tab:SimRes-N}\scalebox{0.75}{\begin{tabular}{cc|ccc|ccc|ccc}\toprule
&&\multicolumn{3}{c|}{\scs N-1/3}&\multicolumn{3}{c|}{\scs N-1}&\multicolumn{3}{c}{\scs N-3}\\
&&{\scs MZE}&{\scs MAE}&{\scs RMSE}&{\scs MZE}&{\scs MAE}&{\scs RMSE}&{\scs MZE}&{\scs MAE}&{\scs RMSE}\\
\midrule
\multicolumn{2}{c|}{\scs optimal}&.726&1.887&2.372&.571&1.242&1.666&.349&.672&1.068\\
\midrule\multirow{7}{*}{\rotatebox{90}{\tiny AT-O\hspace{1em}}}
&{\scs logi}&.788&1.905&-&.687&1.280&-&.517&.727&-\\
&{\scs hing}&.778&1.924&2.386&.661&1.272&1.672&.487&.721&1.077\\
&{\scs smhi}&.787&1.903&-&.682&1.276&-&.509&.718&-\\
&{\scs sqhi}&.788&1.905&-&.687&1.280&-&.522&.727&-\\
&{\scs expo}&.788&1.905&-&.687&1.279&-&.531&.735&-\\
&{\scs abso}&.822&2.217&2.652&.687&1.347&1.734&.626&1.072&1.415\\
&{\scs squa}&.788&1.905&-&.687&1.280&-&.517&.727&-\\
\midrule\multirow{7}{*}{\rotatebox{90}{\tiny IT-N\hspace{1em}}}
&{\scs logi}&.768&1.928&2.396&.666&1.287&1.669&.508&.726&-\\
&{\scs hing}&.770&2.076&2.524&.665&1.364&1.793&.446&.797&1.177\\
&{\scs smhi}&.757&1.948&2.402&.655&1.284&1.671&.471&.711&1.070\\
&{\scs sqhi}&.768&1.930&2.396&.666&1.287&1.669&.514&.730&-\\
&{\scs expo}&.769&1.928&2.395&.668&1.287&1.668&.509&.731&1.069\\
&{\scs abso}&.770&2.038&2.478&.668&1.449&1.863&.484&.868&1.211\\
&{\scs squa}&.768&1.930&2.396&.666&1.287&1.669&.515&.732&-\\
\midrule\multirow{7}{*}{\rotatebox{90}{\tiny IT-O\hspace{1em}}}
&{\scs logi}&.769&1.928&2.396&.666&1.287&1.669&.508&.726&-\\
&{\scs hing}&.769&2.035&2.479&.646&1.334&1.735&.496&.774&1.128\\
&{\scs smhi}&.764&1.928&2.394&.655&1.284&1.671&.471&.711&1.070\\
&{\scs sqhi}&.769&1.928&2.396&.666&1.287&1.669&.514&.730&-\\
&{\scs expo}&.769&1.928&2.395&.668&1.287&1.668&.509&.731&1.069\\
&{\scs abso}&.776&2.054&2.510&.678&1.391&1.797&.510&.849&1.166\\
&{\scs squa}&.769&1.928&2.396&.666&1.287&1.669&.515&.732&-\\
\bottomrule
\end{tabular}}\end{table}

Figure~\ref{fig:SimRes} shows the learned 1DT values $\bar{a}_i$,
method's label prediction $\bar{f}_{\ell,i}$,
and optimal label prediction $\tilde{f}_{\ell,i}$
for several instances of the simulation in Section~\ref{sec:SS}.
This indicates representative behaviors of 
threshold methods based on a learning procedure with non-PL and PL losses.

\begin{figure}[H]
\renewcommand{\arraystretch}{0.1}\renewcommand{\tabcolsep}{0pt}\centering%
\begin{tabular}{rrrrr}
~~&
{\tiny\hyt{m1} H-1, logi~~~~~~~}~~~~&
{\tiny\hyt{m2} M-1, logi~~~~~~~}~~~~&
{\tiny\hyt{m3} A-1, logi~~~~~~~}~~~~&
{\tiny\hyt{m4} O-(1/3,3), logi~~~}\\
\rotatebox{90}{\tiny~~~~~~~~~~~~~~~AT-O}~~&
\includegraphics[height=3cm, bb=0 0 399.240798 381.600763]{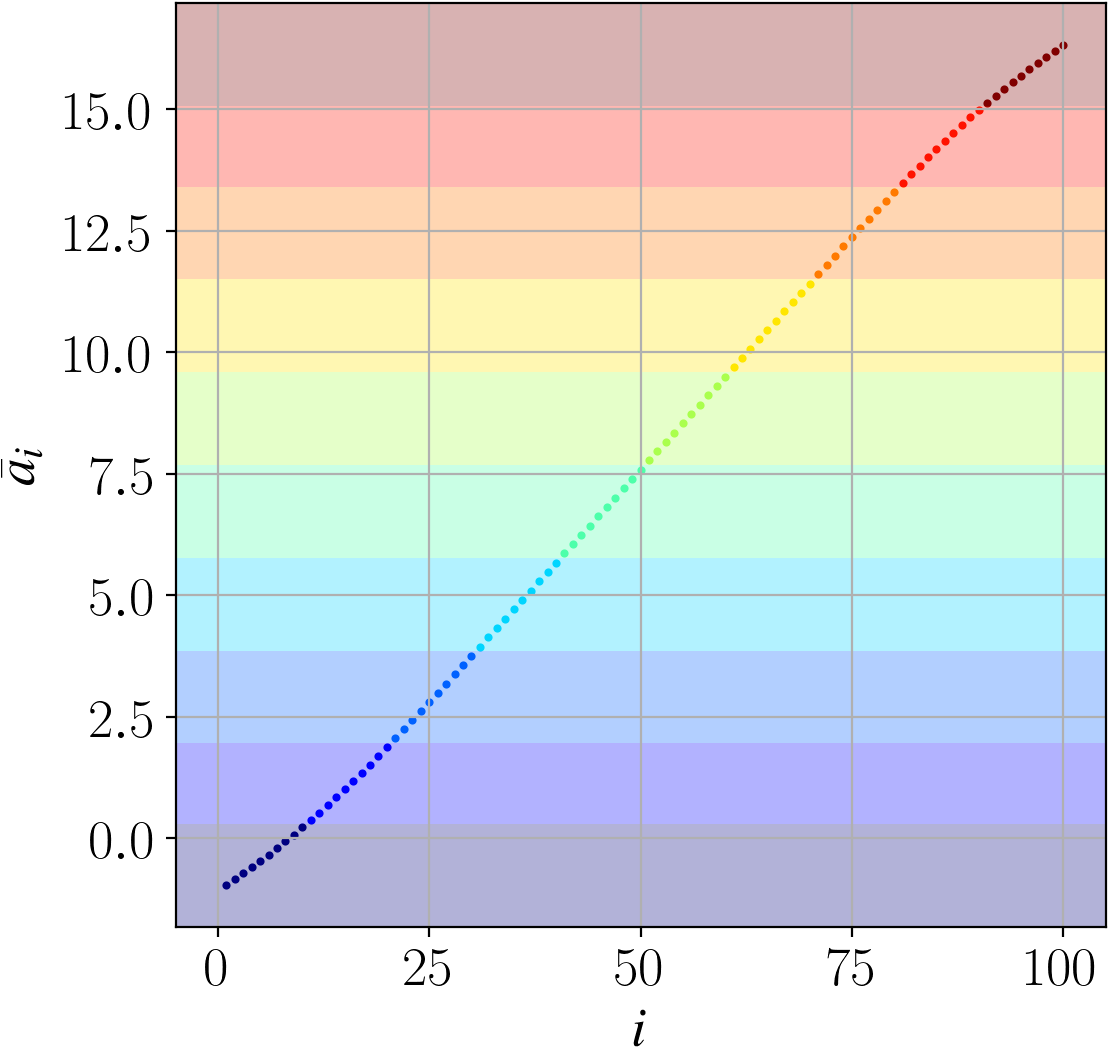}~~~~&
\includegraphics[height=3cm, bb=0 0 399.240798 381.600763]{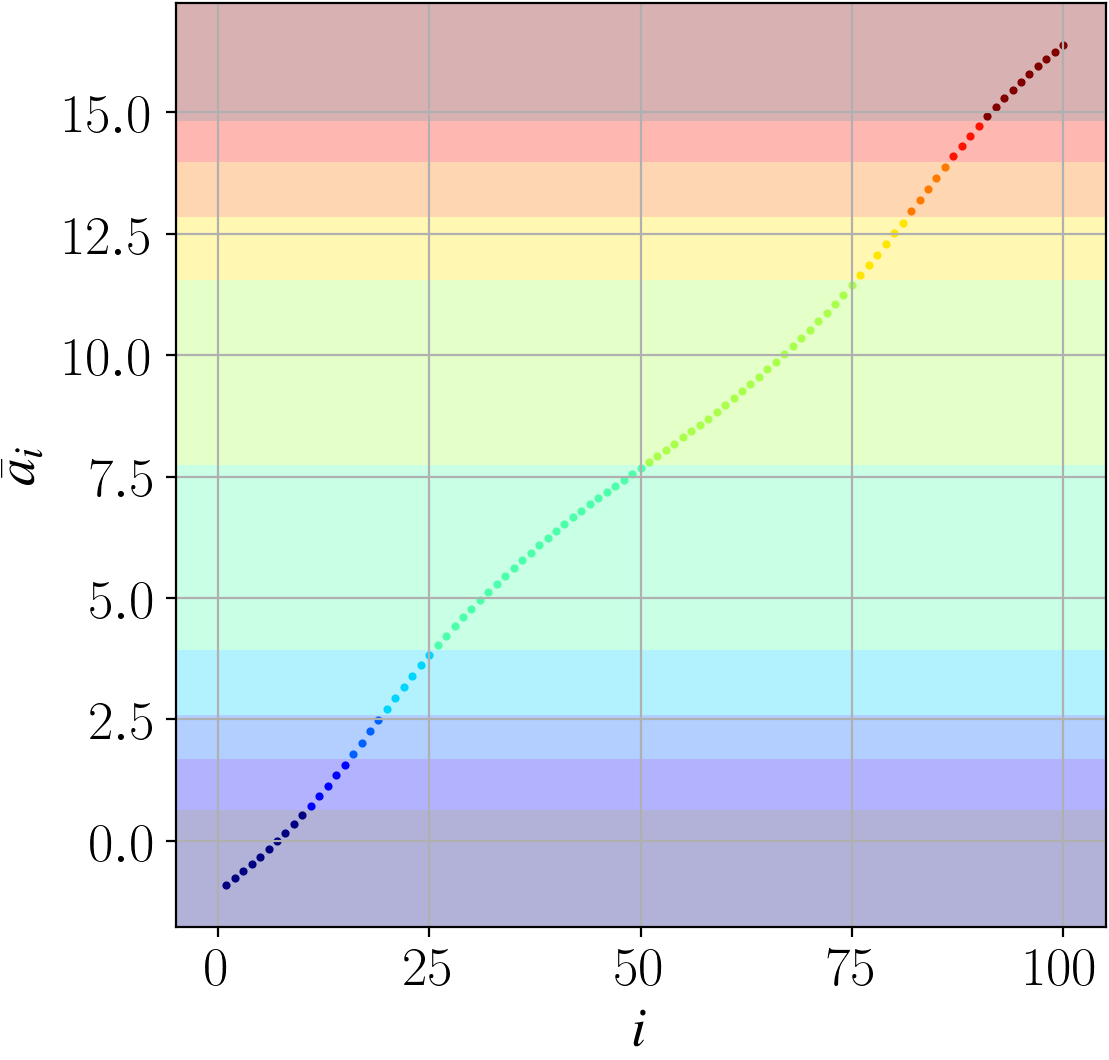}~~~~&
\includegraphics[height=3cm, bb=0 0 399.240798 381.600763]{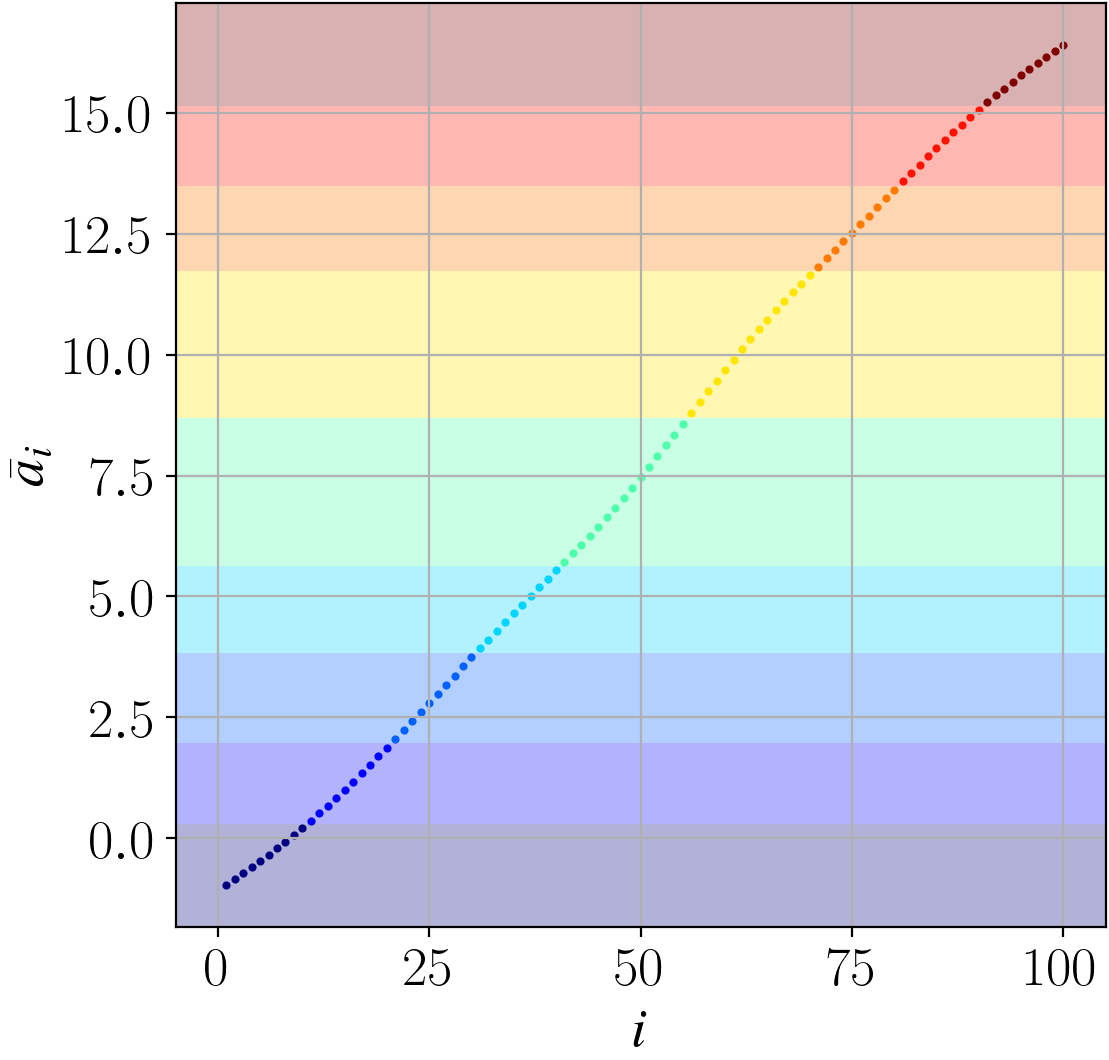}~~~~&
\includegraphics[height=3cm, bb=0 0 384.480769 381.600763]{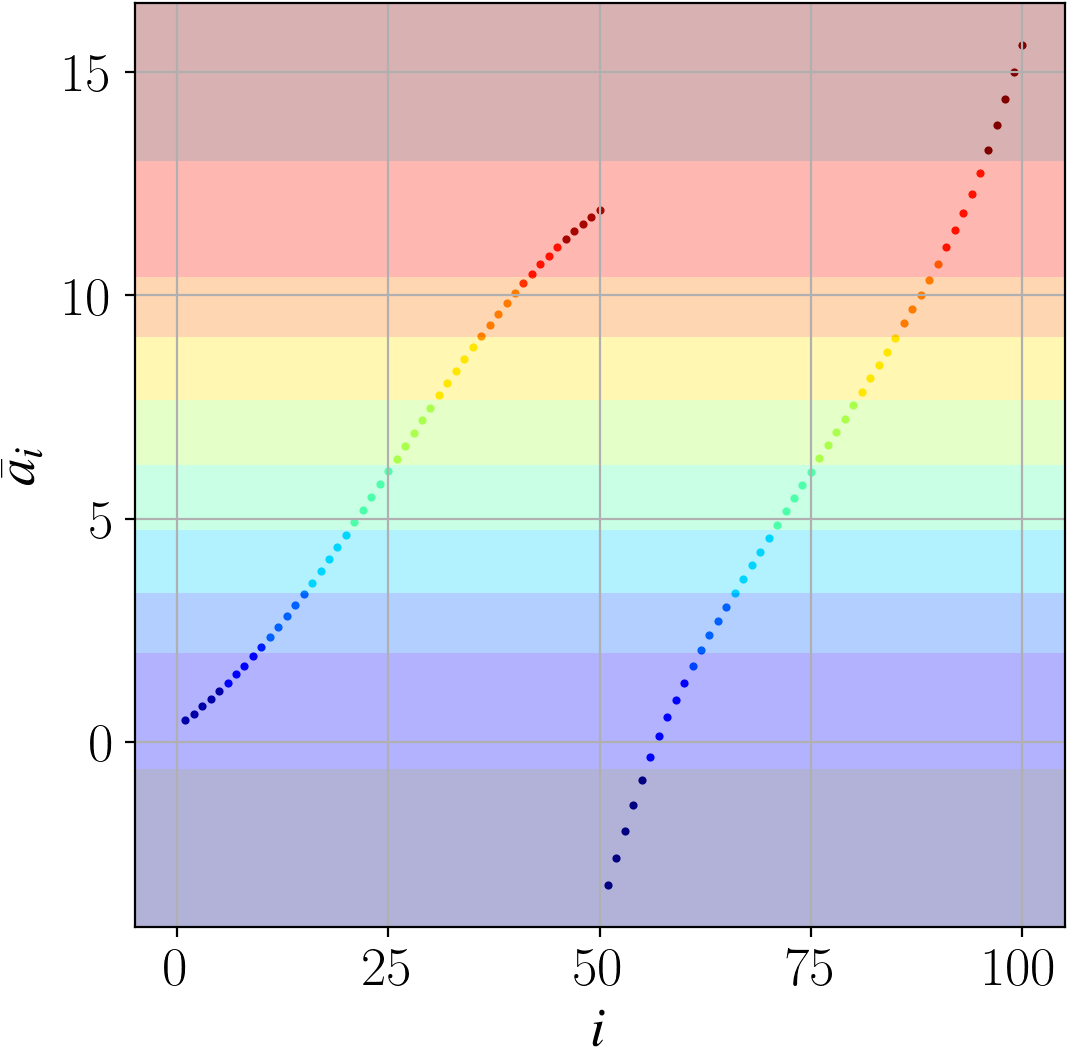}\\
\rotatebox{90}{\tiny~~~~~~~~~~~~~~~IT-N}~~&
\includegraphics[height=3cm, bb=0 0 375 382]{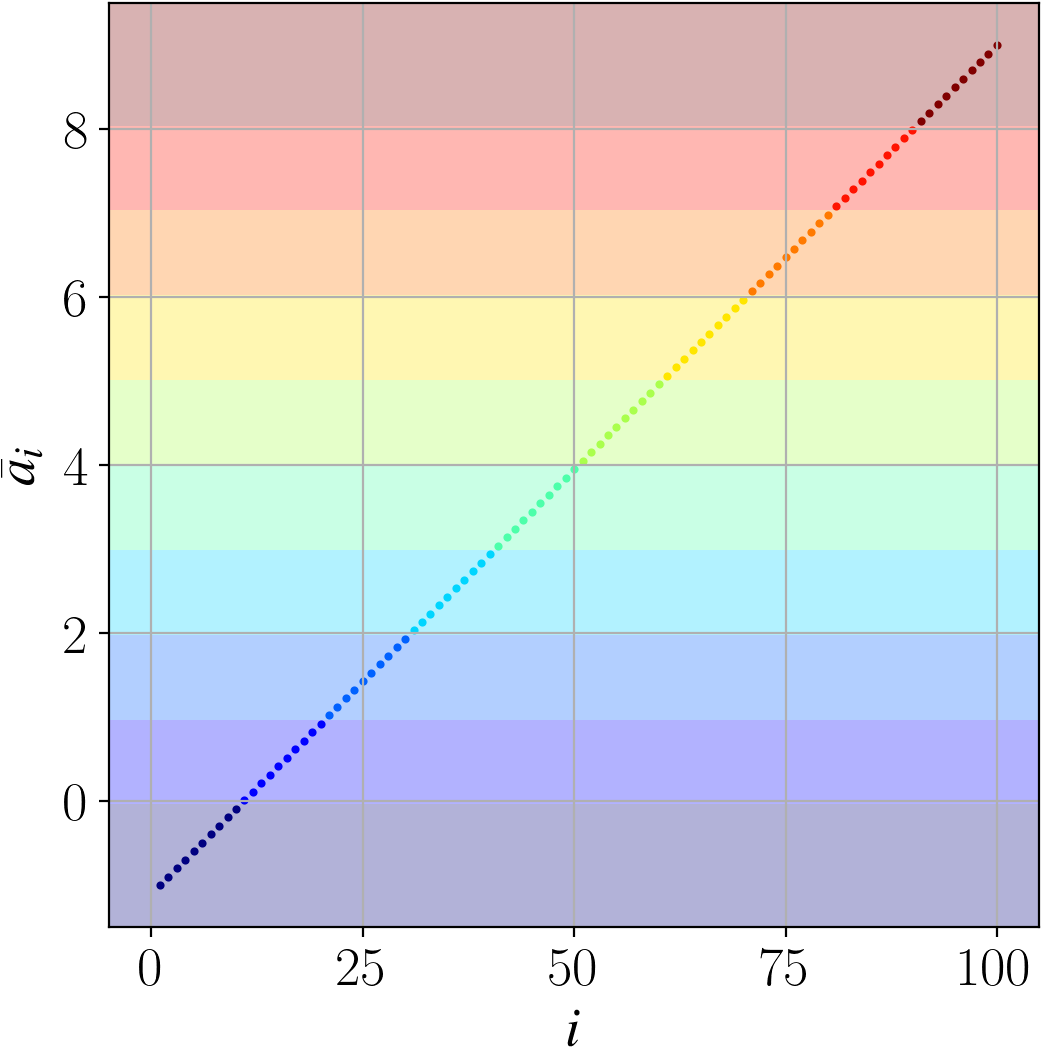}~~~~&
\includegraphics[height=3cm, bb=0 0 375 382]{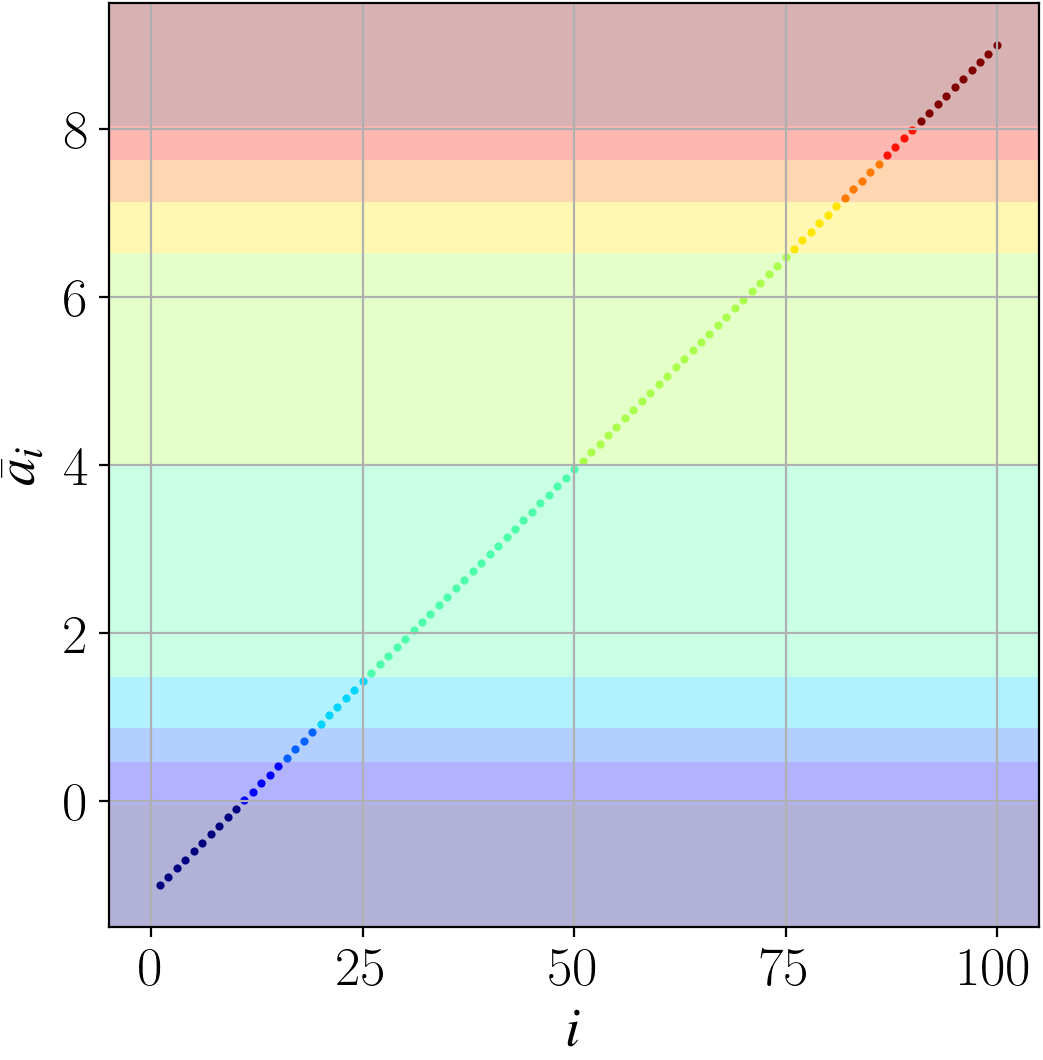}~~~~&
\includegraphics[height=3cm, bb=0 0 375 382]{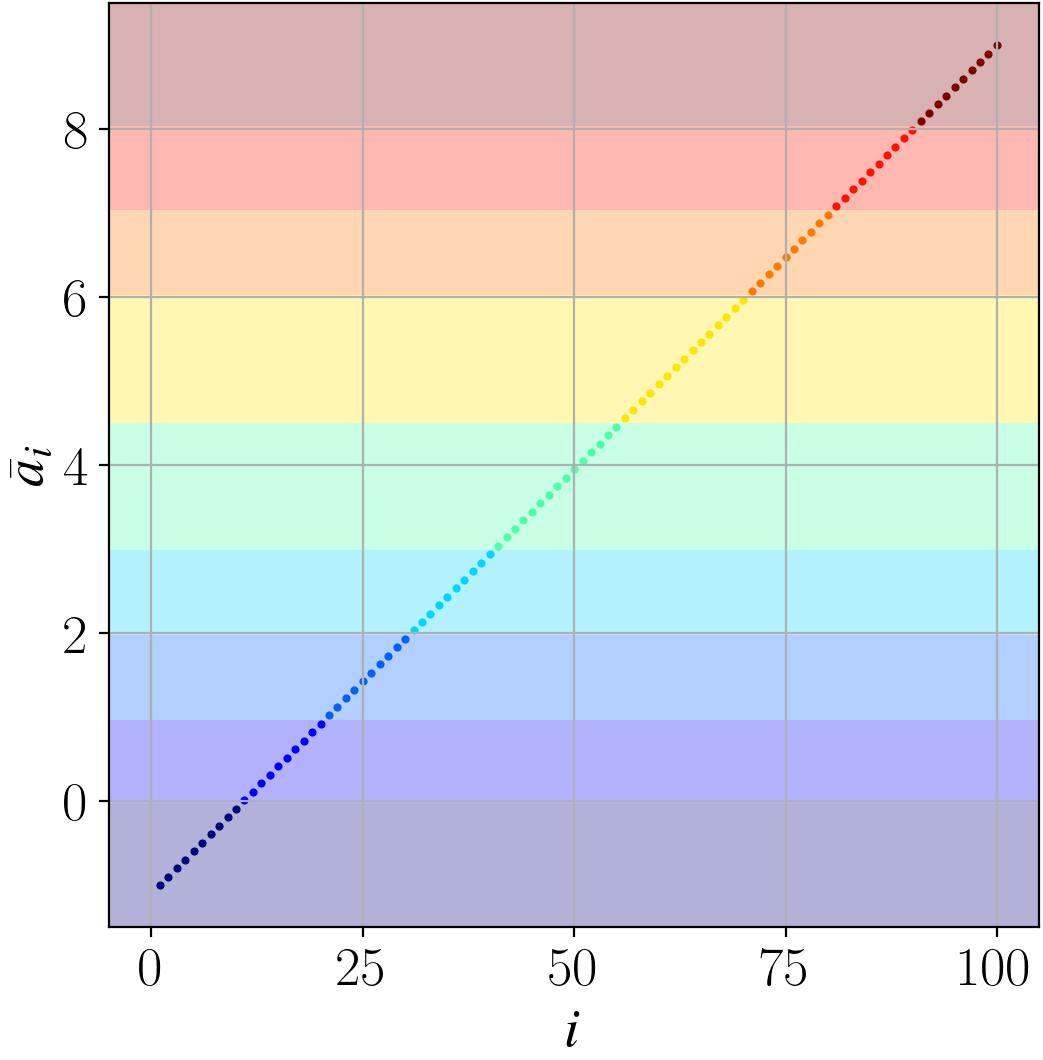}~~~~&
\includegraphics[height=3cm, bb=0 0 388 382]{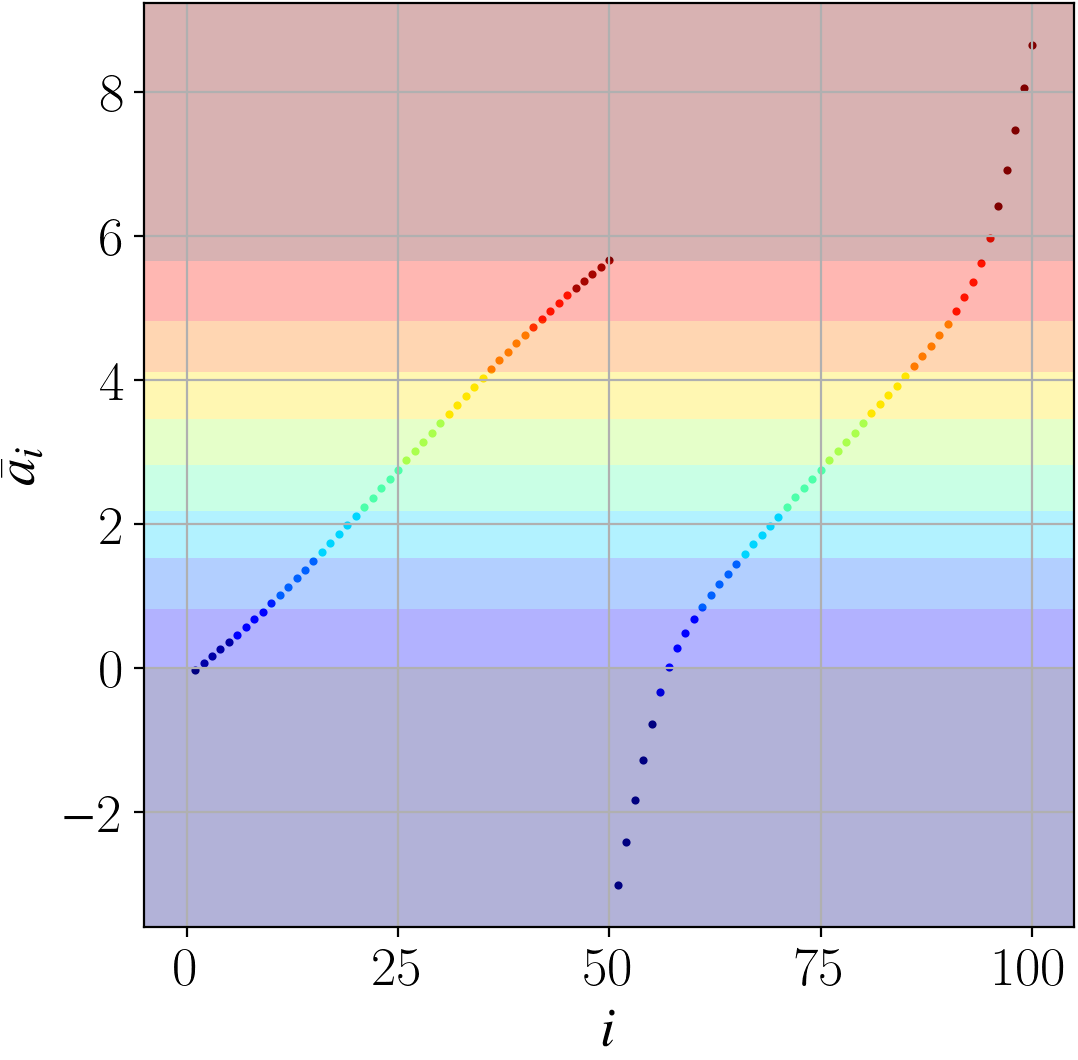}\\
\rotatebox{90}{\tiny~~~~~~~~~~~~~~~IT-O}~~&
\includegraphics[height=3cm, bb=0 0 375 382]{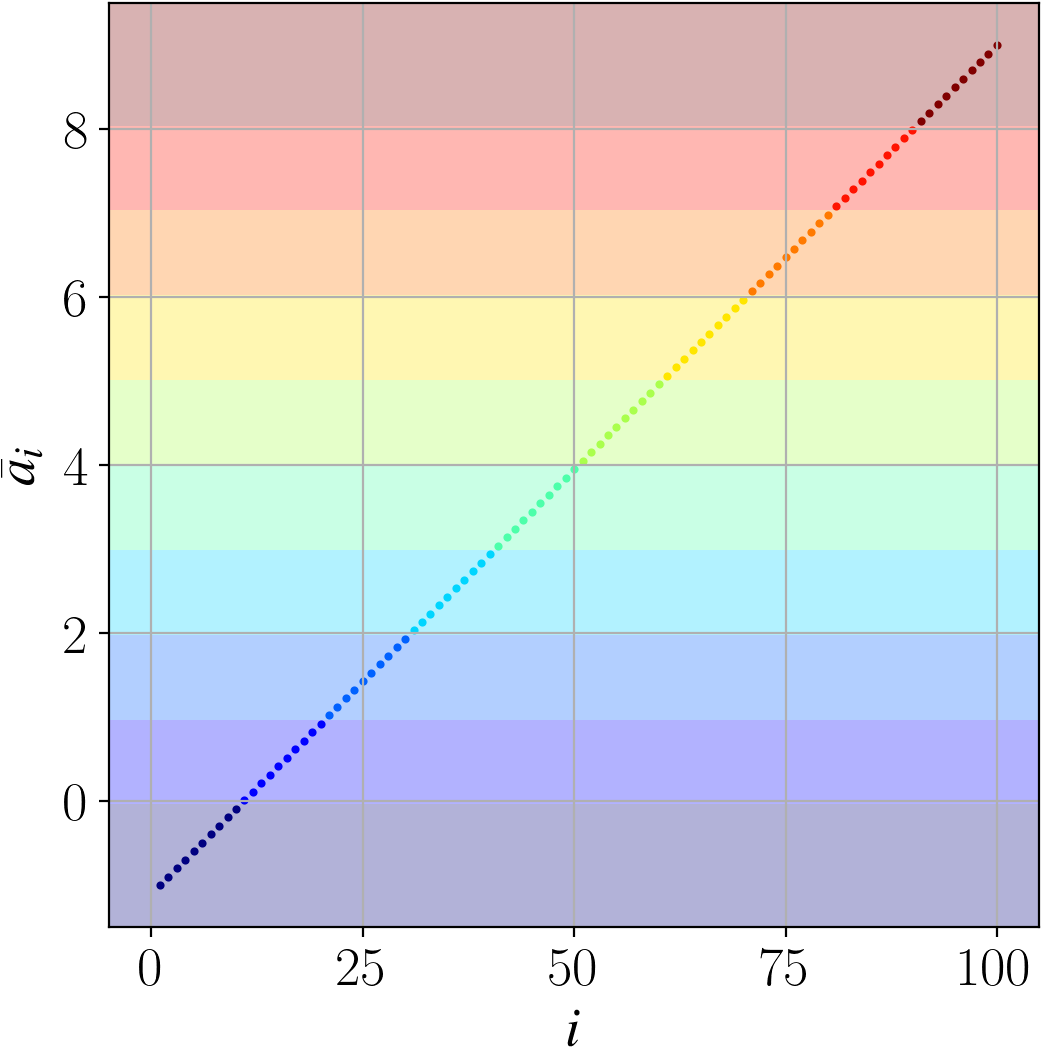}~~~~&
\includegraphics[height=3cm, bb=0 0 375 382]{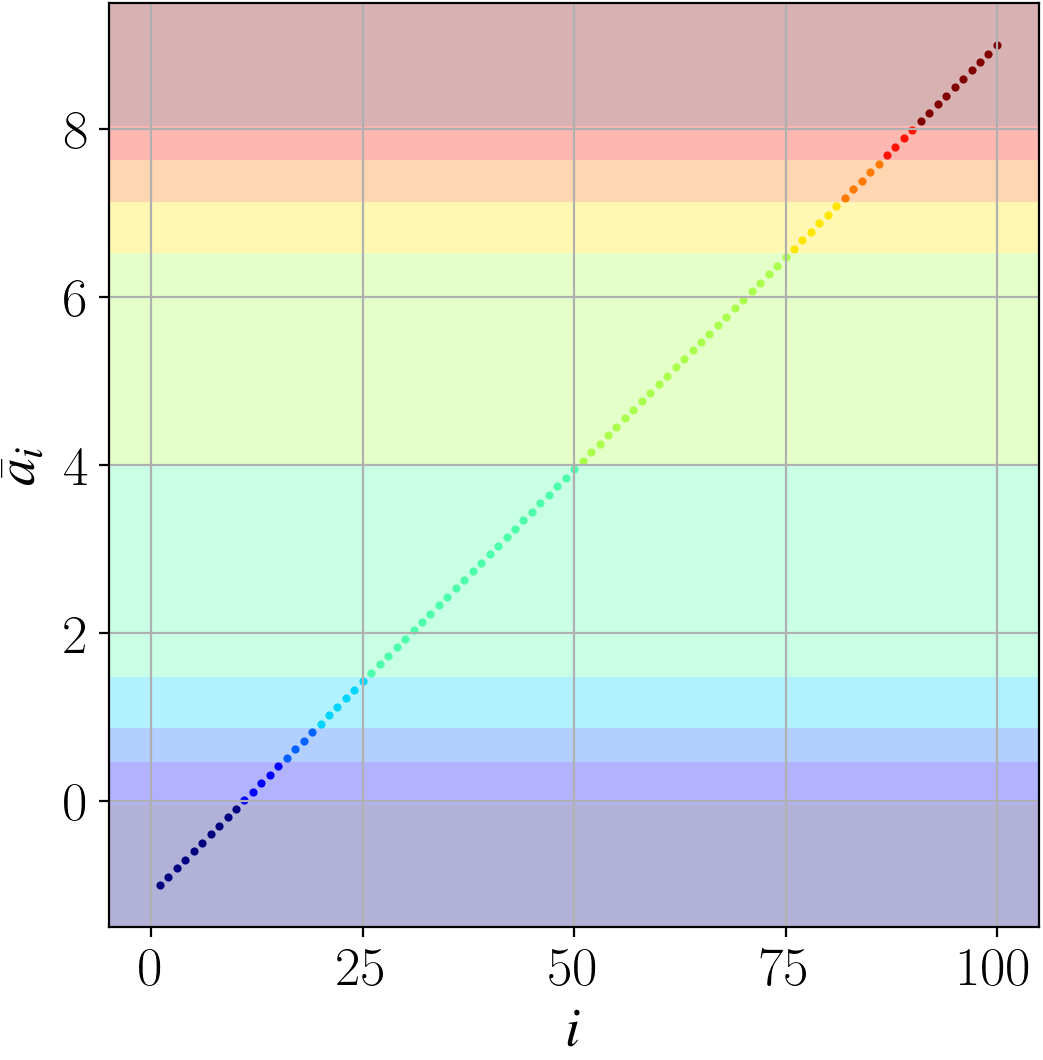}~~~~&
\includegraphics[height=3cm, bb=0 0 375 382]{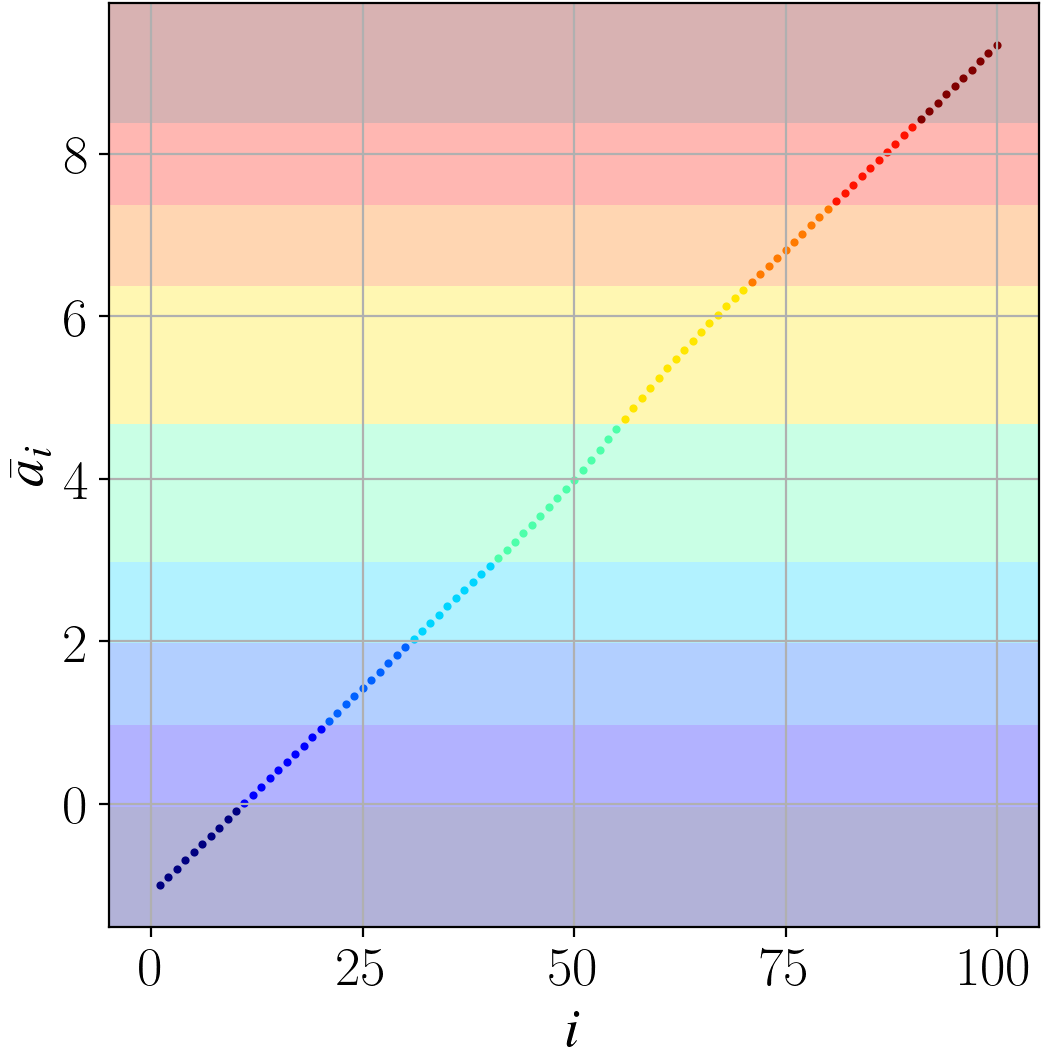}~~~~&
\includegraphics[height=3cm, bb=0 0 388 382]{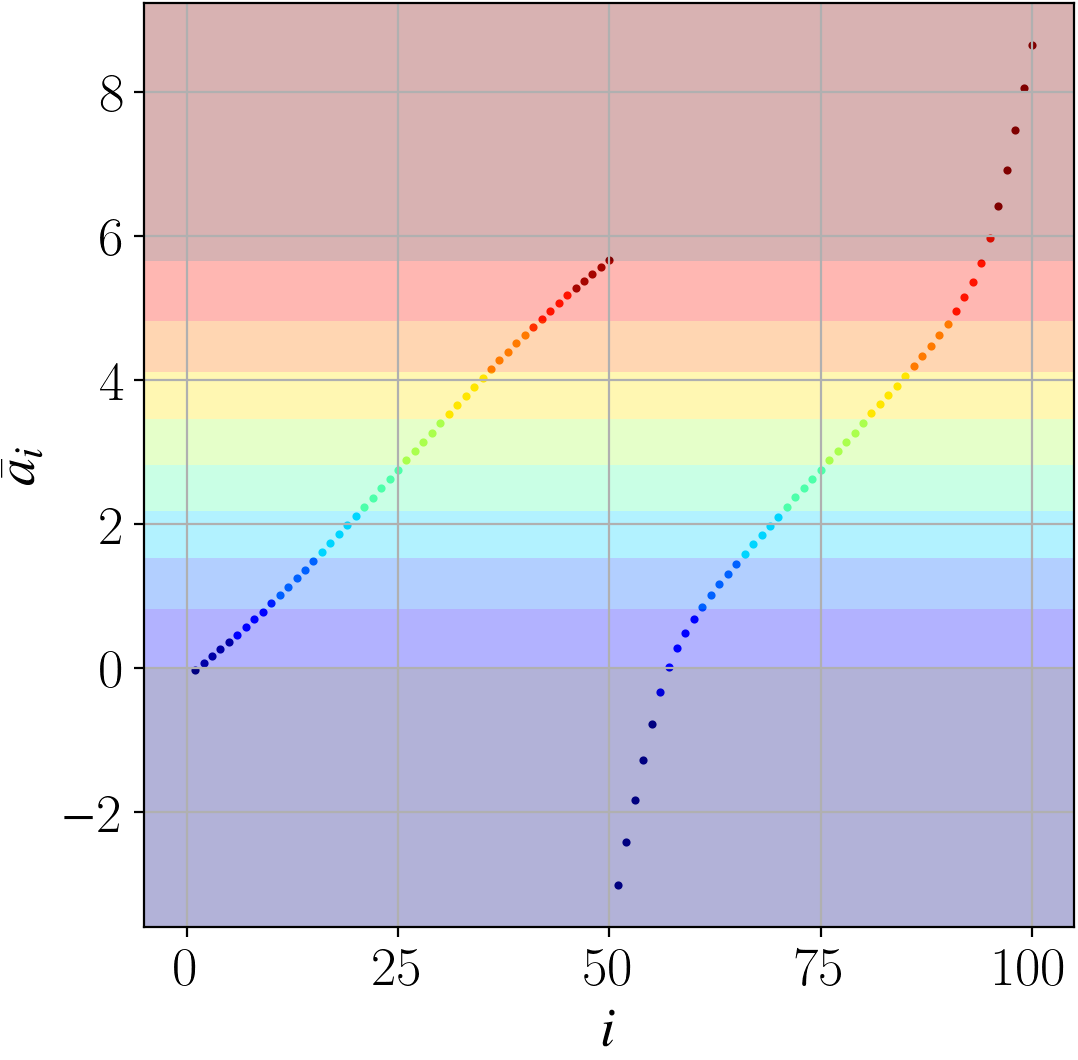}\\
~~&
{\tiny\hyt{m5} N-1, logi~~~~~~~}~~~~&
{\tiny\hyt{m6} H-1/3, hing~~~~~}~~~~&
{\tiny\hyt{m7} H-1, hing~~~~~~~}~~~~&
{\tiny\hyt{m8} H-3, hing~~~~~~~}\\
\rotatebox{90}{\tiny~~~~~~~~~~~~~~~AT-O}~~&
\includegraphics[height=3cm, bb=0 0 375 386]{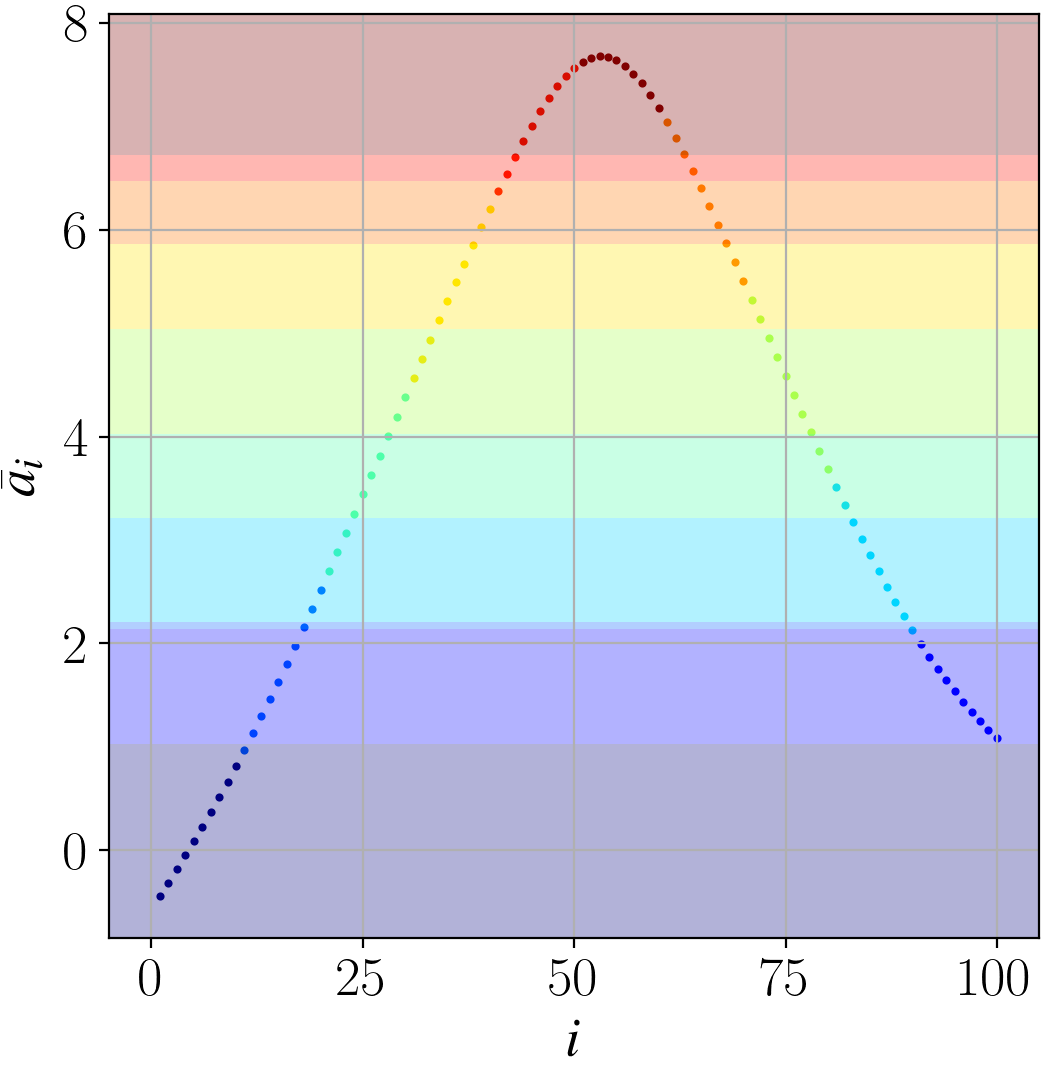}~~~~&
\includegraphics[height=3cm, bb=0 0 375.480751 381.600763]{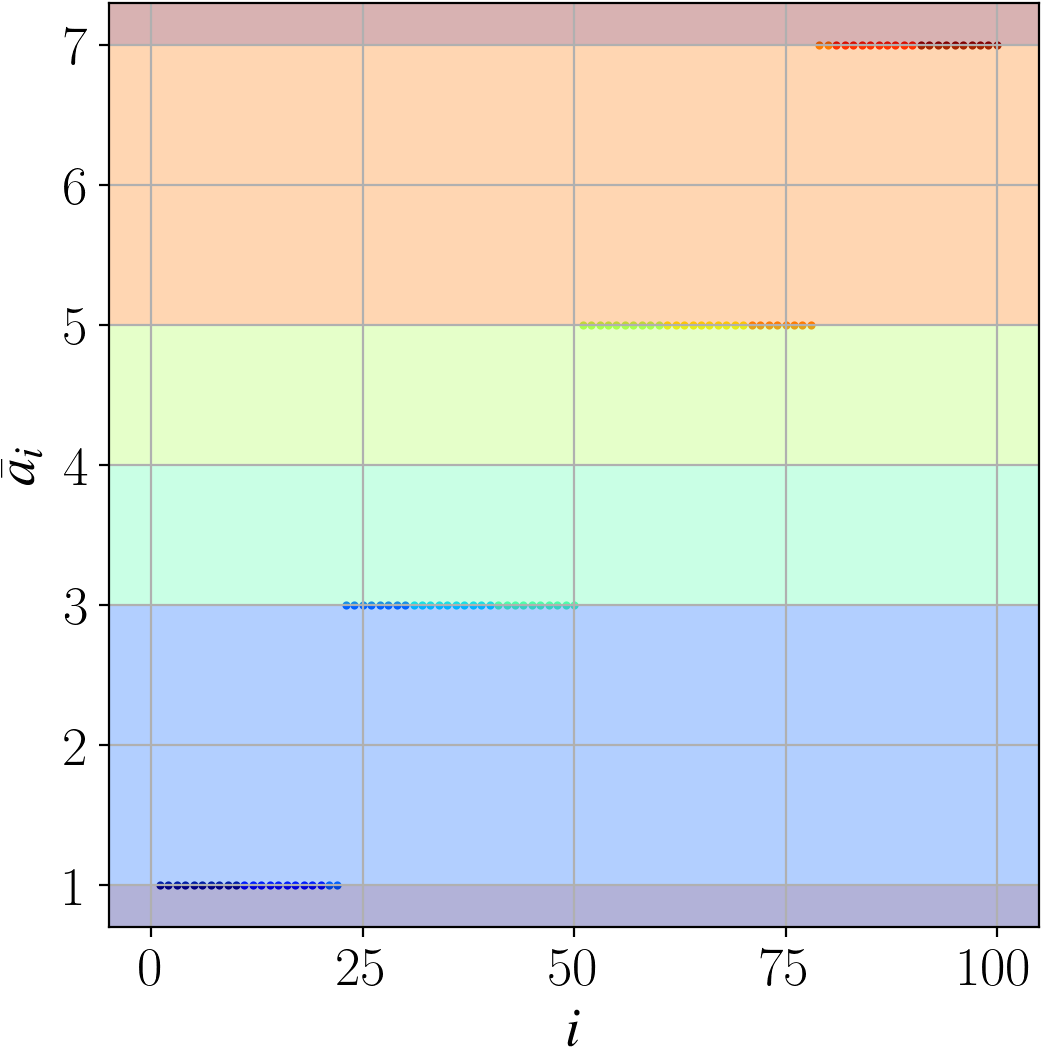}~~~~&
\includegraphics[height=3cm, bb=0 0 384.480769 381.600763]{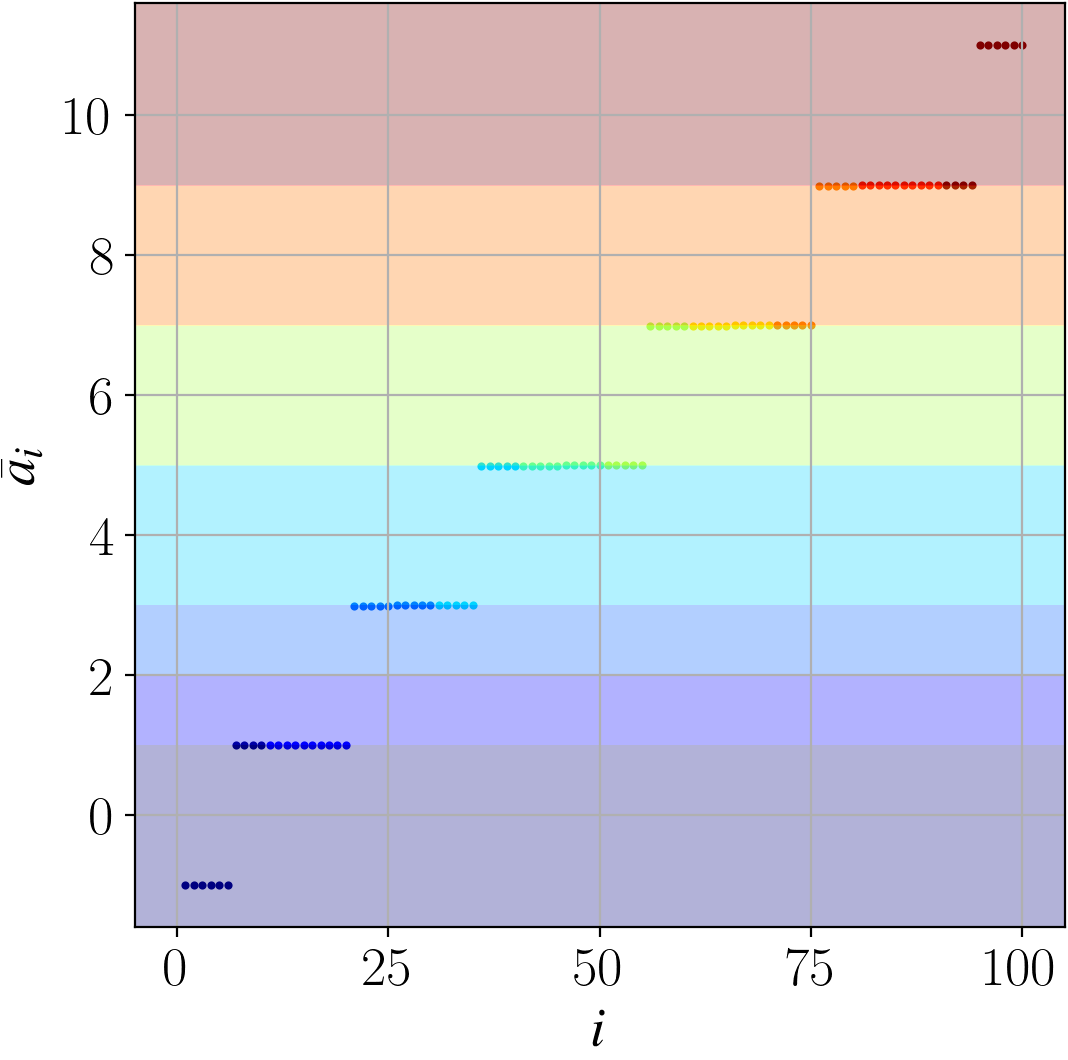}~~~~&
\includegraphics[height=3cm, bb=0 0 399.240798 381.960764]{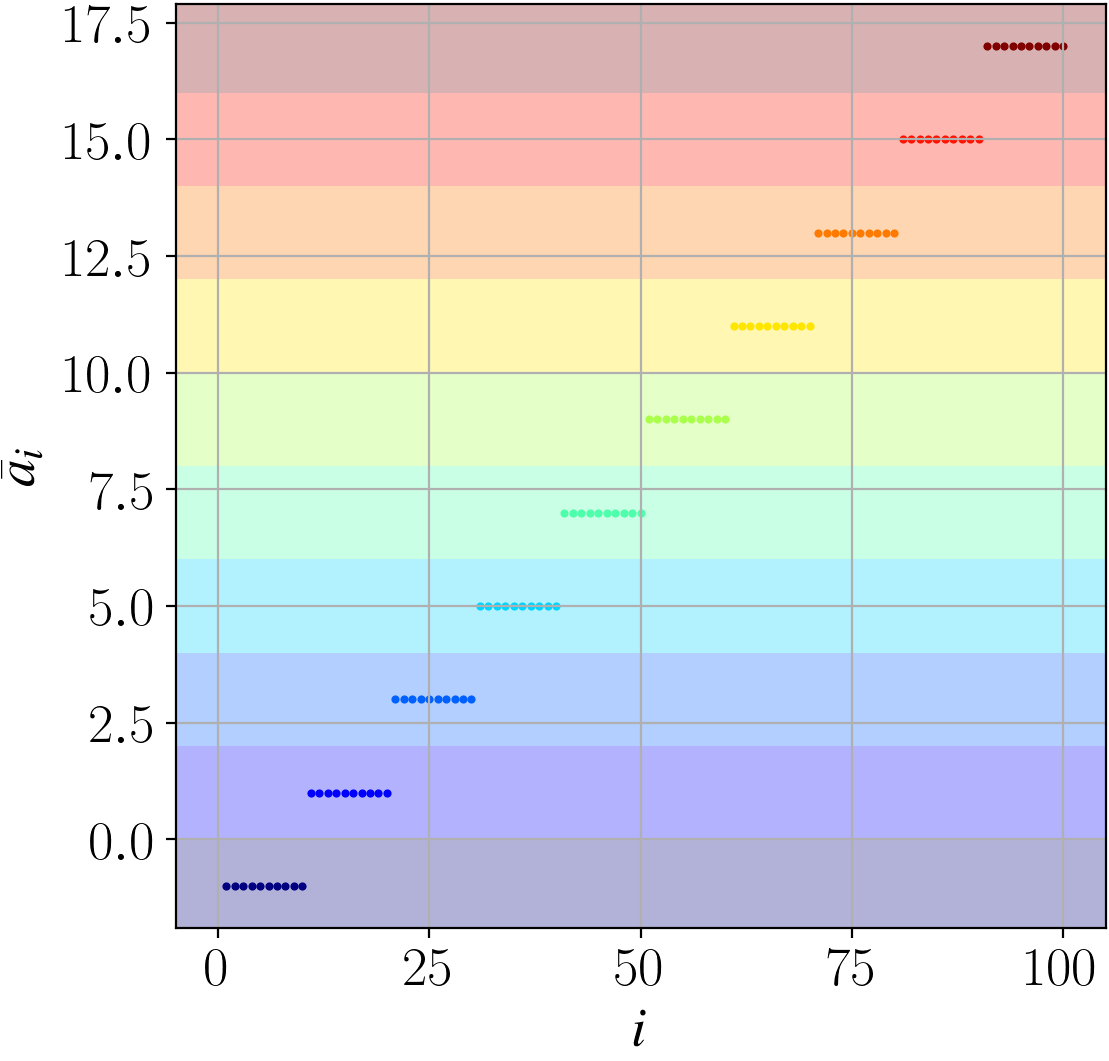}\\
\rotatebox{90}{\tiny~~~~~~~~~~~~~~~IT-N}~~&
\includegraphics[height=3cm, bb=0 0 403 382]{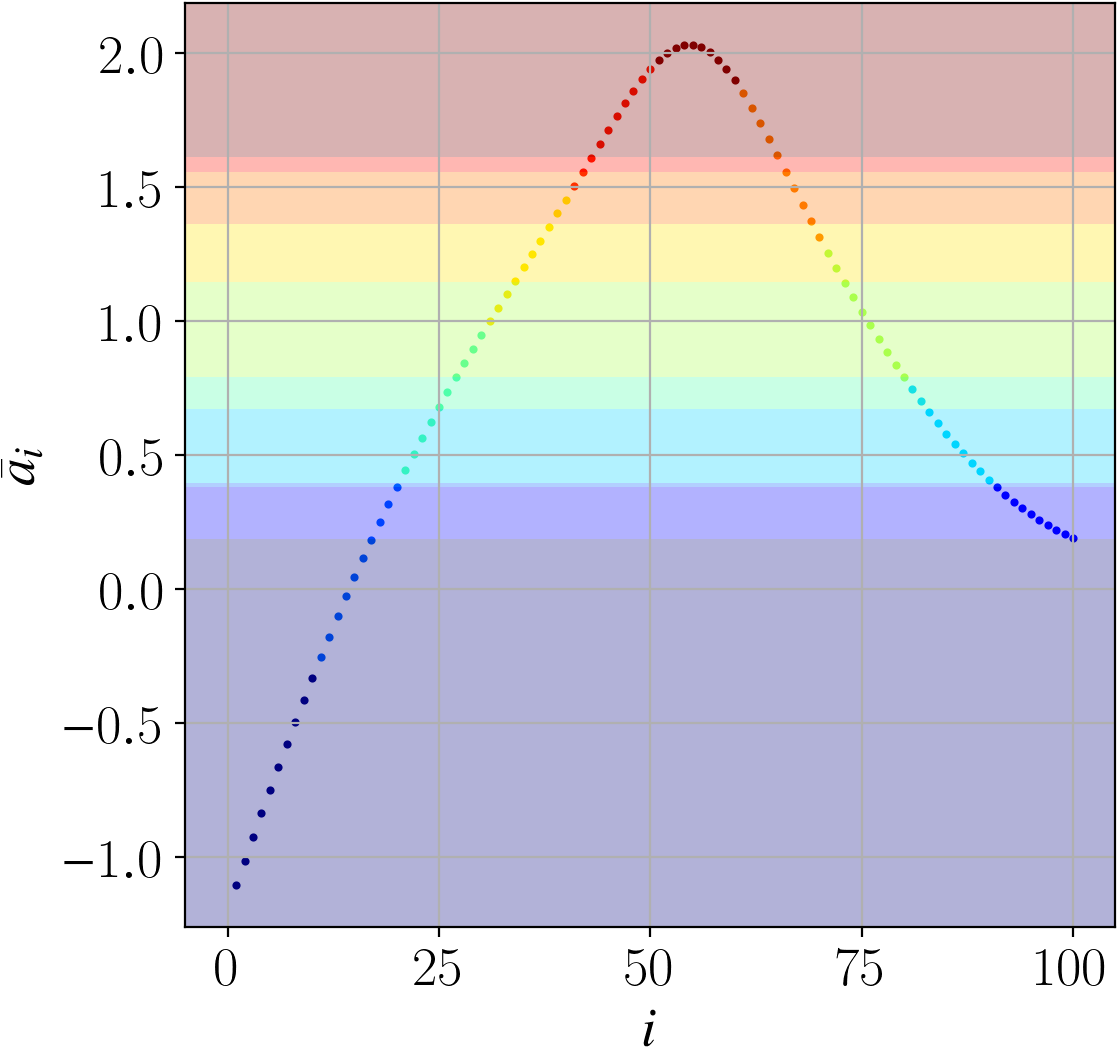}~~~~&
\includegraphics[height=3cm, bb=0 0 375 382]{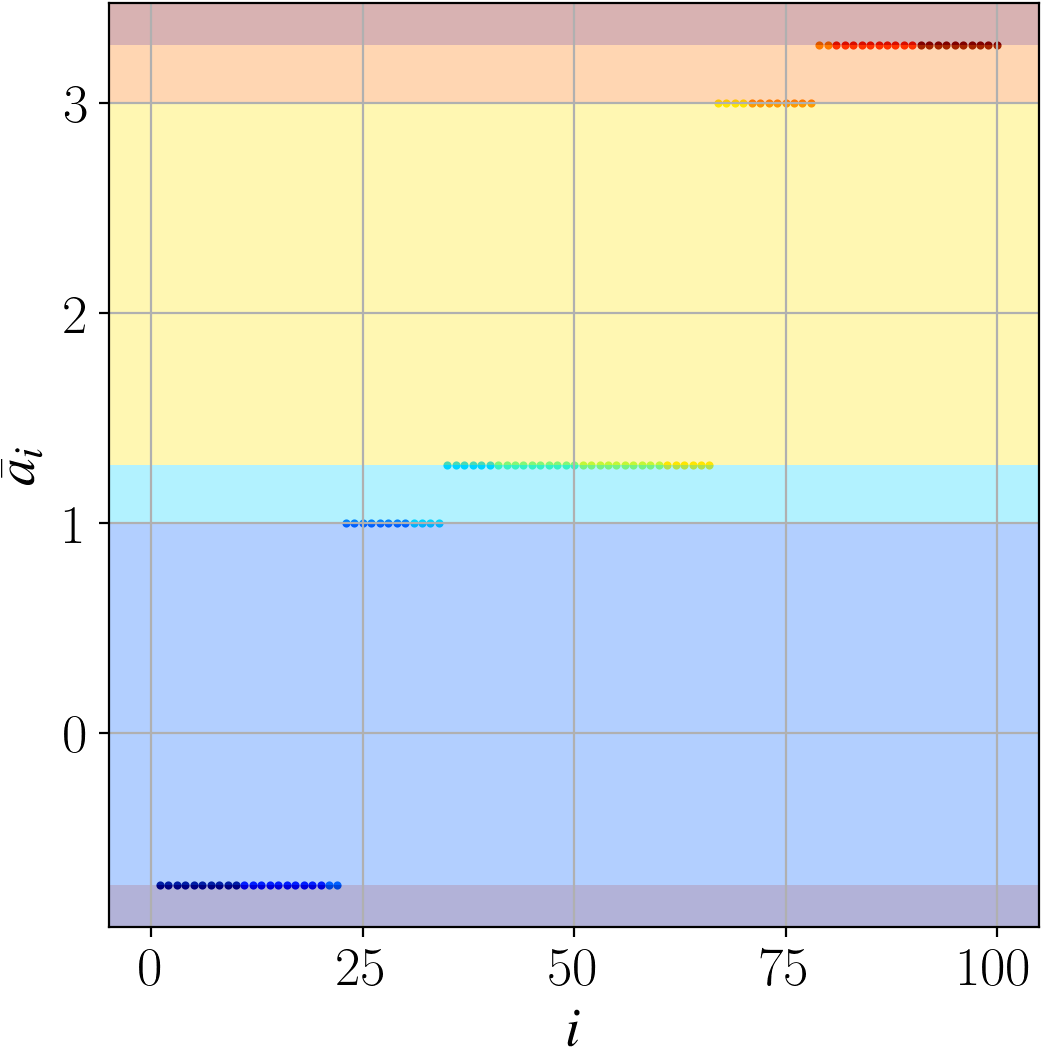}~~~~&
\includegraphics[height=3cm, bb=0 0 375 382]{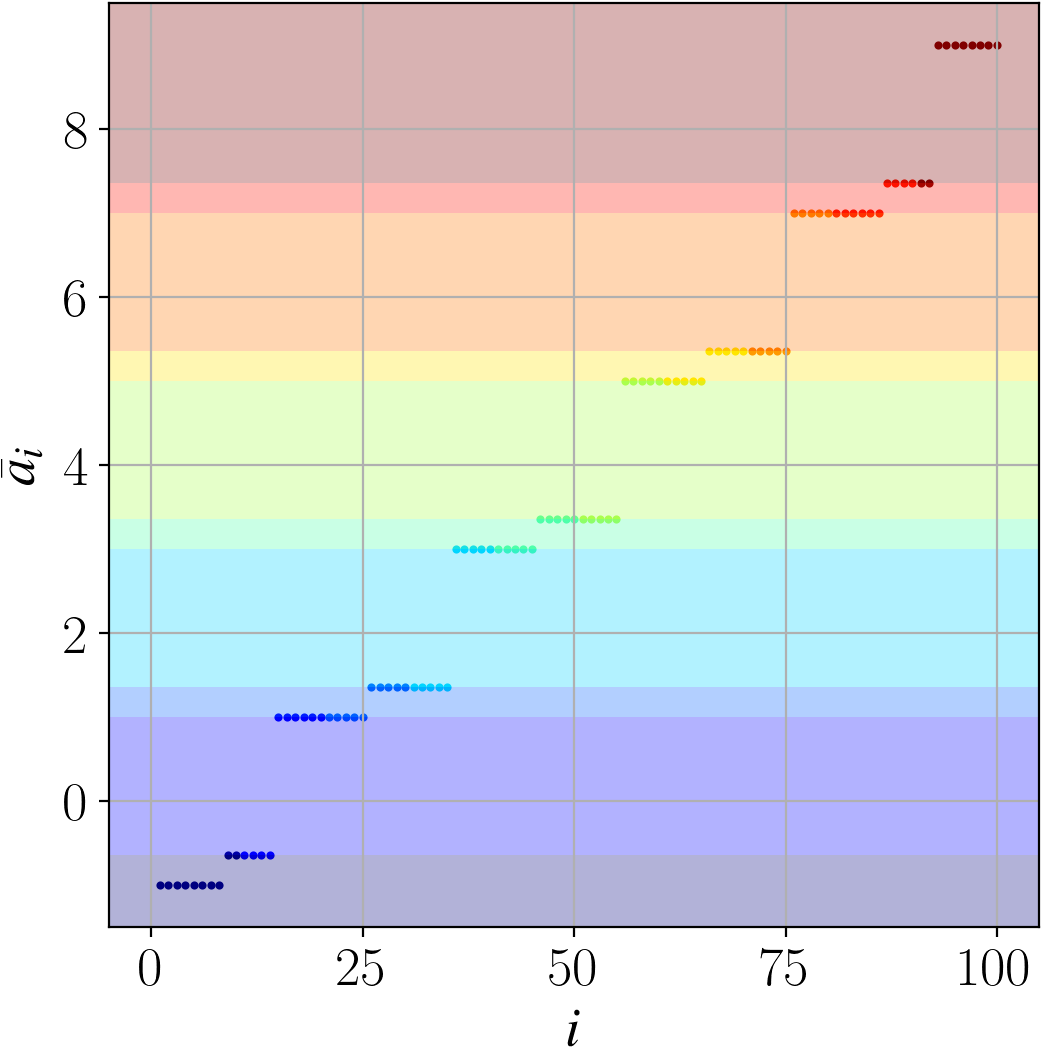}~~~~&
\includegraphics[height=3cm, bb=0 0 399 382]{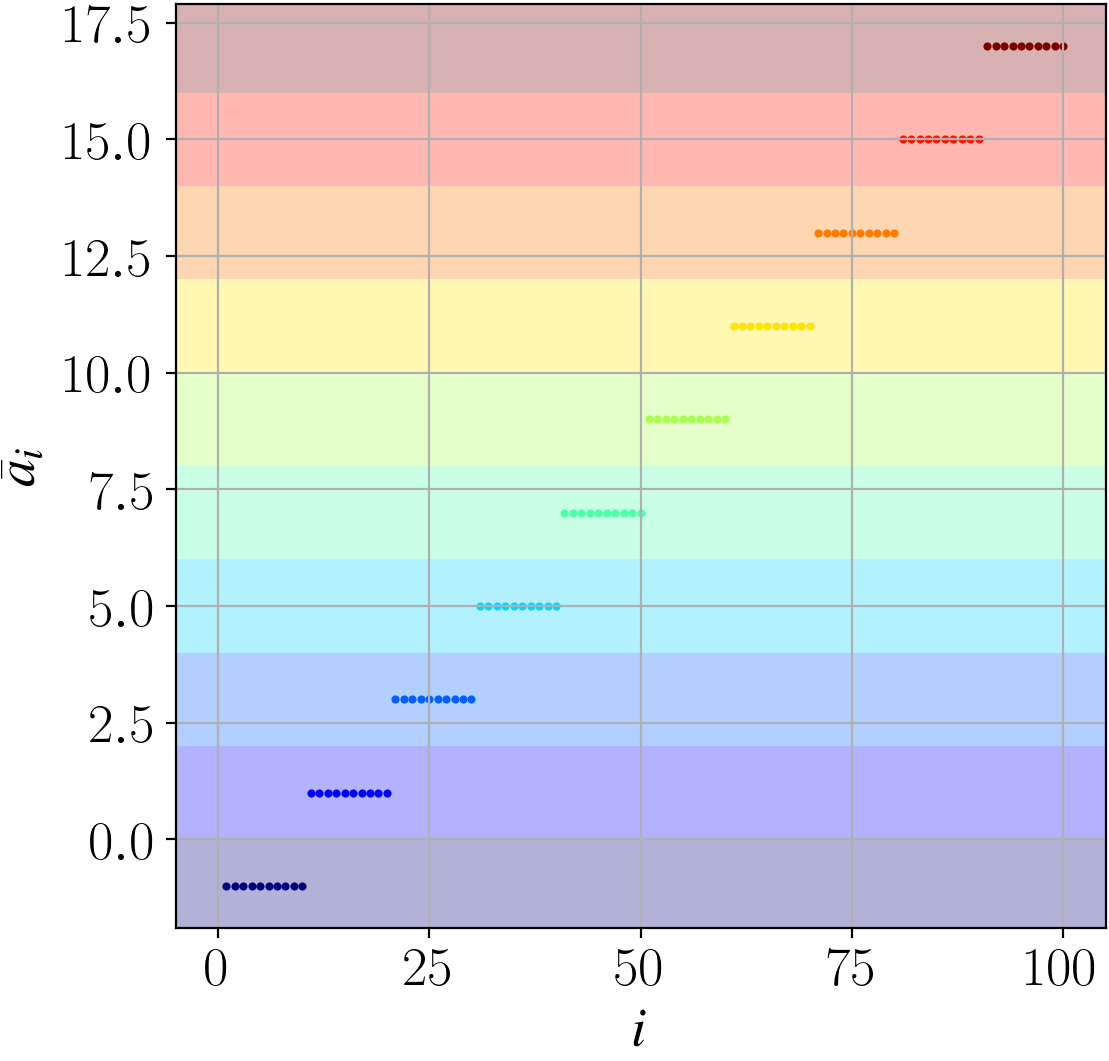}\\
\rotatebox{90}{\tiny~~~~~~~~~~~~~~~IT-O}~~&
\includegraphics[height=3cm, bb=0 0 403 382]{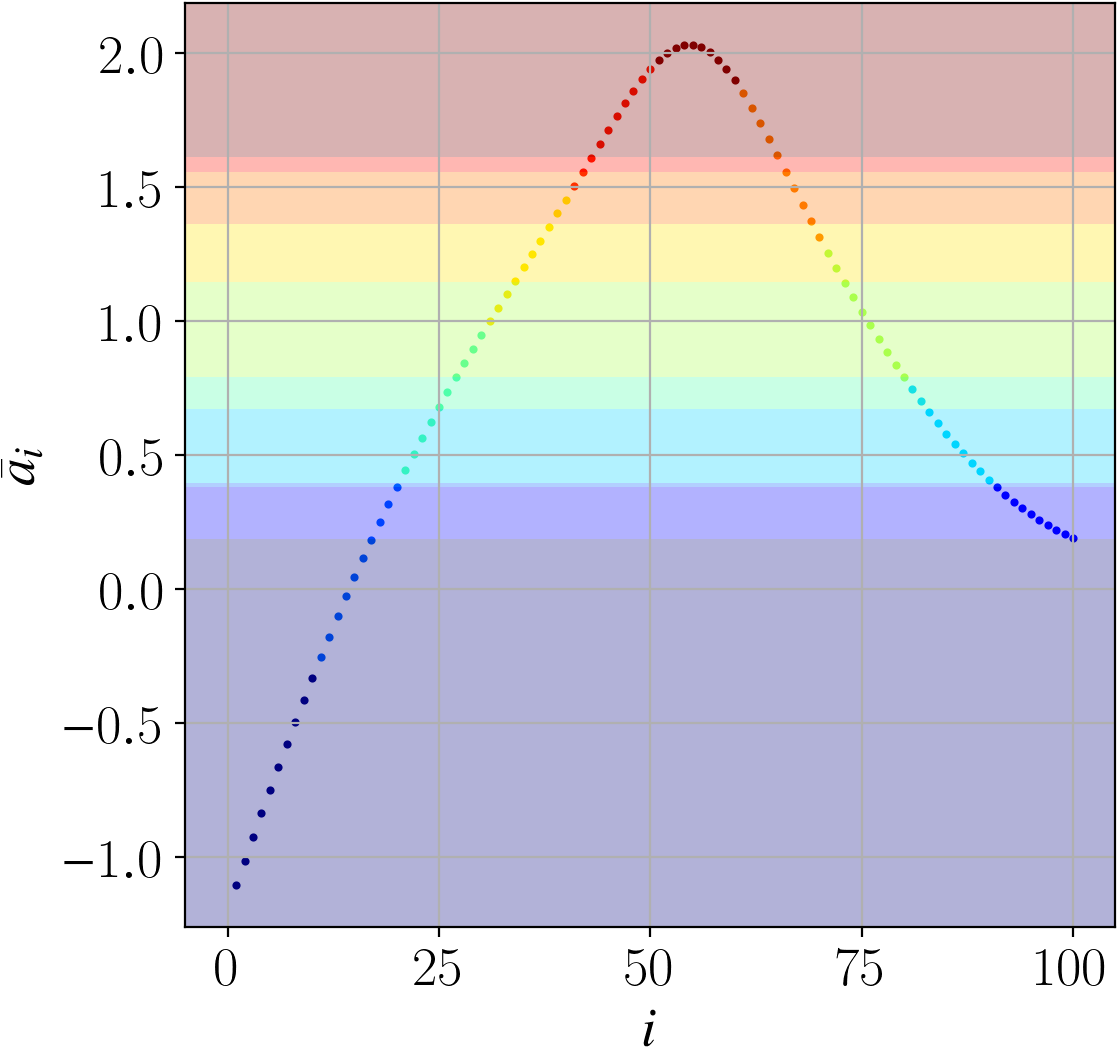}~~~~&
\includegraphics[height=3cm, bb=0 0 388 382]{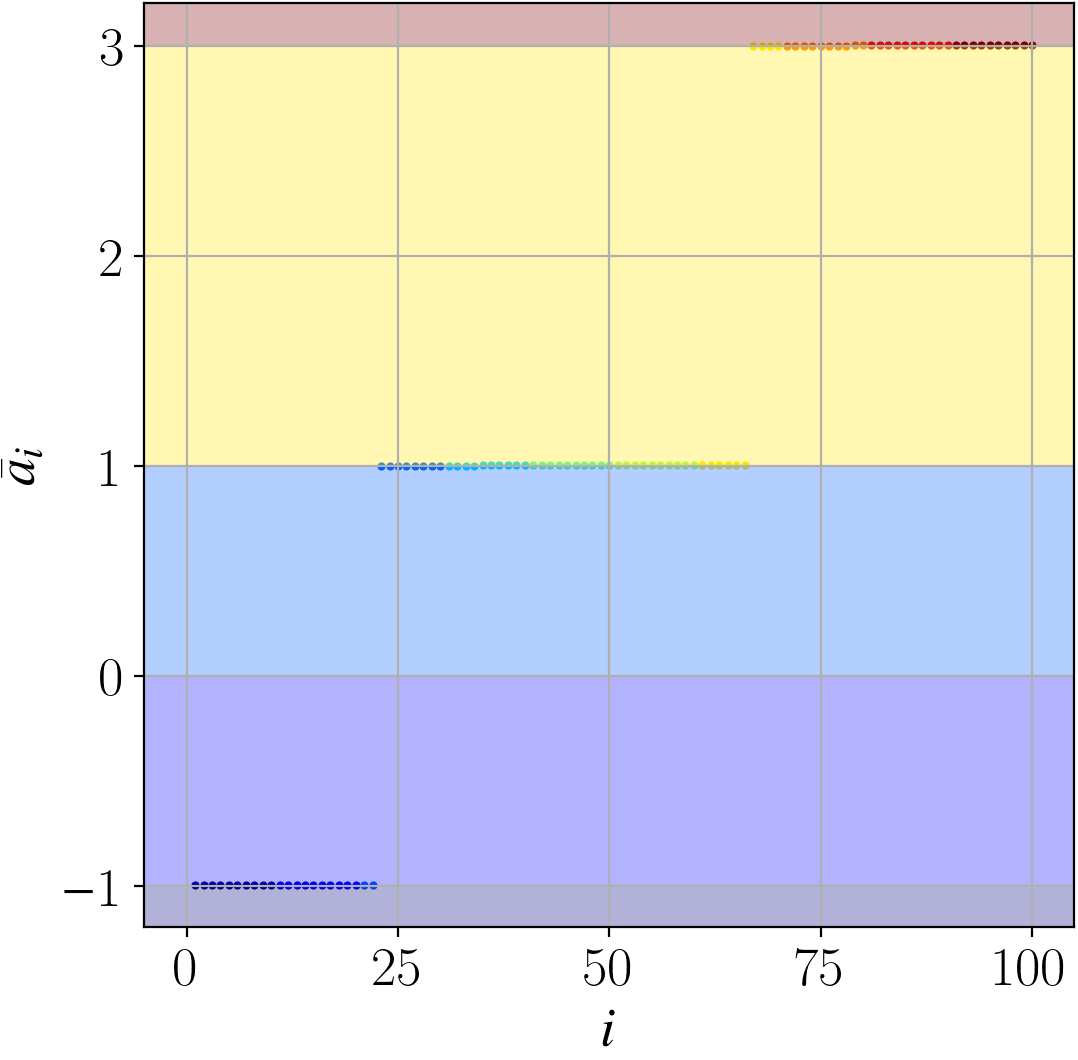}~~~~&
\includegraphics[height=3cm, bb=0 0 375 382]{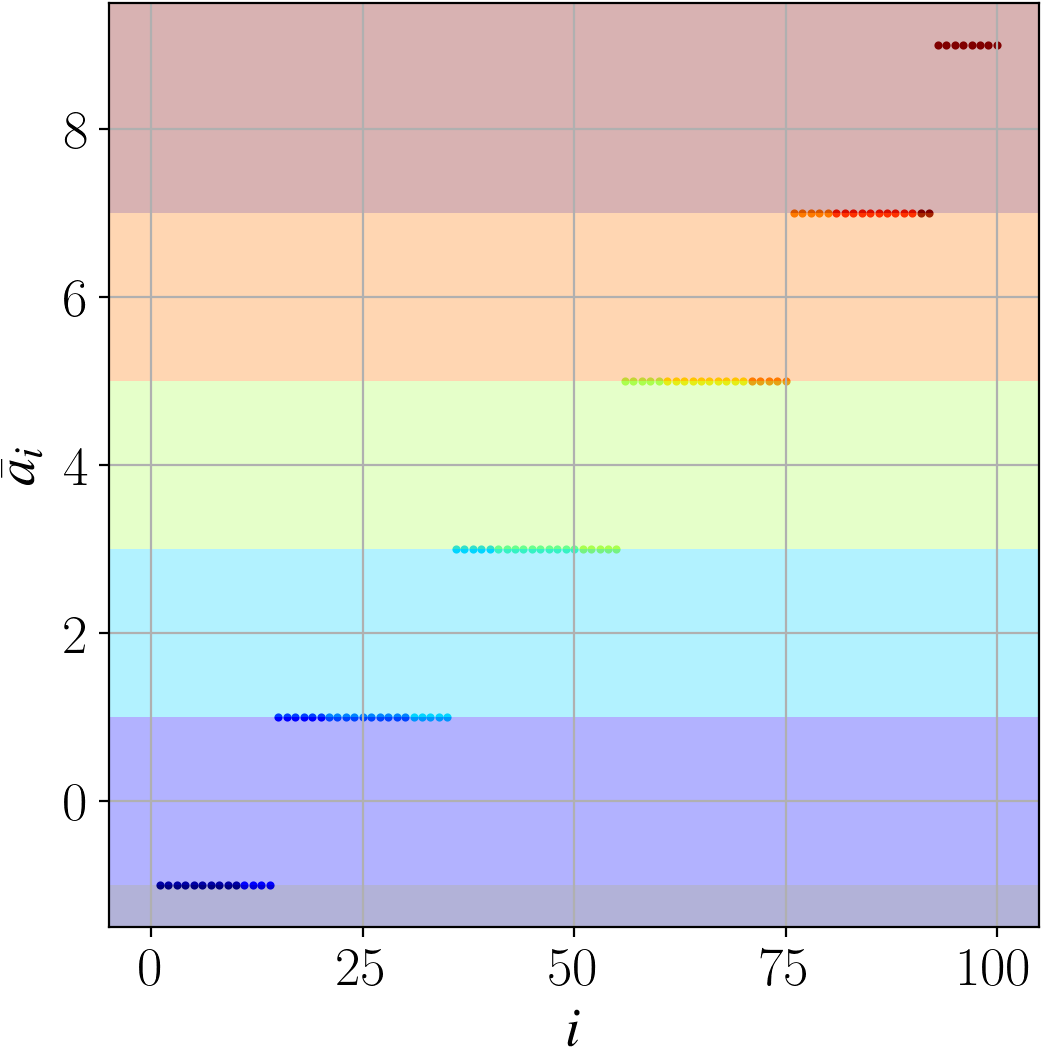}~~~~&
\includegraphics[height=3cm, bb=0 0 399 382]{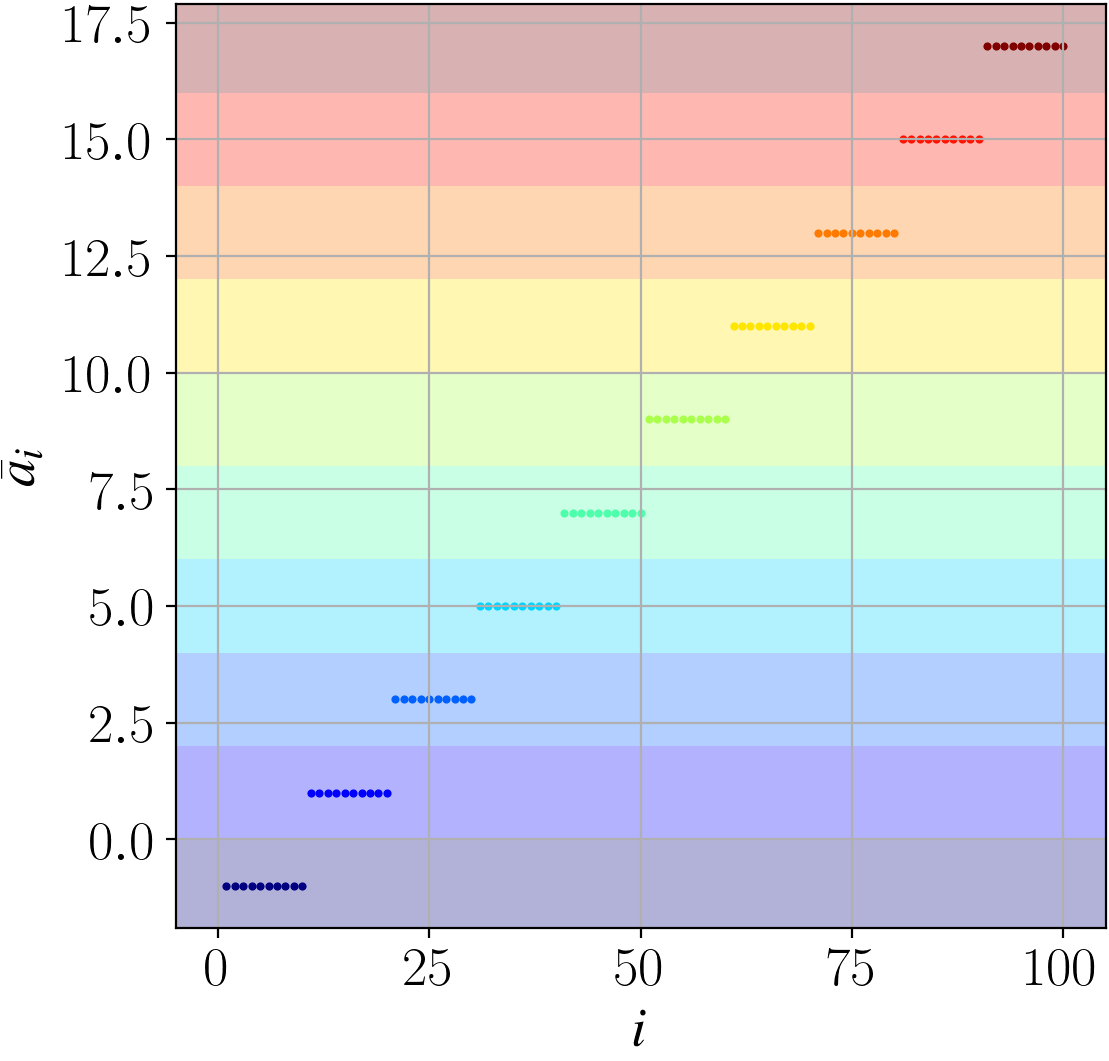}
\end{tabular}\\
\includegraphics[width=15cm, bb=0 0 2520 58]{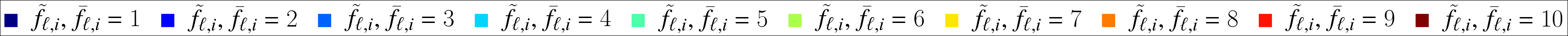}
\caption{%
Learned 1DT value $\bar{a}_i$ and optimal label prediction $\tilde{f}_{\ell,i}$ (color of point) versus $i$, 
and method's label prediction (lightened background color), which represents $\bar{f}_{\ell,i}$,
for the simulation in Section~\ref{sec:SS} with Task-Z.
For example, 3 plates of \protect\hyl{m1} show 3 results 
for H-1 with logi-AT-O, -IT-N, and -IT-O from top to bottom.}
\label{fig:SimRes}
\end{figure}

\subsection{Results for Synthesis Data Experiment in Section~\protect\ref{sec:SDE}}
\label{sec:ApeSDE}
Tables~\ref{tab:SDERes-H}--\ref{tab:SDERes-N} show 
the mean and SD of the test prediction errors (as `$\text{mean}_{\text{SD}}$') 
of MLR and 21 threshold methods for the experiment in Section~\ref{sec:SDE}.
For each combination of the data distribution and error, 
the best results (with the smallest mean of errors) among 
the results for threshold methods are shown in red, 
the result for MLR is shown in red if it is significantly better than the best results,
and results significantly worse than the best results in blue.
We see these tables graphically, 
concentrating especially on two points: 
whether the results for each threshold method are in blue
(i.e., trouble inherent to the learning procedure), 
and whether the results for MLR are colored
(i.e., comparison between threshold methods and standard classification method).

\begin{table}[H]
\renewcommand{\arraystretch}{0.75}\renewcommand{\tabcolsep}{5pt}\centering%
\caption{Results for the experiment in Section~\ref{sec:SDE} with the data $\text{H-}1/3$, $\text{H-}1$, and $\text{H-}3$.}
\label{tab:SDERes-H}\scalebox{0.75}{\begin{tabular}{cc|ccc|ccc|ccc}\toprule
&&\multicolumn{3}{c|}{\scs H-1/3}&\multicolumn{3}{c|}{\scs H-1}&\multicolumn{3}{c}{\scs H-3}\\
&&{\scs MZE}&{\scs MAE}&{\scs RMSE}&{\scs MZE}&{\scs MAE}&{\scs RMSE}&{\scs MZE}&{\scs MAE}&{\scs RMSE}\\
\midrule
\multicolumn{2}{c|}{\scs MLR}&\tcb{$.752_{.033}$}&$1.224_{.273}$&$1.574_{.284}$&$.589_{.007}$&$.708_{.013}$&$.977_{.014}$&$.365_{.009}$&$.371_{.008}$&$.621_{.008}$\\
\midrule\multirow{7}{*}{\rotatebox{90}{\tiny AT-O\hspace{1em}}}&
{\scs logi}&$.740_{.008}$&$1.166_{.011}$&$1.513_{.011}$&$.587_{.007}$&$.706_{.010}$&$.976_{.012}$&$.361_{.008}$&\tcr{$.369_{.009}$}&$.620_{.007}$\\&
{\scs hing}&$.741_{.009}$&$1.173_{.022}$&\tcb{$1.519_{.016}$}&$.589_{.007}$&$.708_{.012}$&$.980_{.014}$&\tcr{$.360_{.007}$}&$.369_{.008}$&$.620_{.007}$\\&
{\scs smhi}&$.739_{.007}$&$1.168_{.012}$&$1.511_{.012}$&$.587_{.008}$&$.705_{.011}$&$.977_{.013}$&$.361_{.008}$&$.370_{.009}$&$.620_{.007}$\\&
{\scs sqhi}&$.739_{.007}$&$1.166_{.011}$&$1.511_{.011}$&$.587_{.008}$&$.706_{.011}$&$.976_{.012}$&$.361_{.007}$&$.369_{.009}$&$.620_{.008}$\\&
{\scs expo}&$.740_{.007}$&$1.170_{.013}$&$1.514_{.014}$&$.587_{.007}$&$.706_{.009}$&$.976_{.012}$&$.362_{.008}$&$.371_{.009}$&$.621_{.007}$\\&
{\scs abso}&\tcb{$.746_{.014}$}&\tcb{$1.195_{.058}$}&\tcb{$1.545_{.069}$}&\tcb{$.620_{.049}$}&\tcb{$.795_{.133}$}&\tcb{$1.075_{.160}$}&\tcb{$.455_{.099}$}&\tcb{$.541_{.197}$}&\tcb{$.818_{.234}$}\\&
{\scs squa}&$.739_{.007}$&$1.164_{.011}$&$1.512_{.010}$&$.587_{.008}$&$.705_{.011}$&$.976_{.013}$&$.361_{.008}$&$.370_{.009}$&$.621_{.007}$\\
\midrule\multirow{7}{*}{\rotatebox{90}{\tiny IT-N\hspace{1em}}}&
{\scs logi}&$.740_{.009}$&$1.168_{.012}$&$1.513_{.012}$&$.587_{.007}$&$.704_{.010}$&$.976_{.013}$&$.362_{.008}$&$.372_{.009}$&$.623_{.008}$\\&
{\scs hing}&$.743_{.009}$&\tcb{$1.176_{.025}$}&$1.524_{.034}$&$.587_{.008}$&$.706_{.014}$&$.977_{.014}$&$.362_{.008}$&$.371_{.010}$&$.623_{.009}$\\&
{\scs smhi}&$.740_{.007}$&$1.167_{.012}$&$1.512_{.012}$&$.587_{.007}$&$.705_{.011}$&$.976_{.013}$&$.362_{.007}$&$.371_{.009}$&$.621_{.007}$\\&
{\scs sqhi}&$.740_{.007}$&$1.166_{.011}$&$1.511_{.011}$&$.586_{.008}$&$.705_{.012}$&$.976_{.013}$&$.361_{.008}$&$.370_{.009}$&$.621_{.008}$\\&
{\scs expo}&$.741_{.008}$&\tcb{$1.169_{.011}$}&$1.513_{.012}$&$.588_{.007}$&$.708_{.011}$&$.980_{.012}$&$.363_{.009}$&$.371_{.011}$&$.622_{.009}$\\&
{\scs abso}&$.742_{.009}$&\tcb{$1.179_{.022}$}&\tcb{$1.526_{.020}$}&$.590_{.008}$&\tcb{$.715_{.016}$}&$.984_{.020}$&\tcb{$.364_{.008}$}&$.373_{.009}$&\tcb{$.625_{.008}$}\\&
{\scs squa}&$.739_{.008}$&$1.165_{.011}$&$1.512_{.012}$&$.586_{.008}$&\tcr{$.704_{.009}$}&$.975_{.012}$&$.361_{.008}$&$.370_{.009}$&$.620_{.007}$\\
\midrule\multirow{7}{*}{\rotatebox{90}{\tiny IT-O\hspace{1em}}}&
{\scs logi}&\tcr{$.738_{.007}$}&$1.165_{.011}$&$1.511_{.010}$&\tcr{$.586_{.007}$}&$.705_{.010}$&\tcr{$.974_{.012}$}&$.361_{.008}$&$.370_{.009}$&$.621_{.007}$\\&
{\scs hing}&$.742_{.010}$&\tcb{$1.172_{.017}$}&$1.521_{.027}$&$.588_{.008}$&$.709_{.016}$&$.981_{.016}$&$.360_{.007}$&$.369_{.008}$&$.620_{.007}$\\&
{\scs smhi}&$.739_{.008}$&$1.165_{.011}$&$1.511_{.011}$&$.587_{.007}$&$.706_{.011}$&$.976_{.014}$&$.361_{.007}$&$.370_{.009}$&$.621_{.008}$\\&
{\scs sqhi}&$.739_{.007}$&\tcr{$1.164_{.011}$}&\tcr{$1.510_{.011}$}&$.586_{.008}$&$.705_{.011}$&$.976_{.013}$&$.360_{.007}$&$.369_{.009}$&\tcr{$.620_{.007}$}\\&
{\scs expo}&$.740_{.007}$&$1.167_{.013}$&$1.512_{.011}$&$.587_{.008}$&$.705_{.010}$&$.976_{.013}$&$.362_{.008}$&$.370_{.009}$&$.621_{.008}$\\&
{\scs abso}&$.741_{.008}$&$1.177_{.032}$&$1.528_{.040}$&$.589_{.010}$&$.710_{.017}$&$.981_{.018}$&$.364_{.008}$&$.374_{.011}$&\tcb{$.626_{.011}$}\\&
{\scs squa}&$.740_{.008}$&$1.165_{.012}$&$1.510_{.011}$&$.586_{.007}$&$.705_{.010}$&$.975_{.012}$&$.361_{.008}$&$.370_{.009}$&$.620_{.007}$\\
\bottomrule
\end{tabular}}\end{table}
\begin{table}[H]
\renewcommand{\arraystretch}{0.75}\renewcommand{\tabcolsep}{5pt}\centering%
\caption{Results for the experiment in Section~\ref{sec:SDE} with the data $\text{M-}1/3$, $\text{M-}1$, and $\text{M-}3$.}
\label{tab:SDERes-M}\scalebox{0.75}{\begin{tabular}{cc|ccc|ccc|ccc}\toprule
&&\multicolumn{3}{c|}{\scs M-1/3}&\multicolumn{3}{c|}{\scs M-1}&\multicolumn{3}{c}{\scs M-3}\\
&&{\scs MZE}&{\scs MAE}&{\scs RMSE}&{\scs MZE}&{\scs MAE}&{\scs RMSE}&{\scs MZE}&{\scs MAE}&{\scs RMSE}\\
\midrule
\multicolumn{2}{c|}{\scs MLR}&$.730_{.008}$&$1.144_{.016}$&$1.470_{.016}$&$.544_{.008}$&$.683_{.013}$&$.971_{.012}$&\tcb{$.315_{.011}$}&$.332_{.012}$&$.612_{.013}$\\
\midrule\multirow{7}{*}{\rotatebox{90}{\tiny AT-O\hspace{1em}}}&
{\scs logi}&$.729_{.007}$&$1.145_{.011}$&$1.466_{.009}$&$.543_{.009}$&$.679_{.010}$&$.971_{.014}$&$.310_{.008}$&$.332_{.011}$&$.614_{.016}$\\&
{\scs hing}&$.730_{.007}$&\tcb{$1.152_{.012}$}&\tcb{$1.474_{.013}$}&$.544_{.009}$&$.682_{.011}$&$.975_{.017}$&$.310_{.008}$&$.332_{.011}$&$.612_{.016}$\\&
{\scs smhi}&$.728_{.006}$&\tcr{$1.144_{.013}$}&$1.464_{.011}$&$.542_{.008}$&$.678_{.010}$&$.970_{.011}$&$.310_{.008}$&$.331_{.010}$&\tcr{$.612_{.014}$}\\&
{\scs sqhi}&\tcr{$.727_{.006}$}&$1.145_{.013}$&$1.465_{.010}$&$.542_{.009}$&$.680_{.009}$&$.970_{.013}$&$.310_{.008}$&$.332_{.010}$&$.612_{.015}$\\&
{\scs expo}&$.729_{.006}$&$1.146_{.014}$&$1.466_{.011}$&$.543_{.008}$&$.679_{.010}$&$.971_{.013}$&$.309_{.008}$&$.332_{.010}$&$.614_{.015}$\\&
{\scs abso}&$.730_{.008}$&$1.152_{.029}$&$1.480_{.037}$&$.551_{.020}$&\tcb{$.696_{.027}$}&\tcb{$.984_{.029}$}&\tcb{$.387_{.147}$}&$.555_{.435}$&\tcb{$.897_{.555}$}\\&
{\scs squa}&$.728_{.006}$&$1.146_{.013}$&$1.466_{.012}$&$.542_{.009}$&$.679_{.010}$&$.971_{.012}$&$.309_{.008}$&$.332_{.011}$&$.613_{.015}$\\
\midrule\multirow{7}{*}{\rotatebox{90}{\tiny IT-N\hspace{1em}}}&
{\scs logi}&$.728_{.007}$&\tcb{$1.151_{.013}$}&$1.469_{.015}$&$.542_{.009}$&$.679_{.009}$&$.970_{.012}$&$.311_{.009}$&$.333_{.011}$&$.616_{.015}$\\&
{\scs hing}&$.729_{.007}$&\tcb{$1.155_{.014}$}&\tcb{$1.479_{.023}$}&$.544_{.009}$&\tcb{$.687_{.015}$}&$.982_{.027}$&$.312_{.008}$&\tcb{$.337_{.009}$}&\tcb{$.623_{.015}$}\\&
{\scs smhi}&$.729_{.007}$&$1.146_{.014}$&$1.469_{.014}$&$.543_{.009}$&$.679_{.009}$&$.971_{.013}$&$.310_{.008}$&$.332_{.010}$&$.612_{.014}$\\&
{\scs sqhi}&$.728_{.007}$&$1.146_{.012}$&$1.467_{.012}$&$.542_{.008}$&$.680_{.010}$&$.972_{.013}$&$.311_{.009}$&$.332_{.010}$&$.613_{.015}$\\&
{\scs expo}&$.729_{.007}$&$1.150_{.015}$&$1.470_{.014}$&$.544_{.008}$&\tcb{$.686_{.009}$}&$.976_{.014}$&\tcb{$.334_{.019}$}&\tcb{$.381_{.060}$}&\tcb{$.684_{.085}$}\\&
{\scs abso}&\tcb{$.732_{.008}$}&\tcb{$1.165_{.035}$}&\tcb{$1.489_{.049}$}&$.546_{.011}$&\tcb{$.693_{.020}$}&\tcb{$.983_{.025}$}&\tcb{$.314_{.008}$}&\tcb{$.341_{.011}$}&\tcb{$.625_{.015}$}\\&
{\scs squa}&$.728_{.006}$&$1.146_{.012}$&$1.467_{.011}$&$.542_{.009}$&$.680_{.009}$&$.971_{.013}$&$.310_{.008}$&$.332_{.011}$&$.613_{.015}$\\
\midrule\multirow{7}{*}{\rotatebox{90}{\tiny IT-O\hspace{1em}}}&
{\scs logi}&$.728_{.006}$&$1.146_{.012}$&\tcr{$1.464_{.010}$}&$.542_{.009}$&$.678_{.009}$&\tcr{$.969_{.012}$}&$.309_{.008}$&$.332_{.011}$&$.613_{.015}$\\&
{\scs hing}&$.730_{.007}$&\tcb{$1.160_{.021}$}&\tcb{$1.483_{.029}$}&\tcb{$.546_{.009}$}&\tcb{$.697_{.025}$}&\tcb{$.993_{.038}$}&$.310_{.008}$&$.332_{.011}$&$.614_{.016}$\\&
{\scs smhi}&$.727_{.006}$&$1.146_{.014}$&$1.466_{.008}$&$.542_{.009}$&$.679_{.009}$&$.970_{.013}$&$.310_{.008}$&$.332_{.010}$&$.613_{.014}$\\&
{\scs sqhi}&$.728_{.007}$&$1.146_{.013}$&$1.465_{.009}$&$.542_{.009}$&$.679_{.009}$&$.971_{.013}$&\tcr{$.309_{.008}$}&\tcr{$.331_{.010}$}&$.613_{.015}$\\&
{\scs expo}&$.728_{.007}$&$1.146_{.013}$&$1.465_{.015}$&$.543_{.008}$&\tcr{$.678_{.009}$}&$.970_{.012}$&$.310_{.009}$&$.333_{.011}$&$.614_{.015}$\\&
{\scs abso}&$.731_{.009}$&\tcb{$1.156_{.015}$}&\tcb{$1.484_{.019}$}&$.546_{.012}$&\tcb{$.696_{.029}$}&\tcb{$.985_{.039}$}&\tcb{$.317_{.011}$}&\tcb{$.346_{.017}$}&\tcb{$.635_{.026}$}\\&
{\scs squa}&$.728_{.007}$&$1.147_{.014}$&$1.465_{.013}$&\tcr{$.542_{.009}$}&$.679_{.009}$&$.971_{.013}$&$.310_{.008}$&$.332_{.010}$&$.614_{.014}$\\
\bottomrule
\end{tabular}}\end{table}
\begin{table}[H]
\renewcommand{\arraystretch}{0.75}\renewcommand{\tabcolsep}{5pt}\centering%
\caption{Results for the experiment in Section~\ref{sec:SDE} with the data $\text{A-}1/3$, $\text{A-}1$, and $\text{A-}3$.}
\label{tab:SDERes-A}\scalebox{0.75}{\begin{tabular}{cc|ccc|ccc|ccc}\toprule
&&\multicolumn{3}{c|}{\scs A-1/3}&\multicolumn{3}{c|}{\scs A-1}&\multicolumn{3}{c}{\scs A-3}\\
&&{\scs MZE}&{\scs MAE}&{\scs RMSE}&{\scs MZE}&{\scs MAE}&{\scs RMSE}&{\scs MZE}&{\scs MAE}&{\scs RMSE}\\
\midrule
\multicolumn{2}{c|}{\scs MLR}&$.738_{.010}$&$1.162_{.018}$&$1.511_{.019}$&$.572_{.009}$&$.702_{.012}$&$.976_{.013}$&$.321_{.008}$&$.347_{.011}$&$.619_{.012}$\\
\midrule\multirow{7}{*}{\rotatebox{90}{\tiny AT-O\hspace{1em}}}&
{\scs logi}&$.736_{.008}$&$1.164_{.012}$&$1.513_{.012}$&$.570_{.009}$&$.701_{.011}$&$.976_{.011}$&$.320_{.006}$&$.347_{.009}$&$.617_{.008}$\\&
{\scs hing}&$.738_{.008}$&$1.170_{.018}$&\tcb{$1.520_{.015}$}&$.571_{.009}$&$.704_{.012}$&$.978_{.012}$&$.320_{.005}$&$.348_{.008}$&$.617_{.006}$\\&
{\scs smhi}&$.736_{.007}$&$1.163_{.012}$&$1.512_{.014}$&$.570_{.008}$&$.701_{.009}$&$.975_{.011}$&$.321_{.007}$&$.347_{.008}$&$.617_{.007}$\\&
{\scs sqhi}&$.736_{.007}$&\tcr{$1.162_{.012}$}&$1.512_{.013}$&$.570_{.008}$&\tcr{$.699_{.010}$}&$.975_{.011}$&$.320_{.006}$&$.347_{.008}$&\tcr{$.616_{.007}$}\\&
{\scs expo}&$.737_{.007}$&$1.165_{.012}$&$1.515_{.014}$&$.570_{.008}$&$.700_{.010}$&$.975_{.011}$&$.322_{.006}$&$.348_{.009}$&$.618_{.008}$\\&
{\scs abso}&\tcb{$.746_{.015}$}&\tcb{$1.203_{.071}$}&\tcb{$1.559_{.087}$}&\tcb{$.625_{.048}$}&\tcb{$.824_{.115}$}&\tcb{$1.114_{.142}$}&\tcb{$.500_{.089}$}&\tcb{$.679_{.188}$}&\tcb{$.985_{.200}$}\\&
{\scs squa}&$.736_{.007}$&$1.162_{.013}$&$1.514_{.012}$&$.570_{.008}$&$.699_{.009}$&$.975_{.011}$&$.320_{.007}$&$.348_{.009}$&$.618_{.007}$\\
\midrule\multirow{7}{*}{\rotatebox{90}{\tiny IT-N\hspace{1em}}}&
{\scs logi}&$.737_{.008}$&$1.163_{.012}$&$1.513_{.014}$&$.570_{.009}$&$.700_{.009}$&\tcr{$.975_{.011}$}&$.321_{.006}$&$.349_{.008}$&$.618_{.007}$\\&
{\scs hing}&\tcb{$.742_{.011}$}&\tcb{$1.183_{.042}$}&\tcb{$1.542_{.050}$}&$.571_{.008}$&$.704_{.013}$&$.979_{.014}$&$.322_{.007}$&$.349_{.008}$&$.619_{.009}$\\&
{\scs smhi}&$.737_{.007}$&$1.165_{.013}$&$1.514_{.016}$&$.570_{.008}$&$.701_{.010}$&$.975_{.011}$&$.321_{.006}$&$.347_{.008}$&$.617_{.007}$\\&
{\scs sqhi}&$.737_{.007}$&$1.163_{.012}$&$1.514_{.016}$&$.570_{.008}$&$.700_{.009}$&$.976_{.012}$&\tcr{$.319_{.006}$}&\tcr{$.346_{.008}$}&$.616_{.007}$\\&
{\scs expo}&$.737_{.007}$&$1.165_{.012}$&$1.513_{.013}$&$.572_{.008}$&$.702_{.010}$&$.978_{.012}$&\tcb{$.348_{.027}$}&\tcb{$.405_{.060}$}&\tcb{$.668_{.057}$}\\&
{\scs abso}&\tcb{$.741_{.008}$}&$1.179_{.042}$&\tcb{$1.535_{.051}$}&\tcb{$.574_{.010}$}&\tcb{$.712_{.016}$}&$.980_{.014}$&\tcb{$.333_{.021}$}&\tcb{$.367_{.044}$}&\tcb{$.638_{.043}$}\\&
{\scs squa}&$.736_{.007}$&$1.163_{.012}$&$1.512_{.013}$&\tcr{$.569_{.008}$}&$.700_{.010}$&$.976_{.011}$&$.320_{.006}$&$.347_{.008}$&$.617_{.007}$\\
\midrule\multirow{7}{*}{\rotatebox{90}{\tiny IT-O\hspace{1em}}}&
{\scs logi}&$.736_{.007}$&$1.164_{.012}$&$1.512_{.013}$&$.570_{.008}$&$.700_{.009}$&$.975_{.012}$&$.320_{.006}$&$.348_{.008}$&$.617_{.007}$\\&
{\scs hing}&$.740_{.008}$&\tcb{$1.192_{.049}$}&\tcb{$1.546_{.071}$}&$.573_{.008}$&\tcb{$.706_{.010}$}&$.980_{.014}$&$.320_{.006}$&$.348_{.008}$&$.617_{.007}$\\&
{\scs smhi}&$.736_{.007}$&$1.165_{.013}$&$1.514_{.014}$&$.570_{.008}$&$.701_{.010}$&$.975_{.010}$&$.320_{.006}$&$.347_{.008}$&$.617_{.007}$\\&
{\scs sqhi}&$.736_{.007}$&$1.165_{.016}$&$1.513_{.013}$&$.570_{.007}$&$.700_{.010}$&$.975_{.011}$&$.320_{.005}$&$.347_{.008}$&$.617_{.007}$\\&
{\scs expo}&$.737_{.008}$&$1.164_{.013}$&$1.512_{.013}$&$.570_{.008}$&$.699_{.010}$&$.975_{.012}$&$.322_{.007}$&$.350_{.009}$&$.619_{.009}$\\&
{\scs abso}&$.737_{.007}$&$1.170_{.030}$&$1.524_{.035}$&\tcb{$.579_{.013}$}&\tcb{$.721_{.026}$}&\tcb{$.998_{.029}$}&\tcb{$.327_{.015}$}&\tcb{$.355_{.018}$}&$.626_{.019}$\\&
{\scs squa}&\tcr{$.735_{.007}$}&$1.163_{.014}$&\tcr{$1.512_{.013}$}&$.569_{.008}$&$.700_{.010}$&$.976_{.011}$&$.321_{.005}$&$.348_{.008}$&$.617_{.008}$\\
\bottomrule
\end{tabular}}\end{table}
\begin{table}[H]
\renewcommand{\arraystretch}{0.75}\renewcommand{\tabcolsep}{5pt}\centering%
\caption{Results for the experiment in Section~\ref{sec:SDE} with the data $\text{O-}(1/3,1)$, $\text{O-}(1/3,3)$, and $\text{O-}(1/3,3)$.}
\label{tab:SDERes-O}\scalebox{0.75}{\begin{tabular}{cc|ccc|ccc|ccc}\toprule
&&\multicolumn{3}{c|}{\scs O-(1/3,1)}&\multicolumn{3}{c|}{\scs O-(1/3,3)}&\multicolumn{3}{c}{\scs O-(1,3)}\\
&&{\scs MZE}&{\scs MAE}&{\scs RMSE}&{\scs MZE}&{\scs MAE}&{\scs RMSE}&{\scs MZE}&{\scs MAE}&{\scs RMSE}\\
\midrule
\multicolumn{2}{c|}{\scs MLR}&\tcb{$.678_{.033}$}&$.989_{.194}$&$1.335_{.230}$&$.575_{.075}$&$.858_{.374}$&$1.251_{.377}$&$.495_{.059}$&$.600_{.236}$&$.903_{.298}$\\
\midrule\multirow{7}{*}{\rotatebox{90}{\tiny AT-O\hspace{1em}}}&
{\scs logi}&$.668_{.008}$&$.940_{.014}$&\tcr{$1.276_{.012}$}&$.561_{.008}$&$.776_{.016}$&$1.165_{.016}$&\tcb{$.483_{.010}$}&\tcb{$.551_{.016}$}&$.844_{.065}$\\&
{\scs hing}&$.668_{.009}$&$.942_{.016}$&$1.284_{.016}$&$.561_{.013}$&$.779_{.019}$&$1.173_{.037}$&\tcb{$.481_{.009}$}&$.547_{.019}$&$.841_{.067}$\\&
{\scs smhi}&$.670_{.010}$&$.941_{.014}$&$1.278_{.014}$&$.560_{.010}$&$.773_{.013}$&$1.165_{.012}$&$.480_{.009}$&$.543_{.011}$&$.827_{.020}$\\&
{\scs sqhi}&\tcb{$.669_{.007}$}&$.939_{.013}$&$1.278_{.015}$&$.561_{.010}$&$.775_{.012}$&$1.165_{.013}$&$.479_{.011}$&$.542_{.007}$&$.825_{.006}$\\&
{\scs expo}&$.670_{.008}$&$.940_{.015}$&$1.278_{.013}$&\tcb{$.562_{.012}$}&$.776_{.013}$&$1.164_{.013}$&\tcb{$.483_{.011}$}&$.546_{.009}$&$.828_{.013}$\\&
{\scs abso}&\tcb{$.721_{.044}$}&\tcb{$1.126_{.148}$}&\tcb{$1.479_{.156}$}&\tcb{$.670_{.071}$}&\tcb{$1.056_{.158}$}&\tcb{$1.434_{.138}$}&\tcb{$.615_{.057}$}&\tcb{$.848_{.154}$}&\tcb{$1.149_{.158}$}\\&
{\scs squa}&$.668_{.008}$&$.938_{.012}$&$1.278_{.013}$&$.557_{.008}$&$.777_{.018}$&$1.165_{.018}$&\tcr{$.476_{.009}$}&$.542_{.009}$&$.826_{.010}$\\
\midrule\multirow{7}{*}{\rotatebox{90}{\tiny IT-N\hspace{1em}}}&
{\scs logi}&$.667_{.008}$&$.941_{.018}$&$1.279_{.015}$&$.558_{.008}$&$.774_{.012}$&$1.166_{.013}$&$.480_{.011}$&$.542_{.010}$&$.825_{.012}$\\&
{\scs hing}&$.667_{.010}$&\tcb{$.946_{.013}$}&$1.282_{.013}$&$.559_{.008}$&\tcb{$.781_{.016}$}&$1.169_{.020}$&$.480_{.011}$&\tcb{$.547_{.010}$}&$.829_{.011}$\\&
{\scs smhi}&$.667_{.010}$&$.944_{.020}$&$1.286_{.039}$&\tcr{$.556_{.008}$}&$.777_{.016}$&$1.170_{.022}$&$.478_{.010}$&$.543_{.009}$&$.825_{.011}$\\&
{\scs sqhi}&$.667_{.009}$&\tcr{$.937_{.011}$}&$1.279_{.013}$&$.558_{.010}$&$.777_{.018}$&$1.170_{.019}$&$.479_{.010}$&$.544_{.009}$&$.825_{.012}$\\&
{\scs expo}&$.667_{.008}$&\tcb{$.946_{.015}$}&\tcb{$1.287_{.023}$}&$.557_{.011}$&$.776_{.013}$&\tcb{$1.168_{.011}$}&$.483_{.019}$&\tcb{$.557_{.049}$}&$.837_{.053}$\\&
{\scs abso}&\tcb{$.673_{.010}$}&\tcb{$.960_{.031}$}&\tcb{$1.301_{.042}$}&\tcb{$.566_{.010}$}&\tcb{$.796_{.021}$}&\tcb{$1.181_{.021}$}&\tcb{$.485_{.014}$}&$.554_{.025}$&$.837_{.028}$\\&
{\scs squa}&\tcr{$.665_{.007}$}&$.940_{.012}$&$1.279_{.013}$&$.558_{.009}$&\tcr{$.772_{.012}$}&\tcr{$1.162_{.012}$}&$.480_{.008}$&$.544_{.009}$&$.827_{.009}$\\
\midrule\multirow{7}{*}{\rotatebox{90}{\tiny IT-O\hspace{1em}}}&
{\scs logi}&$.669_{.009}$&$.947_{.020}$&$1.288_{.031}$&$.558_{.009}$&\tcb{$.777_{.012}$}&$1.165_{.015}$&$.479_{.007}$&$.544_{.008}$&$.825_{.008}$\\&
{\scs hing}&\tcb{$.670_{.008}$}&\tcb{$.947_{.016}$}&$1.282_{.018}$&$.559_{.009}$&$.777_{.014}$&$1.169_{.017}$&$.479_{.011}$&$.543_{.010}$&$.827_{.012}$\\&
{\scs smhi}&$.668_{.009}$&$.944_{.016}$&$1.281_{.013}$&$.557_{.009}$&$.775_{.017}$&$1.164_{.020}$&$.480_{.009}$&$.545_{.011}$&$.826_{.011}$\\&
{\scs sqhi}&$.668_{.009}$&$.943_{.013}$&$1.280_{.012}$&$.558_{.009}$&\tcb{$.781_{.016}$}&$1.172_{.030}$&$.478_{.011}$&\tcr{$.541_{.007}$}&$.825_{.009}$\\&
{\scs expo}&$.667_{.008}$&$.940_{.023}$&$1.284_{.042}$&$.559_{.011}$&$.776_{.015}$&$1.166_{.012}$&$.478_{.010}$&$.543_{.008}$&\tcr{$.824_{.009}$}\\&
{\scs abso}&\tcb{$.672_{.010}$}&\tcb{$.953_{.028}$}&\tcb{$1.294_{.032}$}&\tcb{$.568_{.010}$}&\tcb{$.800_{.023}$}&\tcb{$1.191_{.020}$}&$.481_{.011}$&\tcb{$.550_{.014}$}&\tcb{$.832_{.014}$}\\&
{\scs squa}&$.670_{.010}$&$.947_{.020}$&$1.283_{.022}$&$.559_{.009}$&$.777_{.020}$&$1.173_{.039}$&$.479_{.010}$&$.548_{.013}$&$.833_{.028}$\\
\bottomrule
\end{tabular}}\end{table}
\begin{table}[H]
\renewcommand{\arraystretch}{0.75}\renewcommand{\tabcolsep}{5pt}\centering%
\caption{Results for the experiment in Section~\ref{sec:SDE} with the data $\text{N-}1/3$, $\text{N-}1$, and $\text{N-}3$.}
\label{tab:SDERes-N}\scalebox{0.75}{\begin{tabular}{cc|ccc|ccc|ccc}\toprule
&&\multicolumn{3}{c|}{\scs N-1/3}&\multicolumn{3}{c|}{\scs N-1}&\multicolumn{3}{c}{\scs N-3}\\
&&{\scs MZE}&{\scs MAE}&{\scs RMSE}&{\scs MZE}&{\scs MAE}&{\scs RMSE}&{\scs MZE}&{\scs MAE}&{\scs RMSE}\\
\midrule
\multicolumn{2}{c|}{\scs MLR}&\tcr{$.742_{.008}$}&\tcr{$1.914_{.013}$}&\tcr{$2.397_{.011}$}&\tcr{$.591_{.010}$}&\tcr{$1.283_{.024}$}&$1.693_{.018}$&\tcr{$.362_{.008}$}&\tcr{$.697_{.015}$}&$1.087_{.010}$\\
\midrule\multirow{7}{*}{\rotatebox{90}{\tiny AT-O\hspace{1em}}}&
{\scs logi}&\tcb{$.782_{.017}$}&$1.940_{.022}$&\tcr{$2.407_{.018}$}&\tcb{$.665_{.021}$}&$1.310_{.019}$&$1.695_{.018}$&\tcb{$.523_{.016}$}&\tcb{$.758_{.015}$}&\tcr{$1.091_{.013}$}\\&
{\scs hing}&\tcb{$.784_{.022}$}&$1.946_{.024}$&\tcb{$2.425_{.024}$}&\tcb{$.657_{.024}$}&\tcb{$1.324_{.021}$}&$1.705_{.024}$&\tcb{$.504_{.066}$}&$.809_{.244}$&\tcb{$1.167_{.273}$}\\&
{\scs smhi}&\tcb{$.780_{.017}$}&$1.940_{.022}$&$2.411_{.025}$&\tcb{$.671_{.015}$}&$1.312_{.026}$&$1.696_{.020}$&\tcb{$.507_{.018}$}&\tcr{$.747_{.019}$}&$1.094_{.015}$\\&
{\scs sqhi}&\tcb{$.782_{.015}$}&$1.935_{.021}$&$2.411_{.022}$&\tcb{$.665_{.019}$}&$1.315_{.030}$&$1.699_{.022}$&\tcb{$.514_{.014}$}&\tcb{$.759_{.021}$}&$1.093_{.013}$\\&
{\scs expo}&\tcb{$.780_{.011}$}&$1.939_{.018}$&$2.409_{.019}$&\tcb{$.674_{.020}$}&$1.311_{.025}$&$1.697_{.022}$&\tcb{$.526_{.020}$}&\tcb{$.763_{.018}$}&$1.094_{.011}$\\&
{\scs abso}&\tcb{$.788_{.018}$}&\tcr{$1.935_{.021}$}&$2.410_{.019}$&\tcb{$.688_{.037}$}&\tcb{$1.360_{.066}$}&\tcb{$1.752_{.071}$}&\tcb{$.684_{.107}$}&\tcb{$1.188_{.333}$}&\tcb{$1.530_{.358}$}\\&
{\scs squa}&\tcb{$.778_{.018}$}&$1.944_{.019}$&$2.415_{.013}$&\tcb{$.665_{.014}$}&\tcr{$1.310_{.023}$}&$1.697_{.020}$&\tcb{$.523_{.013}$}&\tcb{$.764_{.018}$}&\tcb{$1.096_{.011}$}\\
\midrule\multirow{7}{*}{\rotatebox{90}{\tiny IT-N\hspace{1em}}}&
{\scs logi}&$.771_{.014}$&$1.949_{.033}$&\tcb{$2.422_{.024}$}&\tcb{$.661_{.011}$}&$1.317_{.016}$&\tcr{$1.695_{.019}$}&\tcb{$.514_{.010}$}&\tcb{$.764_{.020}$}&$1.097_{.017}$\\&
{\scs hing}&\tcb{$.786_{.023}$}&\tcb{$2.072_{.101}$}&\tcb{$2.533_{.086}$}&\tcr{$.641_{.014}$}&\tcb{$1.384_{.079}$}&\tcb{$1.795_{.089}$}&\tcr{$.464_{.041}$}&\tcb{$.797_{.037}$}&\tcb{$1.151_{.035}$}\\&
{\scs smhi}&$.769_{.013}$&\tcb{$1.974_{.038}$}&\tcb{$2.438_{.025}$}&\tcb{$.654_{.014}$}&\tcb{$1.337_{.029}$}&\tcb{$1.730_{.037}$}&\tcb{$.486_{.016}$}&\tcb{$.760_{.021}$}&\tcb{$1.108_{.020}$}\\&
{\scs sqhi}&$.770_{.011}$&\tcb{$1.954_{.025}$}&\tcb{$2.422_{.015}$}&\tcb{$.659_{.010}$}&$1.319_{.019}$&$1.697_{.021}$&\tcb{$.511_{.014}$}&\tcb{$.761_{.018}$}&$1.095_{.012}$\\&
{\scs expo}&\tcb{$.773_{.011}$}&$1.944_{.028}$&$2.416_{.024}$&\tcb{$.669_{.012}$}&$1.321_{.028}$&$1.702_{.021}$&\tcb{$.511_{.016}$}&\tcb{$.766_{.020}$}&$1.097_{.012}$\\&
{\scs abso}&\tcb{$.780_{.022}$}&\tcb{$2.082_{.095}$}&\tcb{$2.538_{.070}$}&\tcb{$.657_{.015}$}&\tcb{$1.435_{.063}$}&\tcb{$1.861_{.073}$}&\tcb{$.484_{.023}$}&\tcb{$.870_{.053}$}&\tcb{$1.215_{.050}$}\\&
{\scs squa}&\tcr{$.768_{.007}$}&$1.942_{.020}$&\tcb{$2.419_{.019}$}&\tcb{$.656_{.014}$}&$1.312_{.022}$&$1.700_{.019}$&\tcb{$.510_{.014}$}&\tcb{$.760_{.017}$}&\tcb{$1.098_{.016}$}\\
\midrule\multirow{7}{*}{\rotatebox{90}{\tiny IT-O\hspace{1em}}}&
{\scs logi}&\tcb{$.773_{.007}$}&\tcb{$1.951_{.027}$}&$2.420_{.021}$&\tcb{$.666_{.012}$}&$1.317_{.018}$&$1.700_{.017}$&\tcb{$.512_{.014}$}&\tcb{$.759_{.020}$}&\tcb{$1.098_{.012}$}\\&
{\scs hing}&$.775_{.016}$&\tcb{$2.028_{.085}$}&\tcb{$2.481_{.063}$}&\tcb{$.653_{.013}$}&\tcb{$1.340_{.036}$}&\tcb{$1.739_{.044}$}&\tcb{$.510_{.078}$}&\tcb{$.919_{.292}$}&\tcb{$1.291_{.352}$}\\&
{\scs smhi}&\tcb{$.773_{.010}$}&\tcb{$1.955_{.028}$}&$2.424_{.026}$&\tcb{$.658_{.012}$}&$1.323_{.026}$&$1.708_{.026}$&\tcb{$.495_{.017}$}&\tcb{$.761_{.021}$}&\tcb{$1.104_{.012}$}\\&
{\scs sqhi}&\tcb{$.774_{.008}$}&\tcb{$1.959_{.032}$}&\tcb{$2.426_{.026}$}&\tcb{$.665_{.012}$}&$1.318_{.026}$&$1.703_{.024}$&\tcb{$.514_{.013}$}&\tcb{$.765_{.018}$}&\tcb{$1.097_{.013}$}\\&
{\scs expo}&$.772_{.008}$&\tcb{$1.947_{.023}$}&$2.420_{.019}$&\tcb{$.669_{.014}$}&$1.328_{.043}$&$1.709_{.039}$&\tcb{$.542_{.060}$}&\tcb{$.904_{.340}$}&\tcb{$1.269_{.419}$}\\&
{\scs abso}&\tcb{$.788_{.025}$}&\tcb{$2.085_{.114}$}&\tcb{$2.531_{.096}$}&\tcb{$.665_{.043}$}&\tcb{$1.408_{.107}$}&\tcb{$1.823_{.113}$}&\tcb{$.562_{.109}$}&\tcb{$1.161_{.453}$}&\tcb{$1.573_{.552}$}\\&
{\scs squa}&\tcb{$.773_{.011}$}&\tcb{$1.958_{.031}$}&\tcb{$2.429_{.030}$}&\tcb{$.667_{.015}$}&$1.317_{.026}$&$1.705_{.025}$&\tcb{$.518_{.014}$}&\tcb{$.769_{.022}$}&\tcb{$1.098_{.014}$}\\
\bottomrule
\end{tabular}}\end{table}

Figure~\ref{fig:SDERes} shows the learned 1DT value $\hat{a}_i=\hat{a}(x^{[i]})$,
method's label prediction $\hat{f}_{\ell,i}=\hat{f}(x^{[i]})$,
and optimal label prediction $\tilde{f}_{\ell,i}$ 
for several instances of the experiment in Section~\ref{sec:SDE}.
This indicates representative behaviors of threshold methods 
based on a learning procedure with non-PL and PL losses
in a finite-sample situation.

\begin{figure}[H]
\renewcommand{\arraystretch}{0.1}\renewcommand{\tabcolsep}{0pt}\centering%
\begin{tabular}{rrrrr}
~~&
{\tiny\hyt{o1} H-1, logi~~~~~~~}~~~~&
{\tiny\hyt{o2} M-1, logi~~~~~~~}~~~~&
{\tiny\hyt{o3} A-1, logi~~~~~~~}~~~~&
{\tiny\hyt{o4} O-(1/3,3), logi~~~~}\\
\rotatebox{90}{\tiny~~~~~~~~~~~~~~~AT-O}~~&
\includegraphics[height=3cm, bb=0 0 404 386]{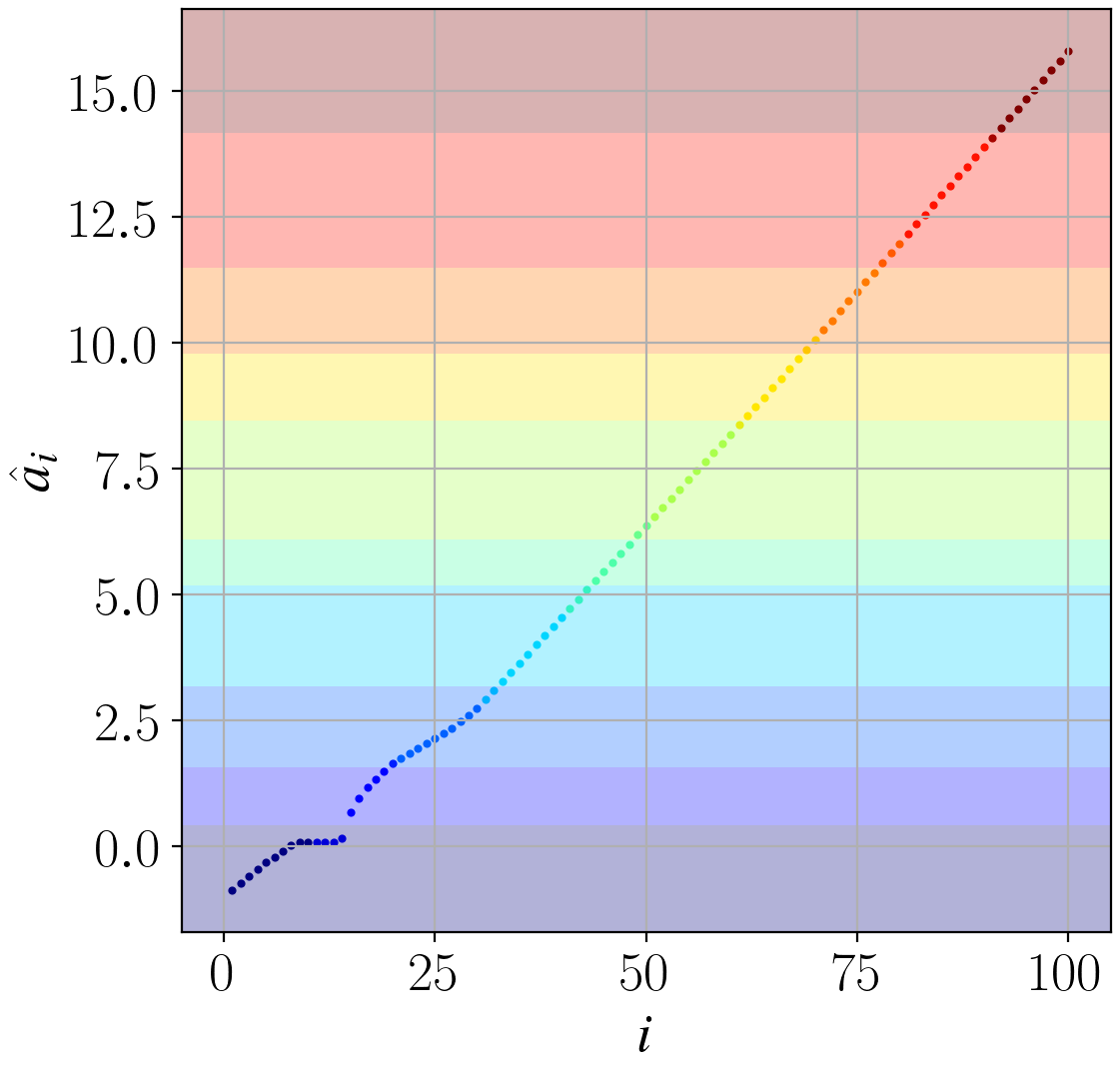}~~~~&
\includegraphics[height=3cm, bb=0 0 404 386]{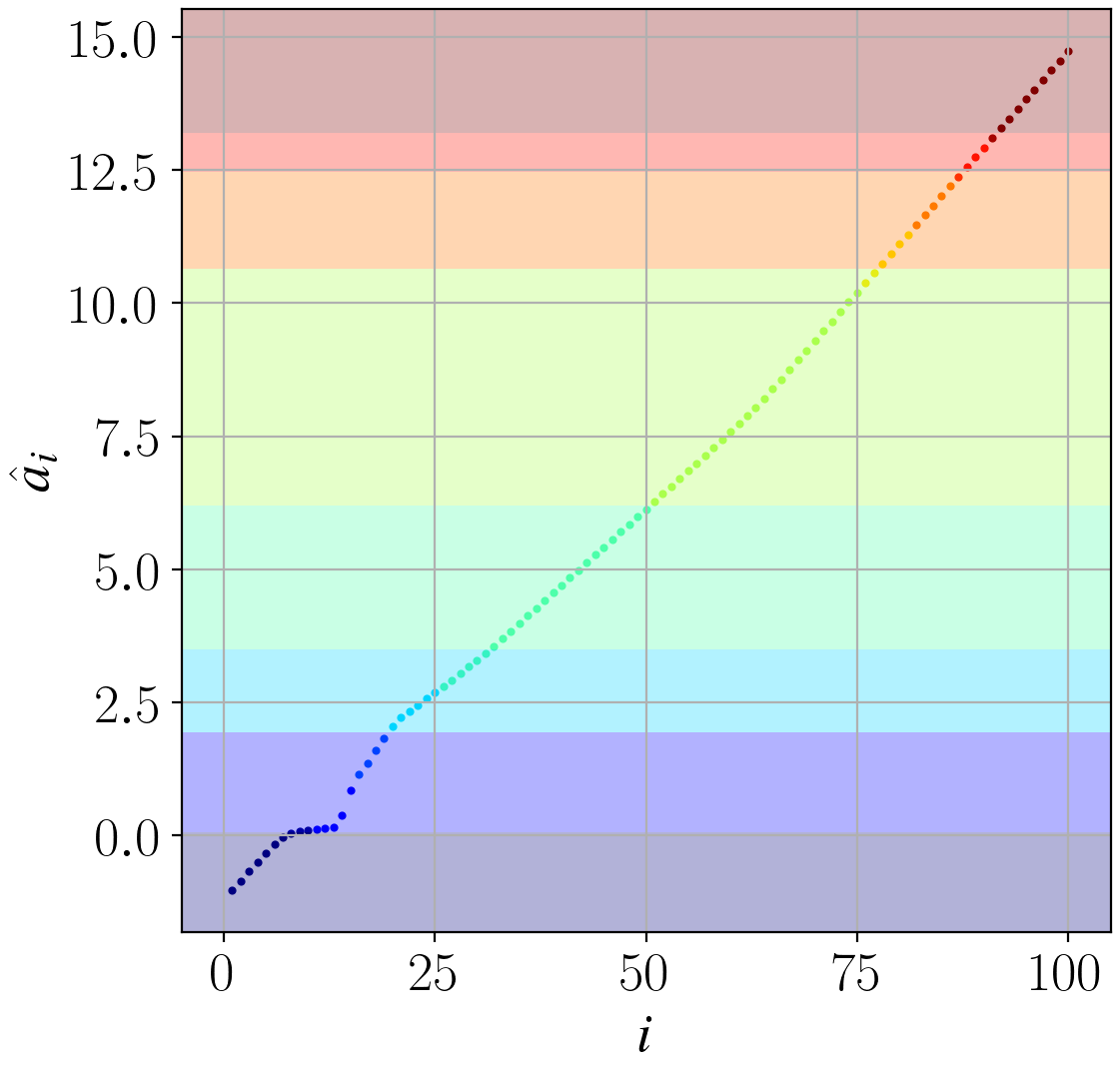}~~~~&
\includegraphics[height=3cm, bb=0 0 404 386]{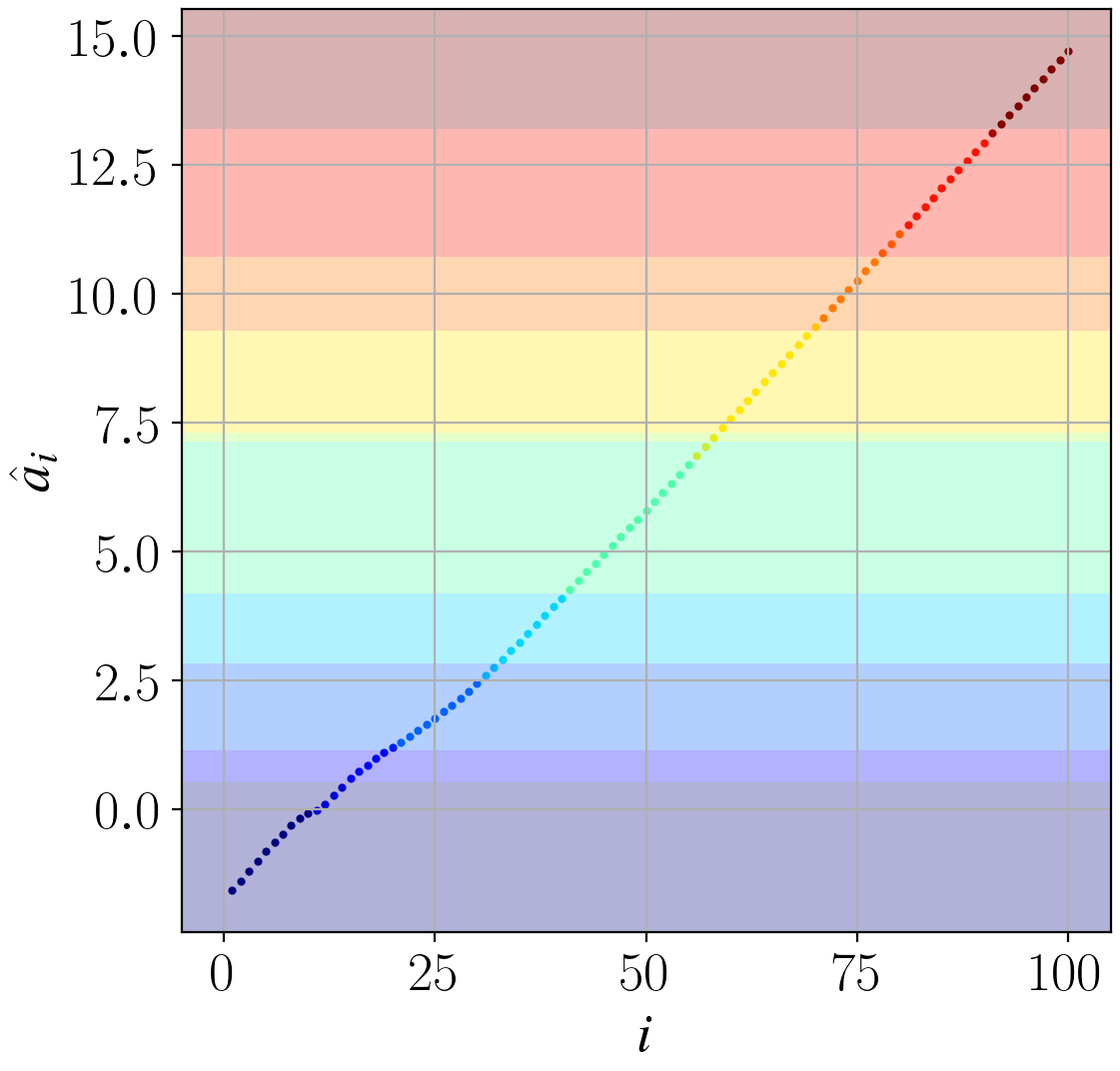}~~~~&
\includegraphics[height=3cm, bb=0 0 407 386]{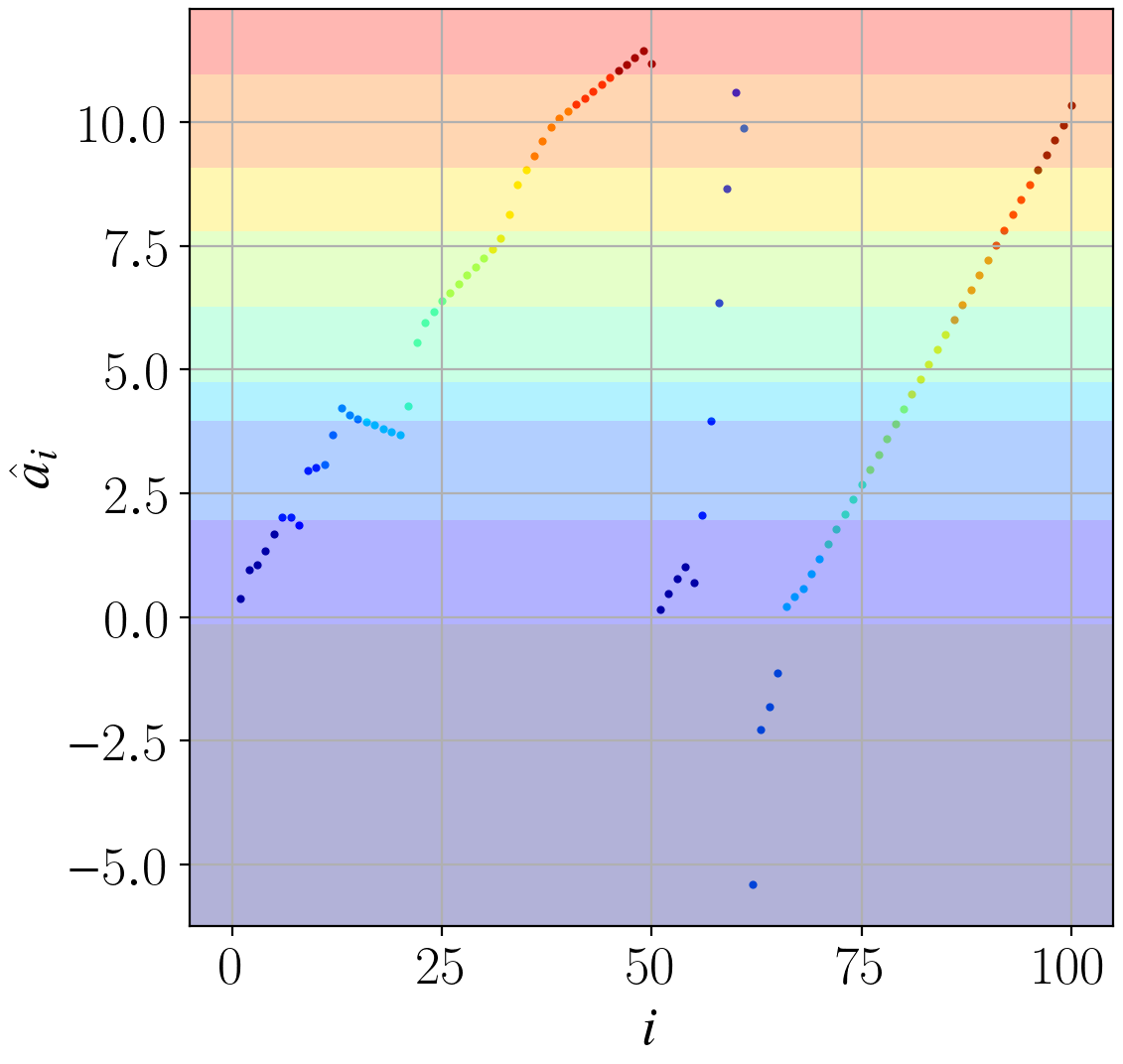}\\
\rotatebox{90}{\tiny~~~~~~~~~~~~~~~IT-N}~~&
\includegraphics[height=3cm, bb=0 0 380 386]{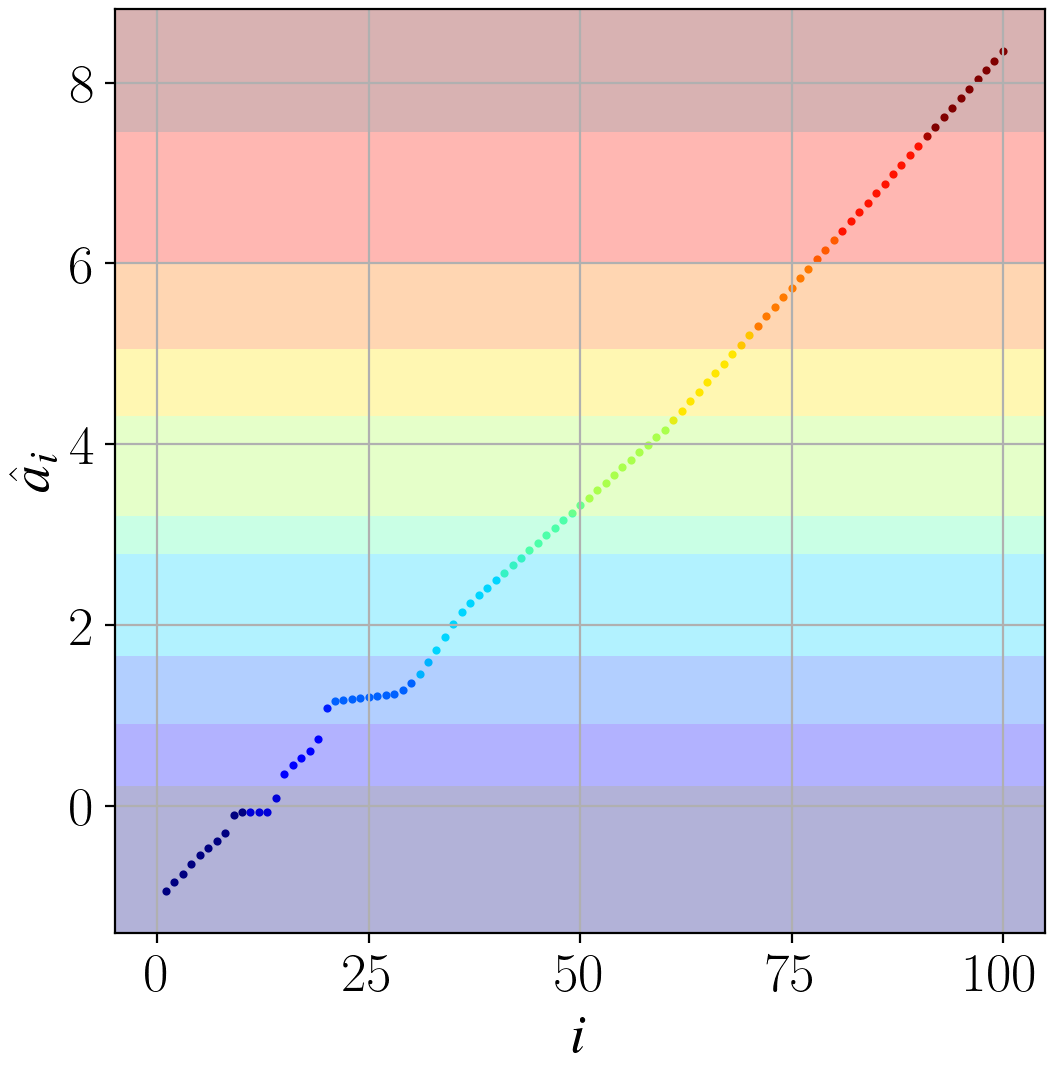}~~~~&
\includegraphics[height=3cm, bb=0 0 380 386]{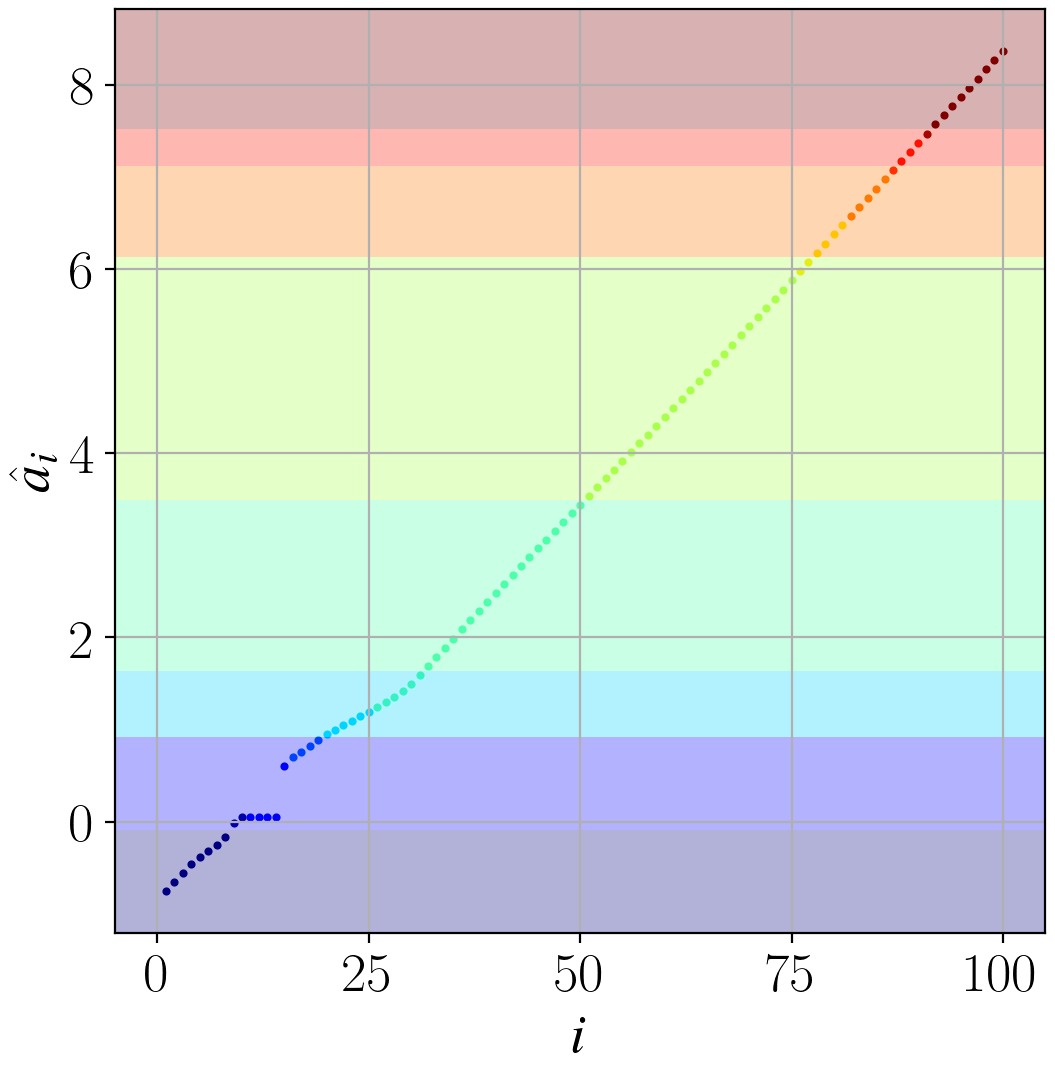}~~~~&
\includegraphics[height=3cm, bb=0 0 380 386]{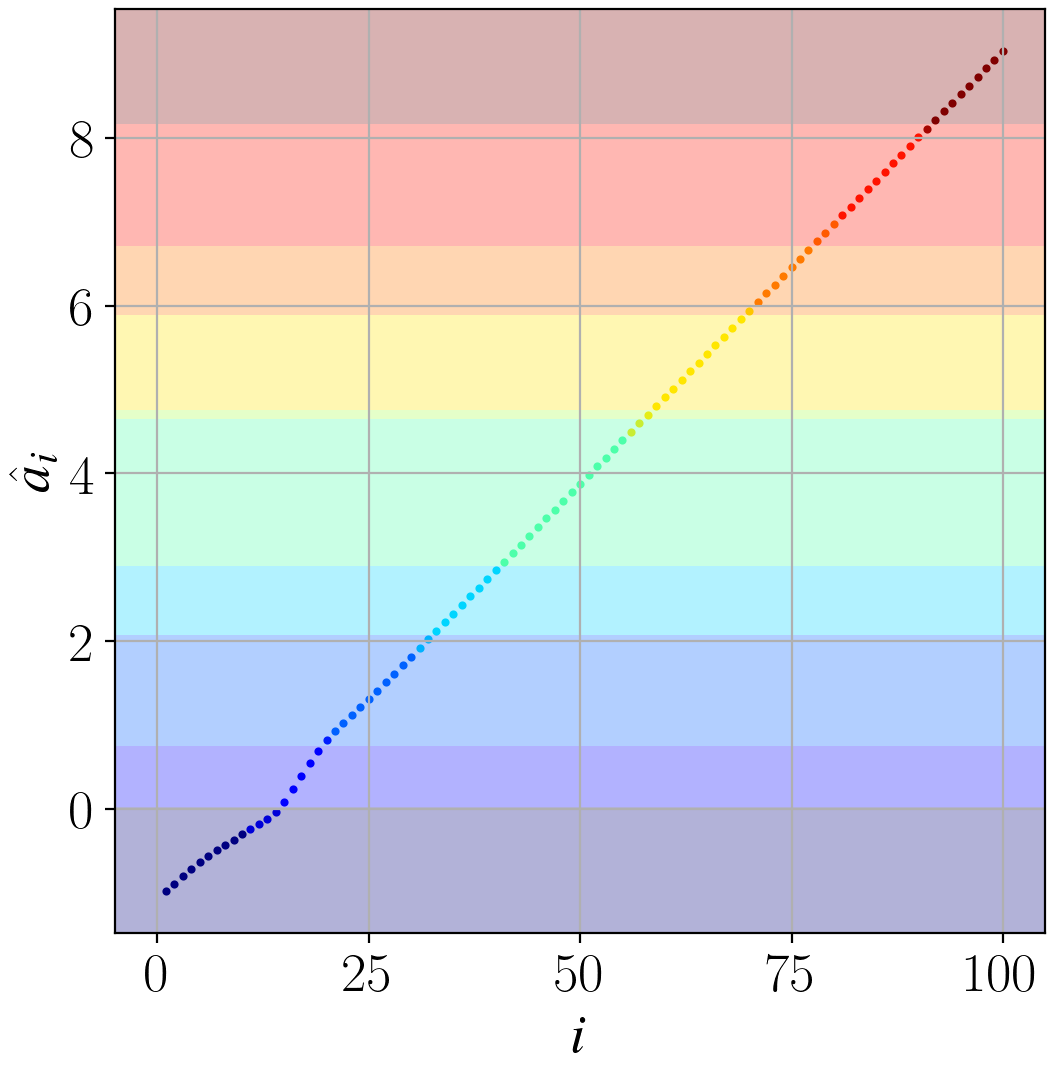}~~~~&
\includegraphics[height=3cm, bb=0 0 392 386]{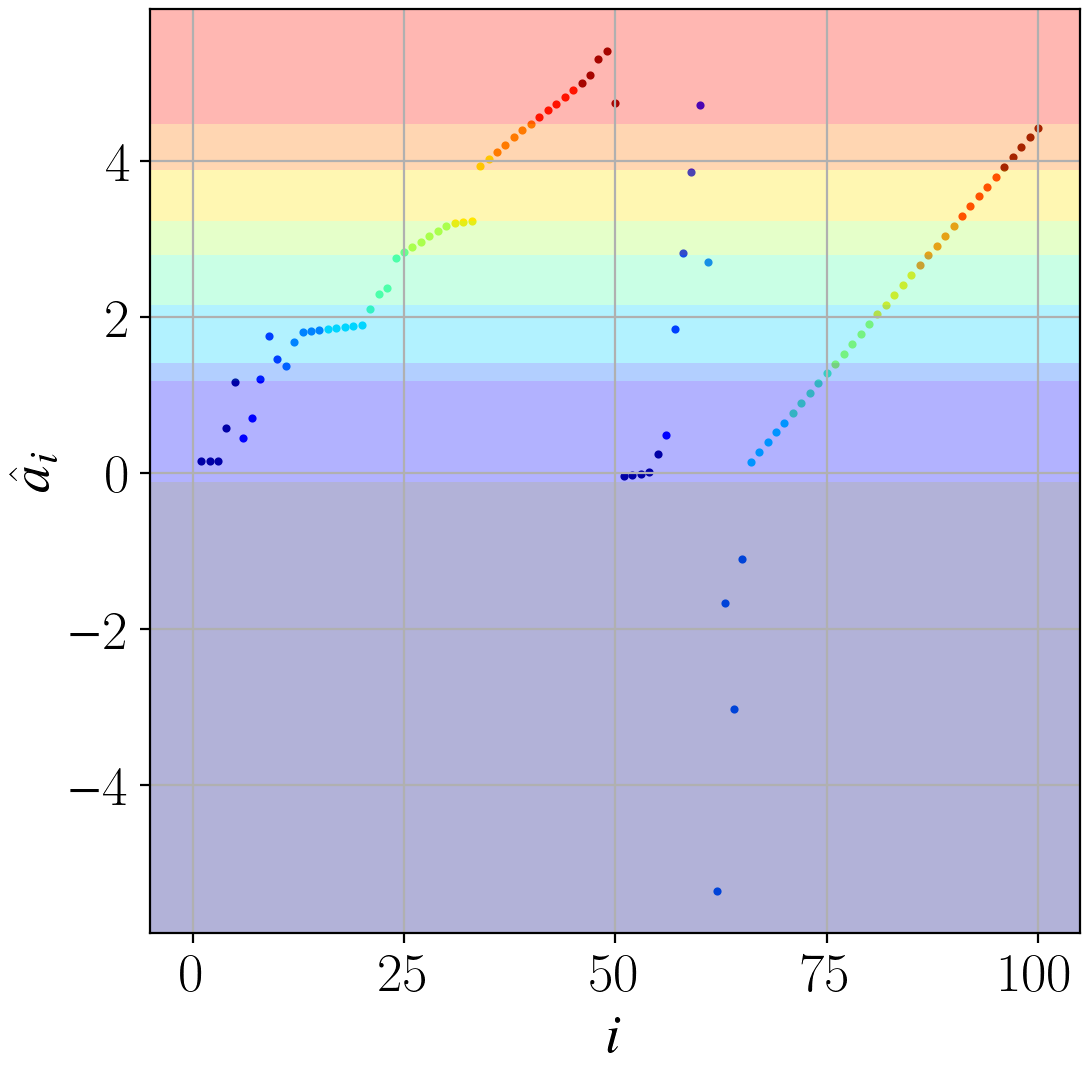}\\
\rotatebox{90}{\tiny~~~~~~~~~~~~~~~IT-O}~~&
\includegraphics[height=3cm, bb=0 0 380 386]{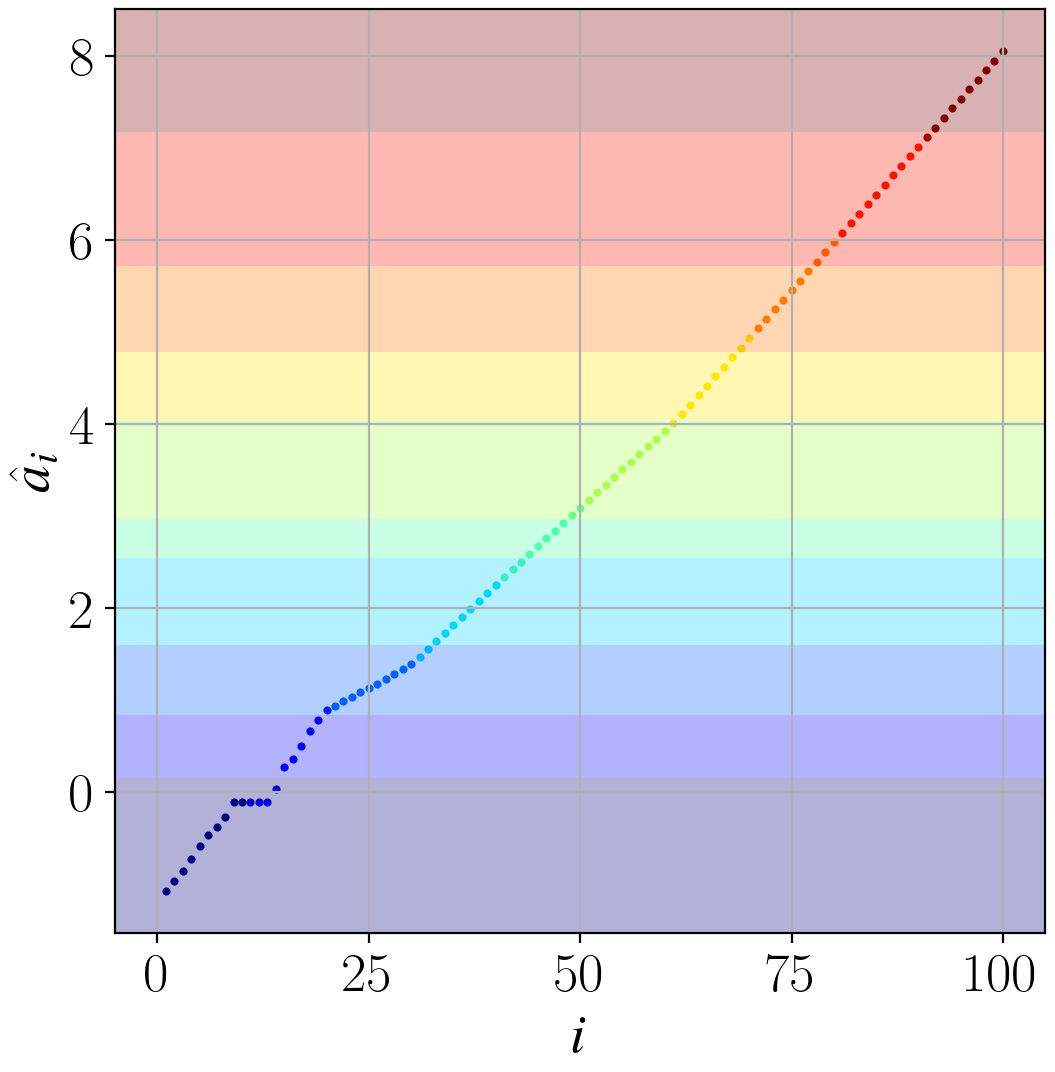}~~~~&
\includegraphics[height=3cm, bb=0 0 380 386]{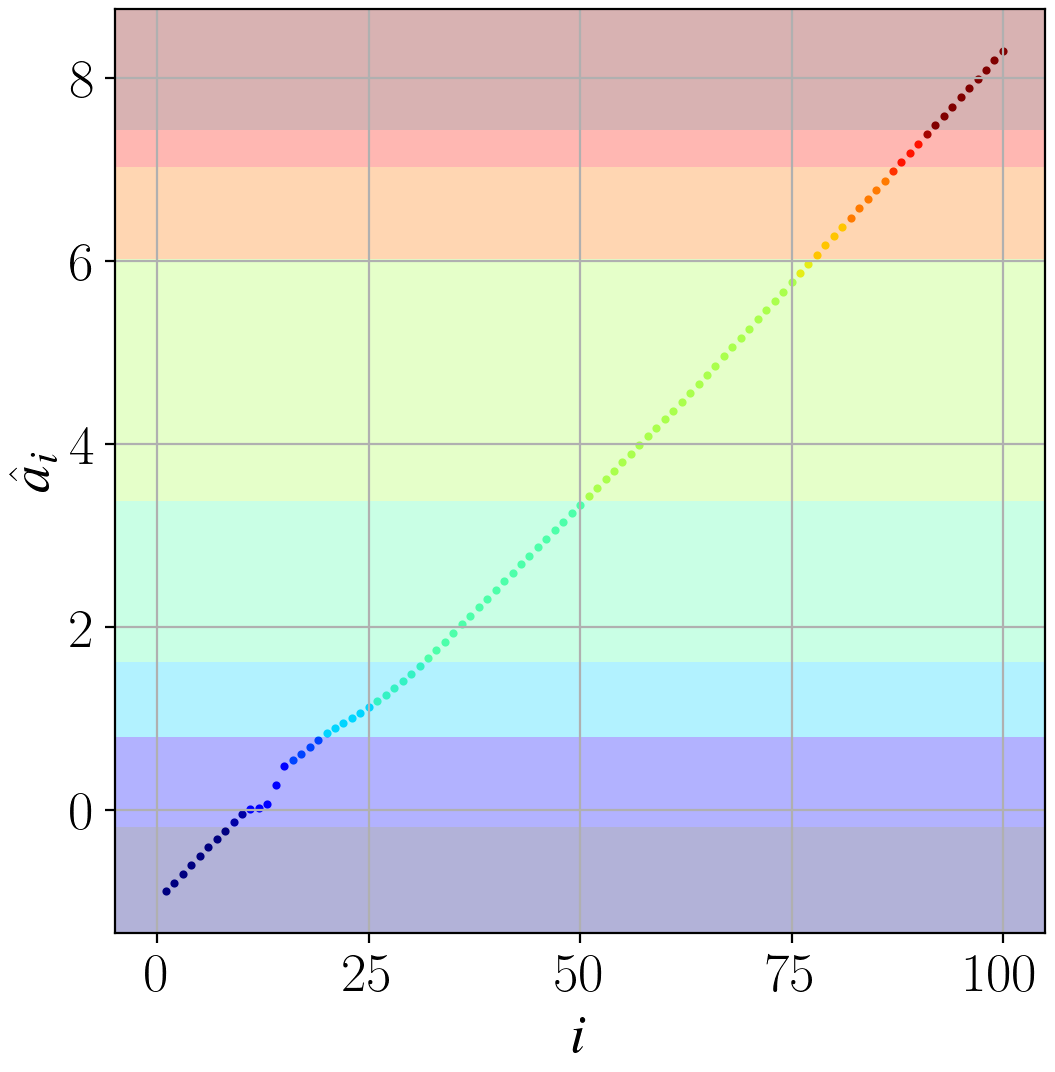}~~~~&
\includegraphics[height=3cm, bb=0 0 380 386]{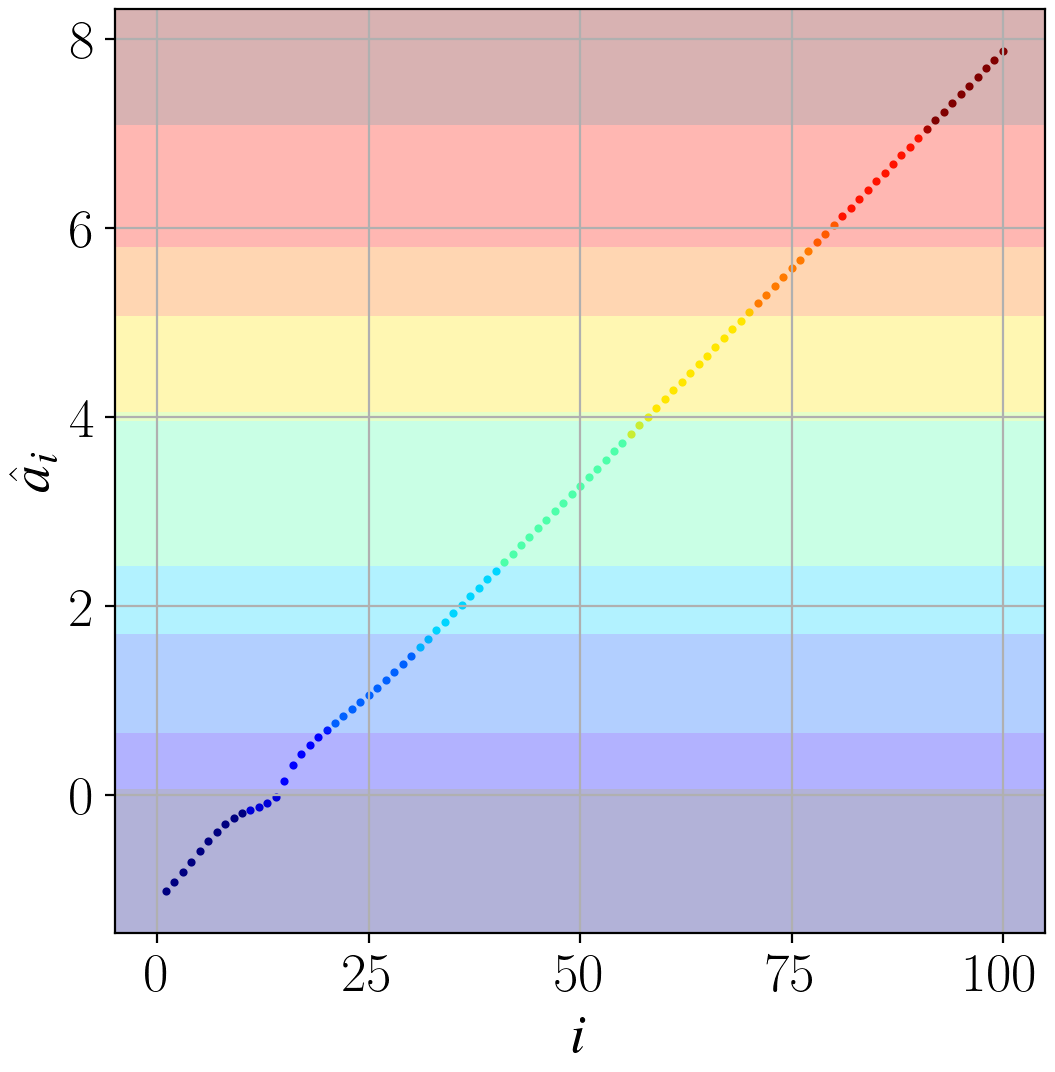}~~~~&
\includegraphics[height=3cm, bb=0 0 392 386]{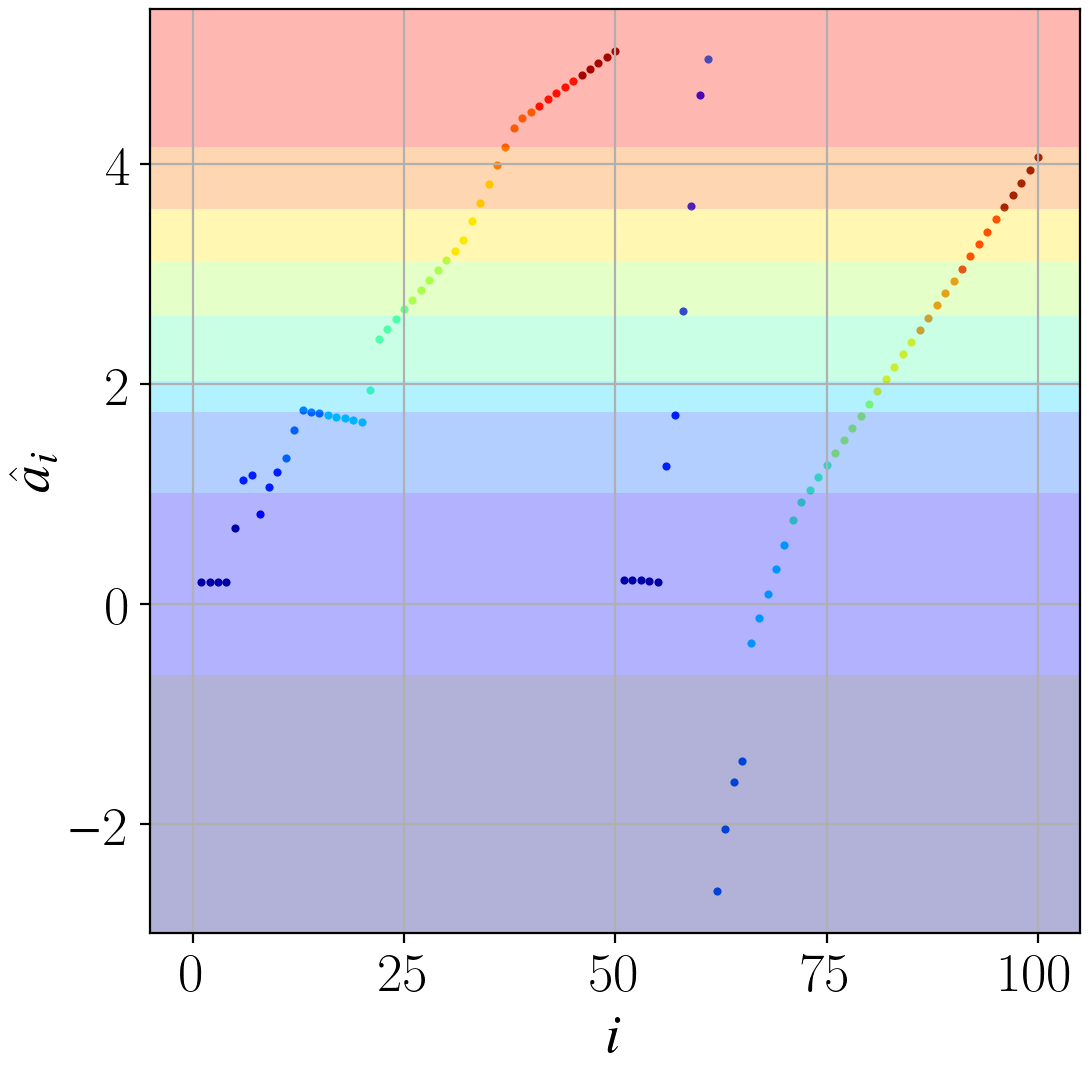}\\
~~&
{\tiny\hyt{o5} N-1, logi~~~~~~~}~~~~&
{\tiny\hyt{o6} H-1/3, hing~~~~~}~~~~&
{\tiny\hyt{o7} H-1, hing~~~~~~~}~~~~&
{\tiny\hyt{o8} H-3, hing~~~~~~~}\\
\rotatebox{90}{\tiny~~~~~~~~~~~~~~~AT-O}~~&
\includegraphics[height=3cm, bb=0 0 380 386]{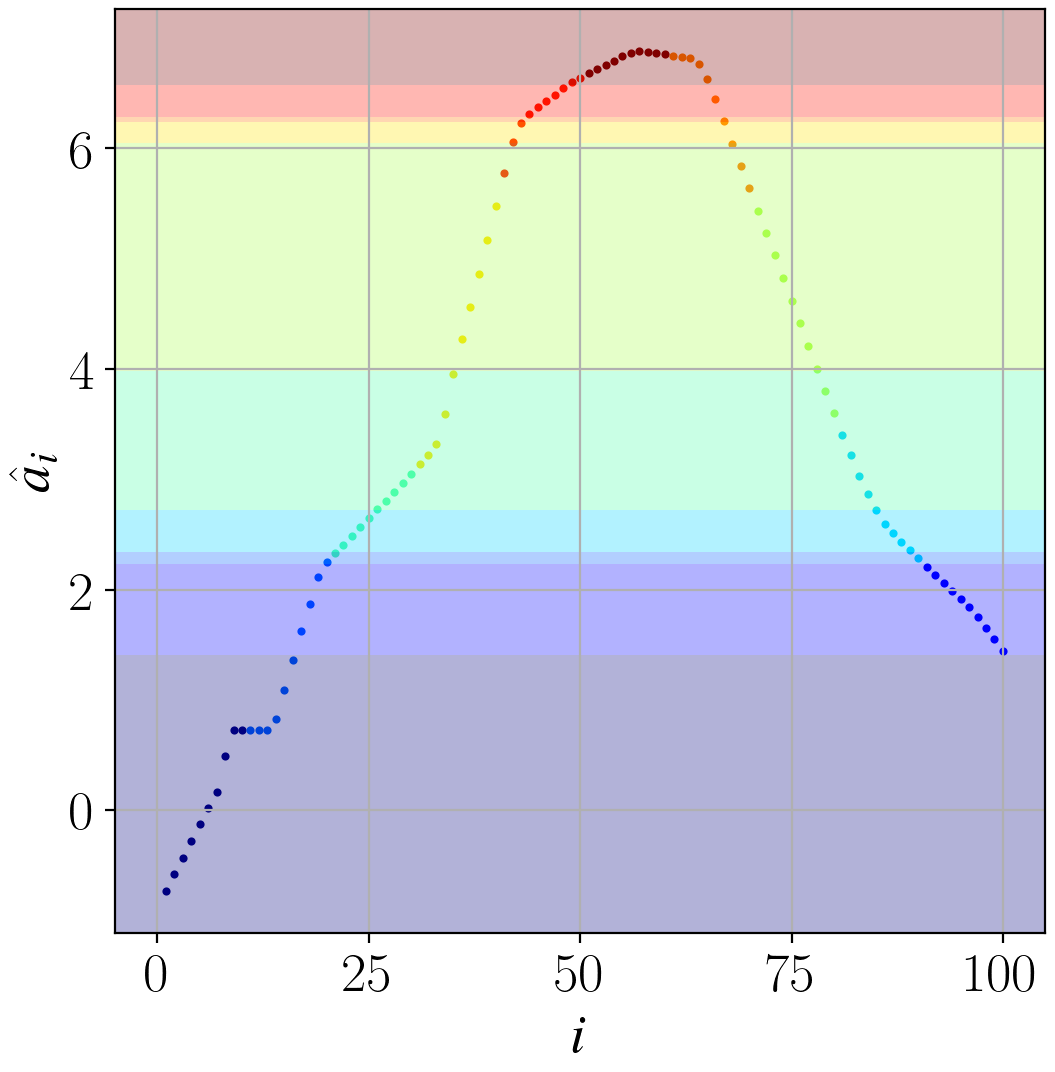}~~~~&
\includegraphics[height=3cm, bb=0 0 380 386]{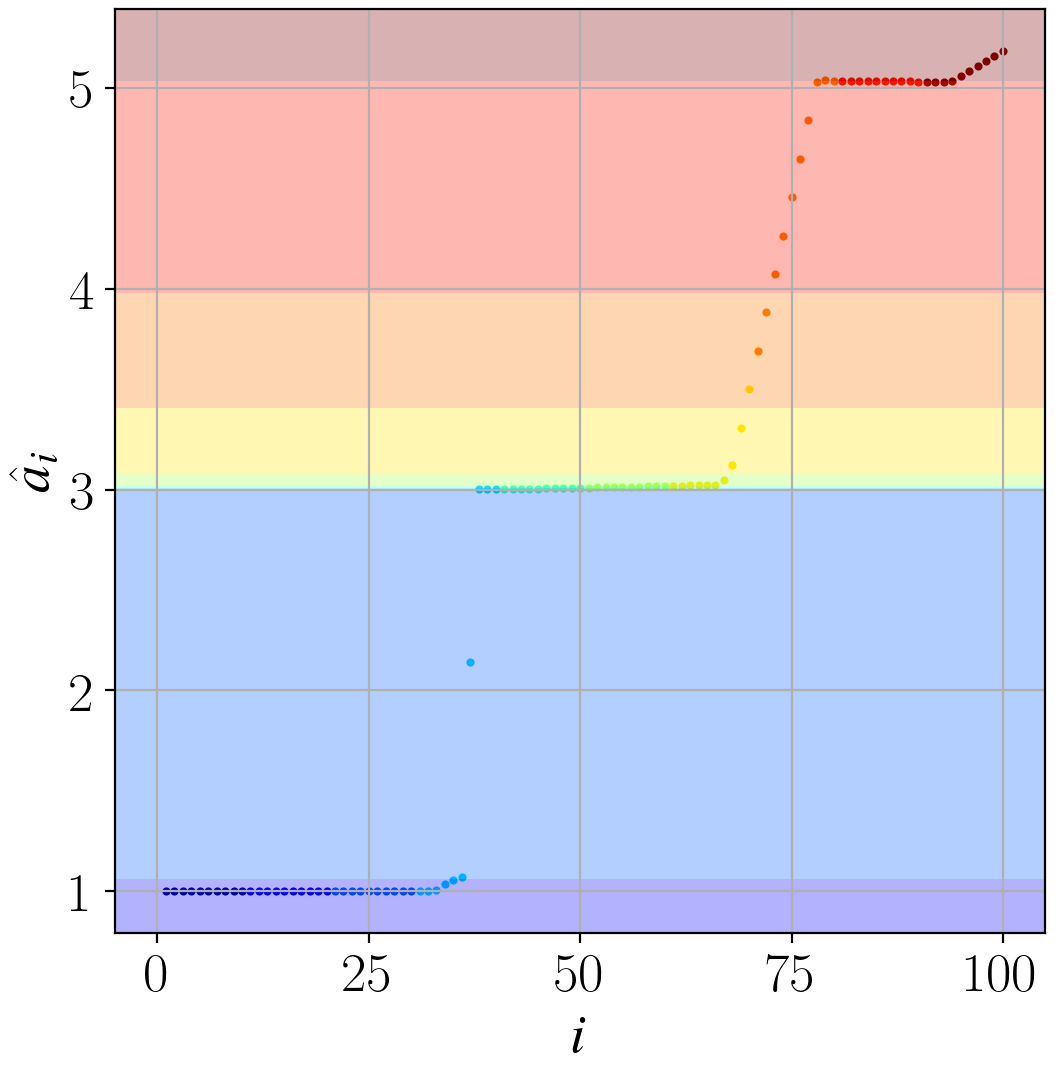}~~~~&
\includegraphics[height=3cm, bb=0 0 389 386]{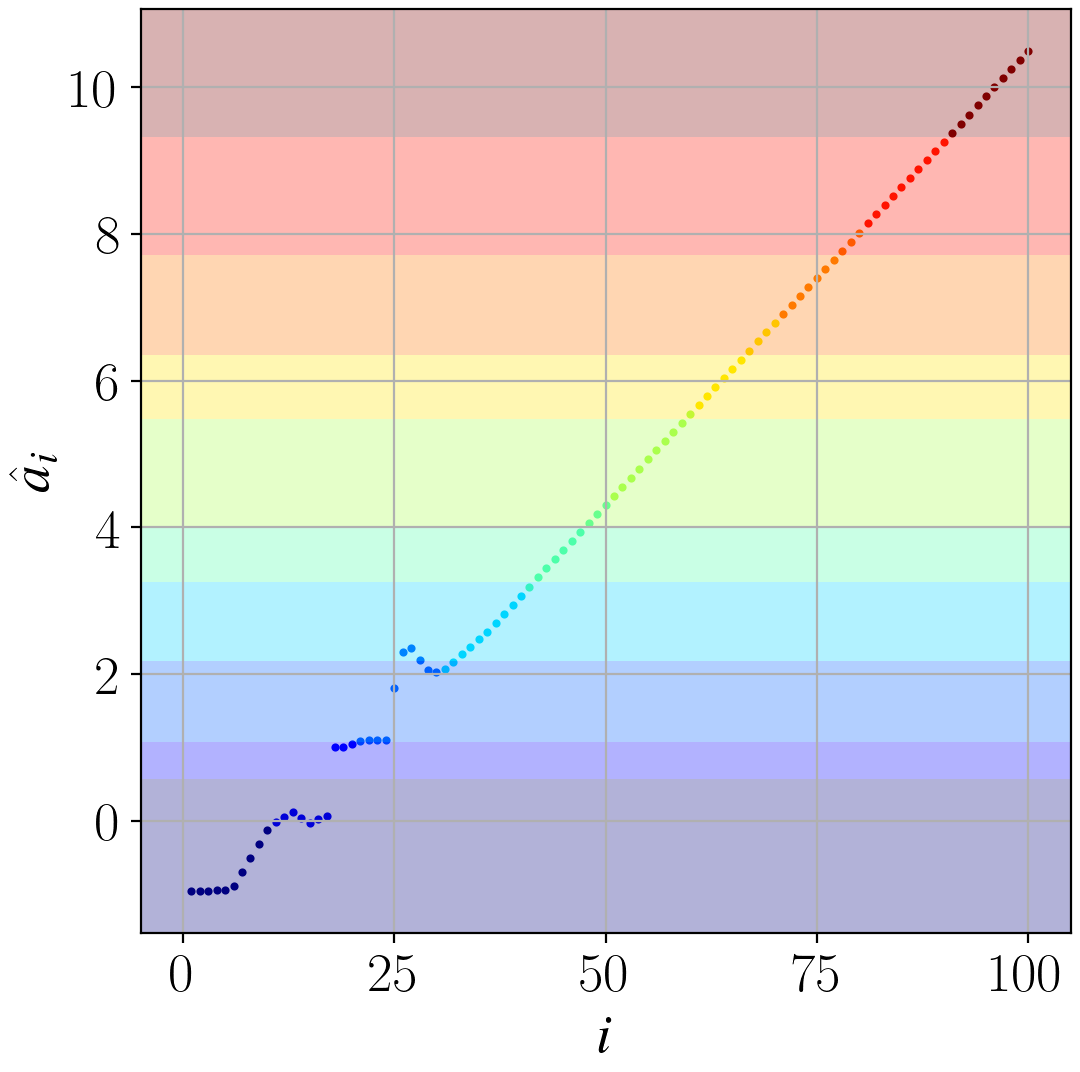}~~~~&
\includegraphics[height=3cm, bb=0 0 389 386]{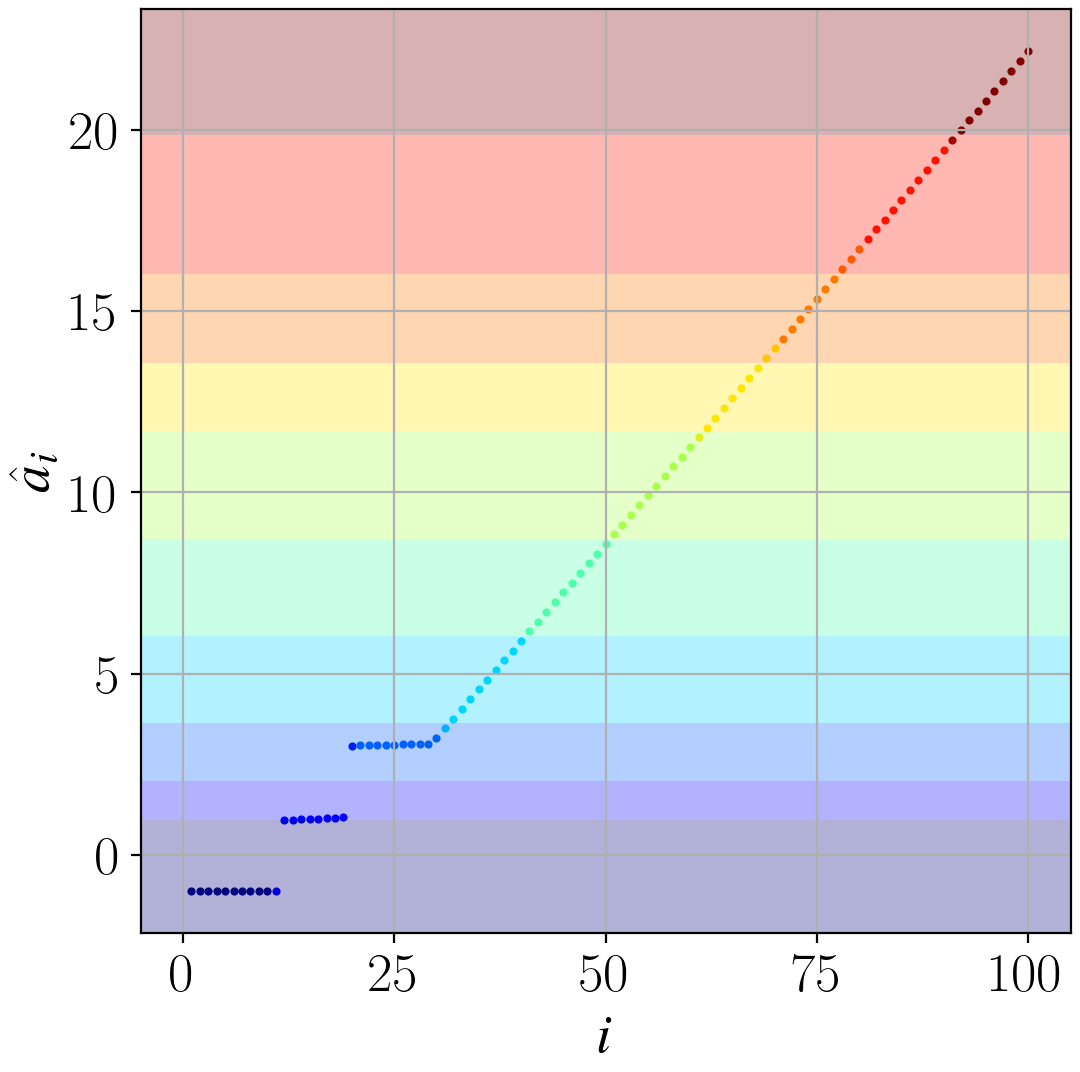}\\
\rotatebox{90}{\tiny~~~~~~~~~~~~~~~IT-N}~~&
\includegraphics[height=3cm, bb=0 0 407 386]{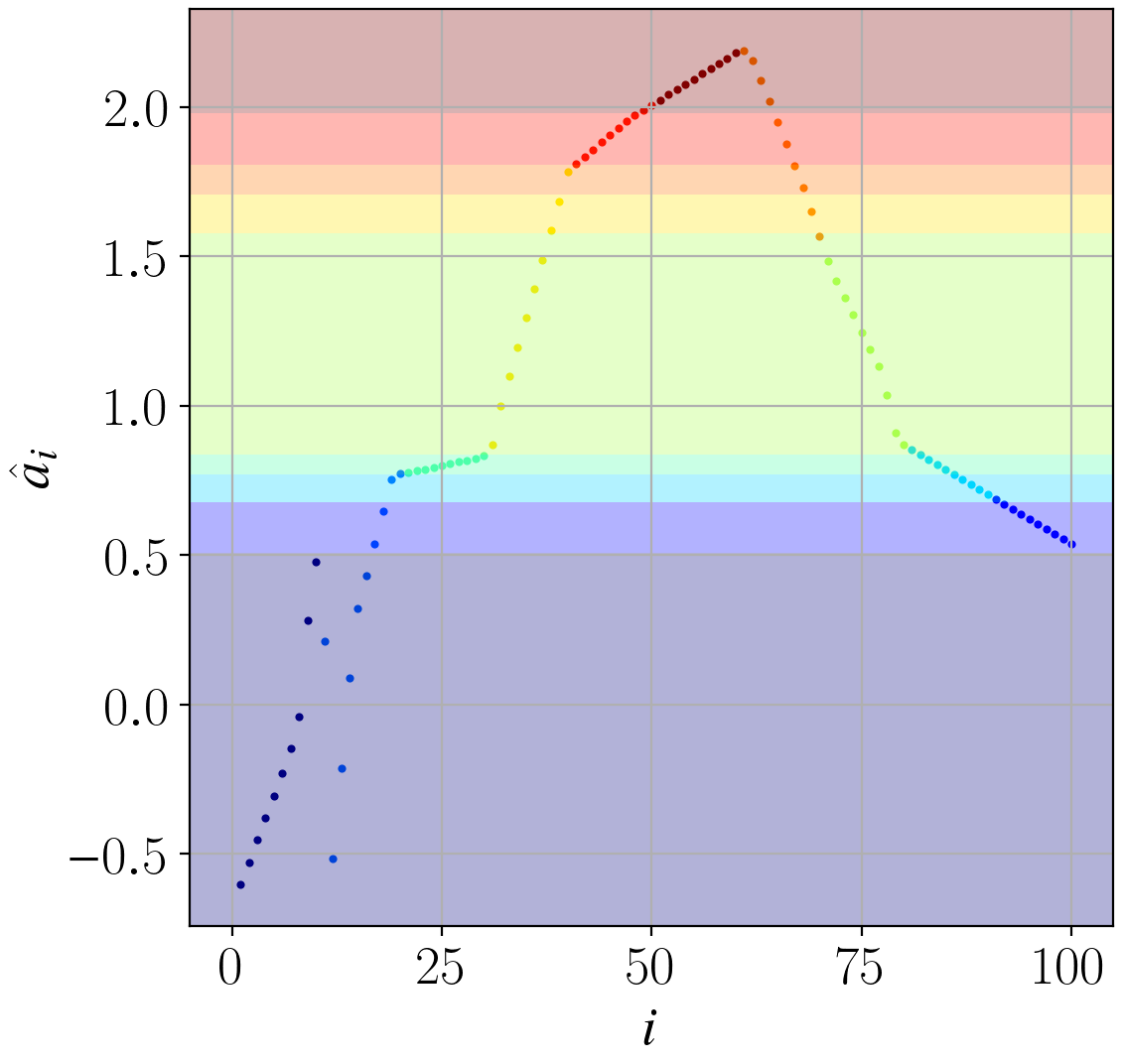}~~~~&
\includegraphics[height=3cm, bb=0 0 380 386]{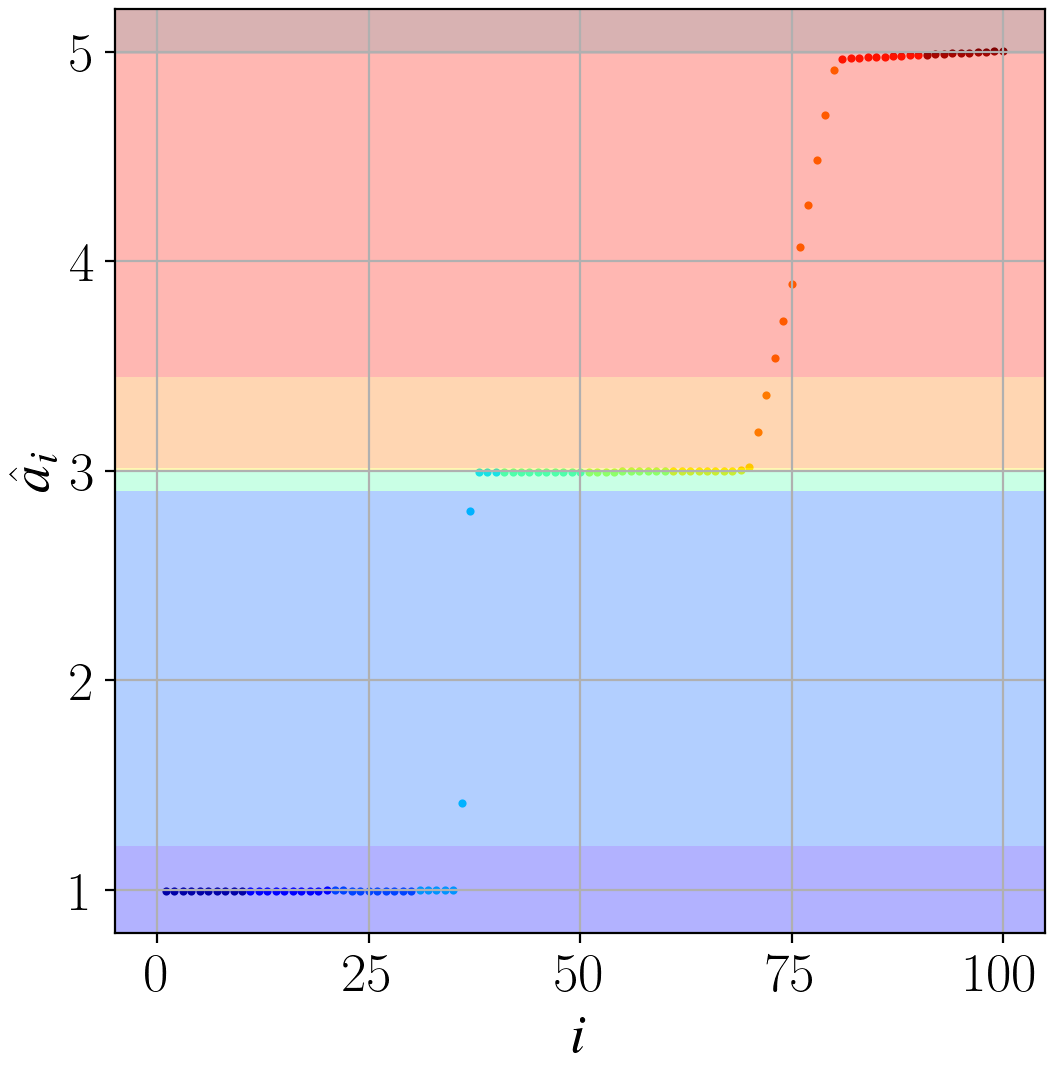}~~~~&
\includegraphics[height=3cm, bb=0 0 380 386]{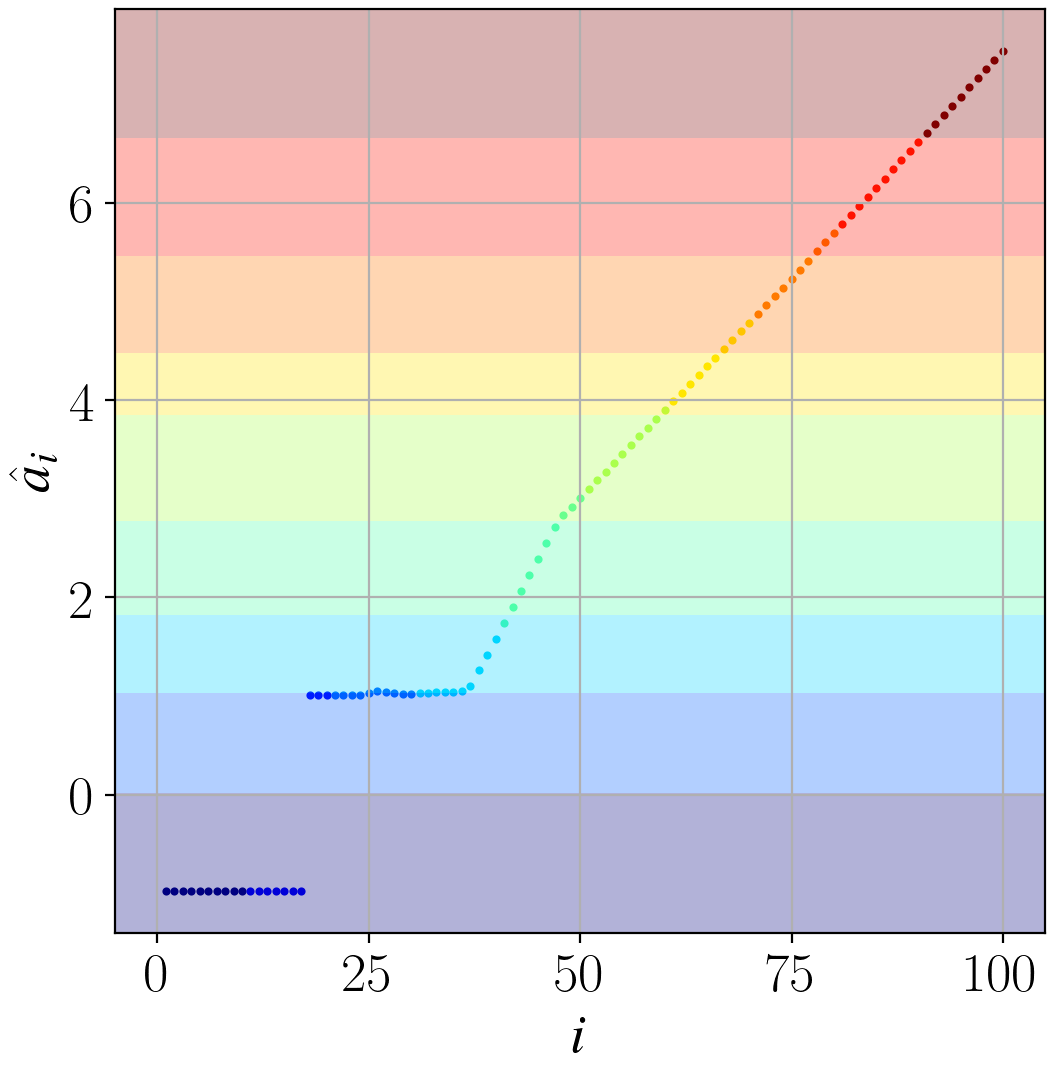}~~~~&
\includegraphics[height=3cm, bb=0 0 389 390]{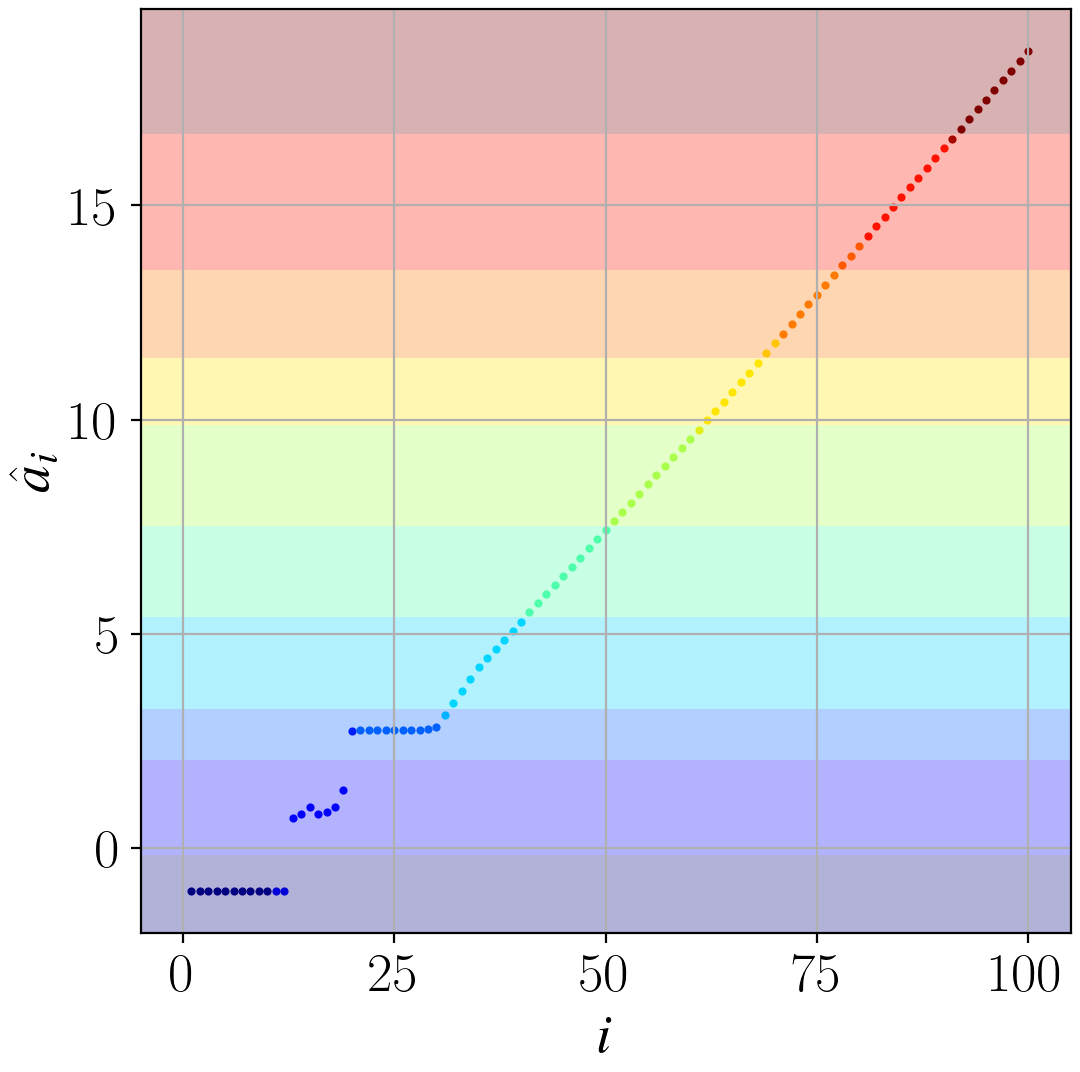}\\
\rotatebox{90}{\tiny~~~~~~~~~~~~~~~IT-O}~~&
\includegraphics[height=3cm, bb=0 0 407 386]{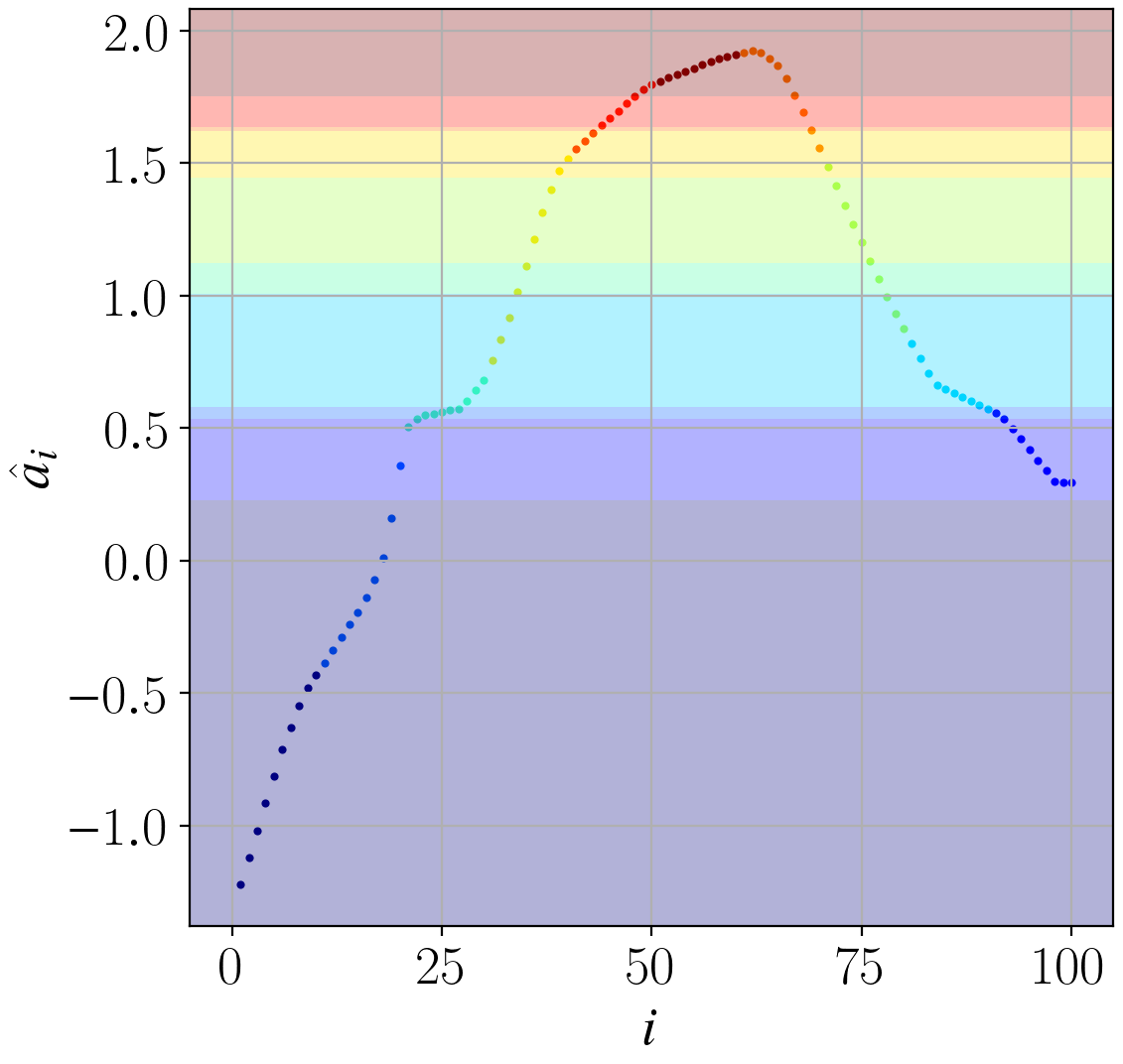}~~~~&
\includegraphics[height=3cm, bb=0 0 392 386]{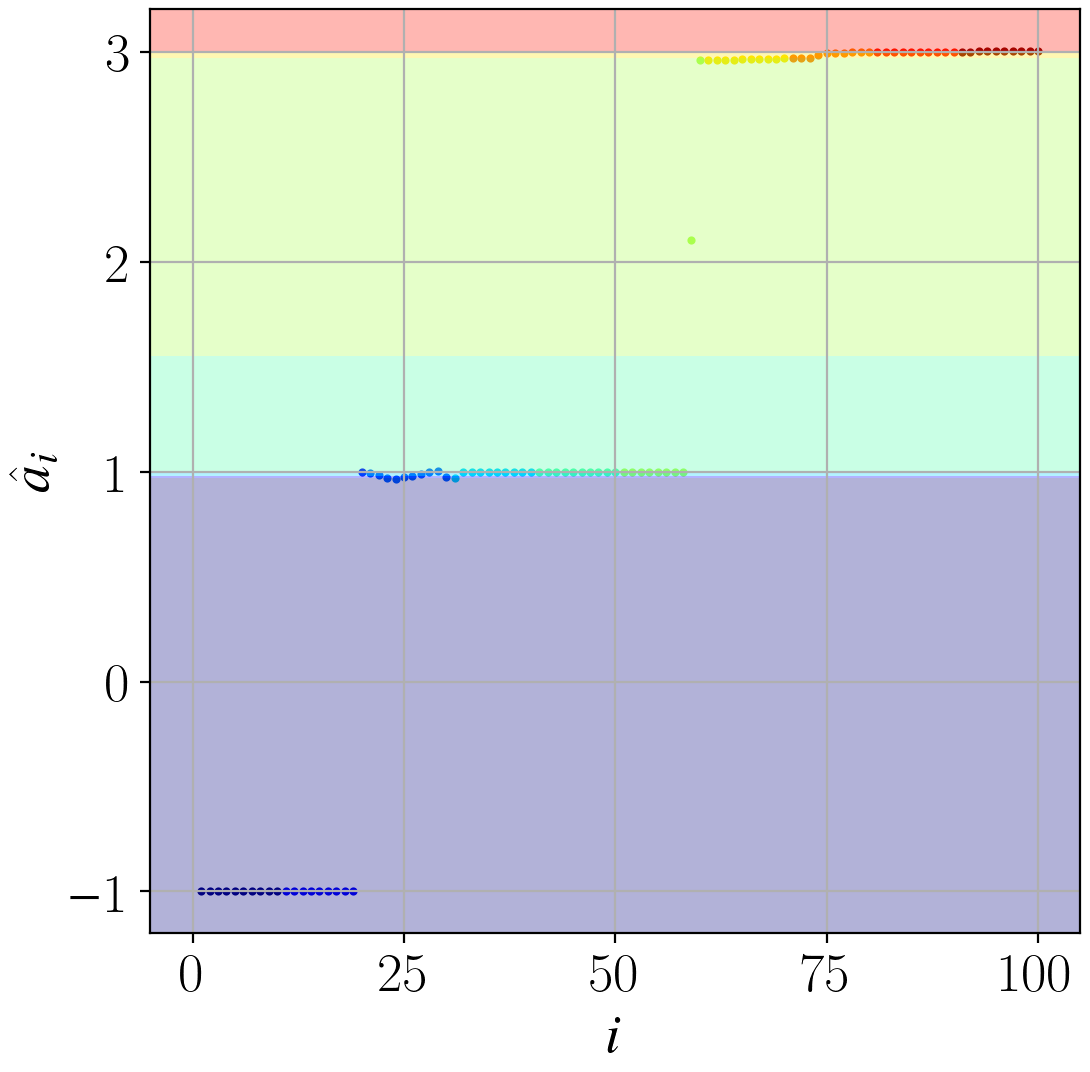}~~~~&
\includegraphics[height=3cm, bb=0 0 380 391]{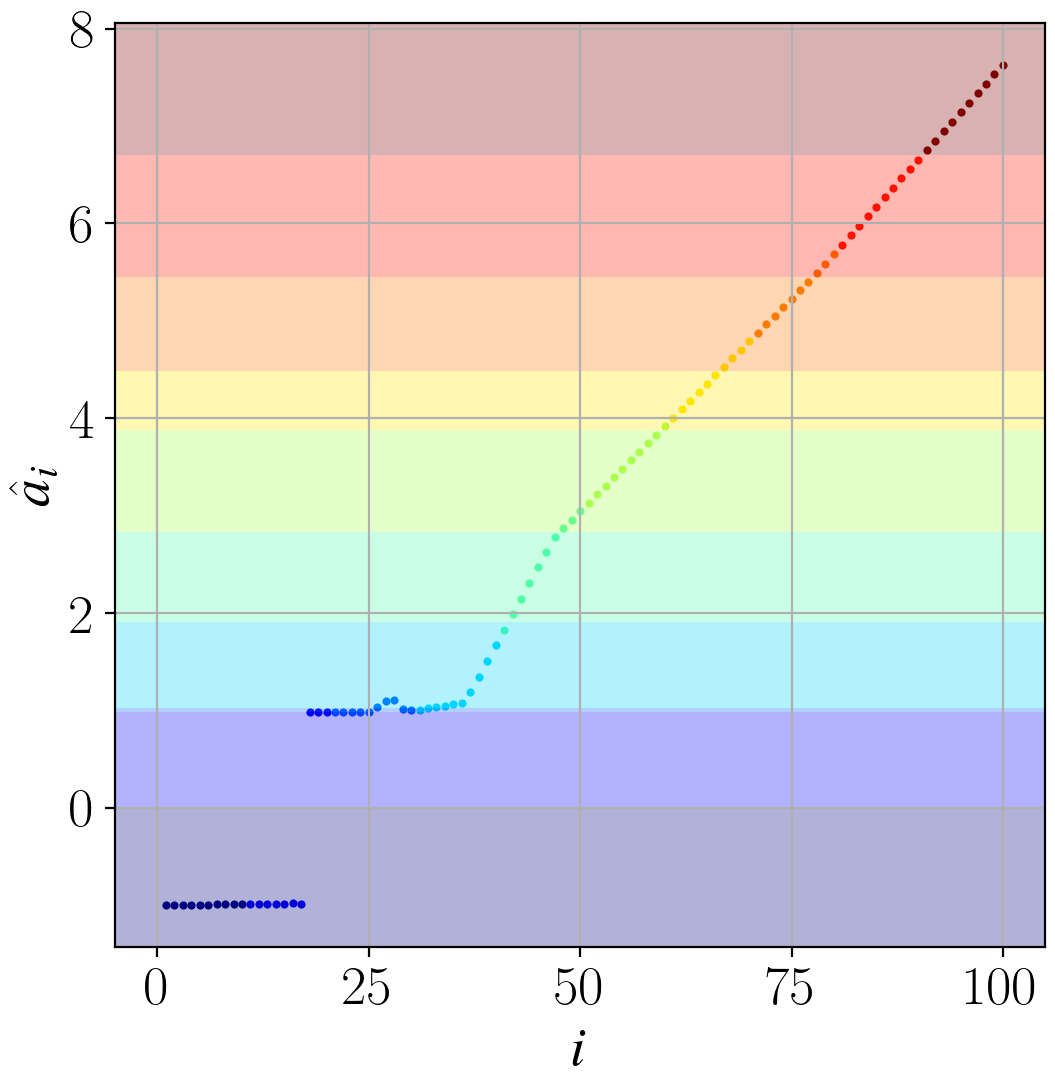}~~~~&
\includegraphics[height=3cm, bb=0 0 404 386]{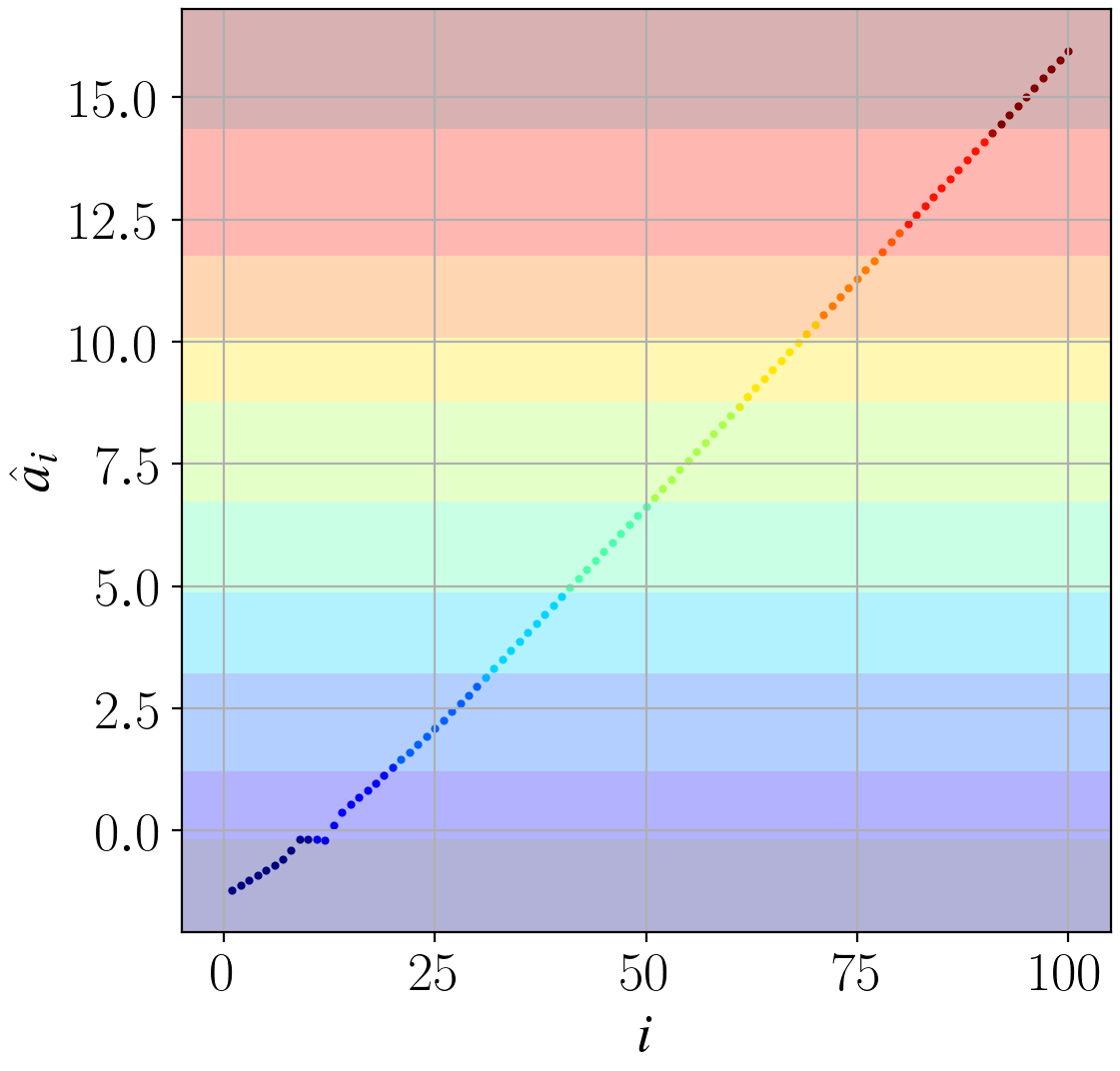}
\end{tabular}\\
\includegraphics[width=15cm, bb=0 0 2520 58]{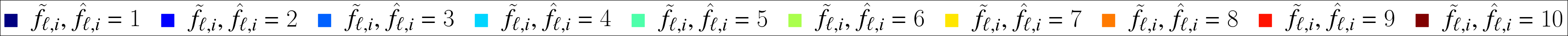}
\caption{%
Learned 1DT value $\hat{a}_i$ and optimal label prediction $\tilde{f}_{\ell,i}$ (color of point) versus $i$, 
and method's label prediction (lightened background color), which represents $\hat{f}_{\ell,i}$,
with $i$-th smallest point $x^{[i]}$ in the domain of $X$, 
in a certain trial for the experiment in Section~\ref{sec:SDE} with Task-Z.
For example, 3 plates of \protect\hyl{o1} show 3 results 
for H-1 with logi-AT-O, -IT-N, and -IT-O from top to bottom.}
\label{fig:SDERes}
\end{figure}

\subsection{Results for Real-World Data Experiment in Section~\protect\ref{sec:RWDE}}
\label{sec:ApeRWDE}
Table~\ref{tab:FAE} shows
the mean and SD of the test prediction errors (as `$\text{mean}_{\text{SD}}$')
of MLR and 21 threshold methods for the experiment in Section~\ref{sec:RWDE}.
This table is colored according to the same rule as 
that for Tables~\ref{tab:SDERes-H}--\ref{tab:SDERes-N}.

\begin{table}[H]
\renewcommand{\arraystretch}{0.75}\renewcommand{\tabcolsep}{5pt}\centering%
\caption{Results for the experiment in Section~\ref{sec:RWDE}.}
\label{tab:FAE}\scalebox{0.75}{\begin{tabular}{cc|ccc|ccc|ccc}\toprule
&&\multicolumn{3}{c|}{\scs MORPH-2}&\multicolumn{3}{c|}{\scs CACD}&\multicolumn{3}{c}{\scs AFAD}\\
&&{\scs MZE}&{\scs MAE}&{\scs RMSE}&{\scs MZE}&{\scs MAE}&{\scs RMSE}&{\scs MZE}&{\scs MAE}&{\scs RMSE}\\
\midrule
\multicolumn{2}{c|}{\scs MLR}&\tcb{$.878_{.006}$}&\tcb{$3.218_{.092}$}&\tcb{$4.303_{.121}$}&\tcb{$.931_{.003}$}&\tcb{$5.541_{.110}$}&\tcb{$7.578_{.120}$}&\tcr{$.873_{.003}$}&$3.402_{.043}$&$4.564_{.051}$\\
\midrule\multirow{7}{*}{\rotatebox{90}{\tiny AT-O\hspace{1em}}}
&{\scs logi}&\tcb{$.882_{.003}$}&\tcb{$3.151_{.097}$}&\tcb{$4.176_{.126}$}&\tcb{$.937_{.001}$}&\tcb{$5.649_{.055}$}&\tcb{$7.629_{.046}$}&$.879_{.001}$&$3.383_{.026}$&\tcr{$4.529_{.026}$}\\
&{\scs hing}&\tcb{$.882_{.002}$}&\tcb{$3.091_{.047}$}&\tcb{$4.105_{.062}$}&\tcb{$.931_{.001}$}&\tcb{$5.424_{.033}$}&\tcb{$7.506_{.047}$}&\tcb{$.885_{.002}$}&\tcb{$3.532_{.016}$}&\tcb{$4.756_{.017}$}\\
&{\scs smhi}&\tcb{$.878_{.004}$}&\tcb{$3.060_{.039}$}&\tcb{$4.075_{.050}$}&\tcb{$.933_{.001}$}&\tcb{$5.526_{.051}$}&\tcb{$7.553_{.056}$}&$.878_{.002}$&$3.376_{.020}$&$4.531_{.019}$\\
&{\scs sqhi}&\tcb{$.880_{.005}$}&\tcb{$3.113_{.053}$}&\tcb{$4.155_{.057}$}&\tcb{$.936_{.002}$}&\tcb{$5.707_{.026}$}&\tcb{$7.690_{.035}$}&$.879_{.002}$&$3.399_{.035}$&$4.556_{.033}$\\
&{\scs expo}&\tcb{$.886_{.007}$}&\tcb{$3.206_{.176}$}&\tcb{$4.231_{.221}$}&\tcb{$.943_{.002}$}&\tcb{$6.187_{.117}$}&\tcb{$8.110_{.124}$}&\tcb{$.883_{.004}$}&\tcb{$3.508_{.096}$}&\tcb{$4.655_{.105}$}\\
&{\scs abso}&\tcb{$.926_{.018}$}&\tcb{$5.327_{1.773}$}&\tcb{$6.744_{2.061}$}&\tcb{$.963_{.005}$}&\tcb{$8.667_{.646}$}&\tcb{$10.588_{.604}$}&\tcb{$.915_{.007}$}&\tcb{$4.771_{.262}$}&\tcb{$6.028_{.278}$}\\
&{\scs squa}&\tcr{$.869_{.003}$}&\tcr{$2.818_{.026}$}&\tcr{$3.824_{.043}$}&\tcr{$.923_{.000}$}&\tcr{$5.120_{.042}$}&\tcr{$7.346_{.061}$}&$.879_{.002}$&\tcr{$3.376_{.013}$}&\tcb{$4.559_{.019}$}\\
\midrule\multirow{7}{*}{\rotatebox{90}{\tiny IT-N\hspace{1em}}}
&{\scs logi}&\tcb{$.894_{.005}$}&\tcb{$3.685_{.113}$}&\tcb{$4.884_{.149}$}&\tcb{$.943_{.002}$}&\tcb{$6.376_{.074}$}&\tcb{$8.368_{.063}$}&\tcb{$.883_{.002}$}&\tcb{$3.775_{.020}$}&\tcb{$4.968_{.025}$}\\
&{\scs hing}&\tcb{$.891_{.003}$}&\tcb{$3.623_{.123}$}&\tcb{$4.794_{.167}$}&\tcb{$.967_{.002}$}&\tcb{$9.304_{.456}$}&\tcb{$11.150_{.411}$}&\tcb{$.908_{.004}$}&\tcb{$4.690_{.125}$}&\tcb{$5.940_{.131}$}\\
&{\scs smhi}&\tcb{$.891_{.003}$}&\tcb{$3.517_{.070}$}&\tcb{$4.671_{.096}$}&\tcb{$.942_{.001}$}&\tcb{$6.253_{.120}$}&\tcb{$8.219_{.116}$}&\tcb{$.884_{.002}$}&\tcb{$3.751_{.022}$}&\tcb{$4.946_{.010}$}\\
&{\scs sqhi}&\tcb{$.896_{.003}$}&\tcb{$3.700_{.094}$}&\tcb{$4.896_{.133}$}&\tcb{$.947_{.003}$}&\tcb{$6.662_{.220}$}&\tcb{$8.630_{.219}$}&\tcb{$.883_{.002}$}&\tcb{$3.747_{.043}$}&\tcb{$4.940_{.039}$}\\
&{\scs expo}&\tcb{$.904_{.002}$}&\tcb{$4.276_{.190}$}&\tcb{$5.627_{.242}$}&\tcb{$.947_{.002}$}&\tcb{$6.571_{.116}$}&\tcb{$8.543_{.119}$}&\tcb{$.889_{.001}$}&\tcb{$3.955_{.095}$}&\tcb{$5.170_{.100}$}\\
&{\scs abso}&\tcb{$.892_{.003}$}&\tcb{$3.643_{.098}$}&\tcb{$4.834_{.143}$}&\tcb{$.970_{.002}$}&\tcb{$9.634_{.220}$}&\tcb{$11.442_{.187}$}&\tcb{$.912_{.005}$}&\tcb{$4.780_{.109}$}&\tcb{$6.032_{.110}$}\\
&{\scs squa}&\tcb{$.894_{.003}$}&\tcb{$3.703_{.079}$}&\tcb{$4.905_{.112}$}&\tcb{$.945_{.001}$}&\tcb{$6.486_{.076}$}&\tcb{$8.457_{.073}$}&\tcb{$.884_{.002}$}&\tcb{$3.749_{.049}$}&\tcb{$4.943_{.050}$}\\
\midrule\multirow{7}{*}{\rotatebox{90}{\tiny IT-O\hspace{1em}}}
&{\scs logi}&\tcb{$.899_{.004}$}&\tcb{$4.015_{.157}$}&\tcb{$5.310_{.206}$}&\tcb{$.944_{.001}$}&\tcb{$6.432_{.125}$}&\tcb{$8.407_{.121}$}&\tcb{$.883_{.003}$}&\tcb{$3.679_{.046}$}&\tcb{$4.861_{.053}$}\\
&{\scs hing}&\tcb{$.895_{.007}$}&\tcb{$3.573_{.172}$}&\tcb{$4.715_{.198}$}&\tcb{$.967_{.003}$}&\tcb{$9.174_{.403}$}&\tcb{$11.037_{.365}$}&\tcb{$.905_{.008}$}&\tcb{$4.445_{.270}$}&\tcb{$5.686_{.281}$}\\
&{\scs smhi}&\tcb{$.893_{.005}$}&\tcb{$3.617_{.145}$}&\tcb{$4.787_{.201}$}&\tcb{$.945_{.003}$}&\tcb{$6.511_{.246}$}&\tcb{$8.479_{.239}$}&\tcr{$.878_{.002}$}&\tcb{$3.579_{.031}$}&\tcb{$4.760_{.039}$}\\
&{\scs sqhi}&\tcb{$.902_{.003}$}&\tcb{$4.068_{.146}$}&\tcb{$5.382_{.196}$}&\tcb{$.945_{.001}$}&\tcb{$6.492_{.080}$}&\tcb{$8.463_{.076}$}&\tcb{$.882_{.002}$}&\tcb{$3.656_{.046}$}&\tcb{$4.840_{.048}$}\\
&{\scs expo}&\tcb{$.895_{.004}$}&\tcb{$3.780_{.110}$}&\tcb{$5.010_{.126}$}&\tcb{$.945_{.002}$}&\tcb{$6.575_{.191}$}&\tcb{$8.551_{.200}$}&\tcb{$.887_{.004}$}&\tcb{$3.849_{.099}$}&\tcb{$5.055_{.104}$}\\
&{\scs abso}&\tcb{$.888_{.003}$}&\tcb{$3.540_{.068}$}&\tcb{$4.683_{.087}$}&\tcb{$.966_{.007}$}&\tcb{$9.140_{.953}$}&\tcb{$10.966_{.888}$}&\tcb{$.918_{.002}$}&\tcb{$4.899_{.013}$}&\tcb{$6.161_{.013}$}\\
&{\scs squa}&\tcb{$.901_{.006}$}&\tcb{$3.965_{.099}$}&\tcb{$5.234_{.143}$}&\tcb{$.946_{.001}$}&\tcb{$6.596_{.103}$}&\tcb{$8.556_{.117}$}&\tcb{$.883_{.002}$}&\tcb{$3.692_{.057}$}&\tcb{$4.874_{.054}$}\\
\bottomrule
\end{tabular}}\end{table}

Table~\ref{tab:Bias} shows the mean and SD of the number of unique elements of 
the learned bias parameter vector for the experiment in Section~\ref{sec:RWDE}.
\begin{table}[H]
\renewcommand{\arraystretch}{0.75}\renewcommand{\tabcolsep}{5pt}\centering%
\caption{Mean and SD of the number of unique elements of the learned 
bias parameter vector for the experiment in Section~\ref{sec:RWDE}.}
\label{tab:Bias}
\begin{minipage}[H]{0.29\columnwidth}\centering
\scalebox{0.75}{\begin{tabular}{cc|c|c|c}\toprule
&&{\scs MORPH-2}&{\scs CACD}&{\scs AFAD}\\
\midrule\multirow{7}{*}{\rotatebox{90}{\tiny AT-O\hspace{1em}}}
&{\scs logi}&$54.0_{.000}$&$48.0_{.000}$&$25.0_{.000}$\\
&{\scs hing}&$54.0_{.000}$&$47.8_{.400}$&$24.4_{.800}$\\
&{\scs smhi}&$54.0_{.000}$&$48.0_{.000}$&$25.0_{.000}$\\
&{\scs sqhi}&$54.0_{.000}$&$48.0_{.000}$&$25.0_{.000}$\\
&{\scs expo}&$54.0_{.000}$&$48.0_{.000}$&$25.0_{.000}$\\
&{\scs abso}&$18.6_{3.878}$&$14.0_{1.414}$&$9.6_{1.356}$\\
&{\scs squa}&$49.2_{1.327}$&$48.0_{.000}$&$25.0_{.000}$\\
\bottomrule\end{tabular}}\end{minipage}
\hspace{0.03\columnwidth}
\begin{minipage}[H]{0.29\columnwidth}\centering
\scalebox{0.75}{\begin{tabular}{cc|c|c|c}\toprule
&&{\scs MORPH-2}&{\scs CACD}&{\scs AFAD}\\
\midrule\multirow{7}{*}{\rotatebox{90}{\tiny IT-N\hspace{1em}}}
&{\scs logi}&$54.0_{.000}$&$48.0_{.000}$&$25.0_{.000}$\\
&{\scs hing}&$54.0_{.000}$&$48.0_{.000}$&$25.0_{.000}$\\
&{\scs smhi}&$54.0_{.000}$&$48.0_{.000}$&$25.0_{.000}$\\
&{\scs sqhi}&$54.0_{.000}$&$48.0_{.000}$&$25.0_{.000}$\\
&{\scs expo}&$54.0_{.000}$&$48.0_{.000}$&$25.0_{.000}$\\
&{\scs abso}&$54.0_{.000}$&$48.0_{.000}$&$25.0_{.000}$\\
&{\scs squa}&$54.0_{.000}$&$48.0_{.000}$&$25.0_{.000}$\\
\bottomrule\end{tabular}}\end{minipage}
\hspace{0.02\columnwidth}
\begin{minipage}[H]{0.29\columnwidth}\centering
\scalebox{0.75}{\begin{tabular}{cc|c|c|c}\toprule
&&{\scs MORPH-2}&{\scs CACD}&{\scs AFAD}\\
\midrule\multirow{7}{*}{\rotatebox{90}{\tiny IT-O\hspace{1em}}}
&{\scs logi}&$42.4_{3.499}$&$28.6_{1.960}$&$8.0_{.000}$\\
&{\scs hing}&$35.8_{4.792}$&$35.4_{10.131}$&$6.2_{.400}$\\
&{\scs smhi}&$42.6_{5.389}$&$27.8_{.748}$&$8.0_{.000}$\\
&{\scs sqhi}&$48.6_{3.611}$&$28.8_{2.638}$&$8.0_{.000}$\\
&{\scs expo}&$40.4_{5.276}$&$32.2_{.748}$&$9.2_{.980}$\\
&{\scs abso}&$26.2_{2.857}$&$36.4_{8.499}$&$6.4_{.490}$\\
&{\scs squa}&$45.0_{4.817}$&$29.6_{1.960}$&$8.0_{.000}$\\
\bottomrule\end{tabular}}\end{minipage}
\hspace{0.02\columnwidth}
\end{table}

Figure~\ref{fig:Violin} gives plots of $(y_i,\hat{a}(\bx_i))$'s for 
representative settings of the experiment in Section~\ref{sec:RWDE}.
In results with the Hinge-IT loss in Figure~\ref{fig:Violin} \hyl{q5} and \hyl{q6},
learned 1DT values $\hat{a}(\bx_i)$'s are concentrated at 1 point.
\begin{figure}[H]
\renewcommand{\arraystretch}{0.1}\renewcommand{\tabcolsep}{5pt}\centering%
\begin{tabular}{rccc}\hskip-6pt
&\hskip-6pt
{\tiny\hyt{q1} MORPH-2, logi}&\hskip-6pt
{\tiny\hyt{q2} CACD, logi}&\hskip-6pt
{\tiny\hyt{q3} AFAD, logi}\\\hskip-6pt
\rotatebox{90}{\tiny~~~~~~~~~~~~AT-O}&\hskip-6pt
\includegraphics[width=4.8cm, bb=0 0 946 442]{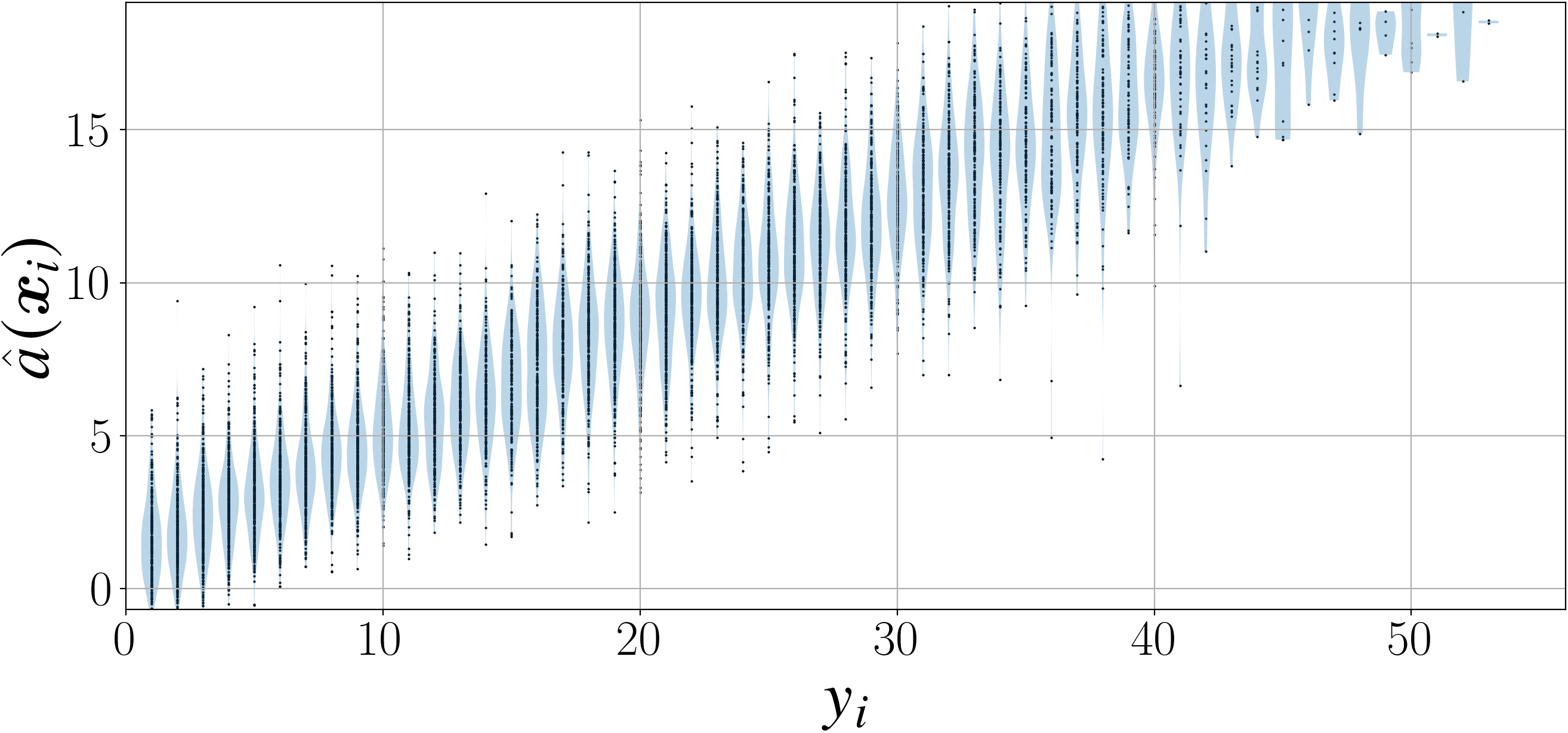}&\hskip-6pt
\includegraphics[width=4.8cm, bb=0 0 946 442]{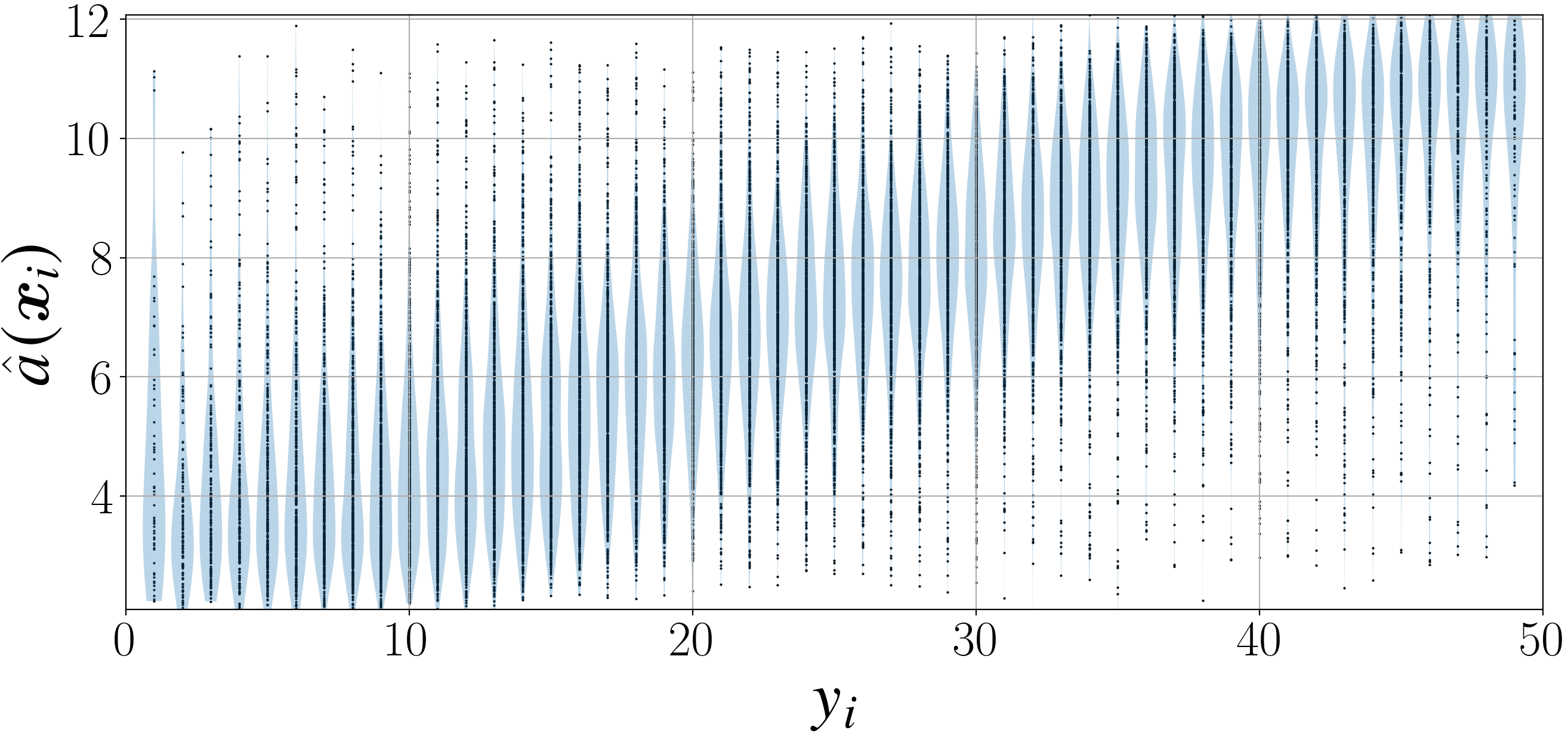}&\hskip-6pt
\includegraphics[width=4.8cm, bb=0 0 946 442]{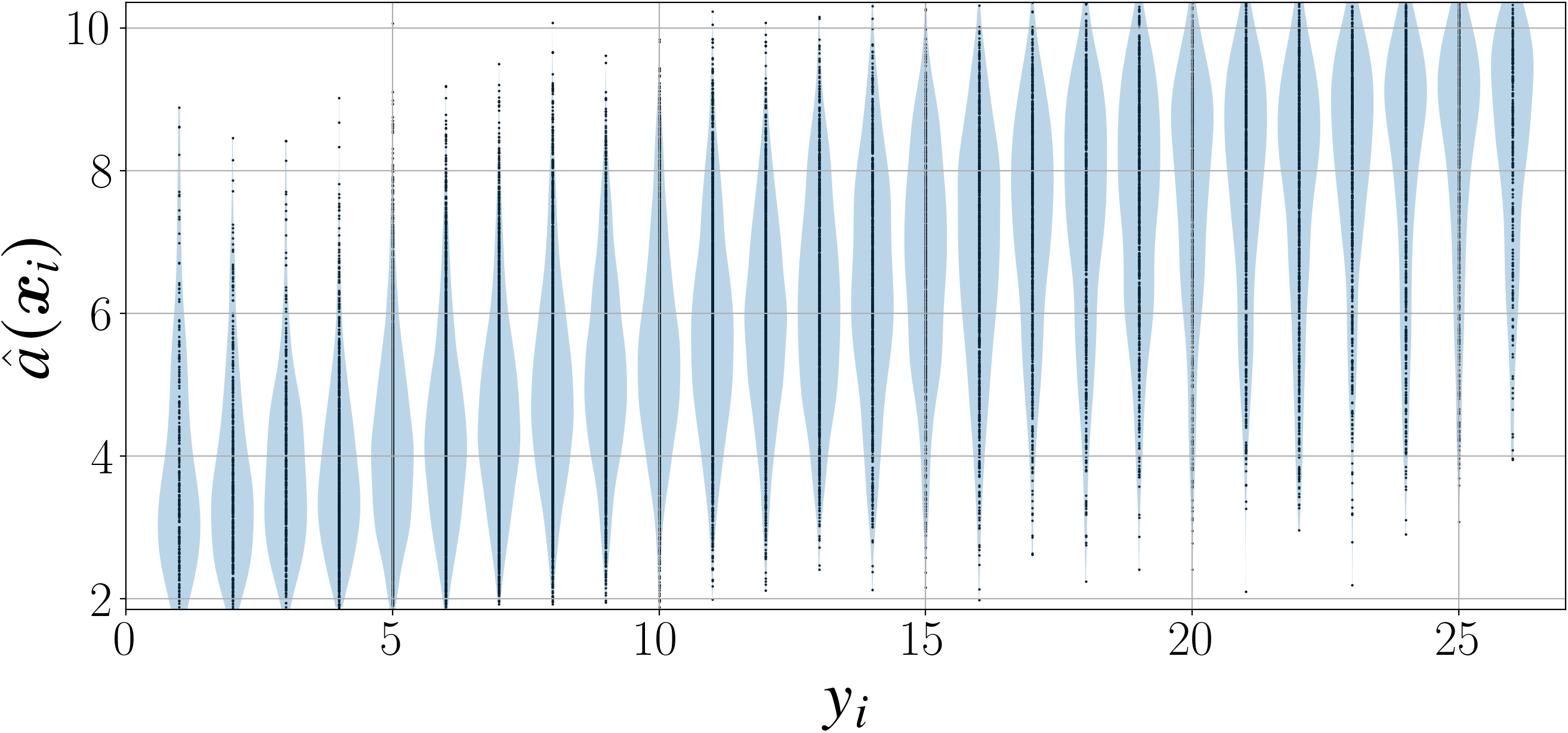}\\\hskip-6pt
\rotatebox{90}{\tiny~~~~~~~~~~~~IT-N}&\hskip-6pt
\includegraphics[width=4.8cm, bb=0 0 946 442]{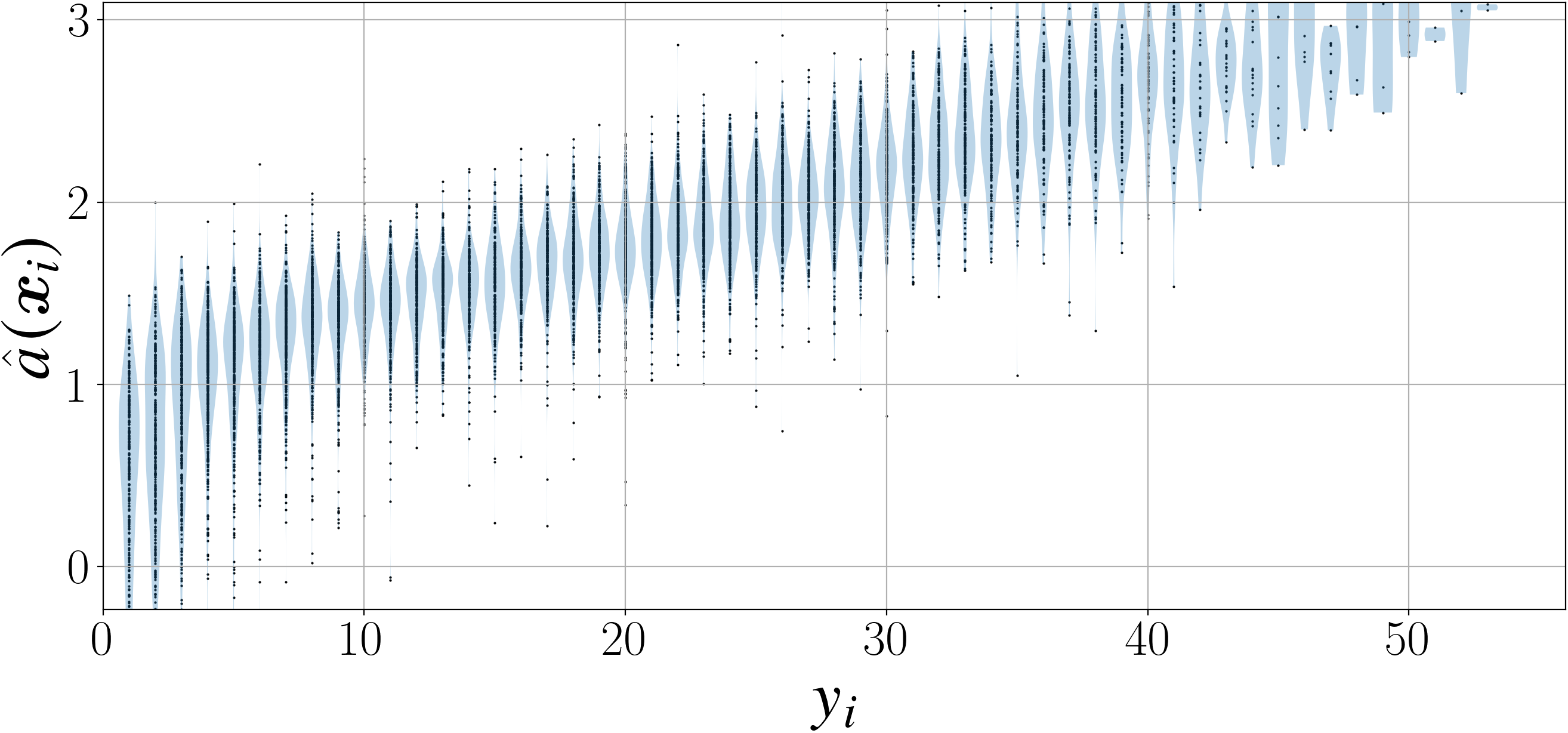}&\hskip-6pt
\includegraphics[width=4.8cm, bb=0 0 946 442]{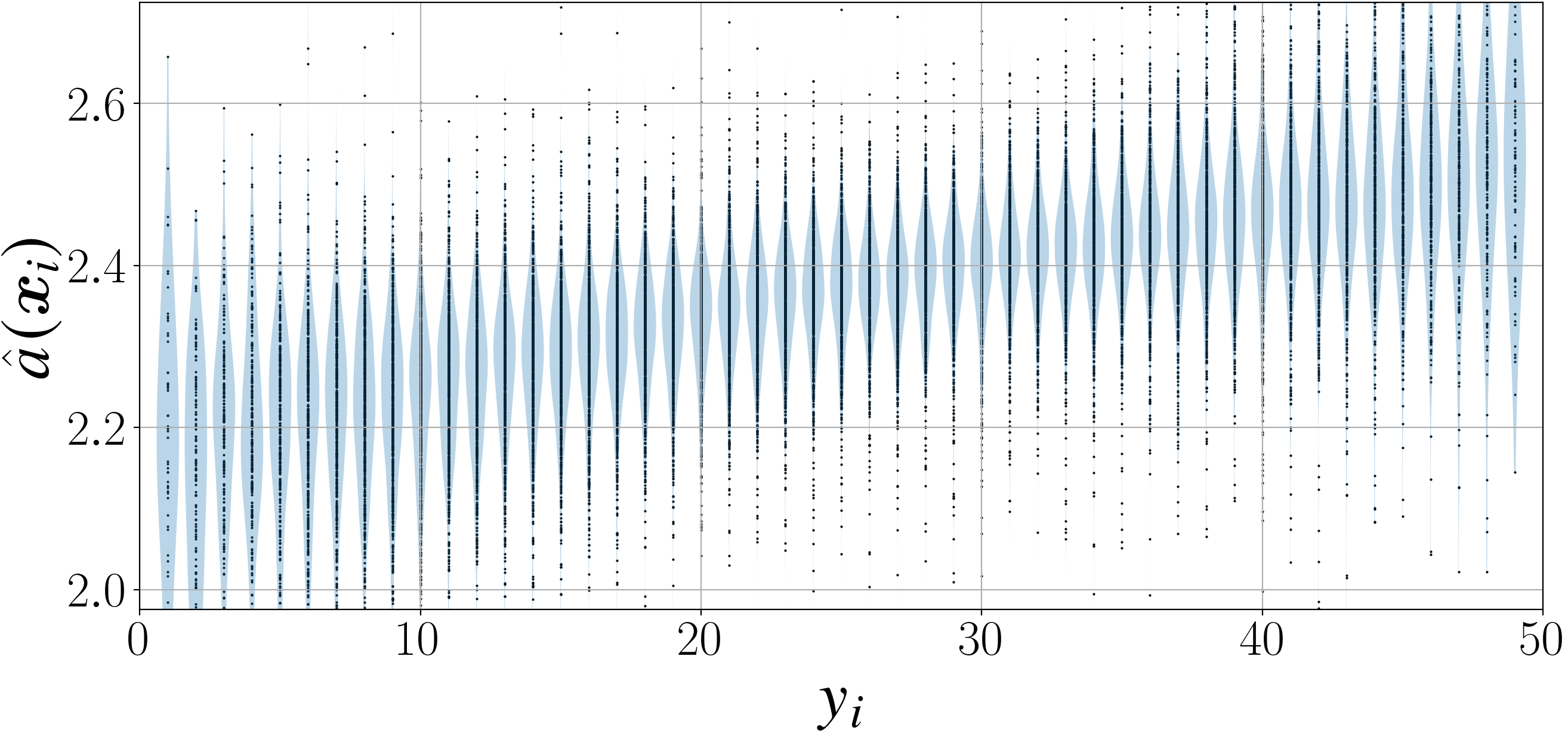}&\hskip-6pt
\includegraphics[width=4.8cm, bb=0 0 946 442]{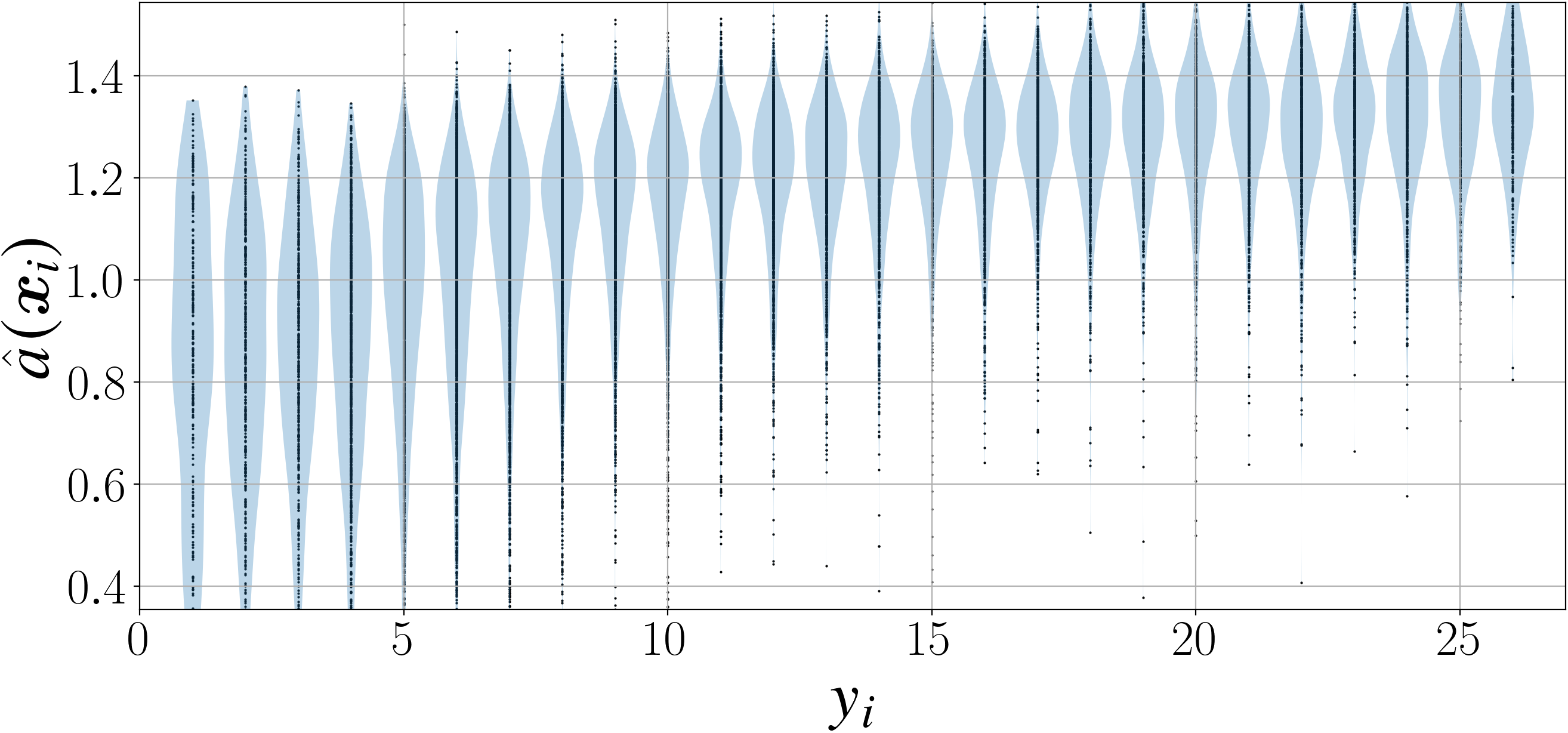}\\\hskip-6pt
\rotatebox{90}{\tiny~~~~~~~~~~~~IT-O}&\hskip-6pt
\includegraphics[width=4.8cm, bb=0 0 946 442]{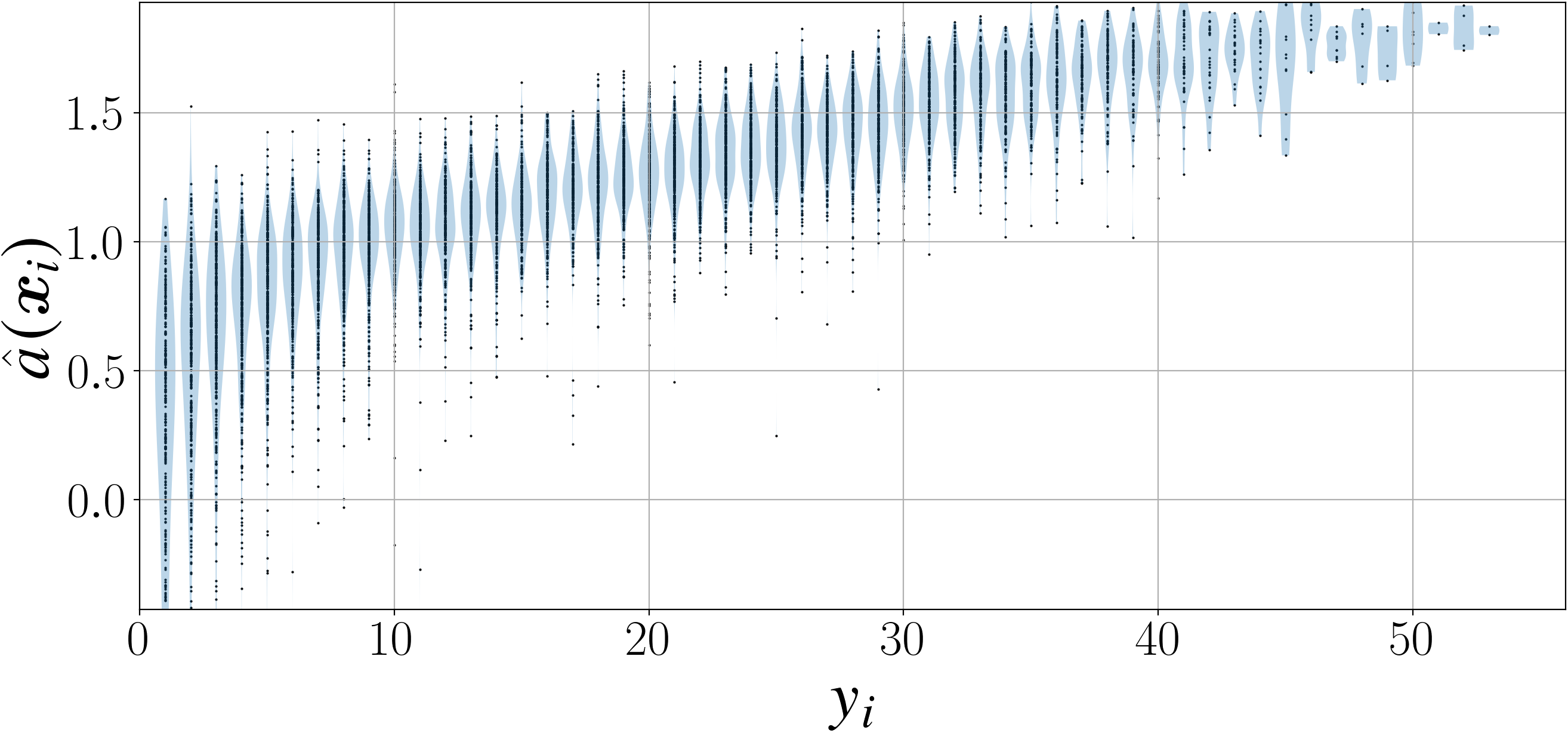}&\hskip-6pt
\includegraphics[width=4.8cm, bb=0 0 946 442]{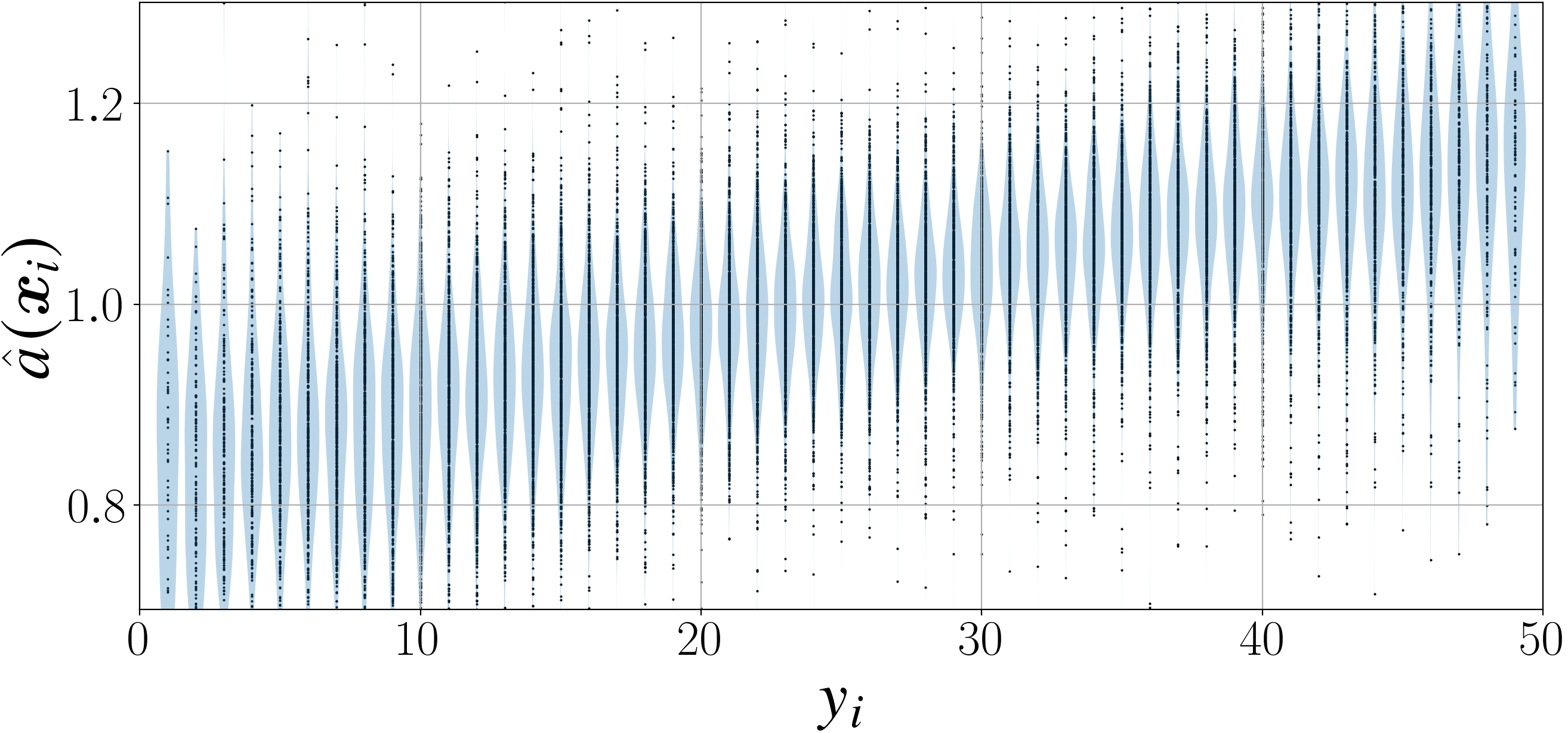}&\hskip-6pt
\includegraphics[width=4.8cm, bb=0 0 946 442]{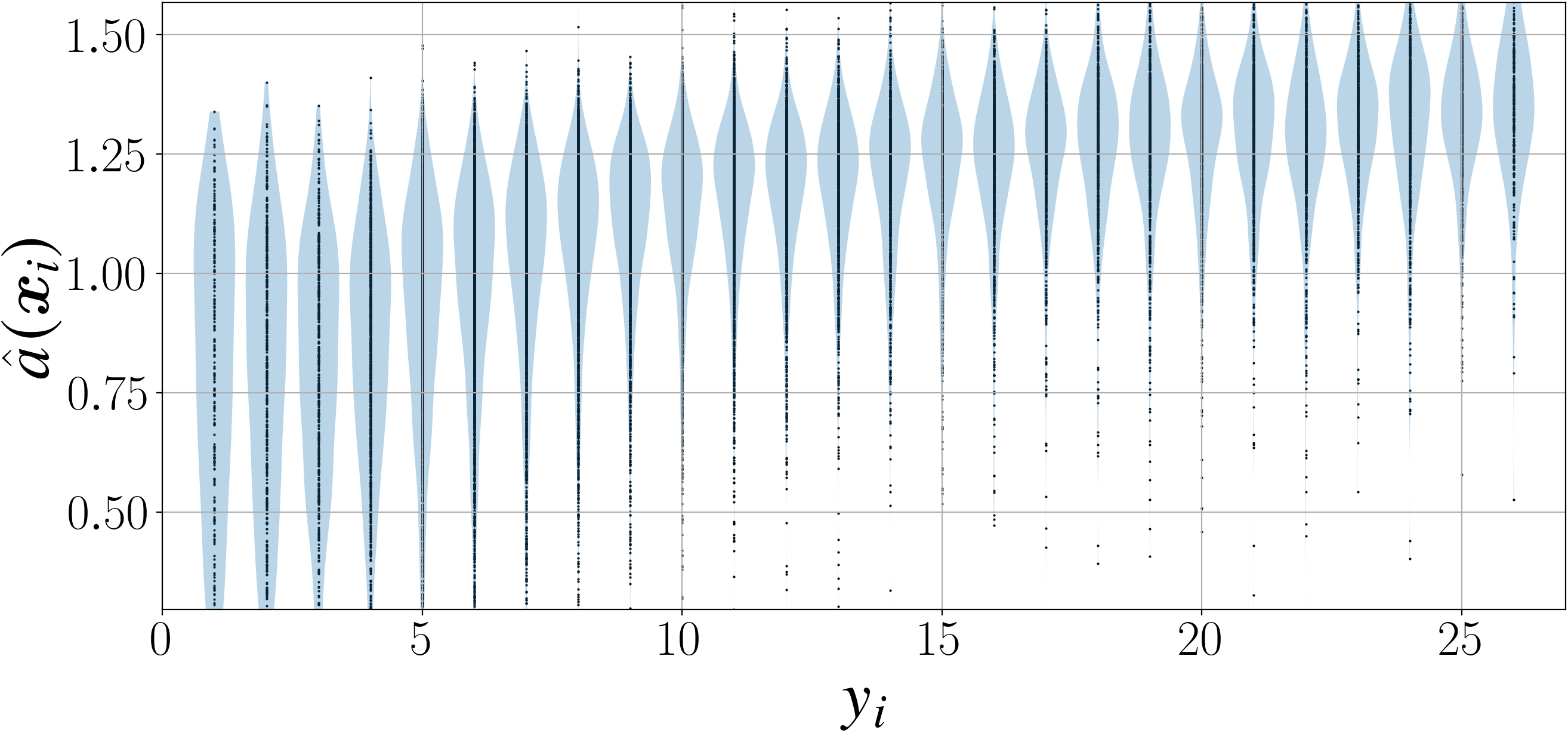}\\\hskip-6pt
&\hskip-6pt
{\tiny\hyt{q4} MORPH-2, hing}&\hskip-6pt
{\tiny\hyt{q5} CACD, hing}&\hskip-6pt
{\tiny\hyt{q6} AFAD, hing}\\\hskip-6pt
\rotatebox{90}{\tiny~~~~~~~~~~~~AT-O}&\hskip-6pt
\includegraphics[width=4.8cm, bb=0 0 946 442]{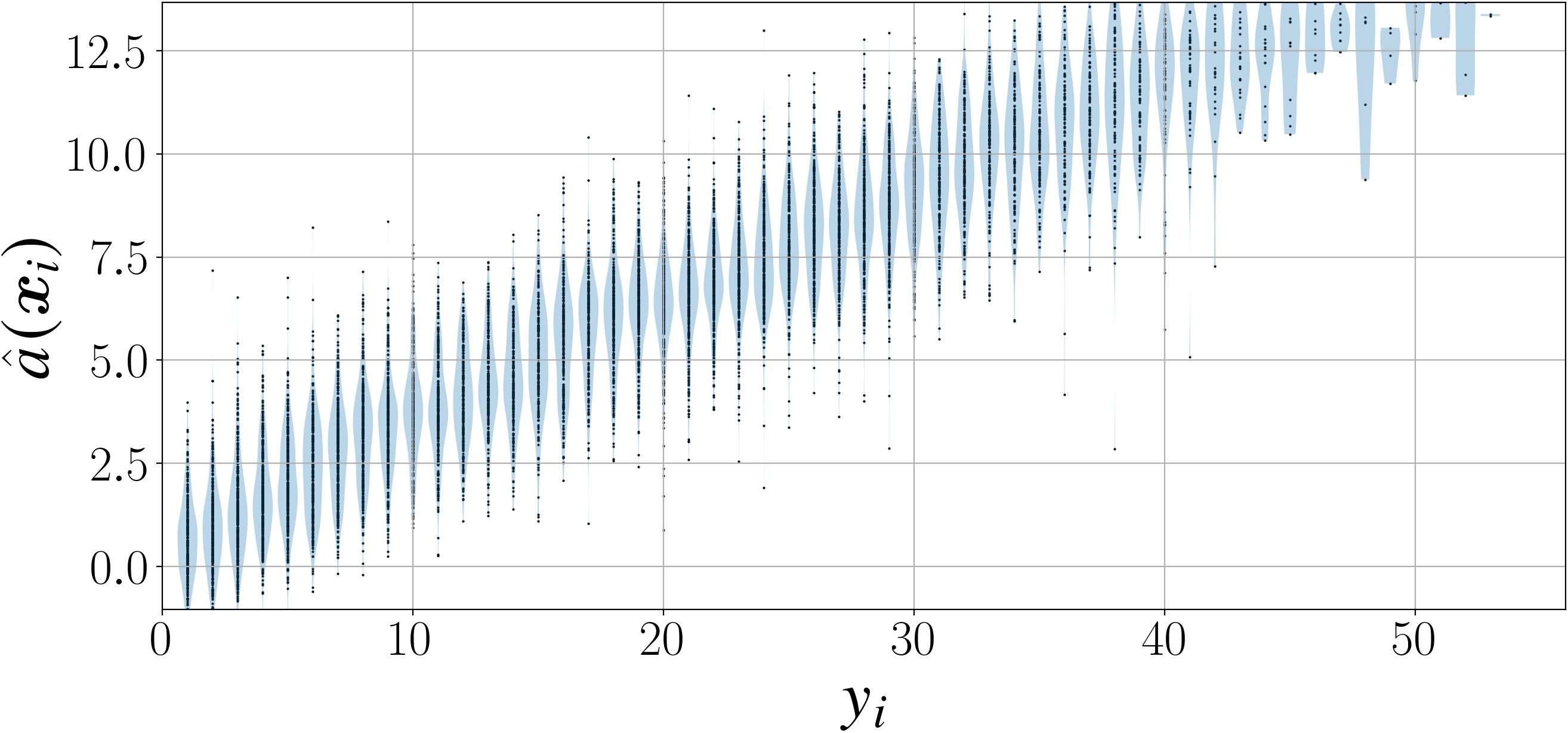}&\hskip-6pt
\includegraphics[width=4.8cm, bb=0 0 946 442]{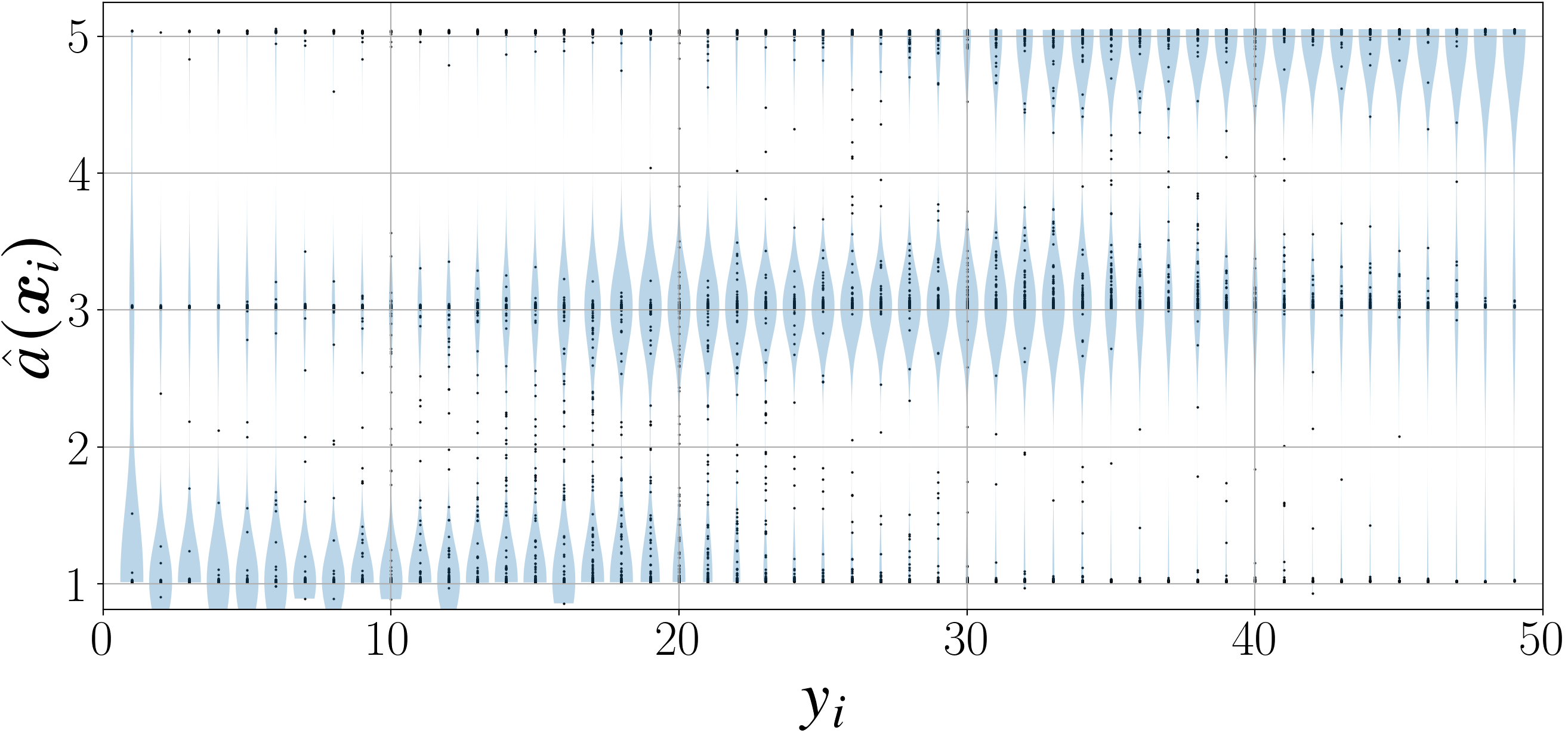}&\hskip-6pt
\includegraphics[width=4.8cm, bb=0 0 946 442]{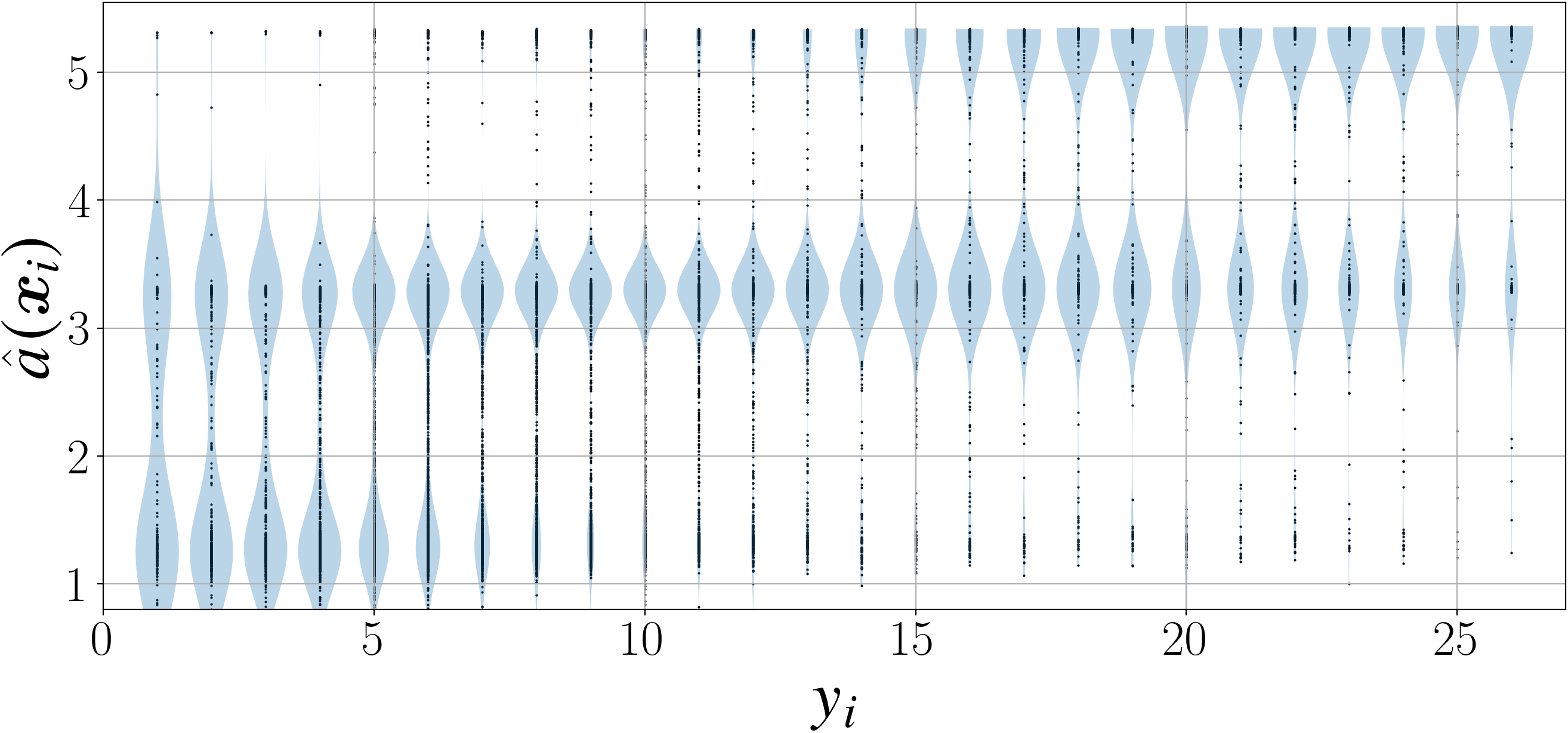}\\\hskip-6pt
\rotatebox{90}{\tiny~~~~~~~~~~~~IT-N}&\hskip-6pt
\includegraphics[width=4.8cm, bb=0 0 946 442]{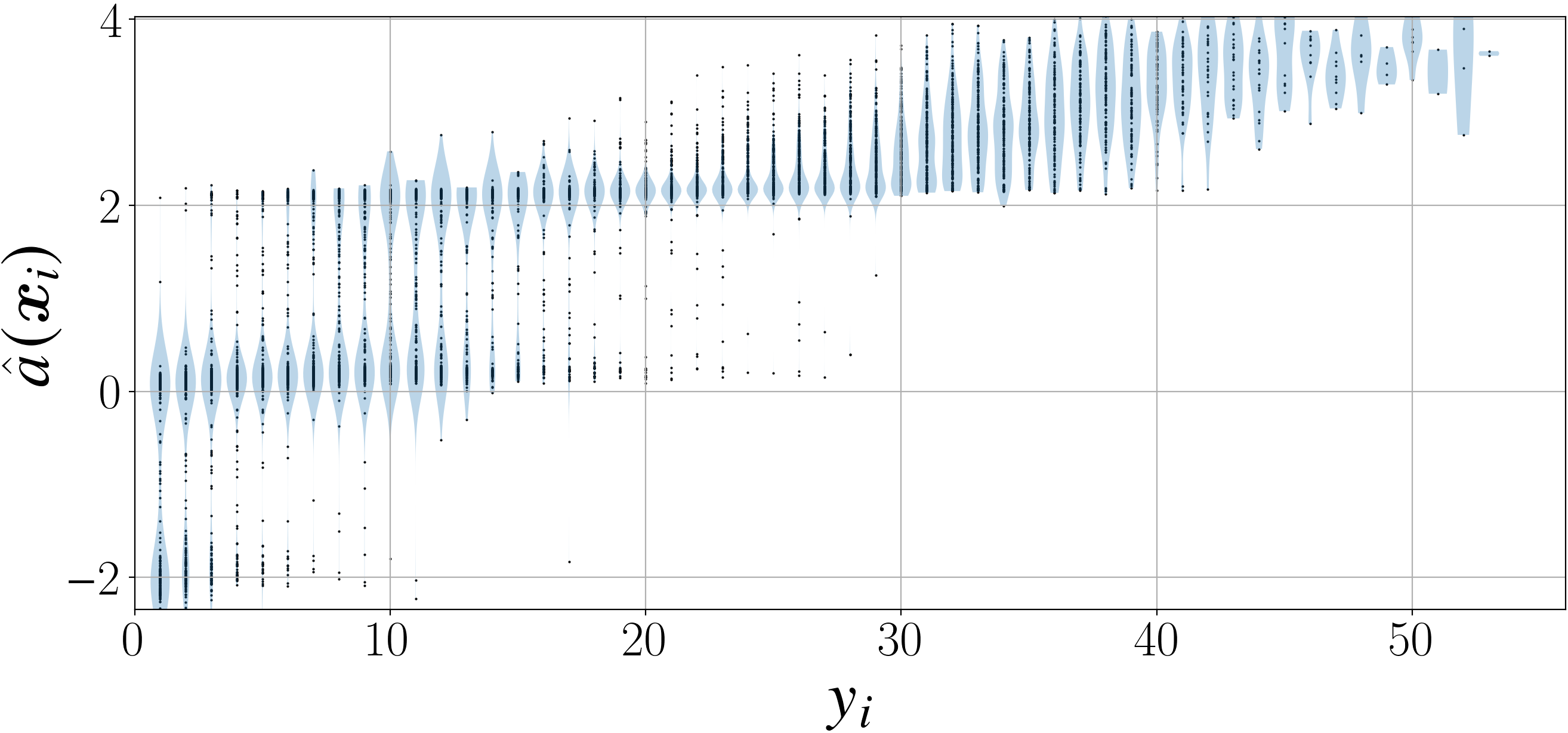}&\hskip-6pt
\includegraphics[width=4.8cm, bb=0 0 946 442]{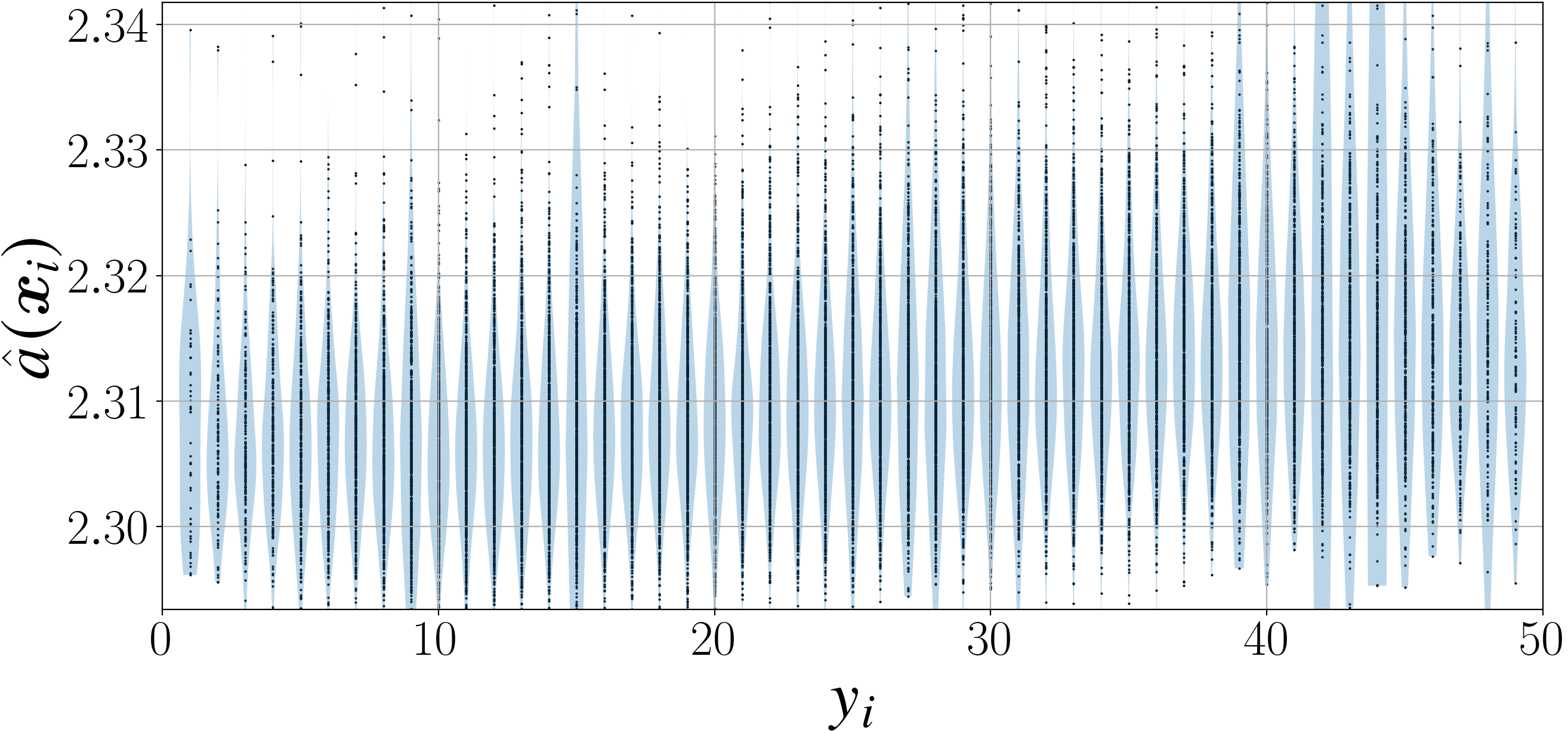}&\hskip-6pt
\includegraphics[width=4.8cm, bb=0 0 946 442]{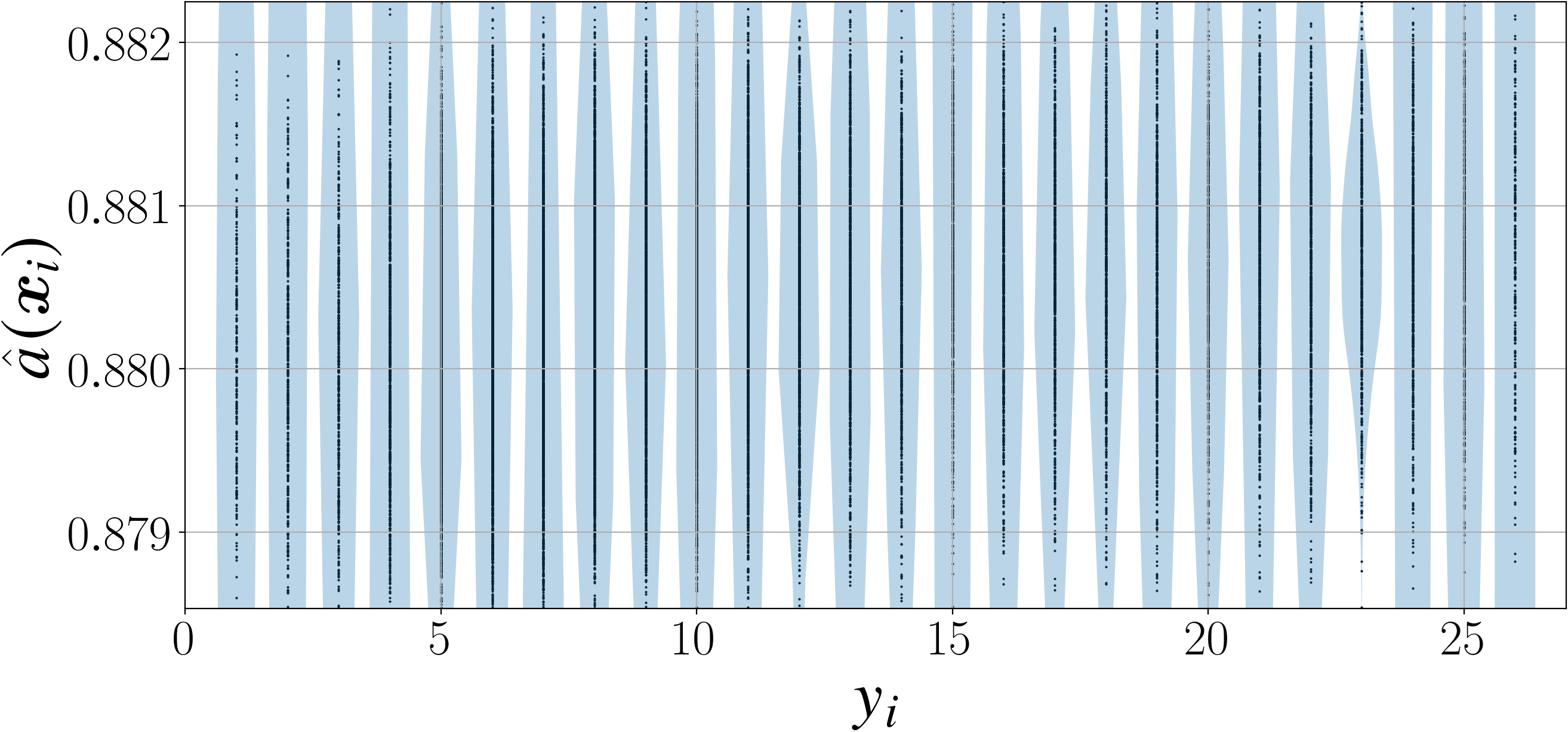}\\\hskip-6pt
\rotatebox{90}{\tiny~~~~~~~~~~~~IT-O}&\hskip-6pt
\includegraphics[width=4.8cm, bb=0 0 946 442]{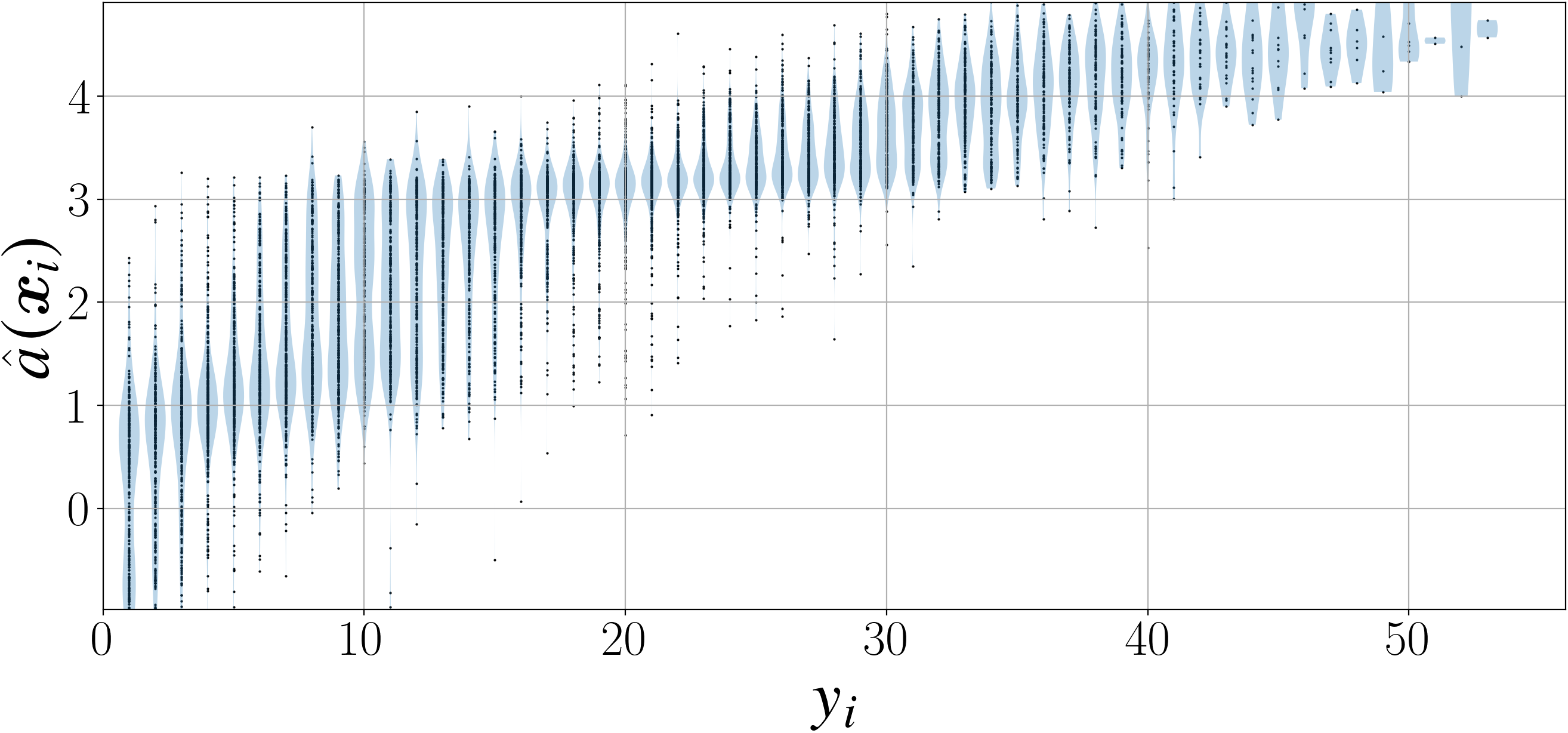}&\hskip-6pt
\includegraphics[width=4.8cm, bb=0 0 946 442]{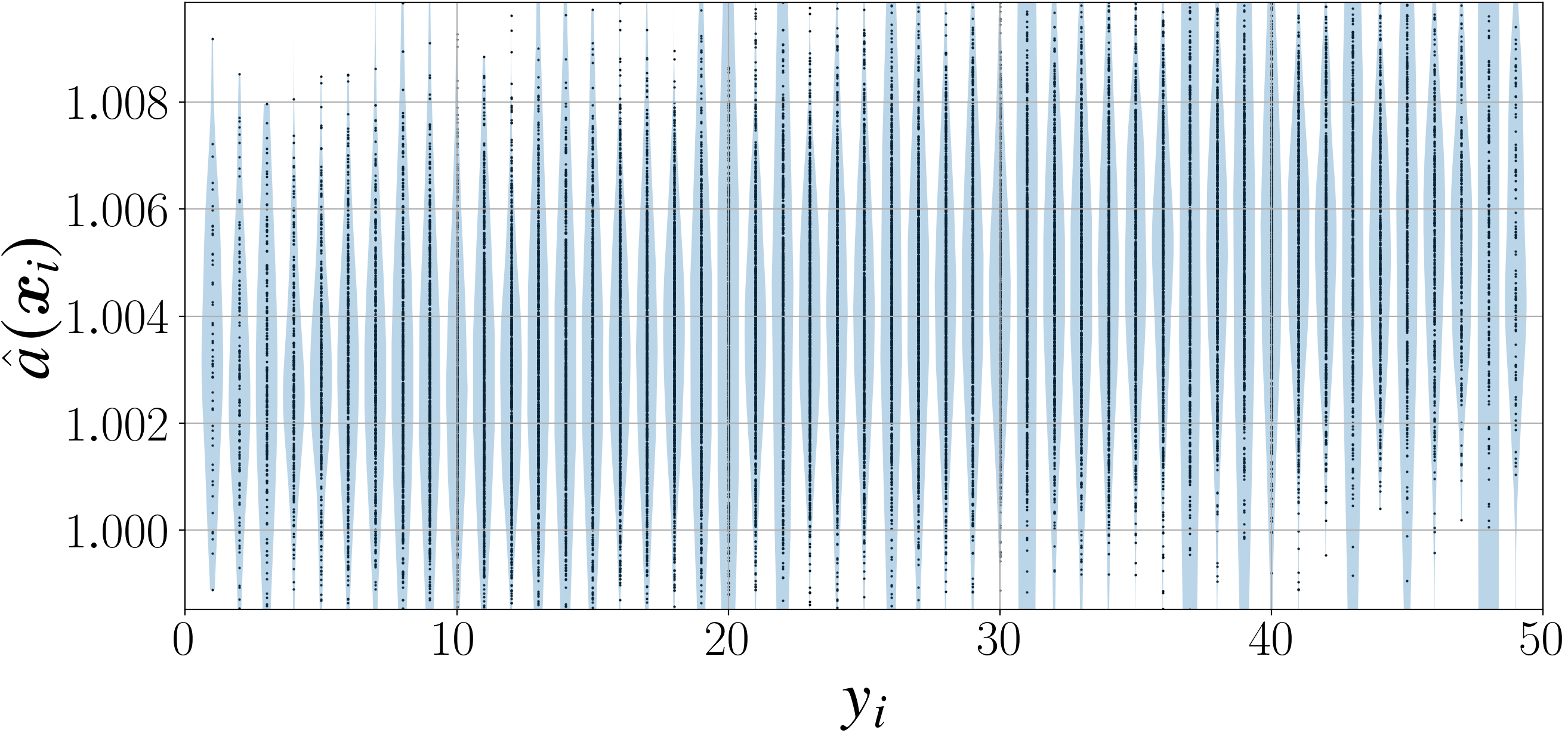}&\hskip-6pt
\includegraphics[width=4.8cm, bb=0 0 946 442]{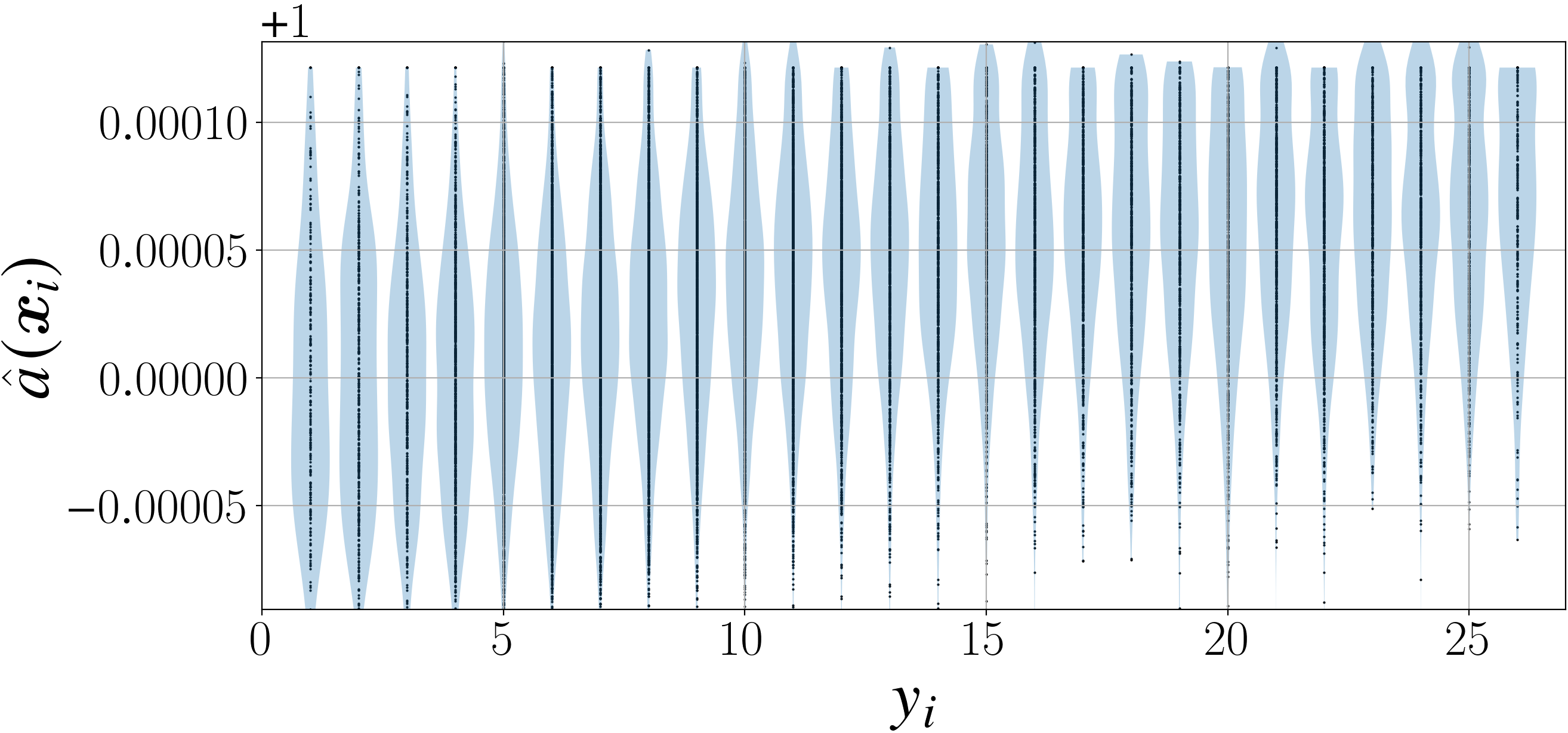}
\end{tabular}
\caption{%
Scatter and violin plots of $(y_i,\hat{a}(\bx_i))$'s for test data 
for the experiment in Section~\ref{sec:RWDE}.
The vertical axis is restricted to $[L-0.05(U-L), U+0.05(U-L)]$ with 
$1$ percentile $L$ and $99$ percentile $U$ of $\hat{a}(\bx_i)$'s.}
\label{fig:Violin}
\end{figure}

\section{Proofs of Theorems}
\label{sec:Proof}
We here provide proofs of the theoretical results presented in this paper.

\paragraph{Proof of Theorem~\protect\ref{thm:AT-BP-Ord}}
We first provide a proof of Theorem~\ref{thm:AT-BP-Ord}:
\begin{proof}[{Proof of Theorem~\ref{thm:AT-BP-Ord}}]
\textbf{Proof of \hyl{a9}~}
We first prove the relationship \hyl{a9} between the conditions \hyl{a1}--\hyl{a6}.
When $\varphi$ is non-increasing \hyl{a2}, 
$\varphi(u_1)\ge\varphi(u_2)$ and $\varphi(-u_2)\ge\varphi(-u_1)$ if $u_1<u_2$.
When $\varphi(\cdot)-\varphi(-\cdot)$ is non-increasing \hyl{a3},
$\varphi(u_1)-\varphi(-u_1)\ge\varphi(u_2)-\varphi(-u_2)$ if $u_1<u_2$.
Since both the results imply \hyl{a1},
one has `\hyl{a2} $\Rightarrow$ \hyl{a1}' and `\hyl{a3} $\Rightarrow$ \hyl{a1}'.
The statements `\hyl{a5} $\Rightarrow$ \hyl{a4}' and `\hyl{a6} $\Rightarrow$ \hyl{a4}' 
can be proved in a similar way.

In the following, we prove the population version of the statements 
\hyl{a10} and \hyl{a11} regarding the order of the optimized bias parameter vector.
Note that the surrogate risk for the AT surrogate loss is 
\begin{align}
\label{eq:ATSR}
    \begin{split}
    \bbE[\phi(a(\bX),\bb,Y)]
    &=\bbE\biggl[\sum_{y=1}^K\Pr(Y=y|\bX)\biggl\{\sum_{k=1}^{y-1}\varphi(a(\bX)-b_k)+\sum_{k=y}^{K-1}\varphi(b_k-a(\bX))\biggr\}\biggr]\\
    &=\bbE\biggl[\sum_{y=1}^{K-1}\{\Pr(Y\le y|\bX)\varphi(b_y-a(\bX))+\Pr(Y>y|\bX)\varphi(a(\bX)-b_y)\}\biggr].
    \end{split}
\end{align}

\textbf{Proof of \hyl{a10}~}
Assume that any surrogate risk minimizer $(\bar{a},\bar{\bb})$ 
has a pair $(\bar{b}_l,\bar{b}_m)$ satisfying 
$\bar{b}_l>\bar{b}_m$ with $l,m\in[K-1]$ such that $l<m$.
For $\check{\bb}=(\check{b}_k)_{k\in[K-1]}$ consisting of
\begin{align}
    \check{b}_k:=
    \begin{cases}
    \bar{b}_m&\text{if }k=l,\\
    \bar{b}_l&\text{if }k=m,\\
    \bar{b}_k&\text{otherwise},
    \end{cases}
\end{align}
one has that
\begin{align}
    \begin{split}
    &\bbE[\phi(\bar{a}(\bX),\bar{\bb},Y)]
    -\bbE[\phi(\bar{a}(\bX),\check{\bb},Y)]\\
    &=\bbE[\Pr(Y\le l|\bX)\varphi(\bar{b}_l-\bar{a}(\bX))+\Pr(Y>l|\bX)\varphi(\bar{a}(\bX)-\bar{b}_l)\\
    &\hphantom{=\bbE[}
    +\Pr(Y\le m|\bX)\varphi(\bar{b}_m-\bar{a}(\bX))+\Pr(Y>m|\bX)\varphi(\bar{a}(\bX)-\bar{b}_m)\\
    &\hphantom{=\bbE[}
    +(\text{$(\bar{b}_l,\bar{b}_m)$-independent term})]\\
    &-\bbE[\Pr(Y\le l|\bX)\varphi(\bar{b}_m-\bar{a}(\bX))+\Pr(Y>l|\bX)\varphi(\bar{a}(\bX)-\bar{b}_m)\\
    &\hphantom{=\bbE[}
    +\Pr(Y\le m|\bX)\varphi(\bar{b}_l-\bar{a}(\bX))+\Pr(Y>m|\bX)\varphi(\bar{a}(\bX)-\bar{b}_l)\\
    &\hphantom{=\bbE[}
    +(\text{$(\bar{b}_l,\bar{b}_m)$-independent term})]\\
    &=\bbE[\Pr(Y\in\{l+1,\ldots,m\}|\bX)\\
    &\hphantom{=\bbE[}
    \times\{\varphi(\bar{a}(\bX)-\bar{b}_l)-\varphi(\bar{b}_l-\bar{a}(\bX))
    -\varphi(\bar{a}(\bX)-\bar{b}_m)+\varphi(\bar{b}_m-\bar{a}(\bX))\}].
    \end{split}
\end{align}
From the assumption \hyl{a8}, 
there exists a point $\bx\in\bbR$ such that $\Pr(Y\in\{l+1,\ldots,m\}|\bX=\bx)>0$.
The condition \hyl{a1} with $u_1=\bar{a}(\bX)-\bar{b}_l$ and 
$u_2=\bar{a}(\bX)-\bar{b}_m$ ($u_1<u_2$ since $\bar{b}_l>\bar{b}_m$)
implies $\varphi(\bar{a}(\bX)-\bar{b}_l)-\varphi(\bar{b}_l-\bar{a}(\bX))
-\varphi(\bar{a}(\bX)-\bar{b}_m)+\varphi(\bar{b}_m-\bar{a}(\bX))\ge0$.
Therefore, one has $\bbE[\phi(\bar{a}(\bX),\check{\bb},Y)]\le\bbE[\phi(\bar{a}(\bX),\bar{\bb},Y)]$.
Also, as changing the bias parameter vector from $\bar{\bb}$ to $\check{\bb}$,
the total number of pairs of bias parameters 
validating the order is reduced by at least one.
Therefore, by repeating the above operations sequentially for pairs of bias parameters validating the order,
one can finally obtain the bias parameter vector satisfying the order condition
without increasing the surrogate risk.
This result contradicts to the assumption that any $\bar{\bb}$ has
a pair $(\bar{b}_l,\bar{b}_m)$ satisfying $\bar{b}_l>\bar{b}_m$ with $l<m$.
Thus, there exists a surrogate risk minimizer $(\bar{a},\bar{\bb})$ 
satisfying $\bar{b}_l\le\bar{b}_m$ for any $l,m\in[K-1]$ such that $l<m$.

\textbf{Proof of \hyl{a11}~}
This can be proved in the same way as the above proof.
\end{proof}

\paragraph{Generalization of Theorem~\protect\ref{thm:AT-BP-Ord}}
Moreover, Theorem~\ref{thm:AT-BP-Ord} can be generalized
to a regularized version as follows:
\begin{corollary}
\label{cor:ATBiasOrdCor}
Let $\calA\subseteq\{a:\bbR^d\to\bbR\}$, 
$\calB_0^\ord\subseteq\calB\subseteq\calB_0$, 
and a real-valued functional $\Lambda$ satisfy
$\Lambda(a,\bb,\bx,y)\ge\Lambda(a,\pi(\bb),\bx,y)$ 
for any $a\in\calA$, $\bb\in\calB$, $\bx\in\bbR^d$, $y\in[K]$ 
and sort $\pi(\bb)=(b_{i_1},\ldots,b_{i_{K-1}})^\top$, where 
$b_{i_1}\le\ldots\le b_{i_{K-1}}$ with $i_1,\ldots,i_{K-1}\in[K-1]$ 
satisfying $i_l\neq i_m$ if $l\neq m$.
Then, \hyl{a10} and \hyl{a11} hold with
\begin{align}
\label{eq:Reguab}
\begin{split}
    &(\hat{a},\hat{\bb})\in\underset{a\in\calA,\bb\in\calB}{\argmin}\,
    \frac{1}{n}\sum_{i=1}^n\{\phi(a(\bx_i),\bb,y_i)+\Lambda(a,\bb,\bx_i,y_i)\},\\
    &(\bar{a},\bar{\bb})\in\underset{a\in\calA,\bb\in\calB}{\argmin}\,
    \bbE[\phi(a(\bX),\bb,Y)+\Lambda(a,\bb,\bX,Y)]
\end{split}
\end{align}
instead of \eqref{eq:Estiab} and \eqref{eq:Idolab}.
\end{corollary}
As the regularization term $\Lambda(a,\bb,\bx,y)$,
\citet{shashua2003ranking} use $\Lambda(a,\bb,\bx,y)\propto\|\bw\|_2^2$ for
a linear 1DT $a(\bx)=\bw^\top\bx+w_0$ (with $\bw\in\bbR^d$ and $w_0\in\bbR$),
and \citet{lin2012reduction} add the term proportional to $\|\bb\|_2^2$,
where $\|\cdot\|_2$ is the $L_2$ norm.
Corollary~\ref{cor:ATBiasOrdCor} covers these methods.

\paragraph{Proof of Theorem~\protect\ref{thm:AT-Consistency}}
Theorem~\ref{thm:AT-Consistency} can be proved
by the proof strategy by \citet{yamasaki2022optimal}:
\begin{proof}[{Proof of Theorem~\ref{thm:AT-Consistency}}]
\textbf{Proof of \hyl{b1}~}
\citet[Theorem 2]{yamasaki2022optimal} showed 
the statement for the Logistic-AT loss.
Thus, we here describe the proof of
the statement for the Exponential-AT loss.
According to \eqref{eq:ATSR},
one has that, under the assumption \eqref{eq:CL}, 
\begin{align}
\begin{split}
    &\bbE[\phi(a(\bX),\bb,Y)]\\
    &=\bbE\biggl[\sum_{y=1}^{K-1}
    \biggl\{\frac{1}{1+e^{-(\tilde{b}_y-\tilde{a}(\bX))}}\varphi(b_y-a(\bX))
    +\frac{1}{1+e^{(\tilde{b}_y-\tilde{a}(\bX))}}\varphi(a(\bX)-b_y)\biggr\}\biggr].
\end{split}
\end{align}
Then, we consider the following optimization problem conditioned on $\bX=\bx$:
\begin{align}
\label{eq:CondOpt}
\begin{split}
    &\min_{a(\bx)\in\bbR, \bb\in\calB_0^\ord}L(a(\bx),\bb)\\
    &\text{with~}
    L(a(\bx),\bb):=
    \sum_{y=1}^{K-1}\biggl\{
    \frac{1}{1+e^{-(\tilde{b}_y-\tilde{a}(\bx))}}e^{-(b_y-a(\bx))}
    +\frac{1}{1+e^{(\tilde{b}_y-\tilde{a}(\bx))}}e^{(b_y-a(\bx))}\biggr\}.
\end{split}
\end{align}
The derivatives of the function $L$ become
\begin{align}
\begin{split}
    \frac{\partial L}{\partial a(\bx)}
    &=\sum_{y=1}^{K-1}\biggl\{
    \frac{1}{1+e^{-(\tilde{b}_y-\tilde{a}(\bx))}}e^{-(b_y-a(\bx))}
    -\frac{1}{1+e^{(\tilde{b}_y-\tilde{a}(\bx))}}e^{(b_y-a(\bx))}\biggr\},\\
    \frac{\partial L}{\partial b_y}
    &=-\frac{1}{1+e^{-(\tilde{b}_y-\tilde{a}(\bx))}}e^{-(b_y-a(\bx))}
    +\frac{1}{1+e^{(\tilde{b}_y-\tilde{a}(\bx))}}e^{(b_y-a(\bx))}.
\end{split}
\end{align}
One can find $\frac{\partial L}{\partial a(\bx)}=\frac{\partial L}{\partial b_1}=\cdots=\frac{\partial L}{\partial b_{K-1}}=0$
at $(a(\bx),\bb)=(\tilde{a}(\bx)/2, \tilde{\bb}/2)$.
Since $\tilde{a}(\bx)/2\in\bbR$ and $\tilde{\bb}/2\in\calB_0^\ord$
and \eqref{eq:CondOpt} is a convex optimization problem,
this pair is a unique solution of \eqref{eq:CondOpt}.
This solution is valid for any $\bx$ in the support of the probability distribution of $\bX$,
so it is ensured that $(\bar{a}(\bX),\bar{\bb})=(\tilde{a}(\bX)/2, \tilde{\bb}/2)$ almost surely.

\textbf{Proof of \hyl{b2}~}
The statement \hyl{b2} can be proved directly 
from \citet[Theorems 3 and 4]{yamasaki2022optimal}:
In general, the conditional task risk at $\bX=\bx$ can be written as
\begin{align}
    \bbE[\ell(f(\bx),Y)]
    =\sum_{y=1}^K \Pr(Y=y|\bX=\bx)\ell(f(\bx), y).
\end{align}
Thus, under the assumption \eqref{eq:CL} of this theorem,
the following labeling function becomes optimal 
in minimizing the task risk:
\begin{align}
\label{eq:Sta-labeling}
    h_\cl(u;\hat{\bb})
    \in\underset{k\in[K]}{\argmin}\,
    \sum_{y=1}^K P_\cl(y;u,\bar{\bb})\ell(k,y).
\end{align}
\citet[Theorems 3 and 4]{yamasaki2022optimal} analyzed
the labeling function \eqref{eq:Sta-labeling},
and showed its relationships to the threshold labeling.
Refer to those theorems for a detailed proof of this statement.
\end{proof}

\paragraph{Proof of Theorem~\protect\ref{thm:POCL-shape}}
We describe the proof of Theorem~\ref{thm:POCL-shape} \hyl{c6} and \hyl{c7}
in the following:
\begin{proof}[{Proof of Theorem~\ref{thm:POCL-shape}}]
\textbf{Proof of \hyl{c1}--\hyl{c5}~}
The results \hyl{c1} and \hyl{c2} can be shown from simple calculations,
and the results \hyl{c3}--\hyl{c5} are shown in \citet[Lemma 1]{yamasaki2022optimal}.
Therefore, we omit the proof for these results.

\textbf{Proof of \hyl{c6}~}
\hyl{c3}--\hyl{c5} show this statement:
On has 
\begin{align}
    P_\cl(k;u,\bb)
    =P_\cl\bigl(k;\tfrac{b_{k-1}+b_k}{2}-\bigl\{\tfrac{b_{k-1}+b_k}{2}-u\bigr\},\bb\bigr)
    =P_\cl\bigl(k;\tfrac{b_{k-1}+b_k}{2}-\bigl|\tfrac{b_{k-1}+b_k}{2}-u\bigr|,\bb\bigr)
\end{align}
from \hyl{c3}.
Under the assumption,
it holds that $|\tfrac{b_{1}+b_2}{2}-u|\ge\cdots\ge|\tfrac{b_{m-1}+b_m}{2}-u|$ 
and $|\tfrac{b_{m-1}+b_m}{2}-u|\le\cdots\le|\tfrac{b_{K-2}+b_{K-1}}{K-1}-u|$.
These results, together with \hyl{c4} and \hyl{c5}, show \hyl{c6}.

\textbf{Proof of \hyl{c7}~}
Consider the comparison
\begin{align}
\begin{split}
    &P_\cl(1;u,\bb)
    =\frac{1}{1+e^{-(b_1-u)}}
    =\frac{1}{1+e^{u}}\\
    &\gtrless 
    P_\cl(2;u,\bb)
    =\frac{1}{1+e^{-(b_2-u)}}-\frac{1}{1+e^{-(b_1-u)}}
    =\frac{1}{1+e^{u-\Delta}}-\frac{1}{1+e^{u}}
\end{split}
\end{align}
with $b_1=0$.
Then, one has that $\Delta_1=-\log\{(1-e^{-u})/2\}$ if $u>0$
and $\Delta_1=0$ if $u<0$.
Although $\Delta_2$ cannot be represented explicitly,
the similar calculation shows that
$\Delta_2=-\log\{(1-e^{u-(K-2)\Delta_2})/2\}$ if $u<(K-2)\Delta_2$,
and $\Delta_2=0$ if $u>(K-2)\Delta_2$.
\end{proof}

\paragraph{Proof of Theorem~\protect\ref{thm:ATSQ-Consistency}}
The following is a proof of Theorem~\ref{thm:ATSQ-Consistency}.
\begin{proof}[{Proof of Theorem~\ref{thm:ATSQ-Consistency}}]
According to \eqref{eq:ATSR}, one has
\begin{align}
    \begin{split}
    &\bbE[\phi(a(\bX),\bb,Y)]\\
    &=\bbE\biggl[\sum_{y=1}^{K-1}\bigl[\Pr(Y\le y|\bX)\{1-(b_y-a(\bX))\}^2+\Pr(Y>y|\bX)\{1+(b_y-a(\bX))\}^2\bigr]\biggr].
    \end{split}
\end{align}
The stationary condition with respect to $a(\bx)$,
\begin{align}
    \begin{split}
    &\frac{\partial \bbE[\phi(a(\bX),\bb,Y)]}{\partial a(\bx)}\biggr|_{a(\bX)=\bar{a}(\bX),\bb=\bar{\bb}}\\
    &=\sum_{y=1}^{K-1}\bigl[\Pr(Y\le y|\bX=\bx)\{1-(\bar{b}_y-\bar{a}(\bx))\}-\Pr(Y>y|\bX=\bx)\{1+(\bar{b}_y-\bar{a}(\bx))\}\bigr]\\
    &=(K-1)\bar{a}(\bx)-\biggl\{(K-1)+\sum_{y=1}^{K-1}\bar{b}_y-2\sum_{y=1}^{K-1}\Pr(Y\le y|\bX=\bx)\biggr\}=0,
    \end{split}
\end{align}
shows 
\begin{align}
\label{eq:ATSQA}
    \bar{a}(\bx)=1+\frac{1}{K-1}\sum_{y=1}^{K-1}\bar{b}_y-\frac{2}{K-1}\sum_{y=1}^{K-1}\Pr(Y\le y|\bX=\bx).
\end{align}
Also, the stationary condition with respect to $b_y$,
\begin{align}
    \begin{split}
    &\frac{\partial \bbE[\phi(a(\bX),\bb,Y)]}{\partial b_y}\biggr|_{a(\bX)=\bar{a}(\bX),\bb=\bar{\bb}}\\
    &=\bbE\bigl[\Pr(Y\le y|\bX)\{1-(\bar{b}_y-\bar{a}(\bX))\}-\Pr(Y>y|\bX)\{1+(\bar{b}_y-\bar{a}(\bX))\}\bigr]\\
    &=\bbE\bigl[\{2\Pr(Y\le y|\bX)-1\}-(\bar{b}_y-\bar{a}(\bX))\bigr]
    =\bigl\{2\Pr(Y\le y)-1+\bbE\bigl[\bar{a}(\bX)\bigr]\bigr\}-\bar{b}_y=0,
    \end{split}
\end{align}
shows 
\begin{align}
\label{eq:ATSQB}
    \bar{b}_y
    =\frac{1}{K-1}\sum_{k=1}^{K-1}\bar{b}_k+2\Pr(Y\le y)-\frac{2}{K-1}\sum_{k=1}^{K-1}\Pr(Y\le k).
\end{align}
On can see that $(\bar{a},\bar{\bb})=(\tilde{a},\tilde{\bb})$
with $(\tilde{a},\tilde{\bb})$ in \eqref{eq:ATSQSRM}
satisfies \eqref{eq:ATSQA} and \eqref{eq:ATSQB}.
Since the corresponding the surrogate risk 
minimization problem is a convex optimization problem,
$(\bar{a},\bar{\bb})=(\tilde{a},\tilde{\bb})$ is its unique solution
under the assumption $(\tilde{a},\tilde{\bb})\in\calA\times\calB$.
\end{proof}

\paragraph{Proof of Theorem~\protect\ref{thm:IT-Consistency}}
Theorem~\ref{thm:IT-Consistency} can be proved
similarly to Theorem~\ref{thm:AT-Consistency}:
\begin{proof}[{Proof of Theorem~\ref{thm:IT-Consistency}}]
\textbf{Proof of \hyl{f1}~}
The surrogate risk for the IT loss $\phi$ based on $\varphi$ is
\begin{align}
\label{eq:ITSR}
\begin{split}
    &\bbE[\phi(a(\bX),\bb,Y)]\\
    &=\bbE\biggl[\sum_{y=1}^{K-1}\bigl\{\Pr(Y=y|\bX)\varphi(b_y-a(\bX))+\Pr(Y=y+1|\bX)\varphi(a(\bX)-b_y)\bigr\}\biggr].
\end{split}
\end{align}
Under the assumption on $\Pr(Y=y|\bX=\bx)$, one has
\begin{align}
    \begin{split}
    &\Pr(Y=y|\bX=\bx)\varphi(b_y-a(\bx))+\Pr(Y=y+1|\bX=\bx)\varphi(a(\bx)-b_y)\\
    &=c_y\biggl\{\frac{1}{1+e^{-(\tilde{b}_y-\tilde{a}(\bx))}}\varphi(b_y-a(\bx))
    +\frac{1}{1+e^{(\tilde{b}_y-\tilde{a}(\bx))}}\varphi(a(\bx)-b_y)\biggr\}
    \end{split}
\end{align}
with a positive constant $c_y$,
because of Theorem~\ref{thm:POACL-shape} \hyl{g1}.
Letting $R(a(\bx),\bb):=\bbE[\phi(a(\bx),\bb,Y)]$,
the derivatives of the function $R$ with $\varphi=\varphi_\logi$ are
\begin{align}
\begin{split}
    \frac{\partial R}{\partial a(\bx)}
    &=\sum_{y=1}^{K-1}c_y\biggl\{\frac{1}{1+e^{-(\tilde{b}_y-\tilde{a}(\bx))}}\frac{1}{1+e^{(b_y-a(\bx))}}
    -\frac{1}{1+e^{(\tilde{b}_y-\tilde{a}(\bx))}}\frac{1}{1+e^{-(b_y-a(\bx))}}\biggr\},\\
    \frac{\partial R}{\partial b_y}
    &=-c_y\biggl\{\frac{1}{1+e^{-(\tilde{b}_y-\tilde{a}(\bx))}}\frac{1}{1+e^{(b_y-a(\bx))}}
    -\frac{1}{1+e^{(\tilde{b}_y-\tilde{a}(\bx))}}\frac{1}{1+e^{-(b_y-a(\bx))}}\biggr\},
\end{split}
\end{align}
and those with $\varphi=\varphi_\expo$ are
\begin{align}
\begin{split}
    \frac{\partial R}{\partial a(\bx)}
    &=\sum_{y=1}^{K-1}c_y\biggl\{\frac{1}{1+e^{-(\tilde{b}_y-\tilde{a}(\bx))}}e^{-(b_y-a(\bx))}
    -\frac{1}{1+e^{(\tilde{b}_y-\tilde{a}(\bx))}}e^{(b_y-a(\bx))}\biggr\},\\
    \frac{\partial R}{\partial b_y}
    &=-c_y\biggl\{\frac{1}{1+e^{-(\tilde{b}_y-\tilde{a}(\bx))}}e^{-(b_y-a(\bx))}
    -\frac{1}{1+e^{(\tilde{b}_y-\tilde{a}(\bx))}}e^{(b_y-a(\bx))}\biggr\}.
\end{split}
\end{align}
Substituting $(a(\bx),\bb)=(\tilde{a}(\bx), \tilde{\bb})$ 
into these derivatives with $\varphi=\varphi_\logi$ 
or $(a(\bx),\bb)=(\tilde{a}(\bx)/2, \tilde{\bb}/2)$
into these derivatives with $\varphi=\varphi_\expo$, 
one can find $\frac{\partial R}{\partial a(\bx)}=\frac{\partial R}{\partial b_1}
=\cdots=\frac{\partial R}{\partial b_{K-1}}=0$.
Therefore, the optimization problem $\min_{a(\bx)\in\bbR, \bb\in\calB_0}R(a(\bx),\bb)$ 
conditioned on $\bX=\bx$, which is a convex optimization problem,
has a unique solution, 
$(\tilde{a}(\bx), \tilde{\bb})$ when $\varphi=\varphi_\logi$ or
$(\tilde{a}(\bx)/2, \tilde{\bb}/2)$ when $\varphi=\varphi_\expo$.
This solution is valid for any $\bx$ in the support of the probability distribution of $\bX$,
so the proof of the statement \hyl{f1} is concluded.


\textbf{Proof of \hyl{f2}~}
Since the statement for the Exponential-IT loss 
can be proved in the same way, 
here we describe the proof for the Logistic-IT loss only.
Since $\argmin_{f(\bx)\in[K]}\bbE[\ell(f(\bx),Y)]=
\argmax_{k\in[K]}P_\acl(k;\bar{a}(\bx),\bar{\bb})$ 
when $\ell=\ell_\zo$ due to \eqref{eq:ACL}
and the threshold labeling $h_\thr(u;\bt)$ is non-decreasing in $u$,
it is sufficient to show that
$\argmax_{k\in[K]}P_\acl(k;u_1,\bar{\bb})\le
\argmax_{k\in[K]}P_\acl(k;u_2,\bar{\bb})$ if $u_1\le u_2$
for \hyl{f2}.
This inequality is guaranteed if it can be shown that 
there is no pair $(l,m)\in[K]^2$ such that $l<m$ and 
$P_\acl(l;u_1,\bar{\bb})\le P_\acl(m;u_1,\bar{\bb})$ and
$P_\acl(l;u_2,\bar{\bb})\ge P_\acl(m;u_2,\bar{\bb})$ with $u_1, u_2\in\bbR$ such that $u_1<u_2$
(imagine $l=\argmax_{k\in[K]}P_\acl(k;u_1,\bar{\bb})$
and $m=\argmax_{k\in[K]}P_\acl(k;u_2,\bar{\bb})$).
From Theorem~\ref{thm:POACL-shape} \hyl{g1}, one has that
$P_\acl(m;u_1,\bar{\bb})/P_\acl(l;u_1,\bar{\bb})=e^{-\sum_{k=l}^{m-1}(\bar{b}_k-u_1)}<
P_\acl(m;u_2,\bar{\bb})/P_\acl(l;u_2,\bar{\bb})=e^{-\sum_{k=l}^{m-1}(\bar{b}_k-u_2)}
=e^{-\sum_{k=l}^{m-1}(b_k-u_1)}\cdot e^{(m-l)(u_2-u_1)}$.
This result shows the absence of the above-described pair $(l,m)$,
which concludes the proof of \hyl{f2}.

\textbf{Proof of \hyl{f3}~}
This is obvious from Theorem~\ref{thm:POACL-shape} \hyl{g2}.

\textbf{Proof of \hyl{f4}~}
This is clear if considering Karush-Kuhn-Tucker conditions
for the constrained optimization problem.
\end{proof}

\paragraph{Proof of Theorem~\protect\ref{thm:POACL-shape}}
Theorem~\ref{thm:POACL-shape} is on 
structural properties of the ACL model $(P_\acl(y;u,\bb))_{y\in[K]}$,
and can be proved as follows:
\begin{proof}[{Proof of Theorem~\ref{thm:POACL-shape}}]
\textbf{Proof of \hyl{g1}~}
For each $y\in[K-1]$, one has
\begin{align}
\begin{split}
    \frac{P_\acl(y+1;u,\bb)}{P_\acl(y;u,\bb)}
    &=\frac{e^{-\sum_{k=1}^{y}(b_k-u)}}{\sum_{l=1}^Ke^{-\sum_{k=1}^{l-1}(b_k-u)}}\biggr/
    \frac{e^{-\sum_{k=1}^{y-1}(b_k-u)}}{\sum_{l=1}^Ke^{-\sum_{k=1}^{l-1}(b_k-u)}}\\
    &=e^{-(b_y-u)}=
    \biggl(\frac{1}{1+e^{(b_y-u)}}\biggr)\biggr/\biggl(\frac{1}{1+e^{-(b_y-u)}}\biggr).
\end{split}
\end{align}

\textbf{Proof of \hyl{g2}~}
When $b_1\le\cdots\le b_{m-1}\le u\le b_m\le\cdots\le b_{K-1}$,
one has $b_1-u\le\cdots\le b_{m-1}-u\le 0\le b_m-u\le\cdots\le b_{K-1}-u$,
which, together with \hyl{g1}, shows
\begin{align}
\label{eq:ACLUni1}
    \begin{split}
    &\frac{P_\acl(2;u,\bb)}{P_\acl(1;u,\bb)},\ldots,
    \frac{P_\acl(m;u,\bb)}{P_\acl(m-1;u,\bb)}\ge e^{-0}=1,\\
    &\frac{P_\acl(m+1;u,\bb)}{P_\acl(m;u,\bb)},\ldots,
    \frac{P_\acl(K;u,\bb)}{P_\acl(K-1;u,\bb)}\le e^{-0}=1.
    \end{split}
\end{align}
This result implies that $(P_\acl(y;u,\bb))_{y\in[K]}$ is unimodal with a mode $m$.

\textbf{Proof of \hyl{g3}~}
\hyl{g3} gives a more general condition than \hyl{g2} 
for the ACL model $(P_\acl(y;u,\bb))_{y\in[K]}$ to be unimodal.
Letting $l_1=\max(\{1\}\cup\{k\in[K-1]\mid b_k\le u\})$ 
and $l_2=\min(\{k\in[K-1]\mid b_k\ge u\}\cup\{K\})$,
one has 
(i) $b_{l_1}-u\le 0$, 
(ii) $b_{l_1+1}-u,\ldots,b_{K-1}-u\ge 0$ if $l_1\le K-2$,
(iii) $b_1-u,\ldots,b_{l_2-1}-u\le 0$ if $l_2\ge2$, 
and (iv) $b_{l_2}-u\ge 0$.
When $l_1=l_2=m$, 
(iii) shows $b_1-u,\ldots,b_{m-1}-u\le 0$,
(iv) shows $b_m-u\ge 0$,
(ii) shows $b_{m+1}-u,\ldots,b_{K-1}-u\ge 0$.
Also, when $m=l_1<l_2$,
(iii) shows $b_1-u,\ldots,b_{m-1}-u(,\ldots,b_{l_2-1}-u)\le 0$,
(i) shows $b_m-u\ge 0$, and
(ii) shows $b_{m+1},\ldots,b_{K-1}-u\ge 0$.
In both the cases, \eqref{eq:ACLUni1} holds, and hence,
$(P_\acl(y;u,\bb))_{y\in[K]}$ becomes unimodal with a mode $m$.

\textbf{Proof of \hyl{g4}~}
\hyl{g4} gives a sufficient condition that
the ACL model $(P_\acl(y;u,\bb))_{y\in[K]}$ is not unimodal.
When $b_l-u<0<b_k-u$ with $k<l$,
it holds that 
$P_\acl(k;u,\bb)>P_\acl(k+1;u,\bb)$ and
$P_\acl(l;u,\bb)<P_\acl(l+1;u,\bb)$,
namely, the ACL model $(P_\acl(y;u,\bb))_{y\in[K]}$ is not unimodal.
\end{proof}

\paragraph{Proof of Theorem~\protect\ref{thm:ITSQ-Consistency}}
The following is a proof of Theorem~\ref{thm:ITSQ-Consistency}.
\begin{proof}[{Proof of Theorem~\ref{thm:ITSQ-Consistency}}]
According to \eqref{eq:ITSR}, one has
\begin{align}
    \begin{split}
    &\bbE[\phi(a(\bX),\bb,Y)]\\
    &=\bbE\biggl[\sum_{y=1}^{K-1}\bigl[\Pr(Y=y|\bX)\{1-(b_y-a(\bX))\}^2+\Pr(Y=y+1|\bX)\{1+(b_y-a(\bX))\}^2\bigr]\biggr].
    \end{split}
\end{align}
The stationary condition with respect to $a(\bx)$,
\begin{align}
    &\frac{\partial \bbE[\phi(a(\bX),\bb,Y)]}{\partial a(\bx)}\biggr|_{a(\bX)=\bar{a}(\bX),\bb=\bar{\bb}}\nonumber\\
    &=\sum_{y=1}^{K-1}\bigl[\Pr(Y=y|\bX=\bx)\{1-(\bar{b}_y-\bar{a}(\bx))\}-\Pr(Y=y+1|\bX=\bx)\{1+(\bar{b}_y-\bar{a}(\bx))\}\bigr]\\
    &=\Pr(Y=1|\bX=\bx)-\Pr(Y=K|\bX=\bx)
    -\sum_{y=1}^{K-1}\Pr(Y\in\{y,y+1\}|\bX=\bx)(\bar{b}_y-\bar{a}(\bx))=0,\nonumber
\end{align}
shows 
\begin{align}
\label{eq:ITSQA}
    \bar{a}(\bx)
    =\frac{\sum_{y=1}^{K-1}\Pr(Y\in\{y,y+1\}|\bX=\bx)\bar{b}_y-\Pr(Y=1|\bX=\bx)+\Pr(Y=K|\bX=\bx)}{\sum_{y=1}^{K-1}\Pr(Y\in\{y,y+1\}|\bX=\bx)}.
\end{align}
The stationary condition with respect to $b_y$,
\begin{align}
    &\frac{\partial \bbE[\phi(a(\bX),\bb,Y)]}{\partial b_y}\biggr|_{a(\bX)=\bar{a}(\bX),\bb=\bar{\bb}}\nonumber\\
    &=\bbE\bigl[\Pr(Y=y|\bX)\{1-(\bar{b}_y-\bar{a}(\bX))\}-\Pr(Y=y+1|\bX)\{1+(\bar{b}_y-\bar{a}(\bX))\}\bigr]\\
    &=\Pr(Y=y)-\Pr(Y=y+1)-\Pr(Y\in\{y,y+1\})\bar{b}_y+\bbE[\Pr(Y\in\{y,y+1\}|\bX)\bar{a}(\bX)]=0,\nonumber
\end{align}
shows 
\begin{align}
\label{eq:ITSQB}
    \bar{b}_y=\frac{\Pr(Y=y)-\Pr(Y=y+1)+\bbE[\Pr(Y\in\{y,y+1\}|\bX)\bar{a}(\bX)]}{\Pr(Y\in\{y,y+1\})}.
\end{align}
On can see that $(\bar{a},\bar{\bb})=(\tilde{a},\tilde{\bb})$
with $(\tilde{a},\tilde{\bb})$ in \eqref{eq:ITSQSRM}
satisfies \eqref{eq:ITSQA} and \eqref{eq:ITSQB}.
Since the corresponding the surrogate risk 
minimization problem is a convex optimization problem,
$(\bar{a},\bar{\bb})=(\tilde{a},\tilde{\bb})$ is its unique solution
under the assumption $(\tilde{a},\tilde{\bb})\in\calA\times\calB$.
\end{proof}

\paragraph{Proof of Theorem~\protect\ref{thm:PLAIT}}
Theorem~\ref{thm:PLAIT} is trivial because of the functional form of the AT and IT loss functions.
We omit its proof.

\paragraph{Theorem~\protect\ref{thm:FBLAIT} and its Proof for Flat-Bottom (FB) Loss}
Many of the function $\varphi$ used in the AT and IT losses are 
non-increasing and zero on the side greater than a certain point,
for example, $\varphi_\hing$, $\varphi_\ramp$, $\varphi_\sqhi$, and $\varphi_\smhi$.
We say that such a function is \emph{flat-bottom} (\emph{FB}):
\begin{definition}
\label{def:PL}
If a function $\varphi:\bbR\to\bbR$ is non-increasing and satisfies that 
$\varphi(u)>0$ for $u\in(-\infty,c)$ and $\varphi(u)=0$ for $u\in[c,\infty)$ with some $c\in\bbR$,
we say that $\varphi$ is FB with the edge $c$.
\end{definition}

For the surrogate risk minimizer with AT and IT losses based on 
a FB function $\varphi$, we found the following result,
which we used in Example~\ref{ex:3hingit}:
\begin{theorem}
\label{thm:FBLAIT}
Let $\calA=\{a:\bbR^d\to\bbR\}$ and $\calB=\calB_0^\ord$.
If $\phi$ is an AT loss \eqref{eq:AT} or IT loss \eqref{eq:IT} 
with $\varphi$ that is FB with the edge $c>0$
(e.g., $\varphi=\varphi_\hing, \varphi_\ramp, \varphi_\smhi, \varphi_\sqhi$ with $c=1$),
then there exists $\hat{\bb}$ defined by \eqref{eq:Estiab}
that satisfies $\hat{b}_{k+1}-\hat{b}_k\in[0,2c]$ for all $k\in[K-2]$,
and there exists $\bar{\bb}$ defined by \eqref{eq:Idolab}
that satisfies $\bar{b}_{k+1}-\bar{b}_k\in[0,2c]$ for all $k\in[K-2]$.
\end{theorem}

\begin{proof}[{Proof of Theorem~\ref{thm:FBLAIT}}]
Here, we describe a proof for the population version only.

Assume any surrogate risk minimizer $(\bar{a},\bar{\bb})$ satisfies
\begin{align}
\label{eq:BiasAsm}
    \bar{b}_{l+1}-\bar{b}_l>2c\text{~for some~}l\in[K-2]
\end{align}
as assumption for proof by contradiction,
and let $\delta=\bar{b}_{l+1}-\bar{b}_l-2c$ and
\begin{align}
\begin{split}
    &\check{a}(\bx):=
    \begin{cases}
    \bar{a}(\bx)&\text{for~}\bx\text{~s.t.~}\bar{a}(\bx)\le\bar{b}_l+c,\\
    \bar{b}_l+c&\text{for~}\bx\text{~s.t.~}\bar{b}_l+c<\bar{a}(\bx)<\bar{b}_{l-1}-c,\\
    \bar{a}(\bx)-\delta&\text{for~}\bx\text{~s.t.~}\bar{a}(\bx)\ge\bar{b}_{l-1}-c,
    \end{cases}\\
    &\check{b}_1:=\bar{b}_1,\cdots,\check{b}_l:=\bar{b}_l,\;
    \check{b}_{l+1}:=\bar{b}_{l+1}-\delta,\ldots,\check{b}_{K-1}:=\bar{b}_{K-1}-\delta.
\end{split}
\end{align}
The surrogate risk $\bbE[\phi(a(\bX),\bb,Y)]$ for the AT or IT loss $\phi$ based on $\varphi$ 
is a weighted mean of $\varphi(a(\bx)-b_k)$'s and $\varphi(b_k-a(\bx))$'s
with non-negative weights.
In the following, 
by studying the change in $\varphi(a(\bx)-b_k)$'s and $\varphi(b_k-a(\bx))$'s 
when changing from $(a,\bb)=(\bar{a},\bar{\bb})$ to $(a,\bb)=(\check{a},\check{\bb})$,
we prove that $\bbE[\phi(\bar{a}(\bX),\bar{\bb},Y)]\ge\bbE[\phi(\check{a}(\bX),\check{\bb},Y)]$.

\textbf{Consider the case $\bar{a}(\bx)\le\bar{b}_l+c$~}
Since $\check{a}(\bx)=\bar{a}(\bx)$ and $\check{b}_k=\bar{b}_k$ for $k=1,\ldots,l$, 
one has that
\begin{align}
    \varphi(\bar{a}(\bx)-\bar{b}_k)=\varphi(\check{a}(\bx)-\check{b}_k)\text{~and~}
    \varphi(\bar{b}_k-\bar{a}(\bx))=\varphi(\check{b}_k-\check{a}(\bx))\text{~for~}k=1,\ldots,l.
\end{align}
Since $\bar{a}(\bx)-\bar{b}_k<\bar{a}(\bx)-\{\bar{b}_k-\delta\}=\check{a}(\bx)-\check{b}_k$ for $k=l+1,\ldots,K-1$
and $\varphi$ is non-increasing, one has that
\begin{align}
    \varphi(\bar{a}(\bx)-\bar{b}_k)\ge\varphi(\check{a}(\bx)-\check{b}_k)\text{~for~}k=l+1,\ldots,K-1.
\end{align}
One has that $\bar{a}(\bx)<\bar{b}_{l+1}-c$ since $\bar{b}_l+c<\bar{b}_{l+1}-c$,
and that $\bar{a}(\bx)\le\{\bar{b}_{l+1}-\delta\}-c=\check{b}_{l+1}-c$ since $\bar{b}_l+c=\{\bar{b}_{l+1}-\delta\}-c$.
These imply that
\begin{align}
    c<\bar{b}_{l+1}-\bar{a}(\bx)\le\cdots\le\bar{b}_{K-1}-\bar{a}(\bx),\text{~and~}
    c\le\check{b}_{l+1}-\bar{a}(\bx)\le\cdots\le\check{b}_{K-1}-\bar{a}(\bx),
\end{align}
from the assumption $\bar{b}_{l+1}\le\cdots\le\bar{b}_{K-1}$ and $\check{b}_{l+1}\le\cdots\le\check{b}_{K-1}$.
Since $\varphi(u)=0$ for any $u\ge c$, one has that
\begin{align}
    \varphi(\bar{b}_k-\bar{a}(\bx))=\varphi(\check{b}_k-\check{a}(\bx))=0\text{~for~}k=l+1,\ldots,K-1.
\end{align}

\textbf{Consider the case $\bar{b}_{l+1}-c\le\bar{a}(\bx)$~}
Since $c\le\bar{a}(\bx)-\bar{b}_l\le\cdots\le\bar{a}(\bx)-\bar{b}_1$ because 
$\bar{b}_{l+1}-c\le\bar{a}(\bx)$ and $\bar{b}_1+c\le\cdots\le\bar{b}_l+c\le\bar{b}_{l+1}-c$,
$c\le\{\bar{a}(\bx)-\delta\}-\bar{b}_l\le\cdots\le\{\bar{a}(\bx)-\delta\}-\bar{b}_1$ because 
$\bar{b}_{l+1}-c\le\bar{a}(\bx)$ and $\bar{b}_1+c+\delta\le\cdots\le\bar{b}_l+c+\delta=\bar{b}_{l+1}-c$,
and $\varphi(u)=0$ for any $u\ge c$, one has that
\begin{align}
    \varphi(\bar{a}(\bx)-\bar{b}_k)=\varphi(\check{a}(\bx)-\check{b}_k)=0\text{~for~}k=l+1,\ldots,K-1.
\end{align}
Since $\bar{b}_k-\bar{a}(\bx)<\bar{b}_k-\{\bar{a}(\bx)-\delta\}=\check{b}_k-\check{a}(\bx)$ for $k=1,\ldots,l$
and $\varphi$ is non-increasing, one has that 
\begin{align}
    \varphi(\bar{b}_k-\bar{a}(\bx))\ge\varphi(\check{b}_k-\check{a}(\bx))\text{~for~}k=1,\ldots,l.
\end{align}
Since $\check{a}(\bx)=\bar{a}(\bx)-\delta$ and $\check{b}_k=\bar{b}_k-\delta$ for $k=l+1,\ldots,K-1$, 
one has that
\begin{align}
    \varphi(\bar{a}(\bx)-\bar{b}_k)=\varphi(\check{a}(\bx)-\check{b}_k)\text{~and~}
    \varphi(\bar{b}_k-\bar{a}(\bx))=\varphi(\check{b}_k-\check{a}(\bx))\text{~for~}k=l+1,\ldots,K-1.
\end{align}

\textbf{Consider the case $\bar{b}_l+c<\bar{a}(\bx)<\bar{b}_{l+1}-c$~}
If letting $\epsilon=\bar{a}(\bx)-\{\bar{b}_l+c\}$, 
then it holds that $\check{a}(\bx)=\bar{a}(\bx)-\epsilon$.
Since $c\le\bar{a}(\bx)-\bar{b}_l\le\cdots\le\bar{a}(\bx)-\bar{b}_1$,
one has that 
$c=\{\bar{a}(\bx)-\epsilon\}-\bar{b}_l\le\cdots\le\{\bar{a}(\bx)-\epsilon\}-\bar{b}_1$
and
\begin{align}
    \varphi(\bar{a}(\bx)-\bar{b}_k)=\varphi(\check{a}(\bx)-\check{b}_k)=0\text{~for~}k=1,\ldots,l.
\end{align}
Since $\bar{b}_k-\bar{a}(\bx)<\bar{b}_k-\{\bar{a}(\bx)-\epsilon\}=\check{b}_k-\check{a}(\bx)$ for $k=1,\ldots,l$
and $\varphi$ is non-increasing, one has that 
\begin{align}
    \varphi(\bar{b}_k-\bar{a}(\bx))\ge\varphi(\check{b}_k-\check{a}(\bx))\text{~for~}k=1,\ldots,l.
\end{align}
Since $\bar{a}(\bx)-\bar{b}_k<\{\bar{a}(\bx)-\epsilon\}-\{\bar{b}_k-\delta\}=\check{a}(\bx)-\check{b}_k$ for $k=l+1,\ldots,K-1$ because $\epsilon<\delta$,
and since $\varphi$ is non-increasing, one has that 
\begin{align}
    \varphi(\bar{a}(\bx)-\bar{b}_k)\ge\varphi(\check{a}(\bx)-\check{b}_k)\text{~for~}k=l+1,\ldots,K-1.
\end{align}
Since $c<\bar{b}_{l+1}-\bar{a}(\bx)\le\cdots\le\bar{b}_{K-1}-\bar{a}(\bx)$,
one has that 
$c=\{\bar{b}_{l+1}-\delta\}-\{\bar{a}(\bx)-\epsilon\}\le\cdots\le\{\bar{b}_{K-1}-\delta\}-\{\bar{a}(\bx)-\epsilon\}$
and
\begin{align}
    \varphi(\bar{b}_k-\bar{a}(\bx))=\varphi(\check{b}_k-\check{a}(\bx))=0\text{~for~}k=l+1,\ldots,K-1.
\end{align}

Correcting these pieces,
one can see $\bbE[\phi(\bar{a}(\bX),\bar{\bb},Y)]\ge\bbE[\phi(\check{a}(\bX),\check{\bb},Y)]$,
which contradicts the optimality of $(\bar{a},\bar{\bb})$ and 
shows that the assumption \eqref{eq:BiasAsm} is false.
This concludes the proof of this theorem.
\end{proof}


\end{document}